\documentclass[twoside]{article}

\usepackage[accepted]{aistats2022}
\usepackage[round]{natbib}

\usepackage[utf8]{inputenc} 
\usepackage[T1]{fontenc}    
\usepackage{hyperref}       
\usepackage{url}            
\usepackage{booktabs}       
\usepackage{amsfonts}       
\usepackage{nicefrac}       
\usepackage{microtype}      
\usepackage{amsmath,dsfont,empheq}
\usepackage{algorithm}
\usepackage{algpseudocode}
\usepackage[textsize=tiny]{todonotes}
\usepackage{mathtools}
\usepackage{amssymb}
\usepackage{syntonly,bm,wrapfig}
\usepackage{xcolor,cases,balance}

\usepackage{subcaption}
\usepackage{ushort}
\usepackage{blindtext}
\usepackage{afterpage}

\usepackage{makecell}

\newcommand\numberthis{\addtocounter{equation}{1}\tag{\theequation}}

\usepackage{graphicx}
\newcommand\smallO{
  \mathchoice
    {{\scriptstyle\mathcal{O}}}
    {{\scriptstyle\mathcal{O}}}
    {{\scriptscriptstyle\mathcal{O}}}
    {\scalebox{.7}{$\scriptscriptstyle\mathcal{O}$}}
  }

\def \btheta {{\boldsymbol{\theta}}}

\definecolor{ERM_cl}{RGB}{51,51,51}
\definecolor{MAML_cl}{RGB}{0,127,0}
\definecolor{BAMAML_cl}{RGB}{0,76,204}
\definecolor{BIMAML_cl}{RGB}{196,51,51}

\newtheorem{theorem}{Theorem}
\newtheorem{lemma}{Lemma}
\newtheorem{proposition}{Proposition}
\newtheorem{remark}[theorem]{Remark}
\newtheorem{corollary}{Corollary}

\newtheorem{assumption}{Assumption}
\newenvironment{proof}{\paragraph{Proof:}}{\hfill$\square$}

\usepackage{float}
\usepackage{placeins}
\usepackage{multirow}

\usepackage[most]{tcolorbox}
\usepackage{pgf}

\graphicspath{{./figures/}}

\newcommand{\emphblockoption}{drop shadow,
    colframe=black!60,
    colback=black!10,
    coltitle=white!, 
    left=.2pt,
    right=.2pt,
    boxrule=1pt,
    arc=1pt}
\tcbset{emphblock/.style={code={\pgfkeysalsofrom{\emphblockoption}}}}

\begin{document}

%

%

\twocolumn[

\aistatstitle{
Is Bayesian Model-Agnostic Meta Learning Better than Model-Agnostic Meta Learning, Provably?
}

\aistatsauthor{ Lisha Chen \And  Tianyi Chen }

\aistatsaddress{ 
Rensselaer Polytechnic Institute  } ]

\begin{abstract}
  Meta learning aims at learning a model that can quickly adapt to unseen tasks.
  Widely used meta learning methods include model-agnostic meta learning (MAML), implicit MAML, Bayesian MAML. Thanks to its ability of modeling uncertainty, Bayesian MAML often has advantageous empirical performance. 
  However, 
  the theoretical understanding of Bayesian MAML is still limited, especially on questions such as if and when Bayesian MAML has provably better performance than MAML. 
  In this paper, 
  we aim to provide theoretical justifications for Bayesian MAML's advantageous performance by comparing the meta test risks of MAML and Bayesian MAML. In the meta linear regression, under both the distribution agnostic and linear centroid cases, we have established that Bayesian MAML indeed has provably lower meta test risks than MAML. 
We verify our theoretical results through experiments,
the code of which is available at \href{https://github.com/lisha-chen/Bayesian-MAML-vs-MAML}{https://github.com/lisha-chen/Bayesian-MAML-vs-MAML}.
\end{abstract}

\section{INTRODUCTION} 
\label{sec:introduction}

Meta learning, also referred to as ``learning to learn'', 
usually learns a model that can quickly adapt to new tasks~\citep{l2l_book_1998,hospedalesmeta,vilalta2002perspective,vanschoren2018meta,bengio_learnsynaptic,Schmidhuber95onlearning,hochreiter2001_l2l}.
The key idea of meta-learning is to learn a ``prior'' model from multiple existing tasks with a hope that the learned model is able to quickly adapt to unseen tasks.
Meta learning has been used in various machine learning scenarios including few-shot learning~\citep{snell2017prototypical,obamuyide2019model}, 
continual learning~\citep{harrison2020continuous,javed2019meta}, 
and personalized learning~\citep{madotto2019personalizing}.
In addition, meta learning has also been successfully implemented in different data limited applications including language and vision tasks~\citep{achille2019task2vec,li2018learning,hsu2018unsupervised,liu2019learning,zintgraf2019fast,Wang_2019_ICCV,obamuyide2019model}.
One of the popular meta-learning approaches is the model agnostic meta-learning (MAML) method~\citep{Finn2017_maml,vuorio2019multimodal,yin2020meta,obamuyide2019model}, 
which learns an initial model that can adapt to new tasks using one step gradient update.
Despite its success, MAML still suffers from overfitting when it is trained with few data, which motivates  Bayesian MAML (BaMAML)~\citep{grant2018recasting,ravi2018_ABML,yoon2018_BMAML}. 
Instead of point estimation of task specific model parameters, as in MAML and its variants~\citep{rajeswaran2019_imaml}, BaMAML obtains a posterior distribution of task specific parameters as a function of the task data and the initial model parameters, as illustrated in Figure~\ref{fig:MAML_BaMAML_intro}.
For example, in 5-way 1-shot classification on TieredImageNet, BaMAML has 35.2\% performance gain over MAML in terms of accuracy ~\citep{nguyen2020_VAMPIRE}.
In spite of BaMAML's impressive empirical performance, 
its theoretical understanding is still very limited, no need to mention a sound justification for its performance gain over MAML.
\begin{figure}[t]
  \centering
  \includegraphics[width=0.99\linewidth]{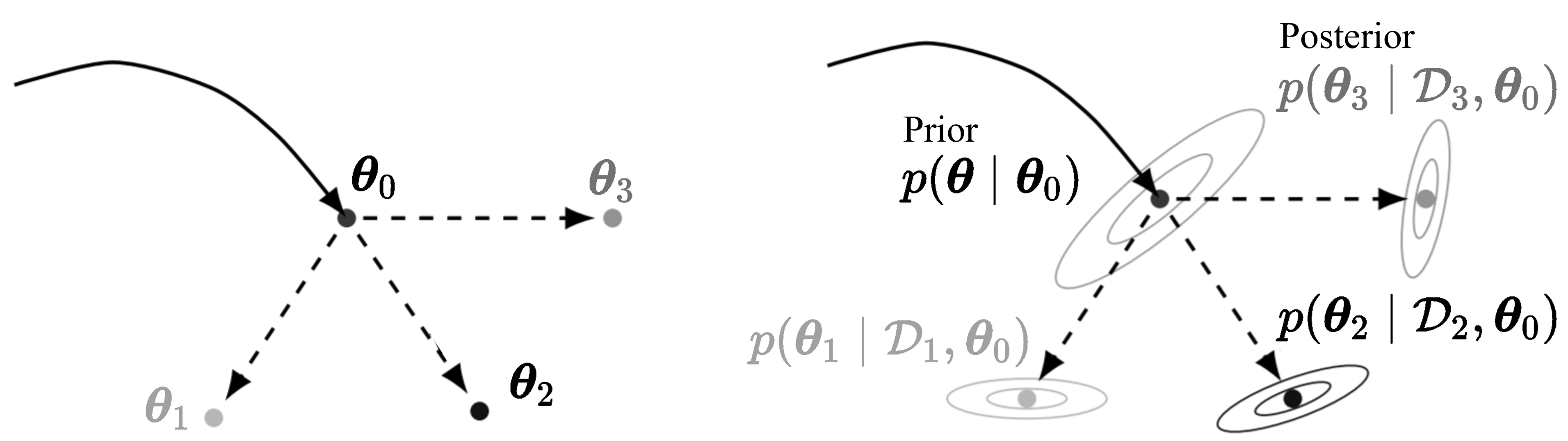}
  \caption{MAML (left) and Bayesian MAML (right).
  }
  \label{fig:MAML_BaMAML_intro}
  \vspace{-0.2cm}
\end{figure}
In this context, this paper aims to answer the following question:
\begin{center}
   \small{
   \textsf{
   Is Bayesian Model-Agnostic Meta Learning Better than Model-Agnostic Meta Learning, Provably?
   }}
\end{center}
In an attempt to provide an affirmative answer to this question, our paper 
analyzes the meta-test risks of one-step MAML and BaMAML to make a fair comparison between them.
In a high level, our theoretical results suggest that compared to one-step MAML, BaMAML
1) harnesses 
flexibility in the trade-off between prior and likelihood  based on their quality to improve model adaptation capacity;
and, 2)
leverages the posterior distribution instead of a point estimation in inference, which allows model averaging to reduce variance.


\subsection{Related Works} 
\label{sub:related_works}
 

Early works of meta learning build black-box recurrent models that can make predictions based on few examples from new tasks~\citep{schmidhuber1993_recurrent,hochreiter2001_l2l,andrychowicz2016_l2l,chen2017_l2l}, or learn shared feature representation among multiple tasks~\citep{snell2017prototypical,vinyals2016matching}. 
More recently, some methods have been developed to find the initialization of model parameters that can quickly adapt to new tasks with few optimization steps~\citep{Finn2017_maml,nichol2018first,rothfuss2018promp}.
The empirical success of meta learning has also stimulated recent interests on building the theoretical foundation of these methods.
To put our work in context, we review prior art that is grouped into the following categories.

\textbf{Theory of meta-learning.}
One line of theoretical works study the convergence of meta-learning algorithms under different settings.
These works include
analysis of the regret bound for an online meta-learning algorithm~\citep{finn2019_onlineML},
the convergence and sample complexity of gradient based MAML~\citep{fallah2020_convergence_maml},
sufficient conditions for its convergence to the exact solution for an approximate bilevel optimization method~\citep{franceschi2018_bilevel_maml_approx},
sample complexity for a bilevel formulation for meta-learning, named implicit MAML (iMAML)~\citep{rajeswaran2019_imaml},
and the global convergence guarantee of MAML with overparameterized deep neural nets (DNNs)~\citep{wang2020_global_converge_maml_dnn,wang2020_global_opt_maml}. There are also works that study the convergence of general compositional~\citep{chen2021composition} or bilevel~\citep{chen2021closing,yang2021provably,liu2021towards} optimzation algorithms which can be applied to analyze the convergence of one-step or bilevel MAML.

Another line of works analyze the generalization error bound  of meta learning  methods under different settings based on their optimization trajectory.
For instance, meta-learning in the linear centroid model for ridge regression~\citep{denevi2018_l2l_linear_centroid},
MAML with sufficiently wide DNNs~\citep{wang2020_global_converge_maml_dnn},
meta-learning in online convex optimization~\citep{balcan2019_gbml_online_convex}, 
and MAML for strongly convex objective functions on recurring and unseen tasks~\citep{fallah2021_generalization_unseen}.
Recently, information theoretical generalization error bounds of meta learning are also proposed by~\cite{jose2021information,rezazadeh2021conditional,jose2021transfer,chen2021generalization}, which bounds the meta learning generalization error in terms of mutual information between the input meta-training data and the output of the meta-learning algorithms rather than gradient norm of the algorithms during optimization.

Our work is also inspired by several pioneering works that analyze the optimization, modeling and statistical errors of meta-learning methods.
Gao et al.~\citeyearpar{gao2020_model_opt_tradeoff_ml} study the modeling and optimization error trade-off in MAML and compare the trade-off with that of empirical risk minimization (ERM). 
Collins et al.~\citeyearpar{collins2020_task_landscape_erm_maml} further analyze the effect of different factors on the optimal population risk, such as task hardness in task landscape.
Bai et al.~\citeyearpar{bai2021_trntrn_trnval} study how the dataset split between the training and validation affects the performance of iMAML under a noiseless realizable centroid model.
But none of them tackle the meta-test risk of BaMAML.
Furthermore, from the technical aspect,
{compared to~\cite{bai2021_trntrn_trnval},
our analysis does not require strong assumption on noiseless realizable model; compared to~\cite{gao2020_model_opt_tradeoff_ml}, our analysis provides a sharper characterization of statistical error bound in the high-dimensional asymptotic case.}

\textbf{Bayesian model agnostic meta-learning.}
From a hierarchical probabilistic modeling perspective, learning the initialization in MAML is tantamount to learning the prior distribution of model parameters shared across different tasks~\citep{grant2018recasting}, which leads to a hierarchical Bayes formulation that we call
BaMAML thereafter.
Empirically, they have better performance in few-shot meta learning settings and tend to reduce over-fitting in the data-limited regimes. 
Several variants of BaMAML have been  proposed  based on different Bayesian inference methods~\citep{grant2018recasting,finn2018_PLATIPUS,yoon2018_BMAML,gordon2018_VERSA,nguyen2020_VAMPIRE}.
Despite the superior empirical performance of BaMAML methods compared to non-Bayesian ones, very few works study their theory.
A related line of works extend the PAC-Bayes framework to meta learning~\citep{amit2018_pac_bayes_ml_prior,rothfuss2021_pacoh_pac_bayes_ml,ding2021bridging,farid2021generalization}, to provide a PAC-Bayes meta-test error bound. 
Different from the PAC-Bayes framework that bounds the Gibbs risk, we bound the Bayes risk~\citep{Sheth2017ExcessRiskBayes}.
While these works provide the meta-test error bound for BaMAML,
exactly when BaMAML is provably better than non-Bayesian methods are not fully understood.
Different from these works,
we explicitly compare MAML and BaMAML in terms of meta-test error, consisting of the optimal population risks and statistical errors.


\subsection{Our Contributions} 
\label{sub:our_contributions}
The goal of this paper is to provide justification on the observed empirical performance gain of BaMAML over MAML.
Our contributions are summarized below.
\begin{itemize}
  \item[\bf C1)] Under the meta-linear regression setting, we decompose the meta-test risk into population risk and statistical error terms, which capture the bias and variance of the estimated parameter, respectively.
  We prove that BaMAML with proper choice of hyperparameters has smaller optimal population risk
  and dominating constant in statistical error than MAML, therefore smaller meta-test risk.
  \item[\bf C2)] With additional linear centroid model assumption for task data distribution, we prove that BaMAML has strictly smaller dominating constant in statistical error than MAML in the high dimensional asymptotic case.
  \item[\bf C3)] We conduct simulations on meta linear regression to verify our theory. And we also perform experiments beyond linear case, where similar conclusions can be drawn. 
\end{itemize}
Our theoretical analysis justifies BaMAML for reducing the optimal population risk and statistical errors, thus the meta-test risk.
And to our best knowledge, we are the first to make a comparison between MAML and BaMAML,
which is complementary to existing works~\citep{gao2020_model_opt_tradeoff_ml,collins2020_task_landscape_erm_maml} that compare MAML against empirical risk minimization. 

\section{PROBLEM DEFINITION AND SOLUTIONS} 
\label{sec:preliminaries}

In this section, we first introduce the general meta-learning setting 
and the formulations of two meta learning methods, MAML and BaMAML.
Then we focus on meta-linear regression, 
where solutions to the empirical and population level risks are obtained in closed form.


\subsection{Problem Setup} 
\label{sub:formal_setup}

In our meta-learning setting, 
assume task $\tau$ are drawn from a task distribution, i.e. $\tau \sim \mathcal{T}$, with input features $\mathbf{x}_{\tau} \in \mathcal{X}_{\tau} \subset \mathbb{R}^d$ and target labels $y_{\tau} \in \mathcal{Y}_{\tau} \subset \mathbb{R}$.
For each task $\tau$, we observe $N$ samples drawn i.i.d. from $\mathcal{P}_{\tau}$ in the dataset
$\mathcal{D}_{\tau} = \{(\mathbf{x}_{\tau,n},y_{\tau,n})\}_{n=1}^N$, 
and $\mathcal{D}_{\tau}$ is divided into the train and validation datasets, denoted as $\mathcal{D}_{\tau}^{{\rm trn}}$ and $\mathcal{D}_{\tau}^{{\rm val}}$, respectively. Here $|\mathcal{D}_{\tau}^{{\rm trn}}|=N_1$ and $|\mathcal{D}_{\tau}^{{\rm val}}|=N_2$ with $N=N_1+N_2$.
Given the data $\mathcal{D}_{\tau}$,
we use the empirical loss  ${\ell_{\tau}}(h_{\tau}, \mathcal{D}_{\tau})$ of per-task hypothesis $h_{\tau} \in \mathcal{H}_{\tau}$ as a measure of the performance. 

And the goal for initialization based meta learning methods, such as MAML~\citep{finn2018_PLATIPUS} and BaMAML~\citep{yoon2018_BMAML},
is to learn an initial parameter $\btheta_0 \in \mathbf{\Theta}_0$,
which, with an adaptation method and the training data, can produce a per-task hypothesis $h_{\tau}$ that performs well on the validation data for task $\tau$.
Formally, for a meta-learning method, 
${\mathcal{A}}: \mathbf{\Theta}_0 \times (\mathcal{X}_{\tau}\times \mathcal{Y}_{\tau})^{N_1} \rightarrow \mathcal{H}_{\tau}$, 
represents the adaptation method or base-learner.
Given $T$ tasks with corresponding data, our meta-learning objective is to find $\btheta_0$ that minimizes the empirical loss, given by
\begin{equation}\label{eq:emp_loss}
  {\mathcal{L}}^{\cal A}(\btheta_0, \mathcal{D})
\coloneqq 
\frac{1}{T} \sum_{\tau=1}^{T} \ell_{\tau}({\mathcal{A}}(\btheta_0,\mathcal{D}_{\tau}^{\mathrm{trn}}),\mathcal{D}_{\tau}^{\mathrm{val}}).
\end{equation}
And the corresponding meta-test risk is defined as
the expectation of the per-task loss $\ell_{\tau} $ over the task and data distribution, given by
\begin{equation}\label{eq:R}
\mathcal{R}^{\cal A}(\btheta_0)
\coloneqq  
\mathbb{E}_{\tau } 
\big[ \mathbb{E}_{\mathcal{D}_{\tau}} \big[\ell_{\tau}({\mathcal{A}}(\btheta_0,\mathcal{D}_{\tau}^{\mathrm{trn}}), \mathcal{D}_{\tau}^{\rm val})\big]\big]. 
\end{equation}

Denote 
$
\mathbf{X}_{\tau}^{\mathrm{all}} := 
[\mathbf{x}_{\tau, 1},\ldots,
\mathbf{x}_{\tau,N}]^{\top} \in \mathbb{R}^{N\times d}$, 
$\mathbf{y}_{\tau}^{\mathrm{all}} := 
[{y}_{\tau, 1},\ldots,
{y}_{\tau,N}]^{\top} \in \mathbb{R}^{N} 
$
for ease of discussion,
where ``all'' can also be ``trn'' for training and ``val'' for validation with $N_1$ and $N_2$ data points, respectively.
Throughout the discussion of this paper, 
we adopt a probabilistic perspective~\citep{grant2018recasting,finn2018_PLATIPUS}, with $\ell_{\tau}$ defined as the negative log likelihood, given by
{\small
\begin{align}\label{eq:l_tau}
  \ell_{\tau}
  &({\mathcal{A}}(\btheta_0,\mathcal{D}_{\tau}^{\mathrm{trn}}),\mathcal{D}_{\tau}^{\mathrm{val}}) 
  \!=\! -\frac{1}{N_2}\log  {p}(\mathbf{y}^{\mathrm{val}}_{\tau}| \mathbf{X}^{\mathrm{val}}_{\tau}, \btheta_0, \mathcal{D}_{\tau}^{\mathrm{trn}}) \nonumber\\
   =&\!- \frac{1}{N_2}\log \int {p}(\mathbf{y}^{\mathrm{val}}_{\tau}\mid \mathbf{X}^{\mathrm{val}}_{\tau}, \btheta_{\tau})
  p_{{\mathcal{A}}}(\btheta_{\tau}\mid \btheta_0, \mathcal{D}_{\tau}^{\mathrm{trn}}) d\btheta_{\tau} \!\!
\end{align}
}
\hspace{-2mm}where $p_{{\mathcal{A}}}(\btheta_{\tau}\mid \btheta_0, \mathcal{D}_{\tau}^{\mathrm{trn}})$ is the posterior distribution induced by ${\mathcal{A}}$.
And the likelihood ${p}(\mathbf{y}^{\mathrm{val}}_{\tau}\mid \mathbf{X}^{\mathrm{val}}_{\tau}, \btheta_{\tau}, \mathcal{D}_{\tau}^{\mathrm{trn}})
=\prod_{n=1}^{N_2} p(y_{\tau,n}\mid \mathbf{x}_{\tau,n},\btheta_{\tau})$.
Note that,
for a point estimate method $\mathcal{A}$, such as MAML,
the posterior distribution $p_{{\mathcal{A}}}(\btheta_{\tau}\mid \btheta_0, \mathcal{D}_{\tau}^{\mathrm{trn}})$ reduces to a Dirac delta function
$\delta(\btheta_{\tau} - \hat{\btheta}_{\tau}^{\mathcal{A}})$.
And $\cal{A}$ specifies a mapping from the initial parameter $\btheta_{0}$ to the task-specific parameter $\hat{\btheta}_{\tau}^{\cal{A}}(\btheta_{0}, \mathcal{D}_{\tau}^{\rm trn})$.

In the meta training stage, we obtain $\hat{\btheta}_0^{\mathcal{A}}$ by minimizing~\eqref{eq:emp_loss} under each meta learning method $\mathcal{A}$.
And in the meta testing stage, we evaluate the test error of $\hat{\btheta}_0^{\mathcal{A}}$ on~\eqref{eq:R} for different methods.
\paragraph{Methods.}
We proceed to introduce the general formulations of MAML and BaMAML.
Considering MAML with one step gradient update as the baseline method for meta-learning \citep{Finn2017_maml},
the task-specific parameter $\hat{\btheta}_{\tau}^{\mathrm{ma}}(\btheta_0)$ is obtained from the initial parameter $\btheta_0$ by taking one step gradient descent with step size $\alpha$ of the per-task loss function $\ell_{\tau}$.
Combined with the empirical loss defined in~\eqref{eq:emp_loss}, we have
\begin{tcolorbox}[emphblock]
the empirical loss of MAML is given by
\vspace{-1.5mm}
\begin{align}\label{eq:MAML_emp_risk}
  &{\mathcal{L}}^{\mathrm{ma}}({\btheta_0},\mathcal{D} )
  = 
  \frac{1}{T} \sum_{\tau=1}^{T} {\ell}_{\tau}(\hat{\btheta}_{\tau}^{\mathrm{ma}}(\btheta_0,\mathcal{D}_{\tau}^{\mathrm{trn}}), \mathcal{D}_{\tau}^{\mathrm{val}}) \\
   \nonumber
  &\mathrm{s.t.}\,\, 
  \hat{\btheta}_{\tau}^{\mathrm{ma}}(\btheta_0,\mathcal{D}_{\tau}^{\mathrm{trn}})
  = \btheta_0  - \frac{\alpha}{2}  \nabla_{\btheta_0} {\ell}_{\tau}(\btheta_0,\mathcal{D}_{\tau}^{\mathrm{trn}}).
\end{align}
\end{tcolorbox}

BaMAML obtains an approximation of the posterior distribution $p(\btheta_{\tau}\mid \mathcal{D}_{\tau}^{{\rm trn}}, \btheta_0)$
instead of a point estimate $\hat{\btheta}_{\tau}^{\cal{A}}(\btheta_{0}, \mathcal{D}_{\tau}^{{\rm trn}})$.
In general,
the true posterior distribution can be difficult to compute exactly.
Alternatively, the approximate distribution $\hat{p} (\btheta_{\tau} \mid \mathcal{D}_{\tau}^{{\rm trn}}, \btheta_0)$ can be obtained via variational inference~\citep{nguyen2020_VAMPIRE}, Markov chain Monte-Carlo sampling or Laplace approximation~\citep{grant2018recasting}.
Here we adopt the variational inference formulation, 
by minimizing the divergence between the approximate and the true posterior distribution. 
Define $\mathrm{D}_{\mathrm{KL}}(\cdot\|\cdot)$ as the KL-divergence between two distributions, we have
\begin{tcolorbox}[emphblock]
the empirical loss of BaMAML is given by
\vspace{-1.5mm}
\small
{\begin{align}\label{eq:theta_tau_BaMAML_theta0}
&\mathcal{L}^{\mathrm{ba}}(\btheta_0, \mathcal{D})
= \frac{1}{T} \sum_{\tau=1}^{T} {\ell}_{\tau}(\hat{p} (\btheta_{\tau} \mid \mathcal{D}_{\tau}^{{\rm trn}}, \btheta_0), \mathcal{D}_{\tau}^{\mathrm{val}}) \\
\nonumber
  &\mathrm{s.t.}\, \hat{p} (\btheta_{\tau} | \mathcal{D}_{\tau}^{{\rm trn}}, \btheta_0)
  \! = \!
   \mathop{\arg\min}_{q({\btheta}_{\tau})\in \mathcal{Q}}
    \mathrm{D}_{\mathrm{KL}}\big(q(\btheta_{\tau})\| 
    p(\btheta_{\tau}| \mathcal{D}_{\tau}^{{\rm trn}}, \btheta_0)\big)
\end{align}
}
\end{tcolorbox}
It is worth mentioning that BaMAML formulation in this paper contains iMAML, or iMAML~\citep{rajeswaran2019_imaml} as a special case. Therefore, results obtained for BaMAML naturally implies the results for iMAML with small difference. We point out this reduction in the next remark, and provide detailed discussion in the appendix.

\begin{table*}[ht!]
  \caption{Weight matrices for the closed form solutions of method $\mathcal{A}$.}
  \label{tab:weight_matrices}
  \centering
  \begin{tabular}{ l| l}
  \hline
  \hline
  Method & Weight matrices \\
  \hline
  MAML~\citep{gao2020_model_opt_tradeoff_ml}
  & $\mathbf{W}_{\tau}^{\mathrm{ma}} =  (\mathbf{I}-\alpha  \mathbf{Q}_{\tau}) \mathbf{Q}_{\tau}(\mathbf{I}-\alpha  \mathbf{Q}_{\tau})$ \\
  & $\hat{\mathbf{W}}_{\tau}^{\mathrm{ma}} 
  = 
  (\mathbf{I}-
  {\alpha} \hat{\mathbf{Q}}_{\tau,N_1}) 
  \hat{\mathbf{Q}}_{\tau,N_2}
  (\mathbf{I}- 
  {\alpha} \hat{\mathbf{Q}}_{\tau,N_1})$ \\
  \hline
  BaMAML
  & $\mathbf{W}_{\tau}^{\mathrm{ba}} =  ((s\gamma)^{-1} \mathbf{Q}_{\tau}+\mathbf{I})^{-1}\mathbf{Q}_{\tau}(\gamma ^{-1} \mathbf{Q}_{\tau}+\mathbf{I})^{-1}$\\
  & $\hat{\mathbf{W}}_{\tau}^{\mathrm{ba}} 
      =  ((s\gamma)^{-1} \hat{\mathbf{Q}}_{\tau,N}+\mathbf{I})^{-1}\hat{\mathbf{Q}}_{\tau,N_2}(\gamma ^{-1} \hat{\mathbf{Q}}_{\tau,N_1}+\mathbf{I})^{-1}$\\
  \hline
  \hline
\end{tabular}
\end{table*}
\begin{remark}[Reduction to iMAML]
\label{rmk: reduce_bimaml}
When $\mathcal{Q}$ is chosen to be the set of Dirac Delta functions and the KL-divergence in~\eqref{eq:theta_tau_BaMAML_theta0} is replaced by the cross entropy, then
\eqref{eq:theta_tau_BaMAML_theta0} reduces to 
\begin{align}
  &\hat{p} (\btheta_{\tau} \mid \mathcal{D}_{\tau}^{{\rm trn}}, \btheta_0)
  = \delta (\btheta_{\tau}
  - \hat{\btheta}_{\tau}^{\rm map}), \\
  \nonumber
  &\mathrm{with}\,\,\hat{\btheta}_{\tau}^{\rm map} = {\arg\max}_{\btheta_{\tau}} {p} (\btheta_{\tau} \mid \mathcal{D}_{\tau}^{{\rm trn}}, \btheta_0).
\end{align}
\end{remark}

\subsection{Meta Linear Regression} 
\label{sub:solutions}

\paragraph{Data model.}
Under the meta linear regression setting, 
with the feature $\mathbf{x}_{\tau} \in \mathbb{R}^d$, the target $y_{\tau} \in \mathbb{R}$, 
and the ground truth parameter of task $\tau$, $\btheta_{\tau}^{\mathrm{gt}} \in \mathbb{R}^d$, 
we assume the data generation model for task $\tau$ is 
\begin{equation}\label{eq:linear_data_generate}
  y_{\tau}={\btheta}^{\mathrm{gt}\top}_{\tau} \mathbf{x}_{\tau}+\epsilon_{\tau},{\rm with}~\epsilon_{\tau} \stackrel{\text{iid}}{\sim} \mathcal{N}\left(0, \sigma_{\tau}^{2}\right), 
  \mathbf{Q}_{\tau}\coloneqq \mathbb{E}[\mathbf{x}_{\tau} \mathbf{x}_{\tau}^{\top}].
\end{equation}


Given the estimate of ${\btheta}^{\mathrm{gt}}_{\tau}$ denoted as $\hat{\btheta}^{\mathcal{A}}_{\tau}$, then the conditional probability $p(y_{\tau}\mid \mathbf{x}_{\tau}, \hat{\btheta}^{\mathcal{A}}_{\tau}) = \mathcal{N}(\hat{\btheta}^{\mathcal{A}\top}_{\tau} \mathbf{x}_{\tau}, \sigma_{\tau}^2)$.
Thus, ignoring the constant, the negative log likelihood in~\eqref{eq:l_tau}, $ -\log p(\mathbf{y}_{\tau}^{\rm val}\mid \mathbf{X}_{\tau}^{\rm val}, {\btheta}_{\tau}) $,  becomes the squared error $ \|\mathbf{y}_{\tau}^{\rm val} - \mathbf{X}_{\tau}^{\rm val} {\btheta}_{\tau}\|^2$. 
Note that $\sigma_{\tau}$ depends on task $\tau$ generally, but does not pose challenges to analysis, therefore we assume $\sigma_{\tau} = 1$  in this paper for simplicity.

By plugging $\hat{\btheta}_{\tau}^{\cal A}$ into~\eqref{eq:R},  and with the squared error as the meta-linear regression loss, the empirical loss, meta-test risk along with their optimal solutions can be computed analytically with closed-form, whose derivations are deferred to the appendix.
We summarize the results for different methods in Proposition~\ref{prop:solutions}, 
where the optimal solutions for MAML
 derived in previous work~\citep{gao2020_model_opt_tradeoff_ml} are also included.

\begin{proposition}
\label{prop:solutions}
\emph{\textbf{(Empirical and population level solutions)}} Under data model \eqref{eq:linear_data_generate},
the meta-test risk of method $\mathcal{A}$ can be computed by
\begin{align}
   &\mathcal{R}^{\mathcal{A}}({\btheta_0} )
    \label{eq:meta_risk_form}
    = \mathbb{E}_{\tau}\big[\|\btheta_0-\btheta^{\mathrm{gt}}_{\tau}\|^2_{\mathbf{W}_{\tau}^{\mathcal{A}}}\big] + 1.
 \end{align} 
The optimal solutions to the meta-test risk and empirical loss are given below respectively
\begin{subequations}
\begin{align}\label{eq:theta_0_star_A_sln}
  \hspace{-2mm}&\btheta_{0}^{\mathcal{A}} 
  \!\coloneqq\! \mathop{\arg\min}_{\btheta_0}
  \mathcal{R}^{\mathcal{A}}({\btheta_0} )
  \!=\!\mathbb{E}_{\tau}\big[\mathbf{W}_{\tau}^{\mathcal{A}}\big]^{-1} \mathbb{E}_{\tau}\big[\mathbf{W}_{\tau}^{\mathcal{A}} {\btheta}_{\tau}^{gt}\big] \\
  \label{eq:theta_0_hat_A_sln}
  \hspace{-2mm}&\hat{\btheta}_{0}^{\mathcal{A}} 
  \coloneqq \mathop{\arg\min}_{\btheta_0}
  \mathcal{L}^{\mathcal{A}} (\btheta_0, \mathcal{D} ) \nonumber \\
  \hspace{-2mm}&\quad ~ = \Big(\sum_{\tau=1}^{T}
  \hat{\mathbf{W}}_{\tau}^{\mathcal{A}}\Big)^{-1}
  \Big(\sum_{\tau=1}^{T}\hat{\mathbf{W}}_{\tau}^{\mathcal{A}}\btheta_{\tau}^{\text{gt}} 
  \Big)
  + \Delta_{T}^{\mathcal{A}}  . 
\end{align}
\end{subequations}
where the error term $\Delta_{T}^{\mathcal{A}}$ is a polynomial function of $T,N,d$ caused by the noise $\epsilon$, and specified in the appendix. 
And $\hat{\mathbf{Q}}_{\tau,N} \coloneqq \frac{1}{N} \mathbf{X}^{\mathrm{all} \top}_{\tau}\mathbf{X}_{\tau}^{\mathrm{all}}$, $s = N_1 / N$.
The weight matrices of different methods, $\mathbf{W}_{\tau}^{\mathcal{A}}$ and $\hat{\mathbf{W}}_{\tau}^{\mathcal{A}}$, are given in Table~\ref{tab:weight_matrices}.

\end{proposition}

Note that, in the meta linear regression case in Proposition~\ref{prop:solutions}, BaMAML further assumes the prior distribution $\btheta_{\tau} \sim \mathcal{N}(\btheta_{0}, 1/\gamma_b)$ with $\gamma_b = \gamma N_1$, resulting in the weight matrices $\mathbf{W}_{\tau}^{\mathrm{ba}}, \hat{\mathbf{W}}_{\tau}^{\mathrm{ba}}$ in Table~\ref{tab:weight_matrices} depending on $\gamma$ and $s$.
The posterior follows a Gaussian distribution,
$p(\btheta_{\tau}\mid \mathcal{D}_{\tau}^{{\rm trn}}, \btheta_0)
= \mathcal{N}(\mu_{\btheta_{\tau}}, \Sigma_{\btheta_{\tau}}) $,
 where the parameters $\Sigma_{\btheta_{\tau}}$ and $\mu_{\btheta_{\tau}}$ are given by
 \begin{subequations}
\begin{align}
  \Sigma_{\btheta_{\tau}}&= (  N_1 \hat{\mathbf{Q}}_{\tau,N_1} + \gamma_b  \mathbf{I} )^{-1} , \\
  \mu_{\btheta_{\tau}}&= \Sigma_{\btheta_{\tau}} (   \mathbf{X}_{\tau}^{{{\rm trn}}\top} \mathbf{y}_{\tau}^{{\rm trn}} + \gamma_b \btheta_0).
\end{align}
\end{subequations}
If $p(\btheta_{\tau}\mid \mathcal{D}_{\tau}^{{\rm trn}}, \btheta_0) \in \mathcal{Q}$,
then $\hat{p} (\btheta_{\tau} \mid \mathcal{D}_{\tau}^{{\rm trn}}, \btheta_0) = p(\btheta_{\tau}\mid \mathcal{D}_{\tau}^{{\rm trn}}, \btheta_0)$, which holds for the meta linear regression case analyzed in this paper,
with $\cal Q$ specified as the set of Gaussian distributions.

Next, we will use the closed-form solutions of different methods in Proposition~\ref{prop:solutions} to compute their generalization errors in Section~\ref{sec:generalization_error_analysis}.

\section{META-TEST RISK ANALYSIS} 
\label{sec:generalization_error_analysis}

In this section, we will compare the meta-test risk of MAML and BaMAML.
By the definition of the meta-test risk $\mathcal{R}^{\mathcal{A}}$ in~\eqref{eq:R},
it can be decomposed into the optimal population risk and  statistical errors, as summarized in Proposition~\ref{prop:meta_test_risk_decompose}.
\begin{tcolorbox}[emphblock]
\begin{proposition}[Meta-test risk decomposition]
\label{prop:meta_test_risk_decompose}
In meta-linear regression, the meta-test risk for method $\mathcal{A}$ can be decomposed into optimal population risks and statistical errors, given by
\begin{equation}\label{eq:meta_test_risk_decompose_final}
 \hspace{-2mm}
 \mathcal{R}^{\mathcal{A}}(\hat{\btheta}_0^{\mathcal{A}})
 = \hspace{-2mm}
  \underbracket{  \mathcal{R}^{\mathcal{A}}({\btheta}_0^{\mathcal{A}})}_{\text{\tiny \makecell{optimal\\population~risk}}}
  +
  \underbracket{\|\hat{\btheta}_0^{\mathcal{A}}- {\btheta}_0^{\mathcal{A}}\|^2_{\mathbb{E}_{\tau}[\mathbf{W}_{\tau}^{\mathcal{A}}]}}_{\text{statistical~error}~\mathcal{E}^{2}_{\mathcal{A}} (\hat{\btheta}_0^{\mathcal{A}})}
  .
\end{equation}
\end{proposition}
\end{tcolorbox}
Invoking the definition of $\btheta_0^{\mathcal{A}}$ in~\eqref{eq:theta_0_star_A_sln} as the optimal solution for $\min_{\btheta_0} \mathcal{R}^{\mathcal{A}}(\btheta_0)$, the optimal population risk $\mathcal{R}^{\mathcal{A}}(\btheta_0^{\mathcal{A}})$ is defined as the minimum meta-test risk, 
which captures the error resulting from limited model adaptation capacity.
On the other hand, the statistical error captures error resulting from using finite samples instead of population statistics.
We will next show that both errors are smaller under BaMAML than those under MAML in Sections~\ref{sub:model_error_analysis} and~\ref{sub:statistical_error_analysis}.

\subsection{Optimal Population Risk Analysis} 
\label{sub:model_error_analysis}

We first analyze and compare the optimal population risk of different methods.
Before proceeding to the theoretical results,
we make the following basic assumptions.
\begin{assumption}[Bounded eigenvalues]
\label{asmp:bounded_data_matrix_eigenvalues}
  For any $\tau$, $0<\ushort{\lambda}\leq\lambda(\mathbf{Q}_{\tau})\leq \bar{\lambda}$, where $\lambda(\mathbf{Q}_{\tau})$ represents the eigenvalues of $\mathbf{Q}_{\tau}$.
\end{assumption}

\begin{assumption}
  \emph{\textbf{(Sub-gaussian task parameter and bounded features)}}
  \label{asmp:sub_gaussian_task_para}
  The ground truth parameter 
  $\btheta_{\tau}^{\mathrm{gt}}$ is independent of $\mathbf{X}_{\tau}$ and satisfies
  that the individual entries $\big\{\btheta_{\tau, i}^{\mathrm{gt}}-\btheta_{0, i}^{\mathcal{A}}\big\}_{i \in[d], \tau \in[T]}$ are independent and $\mathcal{O}( R / \sqrt{d})$-sub-gaussian.
  In addition, $\|\mathbb{E}[\btheta_{\tau}^{\mathrm{gt}}-\btheta_{0}^{\mathcal{A}}]\| \leq M $.
  The inputs $\|\mathbf{x}_{\tau,i}\| \leq K$. 
  $R,K$ are constants.
\end{assumption}

Note that, 
these assumptions can be easily satisfied in data generation model~\eqref{eq:linear_data_generate} by controlling the hyperparameters.
And they  are also standard in analyzing the optimal population risks for meta-linear regression~\citep{gao2020_model_opt_tradeoff_ml,collins2020_task_landscape_erm_maml}.

Next we will show in Theorem~\ref{thm:bamaml_lower_model_err} that 
one can always find a range of the regularizer weight $\gamma$ such that BaMAML has smaller optimal population risk than MAML.
\begin{theorem}
\emph{\textbf{(Optimal population risks)}}
\label{thm:bamaml_lower_model_err}
In the meta-linear regression with data model~\eqref{eq:linear_data_generate},
recall that $\btheta_0^{\rm ma}$ and $\btheta_0^{\rm ba}$ are the minimizers of
$\mathcal{R}^{\mathrm{ma}}({\btheta};\alpha)$ and $\mathcal{R}^{\mathrm{ba}}({\btheta}; \gamma)$,
respectively.
Define $r^{\rm ma} \coloneqq \min_{\alpha}\mathcal{R} (\btheta_0^{\rm ma};\alpha) - 1 >0 $,
$C_{\btheta} \coloneqq \max\{ \big((M + \|\btheta_0^{\mathrm{im}} \|)^2 + {R^2}\big)^{\frac{1}{2}}, \big((M + \|\btheta_0^{\mathrm{ma}} \|)^2 + {R^2}\big)^{\frac{1}{2}}\}$.
Under Assumptions~\ref{asmp:bounded_data_matrix_eigenvalues}-\ref{asmp:sub_gaussian_task_para},
when $\gamma$ satisfies
\begin{align}
  0<\gamma < \big((r^{\mathrm{ma}})^{-\frac{1}{2}}
  C_{\btheta}\bar{\lambda}^{\frac{1}{2}} - 1 \big)^{-1}\ushort{\lambda} 
\end{align}
BaMAML has smaller optimal population risk, i.e.
\begin{align}
 \mathcal{R}^{\mathrm{ba}}(\btheta_0^{\mathrm{ba}}; \gamma) < \min_{\alpha}~ \mathcal{R}^{\mathrm{ma}}(\btheta_0^{\mathrm{ma}}; \alpha ) .  
\end{align}

\end{theorem}

{Theorem~\ref{thm:bamaml_lower_model_err} states that regardless of the choice of $\alpha$, we can always find $\gamma > 0$ such that the BaMAML method has smaller meta-test risk than the MAML method.
}

{Note that, the choice of $\gamma $  represents trade-off between adaptation speed and optimal population risk, because $\btheta_{\tau}^{\rm ba}(\btheta_0)$ is a weighted average of the prior $\btheta_0$ and the ground truth paramter $\btheta_{\tau}^{\mathrm{gt}}$. The larger $\gamma$, the higher weight for the prior $\btheta_0$, then the closer the initial parameter $\btheta_0$ is to the optimal $\btheta_{\tau}^{\rm ba}(\btheta_0)$, and the faster the adaptation speed. On the other hand, the larger $\gamma$, the larger the optimal population risk is.
This inspires us to select model hyperparameter
based on our practical needs
for the specific problem.
}
Combined with the optimal population risk  of ERM (or modeling error in the paper) established in~\cite{gao2020_model_opt_tradeoff_ml}, 
our Theorem~\ref{thm:bamaml_lower_model_err}  also implies that BaMAML has lower optimal population risk than ERM.



\subsection{Statistical Error Analysis} 
\label{sub:statistical_error_analysis}

We next study and compare the statistical errors of different methods defined in \eqref{eq:meta_test_risk_decompose_final}.
We first bound the statistical errors of MAML and BaMAML methods.

\begin{theorem}
\label{thm:bound_MAML_stats_err}
\emph{\textbf{(Statistical error of MAML)}}
Suppose Assumptions~\ref{asmp:bounded_data_matrix_eigenvalues}-\ref{asmp:sub_gaussian_task_para} hold. $T = \Omega (d)$, $M = \smallO(R/\sqrt{T})$.
Denote $\|\cdot\|_{\mathrm{op}}$ as the operator norm.
Define function
\begin{align*}\label{eq:C_0_xx}
  \hspace{-3mm} C_0^{\mathcal{A}} 
  \coloneqq &  
  [\inf_{\tau}\lambda_{\min}({\mathbf{W}}^{\mathcal{A}}_{\tau})]^{-1}
  [\sup_{\tau}\lambda_{\max}({\mathbf{W}}^{\mathcal{A}}_{\tau})]^{2}
  \hspace{-2mm} \numberthis
\end{align*}
and define $\varrho$ as a higher order term given by
\begin{align*}
\label{eq:Delta_high_order}
 \varrho
 =& \frac{1}{T}\big(1+ \frac{d}{N}\big) \big(\widetilde{\mathcal{O}}(\frac{1}{\sqrt{d}}) + \widetilde{\mathcal{O}}(\sqrt{\frac{d}{T}})\big)\\
 &+ 
 \Big(\widetilde{\mathcal{O}}(\sqrt{\frac{d}{T}}) + 
 \widetilde{\mathcal{O}}(\frac{d}{N})
 \Big) M^2 + \frac{1}{T}\widetilde{\mathcal{O}}(\frac{d}{N})
 \numberthis
\end{align*}
where $\widetilde{\mathcal{O}}(\cdot)$ hides $\log (TNd)$ factor.
With probability at least $1-Td^{-10}$, we have
\begin{align*}
\mathcal{E}_{\mathrm{ma}}^2 (\hat{\btheta}_0^{\mathrm{ma}})
\leq 
\frac{R^2}{T} 2C_0^{\mathrm{ma}}
+\frac{d}{TN} 2C^{\mathrm{ma}}_{1}
+ \varrho
\numberthis 
\end{align*}
where $C^{\mathrm{ma}}_0$ is given by~\eqref{eq:C_0_xx}, and
\begin{subequations}
\begin{align}
\nonumber
  C^{\mathrm{ma}}_0 = 
  &(1-\alpha \ushort{\lambda}   )^{4} 
  (1-\alpha \bar{\lambda}   )^{-2} 
  \ushort{\lambda}^{-1}\bar{\lambda}^2
  \label{eq:C0_ma_bound} \\
  C_{1}^{\mathrm{ma}} 
= & s^{-1} + (1-s)^{-1}(1-\alpha \bar{\lambda }  )^{-2}
 \alpha^2 \bar{\lambda}^3
\ushort{\lambda}^{-1}   .
\end{align}
\end{subequations}

\end{theorem}

Note that $M= \smallO(R/\sqrt{T})$ can be achieved when different tasks have similar $\mathbf{Q}_{\tau}$, for example, when input feature normalization is performed.
Analogous to Theorem~\ref{thm:bound_MAML_stats_err}, we bound the BaMAML statistical error next.

\begin{theorem}
\label{thm:bound_BaMAML_stats_err}
\emph{\textbf{(Statistical error of BaMAML)}}
Suppose Assumptions~\ref{asmp:bounded_data_matrix_eigenvalues}-\ref{asmp:sub_gaussian_task_para} hold. 
With probability at least $1-Td^{-10}$, 
we have
\begin{align*}
\mathcal{E}_{\mathrm{ba}}^2 (\hat{\btheta}_0^{\mathrm{ba}})
\leq 
\frac{R^2}{T} 2C^{\mathrm{ba}}_0
+\frac{d}{TN} 2C^{\mathrm{ba}}_{1} + \varrho
\numberthis
\end{align*}
where 
$\varrho$ is given by~\eqref{eq:Delta_high_order},
$C^{\mathrm{ba}}_0$ is given by~\eqref{eq:C_0_xx}, 
and 
\begin{subequations}
\begin{align}
\nonumber
  C^{\mathrm{ba}}_0  = 
  &(1+\gamma^{-1} \ushort{\lambda})^{-4} 
  (1+(\gamma s)^{-1} \bar{\lambda} )^2 
  \ushort{\lambda}^{-1}\bar{\lambda}^2  
  \label{eq:C0_ba_bound} \\
  C_{1}^{\mathrm{ba}} = & 1 .
\end{align}
\end{subequations}

\end{theorem}
Theorems~\ref{thm:bound_MAML_stats_err} and~\ref{thm:bound_BaMAML_stats_err}
show that the statistical errors of MAML and BaMAML have similar decreasing rates, that is, $\mathcal{O}(T^{-1})$ and $\mathcal{O}(N^{-1})$.
The difference lies in their coefficients.
For the dominating constants 
$C_0^{\mathrm{ma}}$ in~\eqref{eq:C0_ma_bound} and $C_0^{\mathrm{ba}}$ in~\eqref{eq:C0_ba_bound}, 
given any $\alpha$, choose 
\begin{align*}
  \gamma < \min \{\bar{\lambda}, \frac{1}{2} \bar{\lambda}^{-1} \ushort{\lambda}^{2} s 
  (1-\alpha \bar{\lambda} )^{2} 
  (1-\alpha \ushort{\lambda})^{-1} \} 
  \numberthis
\end{align*} 
then $C_0^{\mathrm{ma}} > C_0^{\mathrm{ba}}$.
In terms of the dependence on $N$, 
given any $\alpha$, since $C_1^{\mathrm{ma}} > s^{-1} > 1$,
thus $C_1^{\mathrm{ma}} > C_1^{\mathrm{ba}}$,
i.e. MAML has larger coefficients than BaMAML.
Therefore the statistical error of BaMAML is lower when $N$ is small, which is typical in few-shot learning.
Nevertheless, Theorems~\ref{thm:bound_MAML_stats_err} and~\ref{thm:bound_BaMAML_stats_err} only give the worst-case upper bounds of the statistical errors of two methods, which can be inaccurate in some cases.
To precisely characterize the statistical errors of BaMAML and MAML, 
we will provide sharper analysis next based on an additional assumption.

\subsection{Sharp Statistical Error Analysis} 
\label{sub:S_statistical_error_analysis}

To precisely quantify the dominating constants in the statistical error,
we further make assumptions on the task and data distributions.

\begin{assumption}[Linear centroid model]
\label{asmp:linear_centroid_model}
1) The inputs are standard Gaussian: $\mathbf{x}_{\tau, i} \stackrel{\mathrm{iid}}{\sim} \mathcal{N}\left(\mathbf{0}, \mathbf{I}_{d}\right)$. 
Then $\mathbf{Q}_{\tau} = \mathbf{I}_d$, therefore $\mathbf{W}^{\cal{A}}_{\tau} = w_{\cal{A}} \mathbf{I}_d$.
This implies that for different methods, 
the optimal initial parameters are the same, that is, $\btheta_0^* = \mathbb{E}_{\tau}[\btheta_{\tau}^{\mathrm{gt}}]$.
2) The ground truth parameter 
$\btheta_{\tau}^{\mathrm{gt}}$ is independent of $\mathbf{X}_{\tau}$ and satisfies
\begin{align}
\mathbb{E}_{\btheta_{\tau}^{\mathrm{gt}}}\big[\big(
\btheta_{\tau}^{\mathrm{gt}}-
\btheta_{0}^{*}\big)
\big(\btheta_{\tau}^{\mathrm{gt}}-
\btheta_{0}^{*}\big)^{\top}\big]
=\frac{R_{}^{2}}{d} \mathbf{I}_{d}
\end{align}
where $R$ is a constant, and the individual entries $\{\btheta_{\tau, i}^{\mathrm{gt}}-\btheta_{0, i}^{*}\}_{i \in[d], \tau \in[T]}$ are i.i.d. mean-zero and $\mathcal{O}( R / \sqrt{d})$-sub-gaussian.
\end{assumption}

Note that  Assumption~\ref{asmp:linear_centroid_model} has also been used in~\cite{bai2021_trntrn_trnval,denevi2018_l2l_linear_centroid}, 
whereas we do not make the noiseless realizable assumption compared to~\cite{bai2021_trntrn_trnval}, thus less restrictive.
Based on this assumption, we can obtain the dominating constant exactly, as stated in Theorems~\ref{thm:concentration_MAML_statistical_error} and \ref{thm:concentration_BaMAML_statistical_error}.

\begin{theorem}
\label{thm:concentration_MAML_statistical_error}
\emph{\textbf{(Statistical error of MAML)}}
Suppose Assumptions~\ref{asmp:bounded_data_matrix_eigenvalues},\ref{asmp:linear_centroid_model} hold,
 $T=\Omega(d), d / N= \eta > 0$, and $\alpha > 0$.
 Define 
 $$w_{\cal A} \coloneqq \frac{1}{d} \mathrm{tr}(\mathbb{E}[\mathbf{W}_{\tau}^{\cal A}]), ~~
 \tilde{C}_0^{\cal A} \coloneqq \frac{1}{d}\big\langle\mathbb{E}^{-2}\big[\hat{\mathbf{W}}_{\tau}^{\mathcal{A}}\big], \mathbb{E}\big[(\hat{\mathbf{W}}_{\tau}^{\mathcal{A}})^{2}\big]\big\rangle$$
With probability at least $1-T{d}^{-10}$, the statistical error in \eqref{eq:meta_test_risk_decompose_final} under MAML satisfies
\begin{align*}
\mathcal{E}_{\mathrm{ma}}^2 (\hat{\btheta}_0^{\mathrm{ma}}) =
\frac{R^{2}}{T} 
 w_{\mathrm{ma}}\tilde{C}^{\mathrm{ma}}_0 
+\frac{d}{TN}
w_{\mathrm{ma}}\tilde{C}^{\mathrm{ma}}_{1}
+ \varrho 
\numberthis 
\end{align*}
where 
$\varrho$ is given by~\eqref{eq:Delta_high_order}.
The dominating constant $\tilde{C}_0^{\mathrm{ma}}$ satisfies
\begin{align}
\inf _{\text{\tiny $\makecell{\alpha >0 \\ s \in(0,1)}$}} \lim _{\text{\tiny $\makecell{d, N \rightarrow \infty\\ d / N \rightarrow \eta}$}} \tilde{C}_0^{\mathrm{ma}}=1 + \eta .
\end{align}

\end{theorem}

\begin{figure*}[t]
  \centering
  \includegraphics[width=0.99\linewidth]{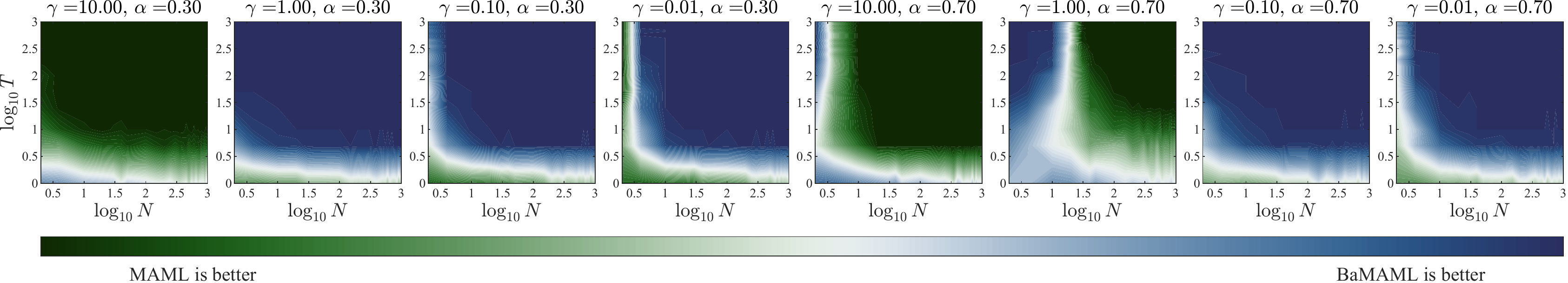}
  \caption{Contour plots of the probability that the \textcolor{BAMAML_cl}{BaMAML} estimate has lower expected loss than the \textcolor{MAML_cl}{MAML} estimate 
  The axes are the log number of tasks ($\log_{10} T$) and the log number of data points ($\log_{10} N$) used for meta-test optimization, 
  and the values of $\alpha,\gamma$ are given in subfigure titles.
  }
  \label{fig:linear_regress_contour}
\end{figure*}

Similarly, we can obtain the concentration of the statistical error of BaMAML.

\begin{theorem}
\label{thm:concentration_BaMAML_statistical_error} 
\emph{\textbf{(Statistical error of BaMAML)}}
Suppose Assumptions~\ref{asmp:bounded_data_matrix_eigenvalues},\ref{asmp:linear_centroid_model} hold,
$T=\Omega(d), d / N=\eta > 0$, and $\gamma >0$. Then with probability at least $1-T{d}^{-10}$, the statistical error in \eqref{eq:meta_test_risk_decompose_final} under BaMAML satisfies
\begin{align*}
\mathcal{E}_{\mathrm{ba}}^2 (\hat{\btheta}_0^{\mathrm{ba}})
=
\frac{R^{2}}{T} w_{\mathrm{ba}} \tilde{C}^{\mathrm{ba}}_0 
+\frac{d}{TN}
w_{\mathrm{ba}}\tilde{C}^{\mathrm{ba}}_{1} 
+ \varrho
\numberthis
\end{align*}
where 
$\varrho$ is given by~\eqref{eq:Delta_high_order}.
In addition, the dominating constant $\tilde{C}_0^{\mathrm{ba}} $ satisfies
\begin{equation}
\inf _{\tiny \makecell{\gamma>0 \\ s \in (0,1)}} \lim _{\text{\tiny $\makecell{d, N \rightarrow \infty\\ d / N \rightarrow \eta}$}} \tilde{C}_0^{\mathrm{ba}}
\begin{cases}
  =1, & \eta \in (0, 1], \\
  \leq \eta, & \eta \in (1, \infty).
\end{cases} 
\end{equation}
\end{theorem}

Theorems~\ref{thm:concentration_MAML_statistical_error} and \ref{thm:concentration_BaMAML_statistical_error}
state that when $T,d$ are large and $T = \tilde{\Omega}(d)$,
the statistical errors of MAML and BaMAML
are dominated by $R^2/T$ times 
$\tilde{C}^{\mathrm{ma}}_0$ and $\tilde{C}^{\mathrm{ba}}_0$, respectively.
Therefore we can compare the statistical errors of MAML and BaMAML based on the optimal hyperparameters $\alpha, \gamma$ and split ratio $s$ below.
\begin{corollary}
\label{crlr:compare_maml_bamaml}
\emph{\textbf{(Dominating constants in statistical errors)}}
The dominating constants in the statistical errors of MAML and BaMAML satisfy
\begin{equation}\label{eq:constant_compare_maml_bamaml}
 \inf_{\text{\tiny $\makecell{\alpha >0 \\ s \in(0,1)}$}} \lim _{\text{\tiny $\makecell{d, N \rightarrow \infty\\ d / N \rightarrow \eta}$}} 
 \tilde{C}^{\mathrm{ma}}_0 >
 \inf_{\text{\tiny $\makecell{\gamma >0 \\ s \in(0,1)}$}} \lim _{\text{\tiny $\makecell{d, N \rightarrow \infty\\ d / N \rightarrow \eta}$}}
\tilde{C}^{\mathrm{ba}}_0.
\end{equation}
\end{corollary}
Corollary~\ref{crlr:compare_maml_bamaml} justifies the provable benefit of BaMAML in terms of strictly smaller statistical error, contributing to smaller meta-test risk. This is achieved in the high dimension limit regime as $d,N \to \infty$ and $d/N \to \eta$, which is common in the overparameterized case.
 



\section{EXPERIMENTS} 
\label{sec:experiments}


In this section, we present empirical experiments on synthetic and real datasets to verify our theorems.
For synthetic datasets, we perform linear and sinusoidal regression.
For real datasets, we use miniImageNet.
By default, the experiments are repeated 5 times with the average, best and worst performance displayed.
In our experiments, we also use ERM as a baseline for comparison.
In the meta learning setting, 
ERM minimizes the average loss over all data,
its meta-test risk and optimal solutions can be obtained by taking $\alpha = 0, N_1 = 0, N_2=N$ in that of MAML~\citep{gao2020_model_opt_tradeoff_ml}. 

All our experiments are conducted on a workstation with
an Intel i9-9960x CPU with 128GB memory and four
NVIDIA RTX 2080Ti GPUs each with 11GB memory. 
Our experiments for linear synthetic data  are conducted on MATLAB R2021a with CPU only.
And our experiments for sinewave regression and real classification  are conducted on Python 3.7, PyTorch 1.9.1 with one GPU. 
More results can be found in the supplementary material.


\subsection{Linear Regression} 
\label{subs:linear_regression}

\noindent\textbf{Experiment settings.} For linear regression, we generate synthetic data according to the following task parameter
$
\label{eq:linear_synthetic_data_gen}
\mathbf{V} \sim U(\mathbb{SO}(d)),
{\btheta}_{\tau} \sim U([0,2]^{d})$,
${\lambda}_{\tau} \sim U([0.1,2]^{d}), 
\mathbf{Q}_{\tau} =  \mathbf{V}\mathrm{diag}({\lambda}_{\tau})\mathbf{V}^{\top},
\mathbf{x}_{\tau} \sim \mathcal{N}(0, \mathbf{Q}_{\tau}),
$
and data model in~\eqref{eq:linear_data_generate},
where $\mathbb{SO}(d)$ is the special orthogonal group in dimension $d$.


\begin{figure}[t]
  \centering
  \begin{subfigure}[t]{0.25\textwidth}
  \centering
  \includegraphics[width=.9\linewidth]{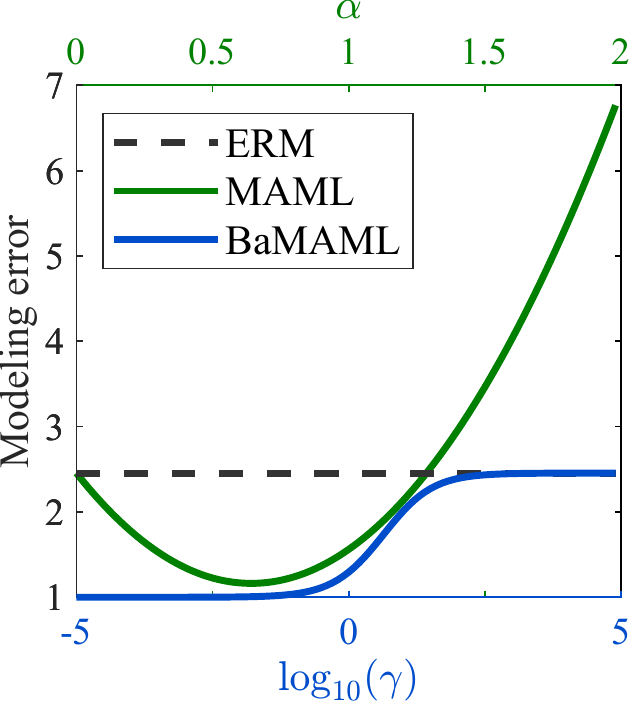}
  \caption{$\mathcal{R}(\btheta_0^{\mathcal{A}})$ v.s. $\alpha,\gamma$}
  \label{fig:modeling_err}
  \end{subfigure}%
  \begin{subfigure}[t]{0.25\textwidth}
  \centering
  \includegraphics[width=.9\linewidth]{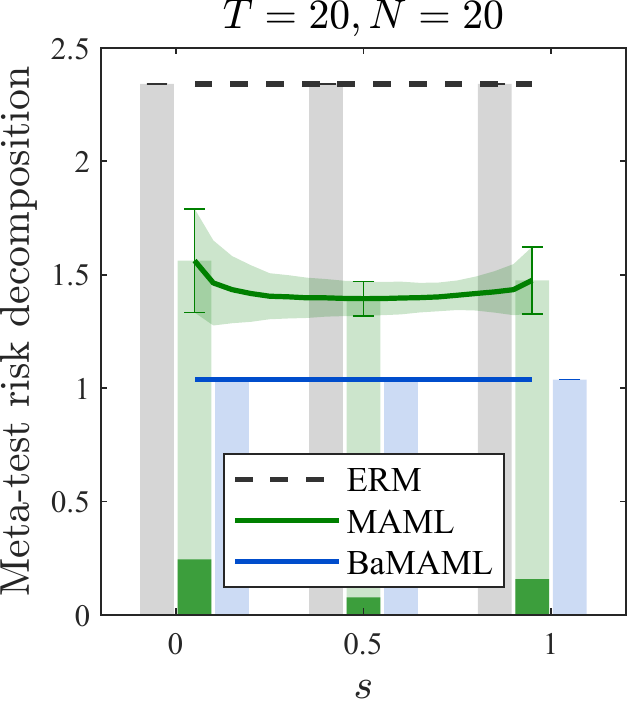}
  \caption{$\mathcal{R}(\hat{\btheta}_0^{\mathcal{A}})$ v.s. $s$}
  \label{fig:meta_test_risks_vs_s}
  \end{subfigure}
  \caption{Optimal population risks $\mathcal{R}(\btheta_0^{\mathcal{A}})$ v.s.  $\alpha,\gamma$,
  and meta-test risks $\mathcal{R}(\hat{\btheta}_0^{\mathcal{A}})$ v.s. train validation split ratio $s$ for \textcolor{ERM_cl}{ERM}, \textcolor{MAML_cl}{MAML} and \textcolor{BAMAML_cl}{BaMAML}. 
  In Figure~\ref{fig:meta_test_risks_vs_s}, the lighter color bars represent the meta-test risks and the darker color bars represent the statistical errors.
  }
  \label{fig:model_err_total_risk}
    \vspace{-0.4cm}
\end{figure}

\begin{figure}[t]
  \centering
  \includegraphics[width=0.99\linewidth]{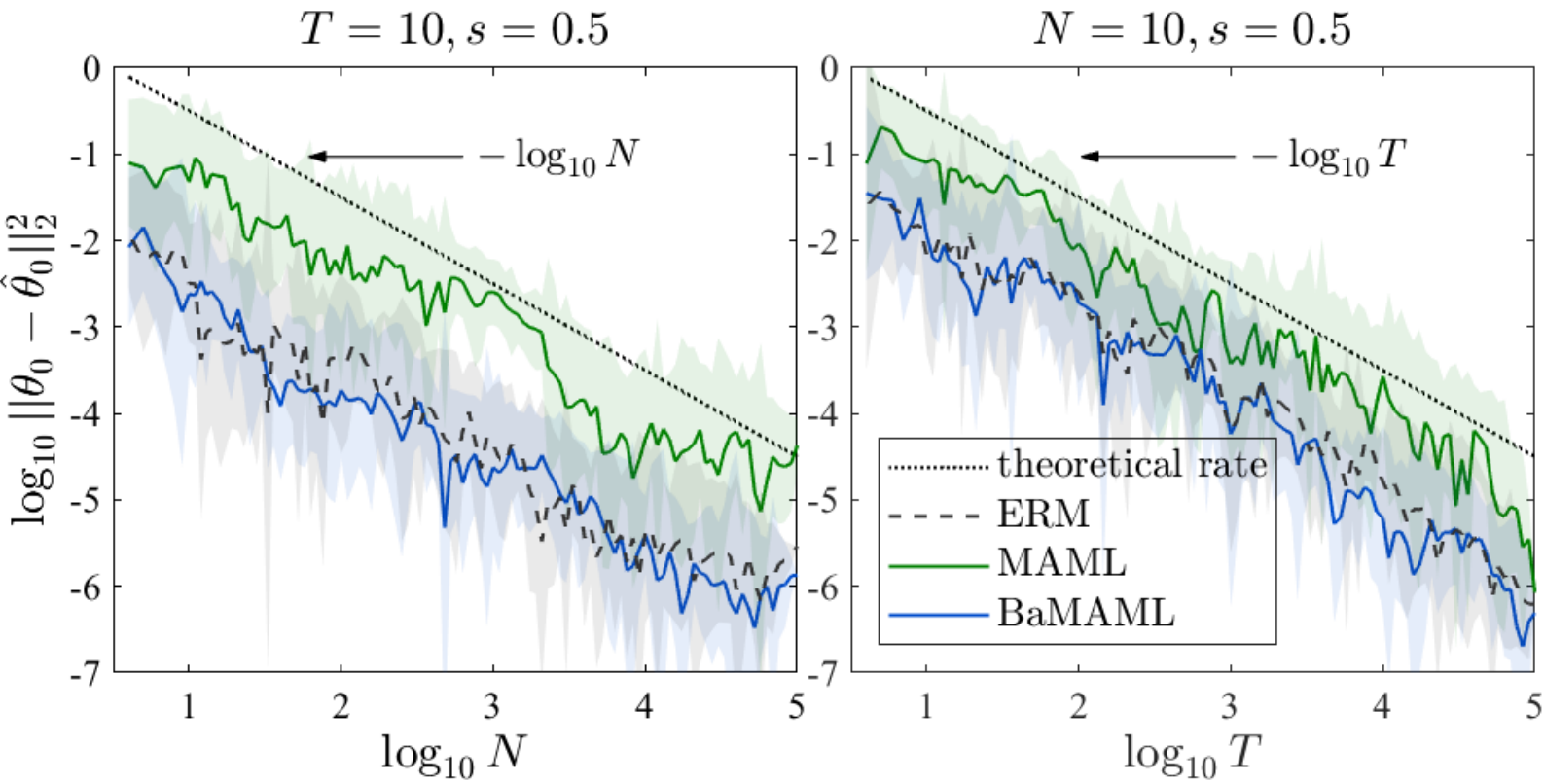}
  \caption{Statistical error v.s. $N$ and $T$ for \textcolor{ERM_cl}{ERM}, \textcolor{MAML_cl}{MAML} and \textcolor{BAMAML_cl}{BaMAML}. The  dotted line serves as a reference of the theoretical decay rate.
  } 
  \label{fig:stats_err_vs_NT}
  \vspace{-0.4cm}
\end{figure}

\begin{figure*}[t]
  \centering
  \begin{subfigure}[t]{0.19\textwidth}
  \centering
  \includegraphics[width=.95\linewidth]{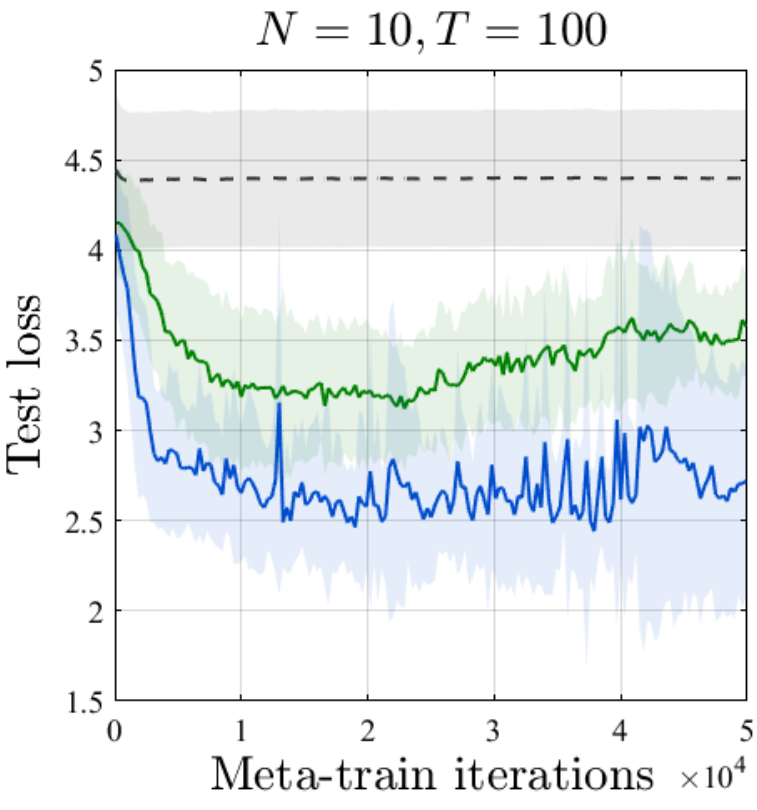}
  \caption{}
  \label{sfig:test_loss_T100_N10}
  \end{subfigure}%
  \begin{subfigure}[t]{0.19\textwidth}
  \centering
  \includegraphics[width=.95\linewidth]{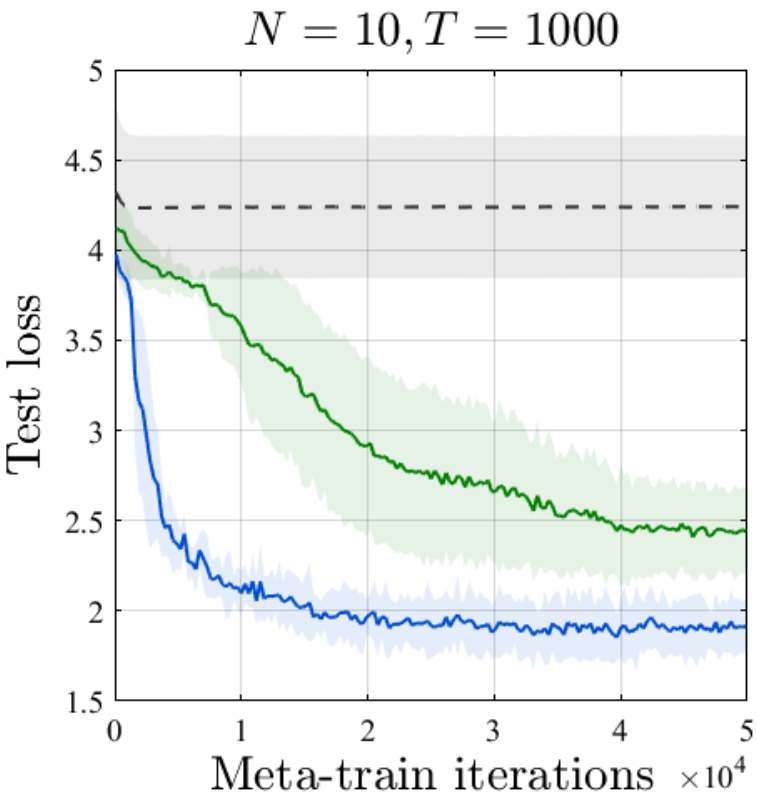}
  \caption{}
  \label{sfig:test_loss_T1000_N10}
  \end{subfigure}
  \begin{subfigure}[t]{0.19\textwidth}
  \centering
  \includegraphics[width=.95\linewidth]{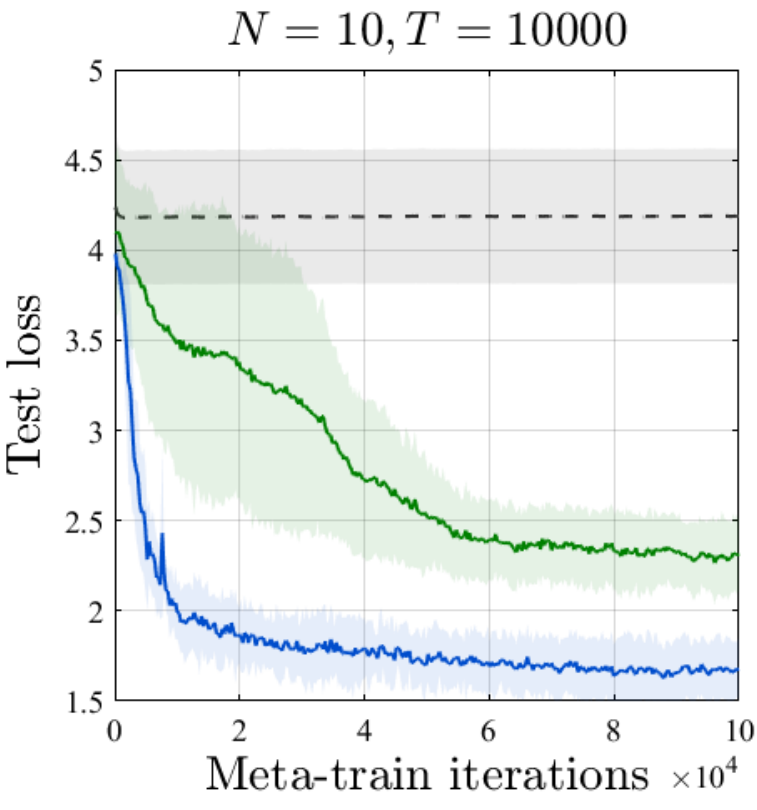}
  \caption{}
  \label{sfig:test_loss_T10000_N10}
  \end{subfigure}
  \begin{subfigure}[t]{0.19\textwidth}
  \centering
  \includegraphics[width=.95\linewidth]{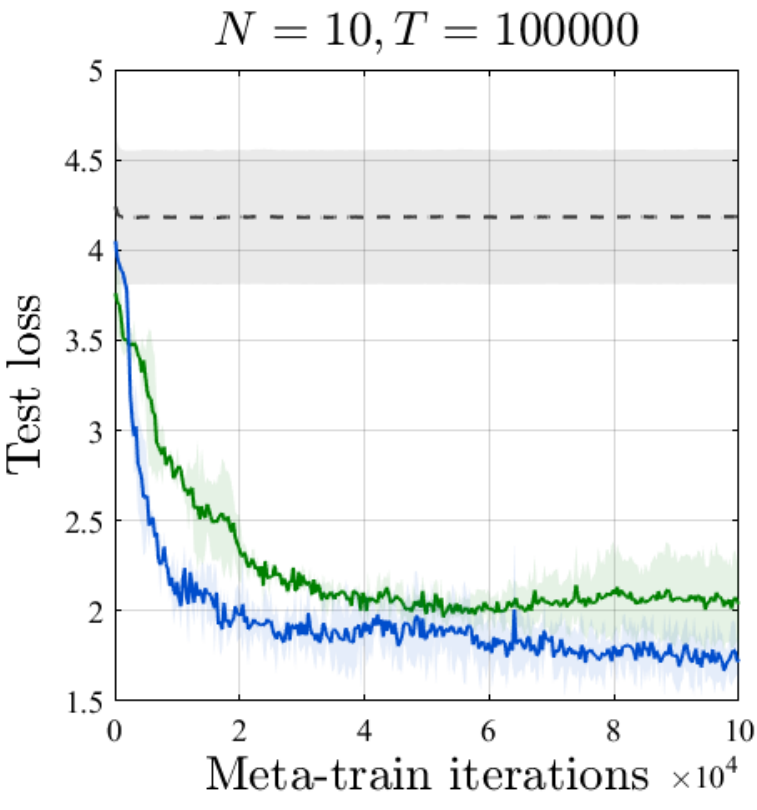}
  \caption{}
  \label{sfig:test_loss_T100000_N10}
  \end{subfigure}
  \begin{subfigure}[t]{0.19\textwidth}
  \centering
  \includegraphics[width=.95\linewidth]{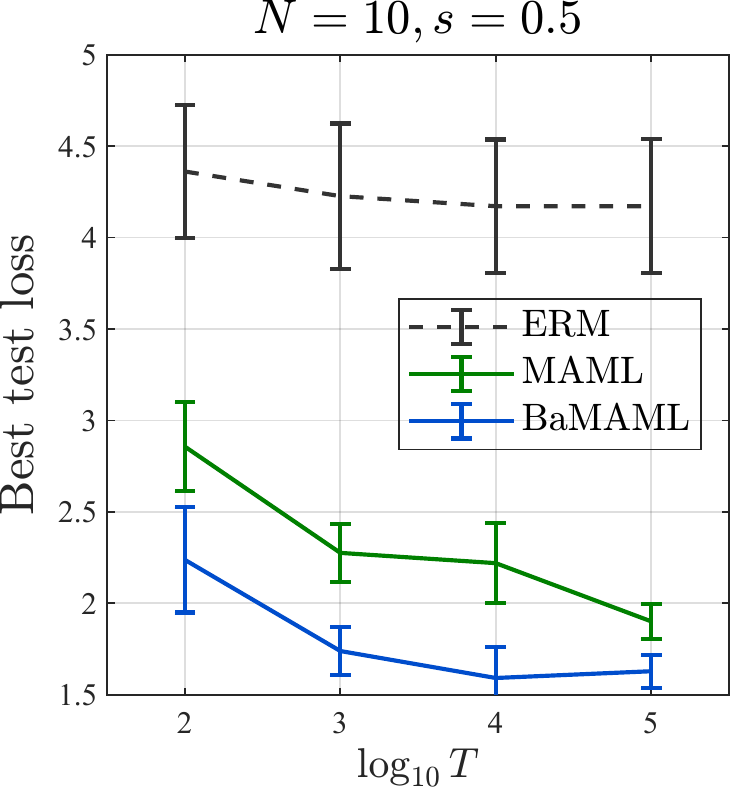}
  \caption{}
  \label{sfig:best_test_loss_T_N10}
  \end{subfigure}
  \begin{subfigure}[t]{0.19\textwidth}
    \centering
    \includegraphics[width=.95\linewidth]{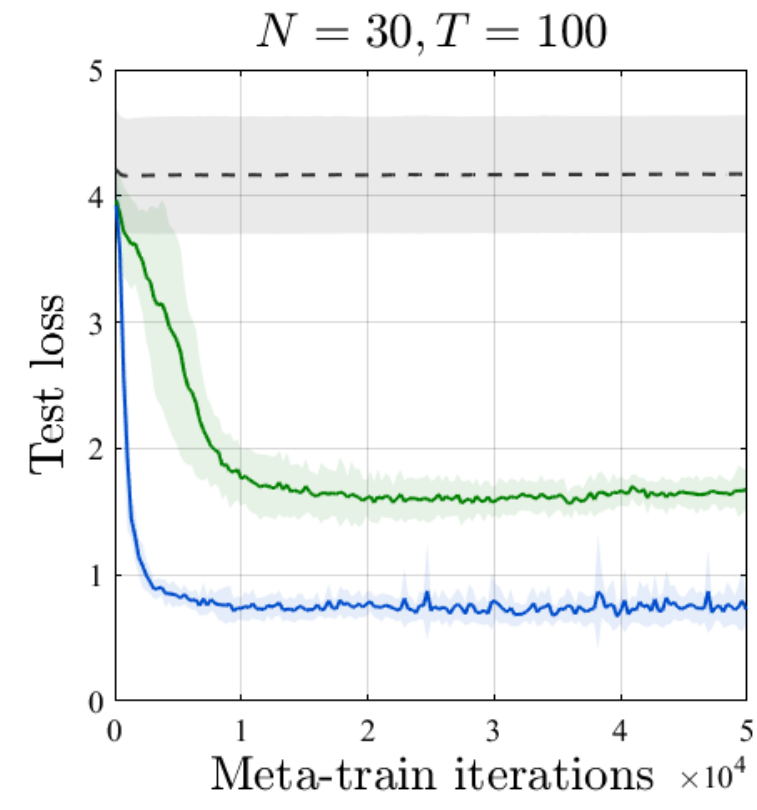}
    \caption{}
    \label{sfig:test_loss_T100_N30}
    \end{subfigure}
  \begin{subfigure}[t]{0.19\textwidth}
  \centering
  \includegraphics[width=.95\linewidth]{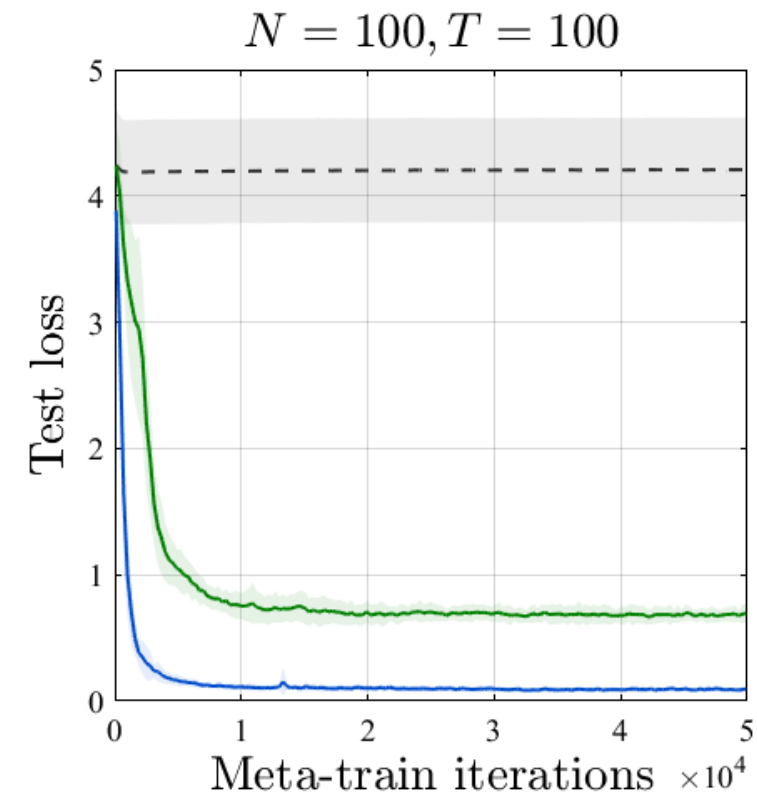}
  \caption{}
  \label{sfig:test_loss_T100_N100}
  \end{subfigure}
  \begin{subfigure}[t]{0.19\textwidth}
  \centering
  \includegraphics[width=.95\linewidth]{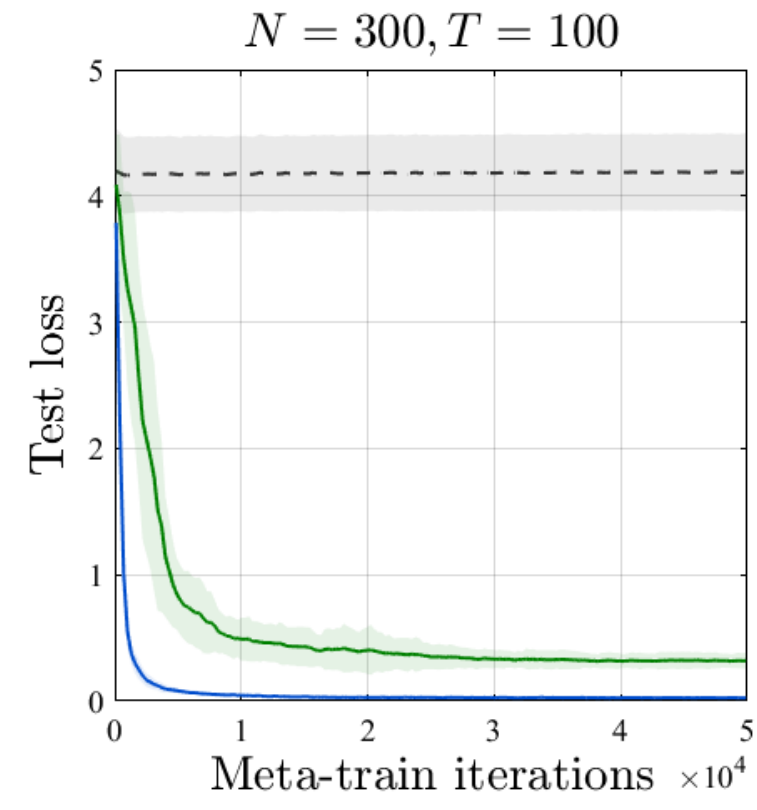}
  \caption{}
  \label{sfig:test_loss_T100_N300}
  \end{subfigure}
  \begin{subfigure}[t]{0.19\textwidth}
  \centering
  \includegraphics[width=.95\linewidth]{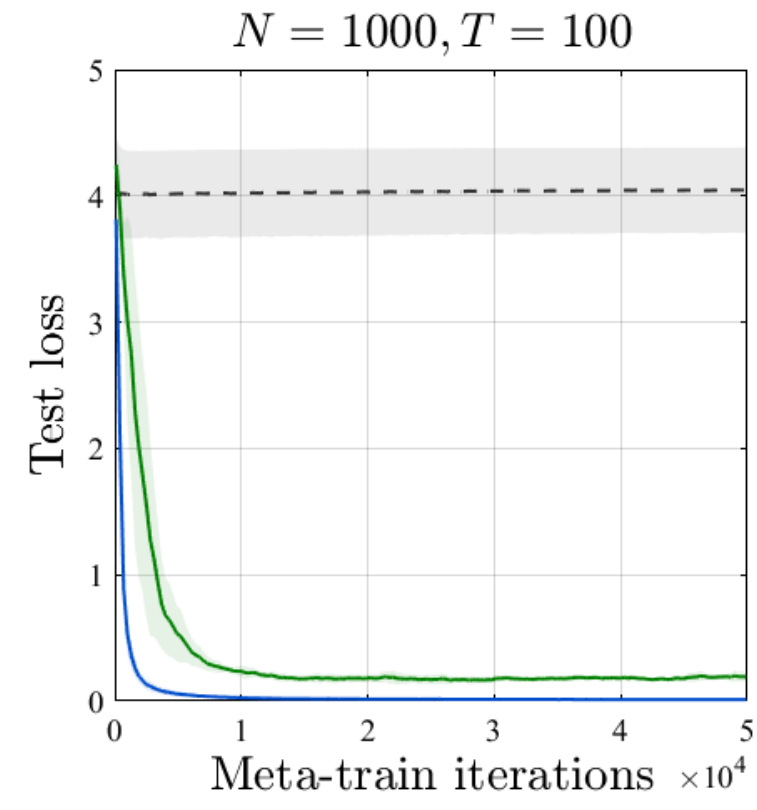}
  \caption{}
  \label{sfig:test_loss_T100_N1000}
  \end{subfigure}
  \begin{subfigure}[t]{0.19\textwidth}
  \centering
  \includegraphics[width=.95\linewidth]{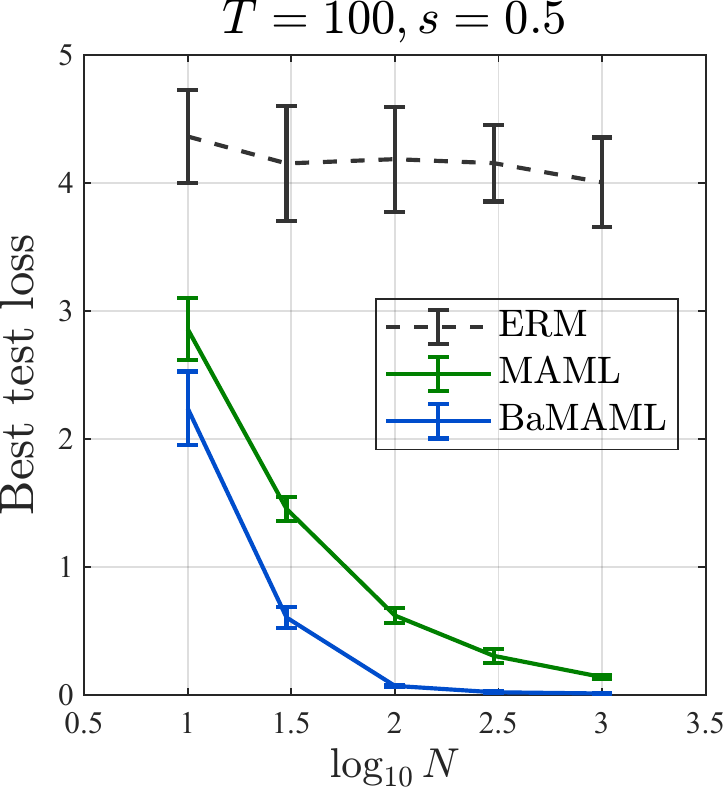}
  \caption{}
  \label{sfig:best_test_loss_T100_N}
  \end{subfigure}
  \caption{Testing mean squared error v.s. meta train iterations for \textcolor{ERM_cl}{ERM}, \textcolor{MAML_cl}{MAML}, 
  and \textcolor{BAMAML_cl}{BaMAML} in sinusoidal regression experiments with varied number of data ($N$) per task and the number of tasks ($T$) used for meta-training.
  }
  \label{fig:sin_regress}
\end{figure*}

\noindent\textbf{Results.} We present experiments for $d=1$. 
To compare the meta-test performance of MAML and BaMAML, we present contour plots of probability that BaMAML has lower loss than MAML in Figure~\ref{fig:linear_regress_contour},
where darker blue represents higher probability that BaMAML is better than MAML, and darker green vice versa.
The results indicate that with sufficient adaptation tasks or data,  and proper choice of $\gamma$, BaMAML performs better than MAML in terms of test error.

In Figure~\ref{fig:modeling_err}, we report the meta-test risks~\eqref{eq:R} for MAML and BaMAML under different hyperparameters 
$\alpha, \gamma$.
Figure~\ref{fig:modeling_err} shows that with proper choice of hyperparameter $\gamma$, BaMAML can achieve lower meta-test risk than MAML, verifying Theorem~1.
Also, when $\gamma \to 0$, the meta-test risk of BaMAML approaches 1, and when $\gamma \to \infty$, the meta-test risk of BaMAML approaches that of ERM.
This further demonstrates the trade-off between fast adaptation and optimal population risk (meta-test risk) that depends on $\gamma$.
On the other hand, MAML is relatively more sensitive to the step size. Too large step size $\alpha$ can lead to very large meta-test risk of MAML, going beyond that of ERM.
Besides, we can see from the empirical optimal solution that, 
contrary to non-Bayesian methods as analyzed in~\cite{bai2021_trntrn_trnval},
BaMAML is robust against different training and validation  data split. This is demonstrated in Figure~\ref{fig:meta_test_risks_vs_s}, where the meta-test risk of BaMAML remains unchanged when $s$ varies while that of MAML is more sensitive to the change in $s$ due to its statistical error, verifying Theorems~\ref{thm:bound_MAML_stats_err}-\ref{thm:concentration_BaMAML_statistical_error}.
Figure~\ref{fig:meta_test_risks_vs_s} also demonstrates the decomposition of meta-test risk into optimal population risk and statistical error, where the meta-test risks are mainly dominated by the optimal population risks in this case.


Next we fix $\alpha = 0.7, \gamma = 10^{-1}$, which are approximately optimal for each method.
We vary number of data samples ($N$) and 
number of tasks ($T$), and compare the statistical error of ERM, MAML, and BaMAML in Figure~\ref{fig:stats_err_vs_NT}.
The statistical errors of all methods decrease as the number of data samples increases.
When the number of data samples is small, MAML has the largest statistical error, followed by BaMAML and ERM.
Similar trends exist with increasing number of tasks.
Figure~\ref{fig:stats_err_vs_NT} have also shown the dotted black line as a reference, indicating the theoretical decay rate of the statistical errors. 
Since MAML and BaMAML has the same slope as the reference line,
it verifies the theoretical decay rate in Theorems~\ref{thm:bound_MAML_stats_err}-\ref{thm:concentration_BaMAML_statistical_error}.





\subsection{Sinusoidal Regression} 
\label{subs:sinusoidal_regression}
\noindent\textbf{Experiment settings.} For sinusoidal regression,
following~\cite{yoon2018_BMAML}, 
the $N$-shot dataset for each task is obtained from $x \sim U([-5.0,5.0])$ and then by computing its corresponding $y$ from the sinusoidal function $y=A \sin (w x+b)+\epsilon$, with task-dependent parameters amplitude $A$, frequency $w$, and phase $b$, and observation noise $\epsilon$. For each task, the parameters are sampled from $A \sim U([0.1,5.0]), b \sim U([0.0,2 \pi]), w \sim U([0.5,2.0]), \epsilon \sim \mathcal{N}\left(0,(0.01 A)^{2}\right)$.
For all experiments under this setting, we used a neural network with 3 layers, each of which consists of 40 hidden units.

\noindent\textbf{Results.} 
Figure~\ref{fig:sin_regress} shows the testing error v.s. the meta-train iterations for the compared methods,
where we can see that BaMAML converges to a point with lowest meta-test error.
For ERM and MAML, when $T$ is small (e.g. $T=100$), the meta test error decreases as the number of meta iterations increases in the beginning, 
but increases later, showing a tendency to overfit.
Besides,
the meta-test error decreases with increasing number of tasks or number of per-task data, with a similar trend as the linear regression even in the nonlinear sinewave regression.
As the number of tasks or number of per-task data increases, the performance gap between MAML and BaMAML reduces, demonstrating that BaMAML has more significant performance gain in limited data settings.

\section{CONCLUSIONS} 
\label{sec:conclusion}

In this paper, we study  what makes BaMAML provably better than MAML under the meta linear regression setting.
The meta-test risk can be decomposed into the optimal population risk and statistical error.
Our analysis shows that, with proper choice of hyperparameters, BaMAML has smaller optimal population risk than MAML, demonstrating better adaptation ability to new data.
And for statistical errors, MAML and BaMAML have the same dependence rate on the number of tasks and the number of data per task for training, while BaMAML has lower upper bound of the corresponding coefficients, thus lower upper bound of statistical error.
And in the high dimensional asymptotic regime, BaMAML has strictly smaller statistical error than MAML. 
The experiments on synthetic and real datasets corroborate our theoretical findings.
Building upon the current work, our future work includes analyzing the performance in nonlinear meta learning algorithms such as BaMAML with overparameterized neural networks.

\section*{Acknowledgments}
This work was supported by 
the Rensselaer-IBM AI Research Collaboration (\href{http://airc.rpi.edu}{http://airc.rpi.edu}), part of the IBM AI Horizons Network (\href{http://ibm.biz/AIHorizons}{http://ibm.biz/AIHorizons}), and the National Science Foundation under the project NSF 2134168.

\newpage
\bibliographystyle{plainnat}
\bibliography{myabrv,bmaml,maml_theory,maml}

\begin{thebibliography}{59}
\providecommand{\natexlab}[1]{#1}
\providecommand{\url}[1]{\texttt{#1}}
\expandafter\ifx\csname urlstyle\endcsname\relax
  \providecommand{\doi}[1]{doi: #1}\else
  \providecommand{\doi}{doi: \begingroup \urlstyle{rm}\Url}\fi

\bibitem[Achille et~al.(2019)Achille, Lam, Tewari, Ravichandran, Maji, Fowlkes,
  Soatto, and Perona]{achille2019task2vec}
Alessandro Achille, Michael Lam, Rahul Tewari, Avinash Ravichandran, Subhransu
  Maji, Charless~C Fowlkes, Stefano Soatto, and Pietro Perona.
\newblock Task2vec: Task embedding for meta-learning.
\newblock In \emph{Proc. International Conference on Computer Vision}, pages
  6430--6439, Seoul, Korea, 2019.

\bibitem[Amit and Meir(2018)]{amit2018_pac_bayes_ml_prior}
Ron Amit and Ron Meir.
\newblock Meta-learning by adjusting priors based on extended pac-bayes theory.
\newblock In \emph{Proc. International Conference on Machine Learning}, pages
  205--214, Stockholm, Sweden, 2018.

\bibitem[Andrychowicz et~al.(2016)Andrychowicz, Denil, Gomez, Hoffman, Pfau,
  Schaul, Shillingford, and De~Freitas]{andrychowicz2016_l2l}
Marcin Andrychowicz, Misha Denil, Sergio Gomez, Matthew~W Hoffman, David Pfau,
  Tom Schaul, Brendan Shillingford, and Nando De~Freitas.
\newblock Learning to learn by gradient descent by gradient descent.
\newblock In \emph{Proc. Advances in Neural Information Processing Systems},
  pages 3981--3989, Barcelona, Spain, 2016.

\bibitem[Bai et~al.(2021)Bai, Chen, Zhou, Zhao, Lee, Kakade, Wang, and
  Xiong]{bai2021_trntrn_trnval}
Yu~Bai, Minshuo Chen, Pan Zhou, Tuo Zhao, Jason Lee, Sham Kakade, Huan Wang,
  and Caiming Xiong.
\newblock How important is the train-validation split in meta-learning?
\newblock In \emph{Proc. International Conference on Machine Learning}, pages
  543--553, Virtual, 2021.

\bibitem[Balcan et~al.(2019)Balcan, Khodak, and
  Talwalkar]{balcan2019_gbml_online_convex}
Maria-Florina Balcan, Mikhail Khodak, and Ameet Talwalkar.
\newblock Provable guarantees for gradient-based meta-learning.
\newblock In \emph{Proc. International Conference on Machine Learning}, pages
  424--433, Long Beach, CA, 2019.

\bibitem[Bengio et~al.(1991)Bengio, Bengio, and Cloutier]{bengio_learnsynaptic}
Yoshua Bengio, Samy Bengio, and Jocelyn Cloutier.
\newblock Learning a synaptic learning rule.
\newblock In \emph{International Joint Conference on Neural Networks},
  volume~2, page 969, Seattle, WA, 1991.

\bibitem[Chen et~al.(2021{\natexlab{a}})Chen, Shui, and
  Marchand]{chen2021generalization}
Qi~Chen, Changjian Shui, and Mario Marchand.
\newblock Generalization bounds for meta-learning: An information-theoretic
  analysis.
\newblock In \emph{Proc. Advances in Neural Information Processing Systems},
  virtual, 2021{\natexlab{a}}.

\bibitem[Chen et~al.(2021{\natexlab{b}})Chen, Sun, and Yin]{chen2021closing}
Tianyi Chen, Yuejiao Sun, and Wotao Yin.
\newblock Closing the gap: Tighter analysis of alternating stochastic gradient
  methods for bilevel problems.
\newblock In A.~Beygelzimer, Y.~Dauphin, P.~Liang, and J.~Wortman Vaughan,
  editors, \emph{Proc. Advances in Neural Information Processing Systems},
  virtual, 2021{\natexlab{b}}.

\bibitem[Chen et~al.(2021{\natexlab{c}})Chen, Sun, and
  Yin]{chen2021composition}
Tianyi Chen, Yuejiao Sun, and Wotao Yin.
\newblock Solving stochastic compositional optimization is nearly as easy as
  solving stochastic optimization.
\newblock \emph{IEEE Transactions on Signal Processing}, 69:\penalty0
  4937--4948, 2021{\natexlab{c}}.

\bibitem[Chen et~al.(2017)Chen, Hoffman, Colmenarejo, Denil, Lillicrap,
  Botvinick, and Freitas]{chen2017_l2l}
Yutian Chen, Matthew~W Hoffman, Sergio~G{\'o}mez Colmenarejo, Misha Denil,
  Timothy~P Lillicrap, Matt Botvinick, and Nando Freitas.
\newblock Learning to learn without gradient descent by gradient descent.
\newblock In \emph{Proc. International Conference on Machine Learning}, pages
  748--756, Sydney, Australia, 2017.

\bibitem[{Collins} et~al.(2020){Collins}, {Mokhtari}, and
  {Shakkottai}]{collins2020_task_landscape_erm_maml}
Liam {Collins}, Aryan {Mokhtari}, and Sanjay {Shakkottai}.
\newblock How does the task landscape affect maml performance?
\newblock \emph{arXiv preprint arXiv:2010.14672}, October 2020.

\bibitem[Denevi et~al.(2018)Denevi, Ciliberto, Stamos, and
  Pontil]{denevi2018_l2l_linear_centroid}
Giulia Denevi, Carlo Ciliberto, Dimitris Stamos, and Massimiliano Pontil.
\newblock Learning to learn around a common mean.
\newblock In \emph{Proc. Advances in Neural Information Processing Systems},
  volume~31, Montreal, Canada, 2018.

\bibitem[Ding et~al.(2021)Ding, Chen, Levinboim, Goodman, and
  Soricut]{ding2021bridging}
Nan Ding, Xi~Chen, Tomer Levinboim, Sebastian Goodman, and Radu Soricut.
\newblock Bridging the gap between practice and pac-bayes theory in few-shot
  meta-learning.
\newblock In \emph{Proc. Advances in Neural Information Processing Systems},
  Virtual, 2021.

\bibitem[Fallah et~al.(2020)Fallah, Mokhtari, and
  Ozdaglar]{fallah2020_convergence_maml}
Alireza Fallah, Aryan Mokhtari, and Asuman Ozdaglar.
\newblock On the convergence theory of gradient-based model-agnostic
  meta-learning algorithms.
\newblock In \emph{Proc. International Conference on Artificial Intelligence
  and Statistics}, pages 1082--1092, Virtual, 2020.

\bibitem[Fallah et~al.(2021)Fallah, Mokhtari, and
  Ozdaglar]{fallah2021_generalization_unseen}
Alireza Fallah, Aryan Mokhtari, and Asuman Ozdaglar.
\newblock Generalization of model-agnostic meta-learning algorithms: Recurring
  and unseen tasks.
\newblock In \emph{Proc. Advances in Neural Information Processing Systems},
  Virtual, 2021.

\bibitem[Farid and Majumdar(2021)]{farid2021generalization}
Alec Farid and Anirudha Majumdar.
\newblock Generalization bounds for meta-learning via {PAC}-bayes and uniform
  stability.
\newblock In \emph{Proc. Advances in Neural Information Processing Systems},
  Virtual, 2021.

\bibitem[Finn et~al.(2017)Finn, Abbeel, and Levine]{Finn2017_maml}
Chelsea Finn, Pieter Abbeel, and Sergey Levine.
\newblock Model-agnostic meta-learning for fast adaptation of deep networks.
\newblock In \emph{Proc. International Conference on Machine Learning}, page
  1126–1135, Sydney, Australia, 2017.

\bibitem[Finn et~al.(2018)Finn, Xu, and Levine]{finn2018_PLATIPUS}
Chelsea Finn, Kelvin Xu, and Sergey Levine.
\newblock Probabilistic model-agnostic meta-learning.
\newblock In \emph{Proc. Advances in Neural Information Processing Systems},
  Montreal, Canada, 2018.

\bibitem[Finn et~al.(2019)Finn, Rajeswaran, Kakade, and
  Levine]{finn2019_onlineML}
Chelsea Finn, Aravind Rajeswaran, Sham Kakade, and Sergey Levine.
\newblock Online meta-learning.
\newblock In \emph{Proc. International Conference on Machine Learning}, pages
  1920--1930, Long Beach, CA, 2019.

\bibitem[Franceschi et~al.(2018)Franceschi, Frasconi, Salzo, Grazzi, and
  Pontil]{franceschi2018_bilevel_maml_approx}
Luca Franceschi, Paolo Frasconi, Saverio Salzo, Riccardo Grazzi, and
  Massimiliano Pontil.
\newblock Bilevel programming for hyperparameter optimization and
  meta-learning.
\newblock In \emph{Proc. International Conference on Machine Learning}, pages
  1568--1577, Stockholm, Sweden, 2018.

\bibitem[Gao and Sener(2020)]{gao2020_model_opt_tradeoff_ml}
Katelyn Gao and Ozan Sener.
\newblock Modeling and optimization trade-off in meta-learning.
\newblock In \emph{Proc. Advances in Neural Information Processing Systems},
  volume~33, Virtual, 2020.

\bibitem[Gordon et~al.(2018)Gordon, Bronskill, Bauer, Nowozin, and
  Turner]{gordon2018_VERSA}
Jonathan Gordon, John Bronskill, Matthias Bauer, Sebastian Nowozin, and Richard
  Turner.
\newblock Meta-learning probabilistic inference for prediction.
\newblock In \emph{Proc. International Conference on Learning Representations},
  Vancouver, Canada, 2018.

\bibitem[Grant et~al.(2018)Grant, Finn, Levine, Darrell, and
  Griffiths]{grant2018recasting}
Erin Grant, Chelsea Finn, Sergey Levine, Trevor Darrell, and Thomas Griffiths.
\newblock Recasting gradient-based meta-learning as hierarchical bayes.
\newblock In \emph{Proc. International Conference on Learning Representations},
  Vancouver, Canada, 2018.

\bibitem[Harrison et~al.(2020)Harrison, Sharma, Finn, and
  Pavone]{harrison2020continuous}
James Harrison, Apoorva Sharma, Chelsea Finn, and Marco Pavone.
\newblock Continuous meta-learning without tasks.
\newblock In \emph{Proc. Advances in Neural Information Processing Systems},
  volume~33, virtual, 2020.

\bibitem[Hochreiter et~al.(2001)Hochreiter, Younger, and
  Conwell]{hochreiter2001_l2l}
Sepp Hochreiter, A~Steven Younger, and Peter~R Conwell.
\newblock Learning to learn using gradient descent.
\newblock In \emph{Proc. International Conference on Artificial Neural
  Networks}, pages 87--94, Vienna, Austria, 2001.

\bibitem[Hospedales et~al.(2020)Hospedales, Antoniou, Micaelli, and
  Storkey]{hospedalesmeta}
Timothy~M Hospedales, Antreas Antoniou, Paul Micaelli, and Amos~J Storkey.
\newblock Meta-learning in neural networks: A survey.
\newblock \emph{IEEE transactions on pattern analysis and machine
  intelligence}, 2020.

\bibitem[Hsu et~al.(2018)Hsu, Levine, and Finn]{hsu2018unsupervised}
Kyle Hsu, Sergey Levine, and Chelsea Finn.
\newblock Unsupervised learning via meta-learning.
\newblock In \emph{Proc. International Conference on Learning Representations},
  Vancouver, Canada, 2018.

\bibitem[Javed and White(2019)]{javed2019meta}
Khurram Javed and Martha White.
\newblock Meta-learning representations for continual learning.
\newblock In \emph{Proc. Advances in Neural Information Processing Systems},
  volume~32, pages 1820--1830, Vancouver, Canada, 2019.

\bibitem[Jose and Simeone(2021)]{jose2021information}
Sharu~Theresa Jose and Osvaldo Simeone.
\newblock Information-theoretic generalization bounds for meta-learning and
  applications.
\newblock \emph{Entropy}, 23\penalty0 (1):\penalty0 126, 2021.

\bibitem[Jose et~al.(2021)Jose, Simeone, and Durisi]{jose2021transfer}
Sharu~Theresa Jose, Osvaldo Simeone, and Giuseppe Durisi.
\newblock Transfer meta-learning: Information-theoretic bounds and information
  meta-risk minimization.
\newblock \emph{IEEE Transactions on Information Theory}, 2021.

\bibitem[Li et~al.(2018)Li, Yang, Song, and Hospedales]{li2018learning}
Da~Li, Yongxin Yang, Yi-Zhe Song, and Timothy~M Hospedales.
\newblock Learning to generalize: Meta-learning for domain generalization.
\newblock In \emph{Proc. Association for the Advancement of Artificial
  Intelligence}, New Orleans, LA, 2018.

\bibitem[Liu et~al.(2019)Liu, Zhou, Long, Jiang, and Zhang]{liu2019learning}
Lu~Liu, Tianyi Zhou, Guodong Long, Jing Jiang, and Chengqi Zhang.
\newblock Learning to propagate for graph meta-learning.
\newblock \emph{Proc. Advances in Neural Information Processing Systems},
  32:\penalty0 1039--1050, 2019.

\bibitem[Liu et~al.(2021)Liu, Liu, Zeng, and Zhang]{liu2021towards}
Risheng Liu, Yaohua Liu, Shangzhi Zeng, and Jin Zhang.
\newblock Towards gradient-based bilevel optimization with non-convex followers
  and beyond.
\newblock In A.~Beygelzimer, Y.~Dauphin, P.~Liang, and J.~Wortman Vaughan,
  editors, \emph{Advances in Neural Information Processing Systems}, 2021.

\bibitem[Madotto et~al.(2019)Madotto, Lin, Wu, and
  Fung]{madotto2019personalizing}
Andrea Madotto, Zhaojiang Lin, Chien-Sheng Wu, and Pascale Fung.
\newblock Personalizing dialogue agents via meta-learning.
\newblock In \emph{Proc. Annual Meeting of the Association for Computational
  Linguistics}, pages 5454--5459, Florence, Italy, 2019.

\bibitem[Nguyen et~al.(2020)Nguyen, Do, and Carneiro]{nguyen2020_VAMPIRE}
Cuong Nguyen, Thanh-Toan Do, and Gustavo Carneiro.
\newblock Uncertainty in model-agnostic meta-learning using variational
  inference.
\newblock In \emph{Proc. Winter Conference on Applications of Computer Vision},
  pages 3090--3100, Snowmass Village, CO, 2020.

\bibitem[Nichol et~al.(2018)Nichol, Achiam, and Schulman]{nichol2018first}
Alex Nichol, Joshua Achiam, and John Schulman.
\newblock On first-order meta-learning algorithms.
\newblock \emph{arXiv preprint arXiv:1803.02999}, 2018.

\bibitem[Obamuyide and Vlachos(2019)]{obamuyide2019model}
Abiola Obamuyide and Andreas Vlachos.
\newblock Model-agnostic meta-learning for relation classification with limited
  supervision.
\newblock In \emph{Proc. Annual Meeting of the Association for Computational
  Linguistics}, pages 5873--5879, Florence, Italy, 2019.

\bibitem[Rajeswaran et~al.(2019)Rajeswaran, Finn, Kakade, and
  Levine]{rajeswaran2019_imaml}
Aravind Rajeswaran, Chelsea Finn, Sham~M Kakade, and Sergey Levine.
\newblock Meta-learning with implicit gradients.
\newblock In \emph{Proc. Advances in Neural Information Processing Systems},
  pages 113--124, Vancouver, Canada, 2019.

\bibitem[Ravi and Beatson(2019)]{ravi2018_ABML}
Sachin Ravi and Alex Beatson.
\newblock Amortized bayesian meta-learning.
\newblock In \emph{Proc. International Conference on Learning Representations},
  New Orleans, LA, 2019.

\bibitem[Rezazadeh et~al.(2021)Rezazadeh, Jose, Durisi, and
  Simeone]{rezazadeh2021conditional}
Arezou Rezazadeh, Sharu~Theresa Jose, Giuseppe Durisi, and Osvaldo Simeone.
\newblock Conditional mutual information-based generalization bound for meta
  learning.
\newblock In \emph{2021 IEEE International Symposium on Information Theory
  (ISIT)}, pages 1176--1181. IEEE, 2021.

\bibitem[Rothfuss et~al.(2018)Rothfuss, Lee, Clavera, Asfour, and
  Abbeel]{rothfuss2018promp}
Jonas Rothfuss, Dennis Lee, Ignasi Clavera, Tamim Asfour, and Pieter Abbeel.
\newblock Promp: Proximal meta-policy search.
\newblock In \emph{Proc. International Conference on Learning Representations},
  Vancouver, Canada, 2018.

\bibitem[Rothfuss et~al.(2021)Rothfuss, Fortuin, Josifoski, and
  Krause]{rothfuss2021_pacoh_pac_bayes_ml}
Jonas Rothfuss, Vincent Fortuin, Martin Josifoski, and Andreas Krause.
\newblock Pacoh: Bayes-optimal meta-learning with pac-guarantees.
\newblock In \emph{Proc. International Conference on Machine Learning}, pages
  9116--9126, Virtual, 2021.

\bibitem[Schmidhuber(1993)]{schmidhuber1993_recurrent}
J.~Schmidhuber.
\newblock A neural network that embeds its own meta-levels.
\newblock In \emph{Proc. IEEE International Conference on Neural Networks},
  volume~1, pages 407--412, 1993.

\bibitem[Schmidhuber(1995)]{Schmidhuber95onlearning}
Jürgen Schmidhuber.
\newblock On learning how to learn learning strategies.
\newblock Technical report, Technical University of Munich, 1995.

\bibitem[Sheth and Khardon(2017)]{Sheth2017ExcessRiskBayes}
Rishit Sheth and Roni Khardon.
\newblock Excess risk bounds for the bayes risk using variational inference in
  latent gaussian models.
\newblock In \emph{Proc. Advances in Neural Information Processing Systems},
  Long Beach, CA, 2017.

\bibitem[Snell et~al.(2017)Snell, Swersky, and Zemel]{snell2017prototypical}
Jake Snell, Kevin Swersky, and Richard Zemel.
\newblock Prototypical networks for few-shot learning.
\newblock In \emph{Proc. Advances in Neural Information Processing Systems},
  pages 4080--4090, Long Beach, CA, 2017.

\bibitem[Thrun and Pratt(1998)]{l2l_book_1998}
Sebastian Thrun and Lorien Pratt.
\newblock \emph{Learning to Learn: Introduction and Overview}.
\newblock Kluwer Academic Publishers, USA, 1998.
\newblock ISBN 0792380479.

\bibitem[Vanschoren(2018)]{vanschoren2018meta}
Joaquin Vanschoren.
\newblock Meta-learning: A survey.
\newblock \emph{arXiv e-prints}, pages arXiv--1810, 2018.

\bibitem[Vershynin(2018)]{vershynin2018high}
Roman Vershynin.
\newblock \emph{High-Dimensional Probability: An Introduction with Applications
  in Data Science}.
\newblock Cambridge Series in Statistical and Probabilistic Mathematics.
  Cambridge University Press, 2018.

\bibitem[Vilalta and Drissi(2002)]{vilalta2002perspective}
Ricardo Vilalta and Youssef Drissi.
\newblock A perspective view and survey of meta-learning.
\newblock \emph{Artificial intelligence review}, 18\penalty0 (2):\penalty0
  77--95, 2002.

\bibitem[Vinyals et~al.(2016)Vinyals, Blundell, Lillicrap, Wierstra,
  et~al.]{vinyals2016matching}
Oriol Vinyals, Charles Blundell, Timothy Lillicrap, Daan Wierstra, et~al.
\newblock Matching networks for one shot learning.
\newblock In \emph{Proc. Advances in Neural Information Processing Systems},
  volume~29, pages 3630--3638, Barcelona, Spain, 2016.

\bibitem[Vuorio et~al.(2019)Vuorio, Sun, Hu, and Lim]{vuorio2019multimodal}
Risto Vuorio, Shao-Hua Sun, Hexiang Hu, and Joseph~J Lim.
\newblock Multimodal model-agnostic meta-learning via task-aware modulation.
\newblock In \emph{Proc. Advances in Neural Information Processing Systems},
  pages 1--12, Vancouver, Canada, 2019.

\bibitem[Wang et~al.(2020{\natexlab{a}})Wang, Sun, and
  Li]{wang2020_global_converge_maml_dnn}
Haoxiang Wang, Ruoyu Sun, and Bo~Li.
\newblock Global convergence and generalization bound of gradient-based
  meta-learning with deep neural nets.
\newblock \emph{arXiv preprint arXiv:2006.14606}, 2020{\natexlab{a}}.

\bibitem[Wang et~al.(2020{\natexlab{b}})Wang, Cai, Yang, and
  Wang]{wang2020_global_opt_maml}
Lingxiao Wang, Qi~Cai, Zhuoran Yang, and Zhaoran Wang.
\newblock On the global optimality of model-agnostic meta-learning.
\newblock In \emph{Proc. International Conference on Machine Learning}, pages
  9837--9846, Virtual, 2020{\natexlab{b}}.

\bibitem[Wang et~al.(2019)Wang, Ramanan, and Hebert]{Wang_2019_ICCV}
Yu-Xiong Wang, Deva Ramanan, and Martial Hebert.
\newblock Meta-learning to detect rare objects.
\newblock In \emph{Proc. International Conference on Computer Vision}, Seoul,
  Korea, October 2019.

\bibitem[Yang et~al.(2021)Yang, Ji, and Liang]{yang2021provably}
Junjie Yang, Kaiyi Ji, and Yingbin Liang.
\newblock Provably faster algorithms for bilevel optimization.
\newblock \emph{Advances in Neural Information Processing Systems}, 34, 2021.

\bibitem[Yin et~al.(2020)Yin, Tucker, Zhou, Levine, and Finn]{yin2020meta}
Mingzhang Yin, George Tucker, Mingyuan Zhou, Sergey Levine, and Chelsea Finn.
\newblock Meta-learning without memorization.
\newblock In \emph{Proc. International Conference on Learning Representations},
  virtual, 2020.

\bibitem[Yoon et~al.(2018)Yoon, Kim, Dia, Kim, Bengio, and Ahn]{yoon2018_BMAML}
Jaesik Yoon, Taesup Kim, Ousmane Dia, Sungwoong Kim, Yoshua Bengio, and Sungjin
  Ahn.
\newblock Bayesian model-agnostic meta-learning.
\newblock In \emph{Proc. Advances in Neural Information Processing Systems},
  volume~31, Montreal, Canada, 2018.

\bibitem[Zintgraf et~al.(2019)Zintgraf, Shiarli, Kurin, Hofmann, and
  Whiteson]{zintgraf2019fast}
Luisa Zintgraf, Kyriacos Shiarli, Vitaly Kurin, Katja Hofmann, and Shimon
  Whiteson.
\newblock Fast context adaptation via meta-learning.
\newblock In \emph{Proc. International Conference on Machine Learning}, pages
  7693--7702, Long Beach, CA, 2019.

\end{thebibliography}


\clearpage
\appendix

\thispagestyle{empty}

\onecolumn \makesupplementtitle

In this appendix, we first present the problem setting, then some basic supporting lemmas, and the missing derivations of some claims, as well as the proofs of all the lemmas and theorems in the paper, which is followed by details on our experiments along with additional experimental results. 


\section{Formulations and closed-form solutions} 
\label{app_sec:methods_formulation_sln}

In this section, we will introduce the definition and computation of meta-test (population) risks and empirical losses for the four methods that we will discuss, including ERM, MAML, iMAML, BaMAML.
This prepares for the analysis of optimal population risk and statistical error of the four methods in later sections.

\subsection{Empirical risk minimization formulation} 
\label{app_sub:erm}
In the meta learning setting, 
ERM minimizes the average loss over all data,
its empirical loss, meta-test risk and their optimal solutions can be obtained by taking $\alpha = 0, N_1 = 0, N_2=N$ in that of MAML~\citep{gao2020_model_opt_tradeoff_ml}, i.e. 
 $\hat{\btheta}_{\tau}^{\mathrm{er}} (\btheta_0, \mathcal{D}_{\tau}^{\mathrm{trn}}) = \btheta_0$, 
and based on the definition in~\eqref{eq:emp_loss}, the empirical loss of ERM is given by
\begin{align}\label{eq:erm_emp_risk}
  {\mathcal{L}}^{\mathrm{er}}({\btheta_0}, \mathcal{D} )
  =
  \frac{1}{T N} \sum_{\tau=1}^{T} 
  {}\|\mathbf{y}_{\tau,N}^{\mathrm{all}} -  \mathbf{X}_{\tau,N}^{\mathrm{all}} \btheta_0 \|^2.
\end{align}

For brevity, denote 
$\mathbf{e}_{\tau, N}^{\mathrm{all}} = 
[\epsilon_{\tau, 1},\ldots,
\epsilon_{\tau, N}]^{\top} \in \mathbb{R}^{N}
$.
And define $\hat{\btheta}_{0}^{\mathrm{er}}$ as the minimizer of the ERM empirical loss, given by
\begin{align}\label{eq:theta_0_hat_erm_derive}
  \hat{\btheta}_{0}^{\mathrm{er}} 
  = \mathop{\arg\min}_{\btheta_0} 
  {\mathcal{L}}^{\mathrm{er}}({\btheta_0}, \mathcal{D} )
  =& \mathop{\arg\min}_{\btheta_0} \frac{1}{TN} \sum_{\tau=1}^{T}
  \|\mathbf{X}_{\tau, N}^{\text{all}}\btheta_{\tau}^{\text{gt}} + \mathbf{e}^{\text{all}}_{\tau,N} - \mathbf{X}^{\text{all}}_{\tau, N}\btheta_0\|^2. 
\end{align}

Using the optimality condition, we have
\begin{subequations}
\begin{align}\label{eq:theta_0_hat_erm_sln}
  &\hat{\btheta}_{0}^{\mathrm{er}} 
  = \Big(\sum_{\tau=1}^{T}
  \hat{\mathbf{W}}_{\tau,N}^{\mathrm{er}}\Big)^{-1} 
  \Big(\sum_{\tau=1}^{T}\hat{\mathbf{W}}_{\tau}^{\mathrm{er}}\btheta_{\tau}^{\text{gt}} \Big) + \Delta_T^{\rm er}\\
  &\Delta_T^{\rm er} = \Big(\sum_{\tau=1}^{T}
  \hat{\mathbf{W}}_{\tau,N}^{\mathrm{er}}\Big)^{-1}
  \Big(\sum_{\tau=1}^{T} \frac{1}{N} \mathbf{X}^{\text{all}\top}_{\tau, N}\mathbf{e}^{\text{all}}_{\tau,N} \Big) \\
  \label{eq:W_hat_erm}
  &\hat{\mathbf{W}}_{\tau}^{\mathrm{er}} =  \frac{1}{N} \mathbf{X}^{\text{all}\top}_{\tau, N}\mathbf{X}_{\tau, N}^{\text{all}}.
\end{align}
\end{subequations}

Based on the definition in~\eqref{eq:R}, denote the number of adaptation data during meta testing as $N_a$, then
the total meta-test risk of ERM can be specified by
\begin{align}\label{eq:erm_metatest_risk}
  \mathcal{R}_{N_a}^{\mathrm{er}}({\btheta_0} )
  \coloneqq
  \mathbb{E}_{\tau}\Big[
  \mathbb{E}_{\mathcal{D}_{\tau,N_a}}\big[
  \mathbb{E}_{p(\mathbf{x}_{\tau},y_{\tau}\mid \tau) }
  [\big({y}_{\tau} - 
  \hat{\btheta}_{\tau}^{\mathrm{er}}(\btheta_0, \mathcal{D}_{\tau,N_a})^{\top}
  \mathbf{x}_{\tau})^2]
  \big]\Big]
\end{align}
where
$\hat{\btheta}_{\tau}^{\mathrm{er}}(\btheta_0, \mathcal{D}_{\tau}) = \btheta_0$, plugging which into~\eqref{eq:erm_metatest_risk}, we have
\begin{subequations}
\begin{align}\label{eq:erm_metatest_risk_simple}
  &\mathcal{R}_{N_a}^{\mathrm{er}}({\btheta_0} )
  = \mathbb{E}_{\tau}[ (y_{\tau}-{\btheta}_0^{\top} \mathbf{x}_{\tau})^{2} ]
  = \mathbb{E}_{\tau}[\|\btheta_0-\btheta^{\mathrm{gt}}_{\tau}\|^2_{\mathbf{W}_{\tau,N_a}^{\mathrm{er}}}] + 1 \\
  \label{eq:W_N_er}
 &\mathbf{W}_{\tau,N_a}^{\mathrm{er}} = \mathbb{E}[\mathbf{x}_{\tau}\mathbf{x}_{\tau}^{\top}\mid \tau] =  \mathbf{Q}_{\tau}.
\end{align}
\end{subequations}

By the general definition of optimal population risk in~\eqref{eq:R},
the ERM optimal population risk is given by
\begin{subequations}
\begin{align}\label{eq:erm_model_err}
  &\mathcal{R}^{\mathrm{er}}({\btheta_0} )
  \coloneqq
  \lim_{N_a \to \infty} 
  \mathcal{R}_{N_a}^{\mathrm{er}}({\btheta_0} )
  =
  \mathbb{E}_{\tau}[\|\btheta_0-\btheta^{\mathrm{gt}}_{\tau}\|^2_{\mathbf{W}_{\tau}^{\mathrm{er}}}] + 1 \\
  \label{eq:W_er}
  & \mathbf{W}_{\tau}^{\mathrm{er}} =  \mathbf{Q}_{\tau}.
\end{align}
\end{subequations}

Define $\btheta_{0}^{\mathrm{er}}$ as the minimizer of the ERM optimal population risk, given by
\begin{align}\label{eq:theta_0_star_erm_sln}
  \btheta_{0}^{\mathrm{er}} 
  = \mathop{\arg\min}_{\btheta_0} 
  \mathcal{R}^{\mathrm{er}}({\btheta_0} )
  = \mathop{\arg\min}_{\btheta_0} 
  \mathbb{E}_{\tau}\big[\|\btheta_0-\btheta^{\mathrm{gt}}_{\tau}\|^2_{\mathbf{W}_{\tau}^{\mathrm{er}}}\big]
  =\mathbb{E}_{\tau}\big[\mathbf{W}_{\tau}^{\mathrm{er}}\big]^{-1} \mathbb{E}_{\tau}\big[\mathbf{W}_{\tau}^{\mathrm{er}} {\btheta}_{\tau}^{\mathrm{gt}}\big].
\end{align}


From \eqref{eq:W_hat_erm}\eqref{eq:W_N_er}\eqref{eq:W_er}, we have the property $\mathbb{E}[\hat{\mathbf{W}}_{\tau,N}^{\mathrm{er}}] = \mathbf{W}_{\tau,N}^{\mathrm{er}} = \mathbf{W}_{\tau}^{\mathrm{er}}$, which will be used in later sections to derive the specific optimal population risks and statistical errors.

\subsection{Model agnostic meta learning method} 
\label{app_sub:maml}

For the one-step model agnostic meta learning (MAML) method, 
the task-specific parameter $\hat{\btheta}_{\tau}^{\mathrm{ma}}$ is computed from the initial parameter $\btheta_0$ by taking one-step gradient descent of the empirical loss function as shown below
\begin{align}\label{eq:theta_tau_maml_theta0}
  \hat{\btheta}_{\tau}^{\mathrm{ma}} (\btheta_0, \mathcal{D}_{\tau}) 
  = \btheta_0 - \frac{\alpha}{2}  \nabla_{\btheta_0} 
  \ell_{\tau}\big(\btheta_0, \mathcal{D}_{\tau} \big)
  = (\mathbf{I} - {\alpha}\hat{\mathbf{Q}}_{\tau,N}) \btheta_0
  + \frac{\alpha}{N} 
  \mathbf{X}_{\tau,N}^{\top}\mathbf{y}_{\tau,N}
\end{align}
\footnote{Note that, here we define the learning rate  as $\alpha/2$ to cancel the scale factor 2 from the derivative for notation simplicity.}
where $\alpha  >0$ is twice the stepsize, and $N$ is the number of adaptation data.
During meta-training, $N=N_1$, is the number of the training data.
From the definition in~\eqref{eq:emp_loss} or~\eqref{eq:R}, and combined with $\hat{\btheta}_{\tau}^{\mathrm{ma}}$ in \eqref{eq:theta_tau_maml_theta0}, the empirical loss of MAML is given by
\begin{align}\label{eq:MAML_emp_risk_app}
  {\mathcal{L}}^{\mathrm{ma}}({\btheta_0}, \mathcal{D})
  = 
  \frac{1}{T N_2} \sum_{\tau=1}^{T} 
  \|\mathbf{y}_{\tau,N_2}^{\rm val} - 
  \mathbf{X}_{\tau,N_2}^{\mathrm{val}}
  \hat{\btheta}_{\tau}^{\mathrm{ma}} (\btheta_0, \mathcal{D}_{\tau}^{\rm trn})\|^2.
\end{align}

The minimizer of the MAML empirical loss is defined as
\begin{align}\label{eq:theta_0_hat_MAML_derive}
  \hat{\btheta}_{0}^{\mathrm{ma}} 
  = \mathop{\arg\min}_{\btheta_0} 
  {\mathcal{L}}^{\mathrm{ma}}({\btheta_0}, \mathcal{D})
  =& \mathop{\arg\min}_{\btheta_0} \frac{1}{TN_2} \sum_{\tau=1}^{T} 
  \|\mathbf{X}_{\tau, N_2}^{\text{val}}\btheta_{\tau}^{\text{gt}} + \mathbf{e}^{\text{val}}_{\tau,N_2} - \mathbf{X}^{\text{val}}_{\tau, N_2}\hat{\btheta}^{\mathrm{ma}}_{\tau}({\btheta_0}, \mathcal{D}_{\tau}^{\text{trn}})\|^2. 
\end{align}

Using the optimality condition, we have
\begin{subequations}
\begin{align}\label{eq:theta_0_hat_MAML_sln_app}
  &\hat{\btheta}_{0}^{\mathrm{ma}} 
  = \Big(\sum_{\tau=1}^{T}
  \hat{\mathbf{W}}_{\tau}^{\mathrm{ma}}\Big)^{-1} 
  \Big(\sum_{\tau=1}^{T}\hat{\mathbf{W}}_{\tau}^{\mathrm{ma}}\btheta_{\tau}^{\text{gt}} \Big) +  \Delta_T^{\rm ma}\\ 
  &\Delta_T^{\rm ma} 
  = \Big(\sum_{\tau=1}^{T}
  \hat{\mathbf{W}}_{\tau}^{\mathrm{ma}}\Big)^{-1} \Big(
  \sum_{\tau=1}^{T} \big(\mathbf{I}- 
  {\alpha }\hat{\mathbf{Q}}_{\tau,N_1}^{\text{trn}}
  \big)
  \big(
  \frac{1}{N_2}\mathbf{X}_{\tau,N_2}^{\text{val}\top} \mathbf{e}_{\tau,N_2}^{\text{val}} - \frac{\alpha  }{N_1}
  \hat{\mathbf{Q}}_{\tau,N_2}^{\text{val}}\mathbf{X}_{\tau,N_1}^{\text{trn}\top} \mathbf{e}_{\tau,N_1}^{\text{trn}}
  \big)
  \Big) \\
  \label{eq:W_hat_tau_ma}
  & 
  \hat{\mathbf{W}}_{\tau}^{\mathrm{ma}} 
  =  
  (\mathbf{I}-
  {\alpha}\hat{\mathbf{Q}}_{\tau,N_1}^{\text{trn}}) 
  \hat{\mathbf{Q}}_{\tau,N_2}^{\text{val}}
  (\mathbf{I}- 
  {\alpha}\hat{\mathbf{Q}}_{\tau,N_1}^{\text{trn}}). 
\end{align}
\end{subequations}

Based on~\eqref{eq:R}, the MAML meta-test risk is defined as~\citep{gao2020_model_opt_tradeoff_ml}
\begin{subequations}
\begin{align}\label{eq:maml_pop_risk}
  &\mathcal{R}_{N_a}^{\mathrm{ma}}({\btheta_0} )
  =
  \mathbb{E}[
  (y_{\tau}-\hat{\btheta}_{\tau}^{\mathrm{ma}} (\btheta_0, \mathcal{D}_{\tau, N_a}) ^{\top} \mathbf{x}_{\tau})^{2} ]
  =
  \mathbb{E}_{\tau}\big[\|\btheta_0-\btheta^{\mathrm{gt}}_{\tau}\|^2_{\mathbf{W}_{\tau,N_a}^{\mathrm{ma}}}\big] 
  + 1 + \frac{\alpha^2}{N_a} \mathbb{E}_{\tau}[\mathrm{tr}(\mathbf{Q}^2_{\tau})] \\
  & \label{eq:W_tau_N_ma}
  \mathbf{W}_{\tau,N_a}^{\mathrm{ma}} 
  =
  \mathbb{E}_{\hat{\mathbf{Q}}_{\tau,N}}\big[(\mathbf{I} - {\alpha} \hat{\mathbf{Q}}_{\tau,N}) \mathbf{Q}_{\tau}(\mathbf{I} - {\alpha} \hat{\mathbf{Q}}_{\tau,N})\big]  \nonumber \\
  &\quad \quad ~~ =
  \big(\mathbf{I}-\alpha \mathbf{Q}_{\tau}\big) \mathbf{Q}_{\tau}\big(\mathbf{I}-\alpha  \mathbf{Q}_{\tau}\big)+\frac{\alpha^{2}}{N_a}\Big(\mathbb{E}_{\mathbf{x}_{\tau, i}}\big[\mathbf{x}_{\tau, i} \mathbf{x}_{\tau, i}^{\top} \mathbf{Q}_{\tau} \mathbf{x}_{\tau, i} \mathbf{x}_{\tau, i}^{\top}\big]-\mathbf{Q}_{\tau}^{3}\Big). 
\end{align}
\end{subequations}
 Note that,
 $\lim_{N_a \to \infty}
 \mathbf{W}_{\tau,N_a}^{\mathrm{ma}} \to \mathbf{W}_{\tau}^{\mathrm{ma}} 
 $, and $\lim_{N_a \to \infty} \frac{\alpha^2}{N_a} \mathrm{tr}(\mathbf{Q}^2_{\tau}) =0$.
 Therefore, from the definition of optimal population risk in~\eqref{eq:R}, we have the MAML optimal population risk is
 \begin{align*}
  \mathcal{R}^{\mathrm{ma}}(\btheta_0)& 
  = \lim_{N_a \to \infty} 
  \mathcal{R}_{N_a}^{\mathrm{ma}}(\btheta_0)
  = \mathbb{E}_{\tau}\big[\|\btheta_0-\btheta^{\mathrm{gt}}_{\tau}\|^2_{\mathbf{W}_{\tau}^{\mathrm{ma}}}\big] 
  + 1
  \numberthis. 
 \end{align*}

In MAML, define $\btheta_{0}^{\mathrm{ma}}$ as the minimizer of the MAML optimal population risk, given by
\begin{align}\label{eq:theta_0_star_MAML_sln_app}
  \btheta_{0}^{\mathrm{ma}} 
  = \mathop{\arg\min}_{\btheta_0} 
  \mathcal{R}^{\mathrm{ma}}({\btheta_0} )
  = \mathop{\arg\min}_{\btheta_0} 
  \mathbb{E}_{\tau}\big[\|\btheta_0-\btheta^{\mathrm{gt}}_{\tau}\|^2_{\mathbf{W}_{\tau}^{\mathrm{ma}}}\big]
  =\mathbb{E}_{\tau}\big[\mathbf{W}_{\tau}^{\mathrm{ma}}\big]^{-1} \mathbb{E}_{\tau}\big[\mathbf{W}_{\tau}^{\mathrm{ma}} {\btheta}_{\tau}^{\mathrm{gt}}\big]
  .
\end{align}

It is worth noting that, from Lemma~\ref{lemma:concentration_W_N}, we have the property $\mathbb{E}_{\mathbf{x}_{\tau}}[\hat{\mathbf{W}}_{\tau,N}^{\mathrm{ma}}] = \mathbf{W}_{\tau,N}^{\mathrm{ma}}$, $\lim_{N \to \infty}
\mathbf{W}_{\tau,N}^{\mathrm{ma}} = \mathbf{W}_{\tau}^{\mathrm{ma}} $, which will be used in later sections to derive the specific optimal population risk and statistical error.


\subsection{Implicit model agnostic meta learning method} 
\label{app_sub:bimaml_method}

  For the iMAML method, 
  the task-specific parameter $\hat{\btheta}_{\tau}^{\mathrm{im}}$ is computed from the initial parameter $\btheta_0$ by optimizing the regularized task-specific empirical loss, given by
  \begin{align}\label{eq:theta_tau_biMAML_theta0}
    \hat{\btheta}_{\tau}^{\mathrm{im}} (\btheta_0) 
    = \mathop{\arg\min}_{\btheta_{\tau}} 
    \frac{1}{ N} 
    \|\mathbf{y}_{\tau,N}- \mathbf{X}_{\tau,N} {\btheta}_{\tau} \|^{2} 
     + \gamma \|\btheta_{\tau} - \btheta_0\|^2
  \end{align}
  where $\gamma $ is the weight of the regularizer,
  and $\mathcal{D}_{\tau,N_a}$ is the adaptation data during meta-testing or training data during meta-training.
    The estimated task-specific parameter can be computed by
  \begin{align}\label{eq:theta_tau_hat_biMAML}
    \hat{\btheta}_{\tau,N}^{\mathrm{im}}(\btheta_0,\mathcal{D}_{\tau})
    &= (\hat{\mathbf{Q}}_{\tau,N} +  \gamma  \mathbf{I})^{-1}(\frac{1}{N} \mathbf{X}_{\tau,N}^{\top}\mathbf{y}_{\tau,N} + \gamma  \btheta_0). 
  \end{align}

  The empirical loss of iMAML is defined as the average per-task loss, which is computed by
  \begin{align}\label{eq:biMAML_emp_risk_app}
    {\mathcal{L}}_{T,N}^{\mathrm{im}}({\btheta_0}, \mathcal{D})
    =\frac{1}{TN_2} \sum_{\tau=1}^{T} 
    \|\mathbf{y}^{\text{val}}_{\tau, N_2} - \mathbf{X}^{\text{val}}_{\tau, N_2}\hat{\btheta}^{\mathrm{im}}_{\tau}({\btheta_0}, \mathcal{D}_{\tau}^{\text{trn}})\|^2
  \end{align}
  whose minimizer is
  \begin{align*}\label{eq:theta_0_hat_biMAML_derive}
    \hat{\btheta}_{0}^{\mathrm{im}} 
    =& \mathop{\arg\min}_{\btheta_0} 
    {\mathcal{L}}^{\mathrm{im}}_{T,N}({\btheta_0}, \mathcal{D} ) 
    = \mathop{\arg\min}_{\btheta_0} \frac{1}{TN_2} \sum_{\tau=1}^{T} 
    \|\mathbf{X}_{\tau, N_2}^{\text{val}}\btheta_{\tau}^{\text{gt}} + \mathbf{e}^{\text{val}}_{\tau,N_2} - \mathbf{X}^{\text{val}}_{\tau, N_2}\hat{\btheta}^{\mathrm{im}}_{\tau}({\btheta_0}, \mathcal{D}_{\tau}^{\text{trn}})\|^2. \numberthis
  \end{align*}
  To solve for $\btheta_0^{\mathrm{im}}$ in the above equation, using the optimality condition, we obtain
  \begin{subequations}
    \begin{align}\label{eq:theta_0_hat_biMAML_sln_app}
      &\hat{\btheta}_{0}^{\mathrm{im}} 
      = \Big(\sum_{\tau=1}^{T}\hat{\mathbf{W}}_{\tau}^{\mathrm{im}}\Big)^{-1}
       \Big(\sum_{\tau=1}^{T}
       \hat{\mathbf{W}}_{\tau}^{\mathrm{im}} \btheta_{\tau}^{\mathrm{gt}} \Big) + \Delta_{T}^{\rm im} \\
       &\Delta_{T}^{\rm im}
      =\Big(\sum_{\tau=1}^{T}\hat{\mathbf{W}}_{\tau}^{\mathrm{im}}\Big)^{-1}
       \Big(\sum_{\tau=1}^{T}\gamma {\Sigma}_{\btheta_{\tau}}
       \frac{1}{N_2}\mathbf{X}^{\text{val}\top}_{\tau, N_2}\mathbf{e}_{\tau, N_2}^{\text{val}}
       -\gamma^{-1}\hat{\mathbf{W}}_{\tau}^{\mathrm{im}}\frac{1}{N_1}  \mathbf{X}_{\tau, N_1}^{\text{trn} \top} \mathbf{e}_{\tau, N_1}^{\text{trn}}\Big) \\
       &{\Sigma}_{\btheta_{\tau}, N_1} 
       = (\frac{1}{N_1} \mathbf{X}_{\tau, N_1}^{\mathrm{trn}\top} \mathbf{X}_{\tau, N_1}^{\mathrm{trn}}
       + \gamma  \mathbf{I})^{-1} 
       = ( \hat{\mathbf{Q}}_{\tau, N_1}
       + \gamma  \mathbf{I})^{-1} \\
       \label{eq:W_hat_bi}
       &\hat{\mathbf{W}}_{\tau}^{\mathrm{im}} = \gamma^2    
       {\Sigma}_{\btheta_{\tau}, N_1}
      \frac{1}{N_2}\mathbf{X}^{\text{val}\top}_{\tau, N_2}\mathbf{X}^{\text{val}}_{\tau, N_2}
      {\Sigma}_{\btheta_{\tau}, N_1}
      = \gamma^2  {\Sigma}_{\btheta_{\tau}, N_1}
      \hat{\mathbf{Q}}_{\tau, N_2}
      {\Sigma}_{\btheta_{\tau}, N_1}. 
    \end{align}
    \end{subequations}

  The iMAML meta-test risk is defined as 
\begin{subequations}\label{eq:bimaml_pop_risk}
  \begin{align}\label{eq:bimaml_pop_risk_R}
  \hspace{-5mm}
  &\mathcal{R}_{N_a}^{\mathrm{im}}({\btheta_0} )
  =
  \mathbb{E}\big[ 
  \big(y_{\tau}-\hat{\btheta}_{\tau}^{\mathrm{im}} (\btheta_0, \mathcal{D}_{\tau,N_a}) ^{\top} \mathbf{x}_{\tau}\big)^{2} \big] \nonumber \\
  &\quad\quad \quad ~~ =
  \mathbb{E}_{\tau}\big[\|\btheta_0-\btheta^{\mathrm{gt}}_{\tau}\|^2_{\mathbf{W}_{\tau,N_a}^{\mathrm{im}}}\big] 
  + 1 
  +\frac{1}{N_a} \mathbb{E}[ \gamma^{-2} \operatorname{tr}(\mathbf{W}_{\tau,N_a}^{\mathrm{im}} \hat{\mathbf{Q}}_{\tau,N_a})] \\
  &\mathbf{W}_{\tau,N_a}^{\mathrm{im}}
  =
  \mathbb{E}_{{\mathbf{x}}_{\tau}}\big[(\hat{\mathbf{Q}}_{\tau,N_a} + \gamma \mathbf{I} )^{-1} \mathbf{Q}_{\tau}
  (\hat{\mathbf{Q}}_{\tau,N_a} + \gamma \mathbf{I} )^{-1}\big] 
  \nonumber 
  = \mathbf{W}_{\tau}^{\mathrm{im}} 
  + \\
  &\mathbb{E}_{{\mathbf{x}}_{\tau}}\big[{\Sigma}_{\theta_{\tau}}\big(\mathbf{Q}_{\tau}-\hat{\mathbf{Q}}_{\tau,N_a}\big) \mathbf{W}_{\tau}^{\mathrm{im}}\big(\mathbf{Q}_{\tau}-\hat{\mathbf{Q}}_{\tau,N_a}\big) {\Sigma}_{\theta_{\tau}} +
  {\Sigma}_{\theta_{\tau}}\big(\mathbf{Q}_{\tau}-\hat{\mathbf{Q}}_{\tau,N_a}\big) \mathbf{W}_{\tau}^{\mathrm{im}}+\mathbf{W}_{\tau}^{\mathrm{im}}\big(\mathbf{Q}_{\tau}-\hat{\mathbf{Q}}_{\tau,N_a}\big) {\Sigma}_{\theta_{\tau}}\big]
\end{align}
\end{subequations}
where $\mathbf{W}_{\tau}^{\mathrm{im}} =
(\gamma ^{-1}\mathbf{Q}_{\tau}+\mathbf{I})^{-1}\mathbf{Q}_{\tau}(\gamma ^{-1}\mathbf{Q}_{\tau}+\mathbf{I})^{-1} $.

Let 
${\Sigma}_{\theta_{\tau}} = (\hat{\mathbf{Q}}_{\tau,N}+\gamma \mathbf{I})^{-1}$, 
and $\mathbf{W}_{\tau,N_a}^{\mathrm{im}} = \gamma^2{\Sigma}_{\btheta_{\tau}}  \mathbf{Q}_{\tau} {\Sigma}_{\btheta_{\tau}} $.
And we simplify the notation of $\mathbf{X}_{\tau,N_a}, \mathbf{y}_{\tau,N_a}, \hat{\mathbf{Q}}_{\tau, N_a}$ as
$\mathbf{X}_{\tau}, \mathbf{y}_{\tau},  \hat{\mathbf{Q}}_{\tau}$.
The derivation of \eqref{eq:bimaml_pop_risk} is given below
\begin{align*}
\label{eq:total_emp_risk_imaml_theta0_a}
    \mathcal{R}_{N_a}^{\mathrm{im}}({\btheta}_0)
    =& \mathbb{E}\big[ %
    \|\hat{\btheta}^{\mathrm{im}}_{\tau}(\btheta_0, \mathcal{D}_{\tau,N_a})-\btheta^{\mathrm{gt}}_{\tau}\|^2_{\mathbf{Q}_{\tau}}\big] + 1 
    = \mathbb{E}\big[ \|(\hat{\mathbf{Q}}_{\tau} + \gamma \mathbf{I})^{-1}(\frac{1}{N_a} \mathbf{X}_{\tau}^{\top}\mathbf{y}_{\tau} + \gamma {\btheta}_0)-\btheta^{\mathrm{gt}}_{\tau}\|^2_{\mathbf{Q}_{\tau}}\big] + 1 \\
    \stackrel{(a)}{=}& \mathbb{E}\Big[
     {\btheta}_0^{\top} 
     \mathbf{W}_{\tau,N_a}^{\mathrm{im}} {\btheta}_0
    +2 \gamma(\frac{1}{N_a}\mathbf{y}_{\tau}^{\top}\mathbf{X}_{\tau} 
    {\Sigma}_{\btheta_{\tau}} - \btheta_{\tau}^{\mathrm{gt}\top})
     \mathbf{Q}_{\tau} {\Sigma}_{\btheta_{\tau}}{\btheta}_0 
    + \frac{1}{N_a}\mathbf{y}_{\tau}^{\top}\mathbf{X}_{\tau}
    {\Sigma}_{\btheta_{\tau}}  \mathbf{Q}_{\tau} {\Sigma}_{\btheta_{\tau}} \frac{1}{N_a} \mathbf{X}_{\tau}^{\top}\mathbf{y}_{\tau} \\
    & - 2\btheta_{\tau}^{\mathrm{gt}\top}
     \mathbf{Q}_{\tau} {\Sigma}_{\btheta_{\tau}}
    \frac{1}{N_a} \mathbf{X}_{\tau}^{\top}\mathbf{y}_{\tau}
    + \btheta_{\tau}^{\mathrm{gt}\top} \mathbf{Q}_{\tau} \btheta_{\tau}^{\mathrm{gt}}
    \Big] + 1 \numberthis
\end{align*}
where $(a)$ follows from the definition of ${\Sigma}_{\btheta_{\tau}}$, $\mathbf{W}_{\tau,N_a}^{\mathrm{im}}$.
Applying the fact that $\mathbf{y}_{\tau} = \mathbf{X}_{\tau}\btheta_{\tau}^{\mathrm{gt}} + \mathbf{e}_{\tau}$ and \(\mathbb{E}_{\mathbf{e}_{\tau}}[\mathbf{e}_{\tau}] = \mathbf{0}\), one can further derive
\begin{align*}
  \mathcal{R}_{N_a}^{\mathrm{im}}({\btheta}_0)
  {=}& 
  \mathbb{E}\Big[  {\btheta}_0^{\top} 
   \mathbf{W}_{\tau,N_a}^{\mathrm{im}} {\btheta}_0
  +2 \gamma(\btheta_{\tau}^{\mathrm{gt}\top}\hat{\mathbf{Q}}_{\tau}
  {\Sigma}_{\btheta_{\tau}} - \btheta_{\tau}^{\mathrm{gt}\top})
   \mathbf{Q}_{\tau} {\Sigma}_{\btheta_{\tau}}{\btheta}_0 \\
  &+ \btheta_{\tau}^{\mathrm{gt}\top}
  \hat{\mathbf{Q}}_{\tau}
  {\Sigma}_{\btheta_{\tau}}  \mathbf{Q}_{\tau} {\Sigma}_{\btheta_{\tau}} \hat{\mathbf{Q}}_{\tau}\btheta_{\tau}^{\mathrm{gt}}
  - 2\btheta_{\tau}^{\mathrm{gt}\top}
   \mathbf{Q}_{\tau} {\Sigma}_{\btheta_{\tau}}
  \hat{\mathbf{Q}}_{\tau}\btheta_{\tau}^{\mathrm{gt}}
  + \btheta_{\tau}^{\mathrm{gt}\top} \mathbf{Q}_{\tau} \btheta_{\tau}^{\mathrm{gt}} \\
  &+\frac{1}{N_a^2}
  \mathbf{e}_{\tau}^{\top}
  \mathbf{X}_{\tau}
  {\Sigma}_{\btheta_{\tau}}  \mathbf{Q}_{\tau} {\Sigma}_{\btheta_{\tau}} \mathbf{X}_{\tau}^{\top}
  \mathbf{e}_{\tau}
  \Big] + 1.
\end{align*}
Based on the linearity of trace and expectation, and the cyclic property of trace, the last term inside the expectation in the above equation can be computed as 
\begin{align*}
&\mathbb{E}_{\mathbf{e}_{\tau}}[\mathbf{e}_{\tau}^{\top}
  \mathbf{X}_{\tau}
  {\Sigma}_{\theta_{\tau}}  \mathbf{Q}_{\tau} {\Sigma}_{\theta_{\tau}} \mathbf{X}_{\tau}^{\top}
  \mathbf{e}_{\tau}^{}] 
  = \mathrm{tr}(
  \mathbf{X}_{\tau}
  {\Sigma}_{\theta_{\tau}}  \mathbf{Q}_{\tau} {\Sigma}_{\theta_{\tau}} \mathbf{X}_{\tau}^{\top}
  \mathbb{E}_{\mathbf{e}_{\tau}}[\mathbf{e}_{\tau}^{}\mathbf{e}_{\tau}^{\top}]) \\
  =& \mathrm{tr}(
  \mathbf{X}_{\tau}
  {\Sigma}_{\theta_{\tau}}  \mathbf{Q}_{\tau} {\Sigma}_{\theta_{\tau}} \mathbf{X}_{\tau}^{\top})
  = N_a \mathrm{tr}(
  {\Sigma}_{\theta_{\tau}} \mathbf{Q}_{\tau} {\Sigma}_{\theta_{\tau}} \hat{\mathbf{Q}}_{\tau})
  = N_a \mathrm{tr}(
  \mathbf{W}_{\tau,N_a}^{\rm im} \hat{\mathbf{Q}}_{\tau}) ;
\end{align*}
also, based on the Woodbury matrix identity,
$\mathbf{I} - \hat{\mathbf{Q}}_{\tau} 
  {\Sigma}_{\theta_{\tau}} = 
\mathbf{I} -  
  {\Sigma}_{\theta_{\tau}}
  \hat{\mathbf{Q}}_{\tau}
  = \gamma{\Sigma}_{\theta_{\tau}}$, therefore\\
\(\begin{aligned}
  (\btheta_{\tau}^{\mathrm{gt}\top}\hat{\mathbf{Q}}_{\tau}
  {\Sigma}_{\btheta_{\tau}} - \btheta_{\tau}^{\mathrm{gt}\top})=
  \btheta_{\tau}^{\mathrm{gt}\top}(\hat{\mathbf{Q}}_{\tau}
  {\Sigma}_{\btheta_{\tau}} - \mathbf{I})
  = - \gamma \btheta_{\tau}^{\mathrm{gt}\top}{\Sigma}_{\theta_{\tau}}
\end{aligned}\), and 
\begin{align*}
  &\btheta_{\tau}^{\mathrm{gt}\top}
  \hat{\mathbf{Q}}_{\tau}
  {\Sigma}_{\btheta_{\tau}}  \mathbf{Q}_{\tau} {\Sigma}_{\btheta_{\tau}} \hat{\mathbf{Q}}_{\tau}\btheta_{\tau}^{\mathrm{gt}}
  - 2\btheta_{\tau}^{\mathrm{gt}\top}
   \mathbf{Q}_{\tau} {\Sigma}_{\btheta_{\tau}}
  \hat{\mathbf{Q}}_{\tau}\btheta_{\tau}^{\mathrm{gt}}
  + \btheta_{\tau}^{\mathrm{gt}\top} \mathbf{Q}_{\tau} \btheta_{\tau}^{\mathrm{gt}}\\
  =& \btheta_{\tau}^{\mathrm{gt}\top}
  \big(
  (\hat{\mathbf{Q}}_{\tau}
  {\Sigma}_{\btheta_{\tau}}-\mathbf{I})  \mathbf{Q}_{\tau} {\Sigma}_{\btheta_{\tau}} \hat{\mathbf{Q}}_{\tau}
  + \mathbf{Q}_{\tau} (\mathbf{I}-{\Sigma}_{\btheta_{\tau}}
  \hat{\mathbf{Q}}_{\tau})
  \big)
  \btheta_{\tau}^{\mathrm{gt}} \\
  =& \btheta_{\tau}^{\mathrm{gt}\top}
  \big(
  -\gamma {\Sigma}_{\btheta_{\tau}}  \mathbf{Q}_{\tau} {\Sigma}_{\btheta_{\tau}} \hat{\mathbf{Q}}_{\tau}
  + \mathbf{Q}_{\tau} \gamma {\Sigma}_{\btheta_{\tau}}
  \big)
  \btheta_{\tau}^{\mathrm{gt}} 
  = \gamma^{-1} \btheta_{\tau}^{\mathrm{gt}\top}
  \big(
  -  \mathbf{W}_{\tau,N_a}^{\rm im}  \hat{\mathbf{Q}}_{\tau}
  + (\hat{\mathbf{Q}}_{\tau}+\gamma \mathbf{I})
  \mathbf{W}_{\tau,N_a}^{\mathrm{im}}
  \big)
  \btheta_{\tau}^{\mathrm{gt}} .
\end{align*}

Combining these equalities and rearranging the equations we obtain
\begin{align*}
  &\mathcal{R}_{N_a}^{\mathrm{im}}({\btheta}_0)
  {=} \mathbb{E}\Big[
   {\btheta}_0^{\top} 
   \mathbf{W}_{\tau,N_a}^{\mathrm{im}} {\btheta}_0
  -2 \btheta_{\tau}^{\mathrm{gt}\top}
  \mathbf{W}_{\tau,N_a}^{\mathrm{im}} {\btheta}_0 \\
  &\quad \quad \quad\quad~~ 
  + \gamma^{-1} \btheta_{\tau}^{\mathrm{gt}\top}
  \big(
  -  \mathbf{W}_{\tau,N_a}^{\rm im}  \hat{\mathbf{Q}}_{\tau}
  + (\hat{\mathbf{Q}}_{\tau}+\gamma \mathbf{I})
  \mathbf{W}_{\tau,N_a}^{\mathrm{im}}
  \big)
  \btheta_{\tau}^{\mathrm{gt}} +\frac{1}{N_a  \gamma^{2}} 
  \mathrm{tr}(
   \mathbf{W}_{\tau,N_a}^{\mathrm{im}} \hat{\mathbf{Q}}_{\tau})
  \Big] + 1 \\
  \stackrel{(b)}{=}& \mathbb{E}\Big[
  \|{\btheta}_0 - {\btheta}_{\tau}^{\mathrm{gt}}\|^2_{ \mathbf{W}_{\tau,N_a}^{\mathrm{im}}}
  + \gamma^{-1}\btheta_{\tau}^{\mathrm{gt}\top}
  \big(
  - \mathbf{W}_{\tau,N_a}^{\mathrm{im}} \hat{\mathbf{Q}}_{\tau}
  +\hat{\mathbf{Q}}_{\tau}
  \mathbf{W}_{\tau,N_a}^{\mathrm{im}} 
  \big)
  \btheta_{\tau}^{\mathrm{gt}} 
   +\frac{1}{N_a  \gamma^{2}} 
  \mathrm{tr}(
   \mathbf{W}_{\tau,N_a}^{\mathrm{im}} \hat{\mathbf{Q}}_{\tau})
  \Big] + 1 \\
  \stackrel{(c)}{=}& \mathbb{E}\Big[
  \|{\btheta}_0 - {\btheta}_{\tau}^{\mathrm{gt}}\|^2_{\mathbf{W}_{\tau,N_a}^{\mathrm{im}}}
   +\frac{1}{N_a  \gamma^{2}} 
  \mathrm{tr}(
   \mathbf{W}_{\tau,N_a}^{\mathrm{im}} \hat{\mathbf{Q}}_{\tau})
  \Big] + 1 
  \numberthis
\end{align*}
where
$(b)$ follows from rearranging the equations;
$(c)$ follows from the fact that\\
$\btheta_{\tau}^{gt\top}
  \big(
  \mathbf{W}_{\tau,N_a}^{\mathrm{im}} \hat{\mathbf{Q}}_{\tau}
  \big)
  \btheta_{\tau}^{\rm gt} 
  = \big(\btheta_{\tau}^{\rm gt \top}
  ( \mathbf{W}_{\tau,N_a}^{\mathrm{im}} \hat{\mathbf{Q}}_{\tau})
  \btheta_{\tau}^{\rm gt} \big)^{\top}
  = \btheta_{\tau}^{\rm gt \top}
  \big(
  \hat{\mathbf{Q}}_{\tau}
  \mathbf{W}_{\tau,N_a}^{\mathrm{im}} 
  \big)
  \btheta_{\tau}^{\rm gt} $.

Since $\lim_{N_a \to \infty} \frac{1}{N_a} \mathbb{E}[ \gamma^{-2} \operatorname{tr}(\mathbf{W}_{\tau,N_a}^{\mathrm{im}} \hat{\mathbf{Q}}_{\tau,N_a})]  =0$,
from the definition of optimal population risk in~\eqref{eq:R}, the  optimal population risk of iMAML is given by
  \begin{subequations}
  \begin{align}\label{eq:biMAML_model_err}
    &\mathcal{R}^{\mathrm{im}}({\btheta_0} )
    \coloneqq 
    \lim_{N_a \to \infty} 
    \mathcal{R}_{N_a}^{\mathrm{im}}({\btheta_0} )
    = \mathbb{E}_{\tau}\big[\|\btheta_0-\btheta^{\mathrm{gt}}_{\tau}\|^2_{\mathbf{W}_{\tau}^{\mathrm{im}}}\big] + 1 \\
    \label{eq:W_tau_bimaml}
    &\mathbf{W}_{\tau}^{\mathrm{im}} =(\gamma ^{-1}\mathbf{Q}_{\tau}+\mathbf{I})^{-1}\mathbf{Q}_{\tau}(\gamma ^{-1}\mathbf{Q}_{\tau}+\mathbf{I})^{-1}
  \end{align}
  \end{subequations}
  whose minimizer is given by
  \begin{align}\label{eq:theta_0_star_biMAML_sln}
  \btheta_{0}^{\mathrm{im}} 
    = \mathop{\arg\min}_{\btheta_0}   \mathcal{R}^{\mathrm{im}}({\btheta_0} )
    =\mathbb{E}_{\tau}\big[\mathbf{W}_{\tau}^{\mathrm{im}}\big]^{-1} \mathbb{E}_{\tau}\big[\mathbf{W}_{\tau}^{\mathrm{im}} {\btheta}_{\tau}^{\mathrm{gt}}\big].
  \end{align}

  It is worth noting that, from Lemma~\ref{lemma:concentration_W_N}, we have the property $\mathbb{E}_{{\mathbf{x}}_{\tau}}[\hat{\mathbf{W}}_{\tau,N}^{\mathrm{im}}] = \mathbf{W}_{\tau,N}^{\mathrm{im}}$, $\lim_{N \to \infty}
 \mathbf{W}_{\tau,N}^{\mathrm{im}} = \mathbf{W}_{\tau}^{\mathrm{im}} $, which will be used in later sections to derive the specific optimal population risk and statistical error.

\subsection{Bayes  model agnostic meta learning method} 
\label{app_sub:bayes_biMAML_method}

For the Bayes  model agnostic meta learning (BaMAML) method, 
instead of obtaining a point estimation of the task-specific parameter, during adaptation, it obtains the posterior distribution or its approximation, given by
\begin{align}\label{eq:theta_tau_baMAML_theta0}
  \hat{p} (\btheta_{\tau} \mid \mathcal{D}_{\tau}, \btheta_0)
  =& \mathop{\arg\min}_{q({\btheta}_{\tau})\in \mathcal{Q}}
    \mathbb{E}_{p(\mathbf{x}_{\tau},y_{\tau}\mid \tau)}\big[
    \mathrm{KL}(q(\btheta_{\tau})\| 
    p(\btheta_{\tau}\mid \mathcal{D}_{\tau}, \btheta_0)) \big]
\end{align}
where $p(\btheta_{\tau}\mid \mathcal{D}_{\tau}, \btheta_0)$ can be computed from Bayes rule via 
\begin{align}\label{eq:posterior_theta_tau}
  p(\btheta_{\tau}\mid \mathcal{D}_{\tau,N}, \btheta_0) \propto p(\mathcal{D}_{\tau,N} \mid \btheta_{\tau}) p(\btheta_{\tau}\mid \btheta_0).
\end{align}
Assuming $y_{\tau}\mid \mathbf{x}_{\tau},\btheta_{\tau} \sim \mathcal{N}(\btheta_{\tau}^{\top}\mathbf{x}_{\tau}, 1)$,
the likelihood $p(\mathcal{D}_{\tau} \mid \btheta_{\tau})$ can be expressed by
\begin{align}\label{eq:likelihood_theta_tau}
  p(\mathcal{D}_{\tau,N} \mid \btheta_{\tau})
  =& \prod_{n=1}^{N} p(y_{\tau,n}\mid \mathbf{x}_{\tau,n},\btheta_{\tau})
  \propto
  \exp\{-\frac{1}{2}\|\mathbf{y}_{\tau,N} - \mathbf{X}_{\tau,N}{\btheta}_{\tau}\|^2\}.
\end{align}
Assuming the prior $\btheta_{\tau}\mid \btheta_{0} \sim \mathcal{N}(\btheta_{0}, \frac{1}{\gamma_b} \mathbf{I}_d)$, with $\gamma_b $ being the prespecified weight of the prior or regularizer. The prior can be expressed by
\begin{align}\label{eq:prior_theta_tau}
  p(\btheta_{\tau}\mid \btheta_{0})
  \propto& \exp\{-\frac{\gamma_b }{2}\|\btheta_{\tau} - \btheta_0\|^2\}.
\end{align}

Combining~\eqref{eq:posterior_theta_tau}-\eqref{eq:prior_theta_tau}, the posterior distribution of the per-task parameter satisfies
\begin{align}\label{eq:theta_tau_posterior}
  p(\btheta_{\tau}\mid \mathcal{D}_{\tau}, \btheta_0) 
  &\propto \exp\{
-\frac{1}{2}\|\mathbf{y}_{\tau, N} - \mathbf{X}_{\tau, N} {\btheta}_{\tau}\|^2
-\frac{\gamma_b }{2} \|{\btheta}_{\tau} - {\btheta}_{0}\|^2\} \\
  &\propto \exp\{-\frac{1}{2}(\btheta_{\tau} - \mu_{\btheta_{\tau},N})^{\top}\Sigma_{\btheta_{\tau},N}^{-1}(\btheta_{\tau} - \mu_{\btheta_{\tau},N})\}= \mathcal{N}(\mu_{\btheta_{\tau},N}, \Sigma_{\btheta_{\tau},N}) \\
  \label{eq:mu_Sigma_theta_tau}
  &\text{with }\Sigma_{\btheta_{\tau},N}=  (\mathbf{X}^{\top}_{\tau,N} \mathbf{X}_{\tau,N} +  \gamma_b \mathbf{I})^{-1},~~
  \mu_{\btheta_{\tau},N}= \Sigma_{\btheta_{\tau},N} ( \mathbf{X}_{\tau,N}^{\top} \mathbf{y}_{\tau,N} + \gamma_b  \btheta_0).
\end{align}

If $p(\btheta_{\tau}\mid \mathcal{D}_{\tau}, \btheta_0) \in \mathcal{Q}$,
then $\hat{p} (\btheta_{\tau} \mid \mathcal{D}_{\tau}, \btheta_0) = p(\btheta_{\tau}\mid \mathcal{D}_{\tau}, \btheta_0)$, which holds in our analysis since $\mathcal{Q}$ is defined to be the set of Gaussian distributions.
The  empirical loss of BaMAML is  
\begin{align}\label{eq:baMAML_emp_risk_app}
  {\mathcal{L}}_{T,N}^{\mathrm{ba}}({\btheta_0} )
  & \coloneqq 
  \frac{1}{T N_2}\sum_{\tau=1}^{T}\Big[ 
  -\int \log p(\mathcal{D}_{\tau}^{\text{val}} \mid \btheta_{\tau})
  \hat{p}^{\mathrm{ba}}(\btheta_{\tau} \mid \mathcal{D}_{\tau}^{\text{trn}}, \btheta_0) d\btheta_{\tau}
   \Big] \nonumber\\
 \mathrm{s.t.} &\quad 
  \hat{p}^{\mathrm{ba}}(\btheta_{\tau} \mid \mathcal{D}_{\tau}^{\text{trn}}, \btheta_0) 
  = \mathop{\arg\min}_{q({\btheta}_{\tau})\in \mathcal{Q}}
    \mathrm{KL}(q(\btheta_{\tau})\| 
    p(\btheta_{\tau}\mid \mathcal{D}_{\tau}^{\text{trn}}, \btheta_0)) 
\end{align}
where $\hat{p}^{\mathrm{ba}}(\btheta_{\tau} \mid \mathcal{D}_{\tau}^{\text{trn}}, \btheta_0) = {p}^{\mathrm{ba}}(\btheta_{\tau} \mid \mathcal{D}_{\tau}^{\text{trn}}, \btheta_0) = \mathcal{N}(\mu_{\btheta_{\tau},N_1}, \Sigma_{\btheta_{\tau},N_1})$ is the solution of the inner problem.
Therefore
\begin{align*}
  {\mathcal{L}}_{T,N}^{\mathrm{ba}}({\btheta_0} ) &=
\frac{1}{T N_2}\sum_{\tau=1}^{T} -\log 
p(\mathcal{D}_{\tau}^{\mathrm{val}} \mid \mathcal{D}_{\tau}^{\mathrm{trn}},\btheta_0) \\
&= - \frac{1}{T N_2}\sum_{\tau=1}^{T} [\log 
p(\mathcal{D}_{\tau}^{\mathrm{val}}, \mathcal{D}_{\tau}^{\mathrm{trn}} \mid \btheta_0) - \log p( \mathcal{D}_{\tau}^{\mathrm{trn}} \mid \btheta_0)] 
\end{align*}
where
\begin{align}\label{eq:derive_L_val_trn_ba}
  &\log p(\mathcal{D}_{\tau}^{\mathrm{val}}, \mathcal{D}_{\tau}^{\mathrm{trn}} \mid \btheta_0) - \log p( \mathcal{D}_{\tau}^{\mathrm{trn}} \mid \btheta_0) 
= \log  {p}(\mathbf{y}^{\mathrm{all}}_{\tau,N} 
  \mid \mathbf{X}^{\mathrm{all}}_{\tau,N}, \btheta_0)
  - \log  {p}(\mathbf{y}^{\mathrm{trn}}_{\tau,N_1} 
  \mid \mathbf{X}^{\mathrm{trn}}_{\tau,N_1}, \btheta_0) \nonumber\\
=&  -\frac{1}{2}\|\mathbf{y}^{\mathrm{all}}_{\tau,N} 
  - \mathbf{X}^{\mathrm{all}}_{\tau,N} \btheta_0\|^2_{\Sigma^{-1}_{y,N}}
  +\frac{1}{2}\|\mathbf{y}^{\mathrm{trn}}_{\tau,N_1} 
  - \mathbf{X}^{\mathrm{trn}}_{\tau,N_1} \btheta_0\|^2_{\Sigma^{-1}_{y,N_1}}
\end{align}
with $\Sigma^{-1}_{y,N} 
= (\mathbf{I}_N + \gamma_b^{-1}\mathbf{X}_{\tau,N}\mathbf{X}^{\top}_{\tau,N})^{-1}$.
The last equation is because $ {p}(\mathbf{y}^{\rm all}_{\tau,N} \mid \mathbf{X}^{\rm all}_{\tau,N}, \btheta_0)$ can be computed by
\begin{align*}\label{eq:theta_0_hat_baMAML_derive}
  & {p}(\mathbf{y}^{\rm all}_{\tau,N} 
  \mid \mathbf{X}^{\rm all}_{\tau,N}, \btheta_0)
  = \int {p}(\mathbf{y}^{\rm all}_{\tau,N} 
  \mid \mathbf{X}^{\rm all}_{\tau,N}, \btheta_{\tau})
  {p}(\btheta_{\tau} \mid \btheta_0) d \btheta_{\tau} \\
  \propto  &
  \int \exp \{-\frac{1}{2 }
  \|\mathbf{y}_{\tau,N}^{\text{all}} - \mathbf{X}_{\tau,N}^{\text{all}}\btheta_{\tau}\|^2 
-\frac{\gamma_b }{2} \|{\btheta}_{\tau} - {\btheta}_{0}\|^2\}d\btheta_{\tau} \\
  =  &
  \int \exp \{-\frac{1}{2}(\btheta_{\tau} - \mu_{\btheta_{\tau}})^{\top}\Sigma_{\btheta_{\tau}}^{-1}
  (\btheta_{\tau} - \mu_{\btheta_{\tau}})
  + \frac{1}{2}\mu_{\btheta_{\tau}}^{\top}\Sigma_{\btheta_{\tau}}^{-1}\mu_{\btheta_{\tau}}
 - \frac{\gamma_b }{2}{\btheta}_{0}^{\top}{\btheta}_{0}
 -\frac{1}{2} 
 \mathbf{y}_{\tau,N}^{\mathrm{all}\top}
 \mathbf{y}_{\tau,N}^{\mathrm{all}}\}
 d\btheta_{\tau} \\
\propto & \exp\{ -\frac{1}{2}
  ( - \mu_{\btheta_{\tau}}^{\top}\Sigma_{\btheta_{\tau}}^{-1}\mu_{\btheta_{\tau}} + {\gamma_b }{\btheta}_{0}^{\top}{\btheta}_{0}
  +
 \mathbf{y}_{\tau,N}^{\mathrm{all}\top}
 \mathbf{y}_{\tau,N}^{\mathrm{all}}
 )\}
 =  \exp\{ -\frac{1}{2}
 \|\mathbf{y}^{\mathrm{all}}_{\tau,N} 
  - \mathbf{X}^{\mathrm{all}}_{\tau,N} \btheta_0\|^2_{\Sigma^{-1}_{y,N}} \}
\numberthis
\end{align*}
where the last equation follows from the Binomial inverse theorem.
Similarly, $ {p}(\mathbf{y}^{\rm trn}_{\tau,N} \mid \mathbf{X}^{\rm trn}_{\tau,N}, \btheta_0) 
\propto \exp\{ -\frac{1}{2}
\|\mathbf{y}^{\mathrm{trn}}_{\tau,N_1} 
 - \mathbf{X}^{\mathrm{trn}}_{\tau,N_1} \btheta_0\|^2_{\Sigma^{-1}_{y,N_1}} \}$. 

To solve for $\hat{\btheta}_0^{\mathrm{ba}}$, using the optimality condition, we obtain
\begin{subequations}
\begin{align}\label{eq:theta_0_hat_baMAML_sln1}
  &\hat{\btheta}_{0}^{\mathrm{ba}} 
  = \Big(\sum_{\tau=1}^{T}\hat{\mathbf{W}}_{\tau, N}^{\mathrm{ba}}\Big)^{-1}
   \Big(\sum_{\tau=1}^{T}
   \hat{\mathbf{W}}_{\tau, N}^{\mathrm{ba}} \btheta_{\tau}^{\mathrm{gt}} \Big)+ \Delta_T^{\rm ba}\\
  &\Delta_T^{\rm ba}
  = \Big(\sum_{\tau=1}^{T}\hat{\mathbf{W}}_{\tau, N}^{\mathrm{ba}}\Big)^{-1}
  \frac{1}{N_2} \Big(\sum_{\tau=1}^{T}\mathbf{X}_{\tau,N}^{\text{all}\top}\Sigma_{y,N}^{-1} \mathbf{e}_{\tau,N}^{\text{all}} - \mathbf{X}_{\tau,N_1}^{\text{trn}\top}\Sigma_{y,N_1}^{-1} \mathbf{e}_{\tau,N_1}^{\text{trn}}
   \Big) \\
   \label{eq:W_hat_ba1}
  &\hat{\mathbf{W}}_{\tau, N}^{\mathrm{ba}} = 
 \Big((\frac{\gamma_b}{N})^{-1}\hat{\mathbf{Q}}_{\tau,N} + \mathbf{I}\Big)^{-1}
\hat{\mathbf{Q}}_{\tau,N_2}
\Big((\frac{\gamma_b}{N_1})^{-1}\hat{\mathbf{Q}}_{\tau,N_1} + \mathbf{I}\Big)^{-1}\\
&\quad\quad~\, 
= \Big(({\gamma s})^{-1}\hat{\mathbf{Q}}_{\tau,N} + \mathbf{I}\Big)^{-1}
\hat{\mathbf{Q}}_{\tau,N_2}
\Big( {\gamma}^{-1}\hat{\mathbf{Q}}_{\tau,N_1} + \mathbf{I}\Big)^{-1}
\end{align}
\end{subequations}
where the last equation is because we choose $\gamma_b = N_1 \gamma $ for a fair comparison with iMAML.

Based on~\eqref{eq:R}, the BaMAML meta-test risk is defined as 
\begin{subequations}
\begin{align}\label{eq:bamaml_pop_risk1}
  &\mathcal{R}_{N_a}^{\mathrm{ba}}({\btheta_0} )
  = 
  \frac{1}{N}\mathbb{E}\big[-\log  {p}(\mathbf{y}_{\tau,N} 
  \mid \mathbf{X}_{\tau,N}, \mathcal{D}_{\tau, N_a}, \btheta_0) \big] \nonumber \\
  =&
  \mathbb{E}_{\tau}\Big[\|\btheta_0-\btheta^{\mathrm{gt}}_{\tau}\|^2_{\mathbf{W}_{\tau,N_a}^{\mathrm{ba}}}\Big] 
  + 1 
  + \frac{1}{N_a}\mathbb{E}\big[\mathrm{tr}\big((\mathbf{I}_d + (\gamma s)^{-1}\hat{\mathbf{Q}}_{\tau,N+N_a})^{-1} 
  - (\mathbf{I}_d + \gamma ^{-1}\hat{\mathbf{Q}}_{\tau,N_a})^{-1}
   \big)\big] \\
&\mathbf{W}_{\tau,N_a}^{\mathrm{ba}} 
\label{eq:W_tau_N_ba_1}
= \mathbb{E}_{\mathbf{x}_{\tau}}
\big[ \frac{\gamma_b}{N} [(\mathbf{I}_{d}+\gamma_b^{-1}N_a \hat{\mathbf{Q}}_{\tau,N_a})^{-1} 
- (\mathbf{I}_{d}+\gamma_b^{-1}(N+N_a) \hat{\mathbf{Q}}_{\tau,N+N_a})^{-1}]\big] \nonumber \\
&\quad\quad~~~ 
=\mathbb{E}_{\mathbf{x}_{\tau}} \big[\big((\gamma s)^{-1}\hat{\mathbf{Q}}_{\tau,N+N_a} + \mathbf{I}\big)^{-1}
\hat{\mathbf{Q}}_{\tau,N}
\big(\gamma^{-1}\hat{\mathbf{Q}}_{\tau,N_a} + \mathbf{I}\big)^{-1} \big]
\end{align}
\end{subequations}

Taking limits of $\mathcal{R}_{N_a}^{\mathrm{ba}}$ w.r.t. $N_a$  further leads to 
\begin{align}
\mathcal{R}^{\mathrm{ba}}({\btheta_0} ) \coloneqq
\lim_{N_a \to \infty}  \mathcal{R}_{N_a}^{\mathrm{ba}}({\btheta_0} )
  = & 
  \lim_{N_a \to \infty}
  \mathbb{E}_{\tau}\big[\|\btheta_0-\btheta^{\mathrm{gt}}_{\tau}\|^2_{\mathbf{W}_{\tau,N_a}^{\mathrm{ba}}}\big] 
  + 1 
  = 
  \mathbb{E}_{\tau}\big[\|\btheta_0-\btheta^{\mathrm{gt}}_{\tau}\|^2_{\mathbf{W}_{\tau}^{\mathrm{ba}}}\big] 
  + 1 
\end{align}

$\btheta_{0}^{\mathrm{ba}}$ is defined to be the minimizer of $\mathcal{R}^{\mathrm{ba}}({\btheta_0} )$, given by
\begin{subequations}
\begin{align}\label{eq:theta_0_star_baMAML_sln}
  &\btheta_{0}^{\mathrm{ba}} 
  = \mathop{\arg\min}_{\btheta_0} 
  \mathcal{R}^{\mathrm{ba}}({\btheta_0} )
  = \mathop{\arg\min}_{\btheta_0} 
  \mathbb{E}_{\tau}\big[\|\btheta_0-\btheta^{\mathrm{gt}}_{\tau}\|^2_{\mathbf{W}_{\tau}^{\mathrm{ba}}}\big]
  =\mathbb{E}_{\tau}\big[\mathbf{W}_{\tau}^{\mathrm{ba}}\big]^{-1} \mathbb{E}_{\tau}\big[\mathbf{W}_{\tau}^{\mathrm{ba}} {\btheta}_{\tau}^{\mathrm{gt}}\big], \\
  &\text{with }
  \mathbf{W}_{\tau}^{\mathrm{ba}} 
  = 
  \big(({\gamma s})^{-1}{\mathbf{Q}}_{\tau} + \mathbf{I}\big)^{-1} {\mathbf{Q}}_{\tau}
  \big({\gamma}^{-1}{\mathbf{Q}}_{\tau} + \mathbf{I}\big)^{-1}.
\end{align}
\end{subequations}

It is worth noting that, from Lemma~\ref{lemma:concentration_W_N}, we have the property $\mathbb{E}_{{\mathbf{x}}_{\tau}}[\hat{\mathbf{W}}_{\tau,N}^{\mathrm{ba}}] = \mathbf{W}_{\tau,N}^{\mathrm{ba}}$, $\lim_{N \to \infty}
\mathbf{W}_{\tau,N}^{\mathrm{ba}} = \mathbf{W}_{\tau}^{\mathrm{ba}} $, which will be used in later sections to derive the specific optimal population risk and statistical error.

The above discussion provides proof for Proposition~\ref{prop:solutions}.
Next we analyze the optimal population risk and the statistical error based on the solutions.

\section{Optimal population risk and statistical error analysis} 
\label{app_sec:model_stats_error}

\paragraph{Meta-test risk decomposition.}
Recall the meta-test risk function for the method ${\cal A}$ of $\btheta_0$ in \eqref{eq:R}. 
For method $\mathcal{A}$, 
{plugging $\hat{\btheta}_0^{\mathcal{A}}$ into \eqref{eq:R} and taking the limit over $N_a$, the number of data during adaptation, we have}

\begin{align*}\label{eq:meta_test_risk_decompose}
&
 \mathcal{R}_{}^{\mathcal{A}}(\hat{\btheta}_0^{\mathcal{A}})=
\mathbb{E}_{\tau}\left[\|\hat{\btheta}_0^{\mathcal{A}}-\btheta^{\mathrm{gt}}_{\tau}\|^2_{\mathbf{W}_{\tau}^{\mathcal{A}}}\right] + 1 
{\stackrel{{(a)}}{=}}
\mathbb{E}_{\tau}\left[\|\hat{\btheta}_0^{\mathcal{A}}- {\btheta}_0^{\mathcal{A}} +{\btheta}_0^{\mathcal{A}} -\btheta^{\mathrm{gt}}_{\tau}\|^2_{\mathbf{W}_{\tau}^{\mathcal{A}}}\right] + 1 \\
=&
\mathbb{E}_{\tau}\left[\|\hat{\btheta}_0^{\mathcal{A}}- {\btheta}_0^{\mathcal{A}}\|^2_{\mathbf{W}_{\tau}^{\mathcal{A}}}
+\|{\btheta}_0^{\mathcal{A}} -\btheta^{\mathrm{gt}}_{\tau}\|^2_{\mathbf{W}_{\tau}^{\mathcal{A}}}
+ 2(\hat{\btheta}_0^{\mathcal{A}}- {\btheta}_0^{\mathcal{A}})^{\top}
\mathbf{W}_{\tau}^{\mathcal{A}}
({\btheta}_0^{\mathcal{A}} -\btheta^{\mathrm{gt}}_{\tau})
\right] + 1 \\
{\stackrel{(b)}{=}}&
\mathbb{E}_{\tau}\left[\|\hat{\btheta}_0^{\mathcal{A}}- {\btheta}_0^{\mathcal{A}}\|^2_{\mathbf{W}_{\tau}^{\mathcal{A}}}
+\|{\btheta}_0^{\mathcal{A}} -\btheta^{\mathrm{gt}}_{\tau}\|^2_{\mathbf{W}_{\tau}^{\mathcal{A}}}
+ 2(\hat{\btheta}_0^{\mathcal{A}}- {\btheta}_0^{\mathcal{A}})^{\top}
\mathbf{W}_{\tau}^{\mathcal{A}}
(\mathbb{E}_{\tau}\left[\mathbf{W}_{\tau}^{\mathcal{A}}\right]^{-1} \mathbb{E}_{\tau}\left[\mathbf{W}_{\tau}^{\mathcal{A}} {\btheta}_{\tau}^{\mathrm{gt}}\right] -\btheta^{\mathrm{gt}}_{\tau})
\right] + 1 \\
{\stackrel{(c)}{=}}&
\mathbb{E}_{\tau}\left[\|\hat{\btheta}_0^{\mathcal{A}}- {\btheta}_0^{\mathcal{A}}\|^2_{\mathbf{W}_{\tau}^{\mathcal{A}}}
+\|{\btheta}_0^{\mathcal{A}} -\btheta^{\mathrm{gt}}_{\tau}\|^2_{\mathbf{W}_{\tau}^{\mathcal{A}}}
\right] + 1\numberthis
\end{align*}
where
{$(a)$ follows from
 $\mathbf{W}_{\tau}^{\mathcal{A}} = \lim_{N_a\to \infty} \mathbf{W}_{\tau, N_a}^{\mathcal{A}}$},
$(b)$ is by plugging 
$\btheta_0^{\mathcal{A}} = 
\mathbb{E}_{\tau}\big[\mathbf{W}_{\tau}^{\mathcal{A}}\big]^{-1} \mathbb{E}_{\tau}\big[\mathbf{W}_{\tau}^{\mathcal{A}} {\btheta}_{\tau}^{\mathrm{gt}}\big]$,
$(c)$ is because
$\mathbb{E}_{\tau}[\mathbf{W}_{\tau}^{\mathcal{A}}
(\mathbb{E}_{\tau}\big[\mathbf{W}_{\tau}^{\mathcal{A}}\big]^{-1} \mathbb{E}_{\tau}\big[\mathbf{W}_{\tau}^{\mathcal{A}} {\btheta}_{\tau}^{\mathrm{gt}}\big] -\btheta^{\mathrm{gt}}_{\tau})]
=\mathbb{E}_{\tau}\big[\mathbf{W}_{\tau}^{\mathcal{A}} {\btheta}_{\tau}^{\mathrm{gt}}\big]
-\mathbb{E}_{\tau}\big[\mathbf{W}_{\tau}^{\mathcal{A}} {\btheta}_{\tau}^{\mathrm{gt}}\big] =0$.
From~\eqref{eq:meta_test_risk_decompose} 
and the definition of $\mathcal{R}_{N_a}^{\mathcal{A}}(\cdot)$ in~\eqref{eq:R},
we can decompose the meta-test risk to the  optimal population risk  and statistical error as follows
\begin{equation}\label{eq:meta_test_risk_decompose_final_app}
  \lim_{N_a\to \infty} \mathcal{R}_{N_a}^{\mathcal{A}}(\hat{\btheta}_0^{\mathcal{A}})= 
  \underbracket{\lim_{N_a\to \infty} \mathcal{R}_{N_a}^{\mathcal{A}}({\btheta}_0^{\mathcal{A}})}_{\text{optimal population risk}}
  +
  \underbracket{\|\hat{\btheta}_0^{\mathcal{A}}- {\btheta}_0^{\mathcal{A}}\|^2_{\mathbb{E}_{\tau}[\mathbf{W}_{\tau}^{\mathcal{A}}]}}_{\text{statistical error}~\mathcal{E}^{2}_{\mathcal{A}} (\hat{\btheta}_0^{\mathcal{A}})}
.
\end{equation}
This completes the proof of Proposition~3 in the main paper.
Note that, the statistical error 
$\mathcal{E}^{2}_{\mathcal{A}} (\hat{\btheta}_0^{\mathcal{A}})$ is resulted from finite random data samples during meta-training to obtain the estimation of the parameter $\hat{\btheta}_0^{\mathcal{A}}$, but not from $N_a$ in~\eqref{eq:meta_test_risk_decompose_final_app}, which is the number of adaptation data during meta-testing.

\subsection{Optimal population risk} 
\label{app_sub:model_error}

The optimal population risk under different methods is given by
$\mathcal{R}^{{\cal A}}(\btheta_0^{{\cal A}}) = \min_{\btheta_0}\mathcal{R}^{{\cal A}}(\btheta_0)$.
Based on the results in Section~\ref{app_sec:methods_formulation_sln}, we compute the optimal population risk of each method.

For \textbf{ERM}, the optimal population risk is computed by
\begin{align}\label{eq:erm_model_error}
  \mathcal{R}^{\mathrm{er}}(\btheta_0^{\mathrm{er}} )
  = \mathbb{E}_{\tau}\big[\|\btheta_0^{\mathrm{er}}-\btheta^{\mathrm{gt}}_{\tau}\|^2_{\mathbf{W}^{\mathrm{er}}_{\tau}}\big] + 1.
\end{align}

For \textbf{MAML}, the optimal population risk is computed by
\begin{align}\label{eq:maml_model_error}
  \mathcal{R}^{\mathrm{ma}}(\btheta_0^{\mathrm{ma}} , \alpha)
  = \mathbb{E}_{\tau}\big[\|\btheta_0^{\mathrm{ma}}-\btheta^{\mathrm{gt}}_{\tau}\|^2_{\mathbf{W}^{\mathrm{ma}}_{\tau}(\alpha)}\big] + 1.
\end{align}
Note that when $\alpha   = 0$, 
 $\mathcal{R}^{\mathrm{ma}}(\btheta_0^{\mathrm{ma}}, \alpha ) = \mathcal{R}^{\mathrm{er}}(\btheta_0^{\mathrm{er}})$.

\textbf{Comparison of ERM and MAML optimal population risk.}
To compare the optimal population risk of ERM and MAML, 
as shown in~\citep{gao2020_model_opt_tradeoff_ml}, when 
$\|\mathbf{Q}_{\tau}\| \leq \bar{\lambda},
0< \alpha \leq 1/\bar{\lambda}$,
$\mathcal{R}^{\mathrm{ma}}(\btheta_0^{\mathrm{ma}}, \alpha)$
is monotonically decreasing.
Therefore 
$\mathcal{R}^{\mathrm{ma}}(\btheta_0^{\mathrm{ma}},\alpha) < \mathcal{R}^{\mathrm{er}}(\btheta_0^{\mathrm{er}})$
when $ 0 < \alpha \leq  1/ \bar{\lambda}$.

For \textbf{iMAML}, the optimal population risk is computed by
\begin{align}\label{eq:bimaml_model_error}
  \mathcal{R}^{\mathrm{im}}(\btheta_0^{\mathrm{im}},\gamma )
  = \mathbb{E}_{\tau}\big[\|\btheta_0^{\mathrm{im}}-\btheta^{\mathrm{gt}}_{\tau}\|^2_{\mathbf{W}^{\mathrm{im}}_{\tau}(\gamma)}\big] + 1.
\end{align}

\textbf{Comparison of MAML and iMAML optimal population risk.}
We can see that as $\gamma  \to \infty, \mathbf{W}_{\tau}^{\mathrm{im}}(\gamma) \to \mathbf{W}_{\tau}^{\mathrm{er}}=  \mathbf{Q}_{\tau}, \mathcal{R}^{\mathrm{im}}({\btheta_0},\gamma) \to \mathcal{R}^{\mathrm{er}}({\btheta_0})$; as 
$\gamma  \to 0, \mathbf{W}_{\tau}^{\mathrm{im}}(\gamma) \to 0, \mathcal{R}^{\mathrm{im}}({\btheta_0}) \to 1$.
To explicitly compare the optimal population risk of MAML and iMAML, 
we will show next that when $\gamma $ takes certain values, the corresponding risks satisfy $\mathcal{R}^{\mathrm{im}}({\btheta_0^{\mathrm{im}}},\gamma) < \mathcal{R}^{\mathrm{ma}}({\btheta_0^{\mathrm{ma}}}, \alpha)$.

\begin{corollary}
\label{crlr:assmp1_2}
Based on Assumption~\ref{asmp:bounded_data_matrix_eigenvalues},
for any $ \tau \sim p(\mathcal{T}) $, and $0< \alpha < 1/\bar{\lambda}$, $\|\mathbf{W}^{\mathrm{ma}}_{\tau}\| >0$.
And generally in a typical multi-task learning setting, the tasks are not all identical, therefore
there exist $ \tau \sim p(\mathcal{T}), \tau \in T, \btheta_{0}^{{\cal A}} - \btheta_{\tau}^{\mathrm{gt}}\neq \mathbf{0}$.
Therefore there exists   $\tau \in T $ such that $(\btheta_0^{\mathrm{ma}}-\btheta^{\mathrm{gt}}_{\tau})^{\top}{\mathbf{W}_{\tau}^{\mathrm{ma}}} 
  (\btheta_0^{\mathrm{ma}}-\btheta^{\mathrm{gt}}_{\tau}) > 0$.
\end{corollary}


First we show that the minimum value of the MAML population risk is larger than $1$, i.e., $\mathcal{R}^{\mathrm{ma}}({\btheta_0^{\mathrm{ma}}},\alpha) > 1$.
According to Corollary~\ref{crlr:assmp1_2}, it is apparent that
\begin{align}\label{eq:bimaml_risk_bigger0}
  \mathbb{E}_{\tau}\big[
  \|\btheta_0^{\mathrm{ma}}-\btheta^{\mathrm{gt}}_{\tau}\|^2_{\mathbf{W}_{\tau}^{\mathrm{ma}}} \big]
  &= \mathbb{E}_{\tau}\big[
  (\btheta_0^{\mathrm{ma}}-\btheta^{\mathrm{gt}}_{\tau})^{\top}{\mathbf{W}_{\tau}^{\mathrm{ma}}} 
  (\btheta_0^{\mathrm{ma}}-\btheta^{\mathrm{gt}}_{\tau})
  \big] > 0.
\end{align}
Therefore, we have
\begin{align}\label{eq:maml_pop_risk_larger_c}
\mathcal{R}^{\mathrm{ma}}({\btheta_0^{\mathrm{ma}}},\alpha) = \mathbb{E}_{\tau}\big[
  \|\btheta_0^{\mathrm{ma}}-\btheta^{\mathrm{gt}}_{\tau}\|^2_{\mathbf{W}_{\tau}^{\mathrm{ma}}} \big] + 1 > 1  .
\end{align}

Note that $\mathcal{R}^{\mathrm{ma}}({\btheta_0^{\mathrm{ma}}},\alpha)$ also depends on $\alpha  $, 
Let $r^{\mathrm{ma}} = \min_{\alpha  } \mathcal{R}^{\mathrm{ma}}({\btheta_0^{\mathrm{ma}}}, \alpha  ) - 1$.
From \eqref{eq:maml_pop_risk_larger_c} we know that $r^{\mathrm{ma}} > 0$.
We will then show one can always find certain $\gamma $ such that 
\begin{align}\label{eq:bimaml_maml_pop_risk_ineq}
  \mathcal{R}^{\mathrm{im}}({\btheta_0^{\mathrm{im}}},\gamma ) < \min_{\alpha  } \mathcal{R}^{\mathrm{ma}}({\btheta_0^{\mathrm{ma}}}, \alpha  ) = r^{\mathrm{ma}}+1
\end{align}
or equivalently 
  $\mathbb{E}_{\tau}\big[
   \|\btheta_0^{\mathrm{im}}-\btheta^{\mathrm{gt}}_{\tau}\|^2_{\mathbf{W}_{\tau}^{\mathrm{im}}} \big] < r^{\mathrm{ma}}$.

From Assumption~1, bounded eigenvalues of per-task data matrix, we can derive 
\begin{align}
  \|(\gamma ^{-1} \mathbf{Q}_{\tau} + \mathbf{I})^{-1}\| 
= 1/\lambda_{\text{min}}(\gamma ^{-1} \mathbf{Q}_{\tau} + \mathbf{I})
= 1/(\gamma ^{-1} \lambda_{\text{min}}(\mathbf{Q}_{\tau}) + 1)
\leq \frac{1}{\gamma ^{-1} \ushort{\lambda} +1}  
\end{align}
from which we can bound the optimal population risk of iMAML by 
\begin{align*}
&\mathbb{E}_{\tau}\big[
  \|\btheta_0^{\mathrm{im}}-\btheta^{\mathrm{gt}}_{\tau}\|^2_{\mathbf{W}_{\tau}^{\mathrm{im}}} \big] 
  \leq  \mathbb{E}_{\tau} [\|\btheta_0^{\mathrm{im}}-\btheta^{\mathrm{gt}}_{\tau}\|^2] \mathrm{sup}_{\tau} \|\mathbf{W}_{\tau}^{\mathrm{im}}\|
\end{align*}
where we have discussed the bound for $\mathrm{sup}_{\tau} \|\mathbf{W}_{\tau}^{\mathrm{im}}\|$ based on Assumption~1. Thus it suffices to bound 
$\mathbb{E}_{\tau} [\|\btheta_0^{\mathrm{im}}-\btheta^{\mathrm{gt}}_{\tau}\|^2]$,  as follows
\begin{align*}
  &\mathbb{E}_{\tau} [\|\btheta_0^{\mathrm{im}}-\btheta^{\mathrm{gt}}_{\tau}\|^2]
  = \|\btheta_0^{\mathrm{im}}\|^2 - 2 \btheta_0^{\mathrm{im}\top} \mathbb{E}_{\tau} [\btheta_{\tau}^{\mathrm{gt}}] + \mathrm{tr}(\mathrm{Cov}_{\tau}[\btheta_{\tau}^{\mathrm{gt}}]) + \|\mathbb{E}_{\tau} [\btheta_{\tau}^{\mathrm{gt}}]\|^2 \\
  \leq & \|\btheta_0^{\mathrm{im}}\|^2 + 2 \|\btheta_0^{\mathrm{im}} \| M + \mathrm{tr}(\mathrm{Cov}_{\tau}[\btheta_{\tau}^{\mathrm{gt}}]) + M^2 
  \leq  \|\btheta_0^{\mathrm{im}}\|^2 + 2 \|\btheta_0^{\mathrm{im}} \| M + \mathrm{tr}(\mathrm{Cov}_{\tau}[\btheta_{\tau}^{\mathrm{gt}}]) + M^2 \\
  \leq & (M + \|\btheta_0^{\mathrm{im}} \|)^2 + \frac{R^2}{d} \cdot d 
  = (M + \|\btheta_0^{\mathrm{im}} \|)^2 + {R^2}
\end{align*}
where the last inequality follows from Assumption~2 that the task parameter distribution is sub-gaussian.
Similarly $\mathbb{E}_{\tau} [\|\btheta_0^{\mathrm{im}}-\btheta^{\mathrm{gt}}_{\tau}\|^2] \leq (M + \|\btheta_0^{\mathrm{ma}} \|)^2 + {R^2}$.
 Therefore
\begin{align*}\label{eq:bound_bimaml_risk}
&\mathbb{E}_{\tau}\Big[
  \|\btheta_0^{\mathrm{im}}-\btheta^{\mathrm{gt}}_{\tau}\|^2_{\mathbf{W}_{\tau}^{\mathrm{im}}} \Big] 
  \leq  \mathbb{E}_{\tau} [\|\btheta_0^{\mathrm{im}}-\btheta^{\mathrm{gt}}_{\tau}\|^2] \mathrm{sup}_{\tau} \|\mathbf{W}_{\tau}^{\mathrm{im}}\| \\
  \leq &  \big((M + \|\btheta_0^{\mathrm{im}} \|)^2 + {R^2}\big)
  \mathbb{E}_{\tau}\Big[\|
   (\gamma ^{-1} \mathbf{Q}_{\tau} + \mathbf{I})^{-1} 
  \mathbf{Q}_{\tau}
  (\gamma ^{-1} \mathbf{Q}_{\tau} + \mathbf{I})^{-1}\| \Big] \\
  \leq &  \big((M + \|\btheta_0^{\mathrm{im}} \|)^2 + {R^2}\big)
  \mathbb{E}_{\tau}\Big[\|
  (\gamma ^{-1} \mathbf{Q}_{\tau} + \mathbf{I})^{-1} \|
  \|\mathbf{Q}_{\tau}\|
  \|(\gamma ^{-1} \mathbf{Q}_{\tau} + \mathbf{I})^{-1}\| \Big] \\
  \stackrel{{(a)}}{\leq}  
  &\frac{  \big((M + \|\btheta_0^{\mathrm{im}} \|)^2 + {R^2}\big) \bar{\lambda}}{(\gamma ^{-1} \ushort{\lambda} +1)^2}
  \numberthis
\end{align*}
where $(a)$ holds because $\|\mathbf{Q}_{\tau}\| \leq \bar{\lambda}$,
$\|(\gamma ^{-1} \mathbf{Q}_{\tau} + \mathbf{I})^{-1}\| 
\leq (\gamma ^{-1} \ushort{\lambda} +1)^{-1}$ from Assumption~1.

Let $C_{\btheta} = \max\{ \big((M + \|\btheta_0^{\mathrm{im}} \|)^2 + {R^2}\big)^{\frac{1}{2}}, \big((M + \|\btheta_0^{\mathrm{ma}} \|)^2 + {R^2}\big)^{\frac{1}{2}}\}$.
In order to ensure $\big((M + \|\btheta_0^{\mathrm{im}} \|)^2 + {R^2}\big) \bar{\lambda} \frac{1}{(\gamma ^{-1} \ushort{\lambda} +1)^2} \leq  C_{\btheta}^2 \bar{\lambda} \frac{1}{(\gamma ^{-1} \ushort{\lambda} +1)^2} < r^{\mathrm{ma}}$ 
it suffices to ensure
\begin{align}\label{eq:gamma_tau_R_ma_bi_ineq}
\gamma ^{-1} \ushort{\lambda} +1 > (r^{\mathrm{ma}})^{-\frac{1}{2}}   C_{\btheta} \bar{\lambda}^{\frac{1}{2}}. 
\end{align}
Since 
$0 < r^{\mathrm{ma}} = \mathbb{E}_{\tau}\big[
  \|\btheta_0^{\mathrm{ma}}-\btheta^{\mathrm{gt}}_{\tau}\|^2_{\mathbf{W}_{\tau}^{\mathrm{ma}}} \big] 
  \leq  C_{\btheta}^2 \mathbb{E}_{\tau}[\|\mathbf{W}_{\tau}^{\mathrm{ma}}\|]
  < C_{\btheta}^2 \bar{\lambda}$.
It follows that
\begin{align}\label{eq:R_ma_1_ineq}
  (r^{\mathrm{ma}})^{-\frac{1}{2}}  C_{\btheta}\bar{\lambda}^{\frac{1}{2}} - 1 > 0.
\end{align}
Then from \eqref{eq:gamma_tau_R_ma_bi_ineq} and \eqref{eq:R_ma_1_ineq} one can derive
\begin{align}\label{eq:gamma_tau_range}
0<\gamma  < \big((r^{\mathrm{ma}})^{-\frac{1}{2}}  C_{\btheta}\bar{\lambda}^{\frac{1}{2}} - 1 \big)^{-1}\ushort{\lambda}.  
\end{align}
In other words, by choosing \eqref{eq:gamma_tau_range}, we have $\mathcal{R}^{\mathrm{im}}(\btheta_0^{\mathrm{im}},\gamma ) < \mathcal{R}^{\mathrm{ma}}(\btheta_0^{\mathrm{ma}},\alpha  ) 
< \mathcal{R}^{\mathrm{er}}(\btheta_0^{\mathrm{er}}), \forall 0<\alpha   \leq 1/\bar{\lambda}$.
We summarize this conclusion in Theorem~\ref{thm:lower_risk_bi_ma} below.

\begin{theorem}[iMAML has lower optimal population risk than MAML]\label{thm:lower_risk_bi_ma}
Under Assumptions~1-2,
for meta-test task $\tau$ and arbitrary ${\btheta}$, 
the population risks for MAML and iMAML, as functions of ${\btheta}$, are
$$
\mathcal{R}^{\mathrm{ma}}({\btheta},\alpha) \equiv \mathbb{E}_{\tau}\big[\|{\btheta}-{\btheta}_{\tau}\|^{2}_{\mathbf{W}_{\tau}^{\mathrm{ma}}(\alpha)}\big] + 1 
\quad \text { and } \quad 
\mathcal{R}^{\mathrm{im}}({\btheta}, \gamma) \equiv \mathbb{E}_{\tau}\big[\|{\btheta}-{\btheta}_{\tau}\|^{2}_{\mathbf{W}_{\tau}^{\mathrm{im}}(\gamma)}\big] + 1
$$
where 
$\mathbf{W}_{\tau}^{\mathrm{ma}} (\alpha) = 
  (\mathbf{I}-\alpha \mathbf{Q}_{\tau}) \mathbf{Q}_{\tau}(\mathbf{I} -\alpha \mathbf{Q}_{\tau})$
and
${\mathbf{W}_{\tau}^{\mathrm{im}}} (\gamma)
=
  (\gamma ^{-1}\mathbf{Q}_{\tau}+\mathbf{I})^{-1}\mathbf{Q}_{\tau}(\gamma ^{-1}\mathbf{Q}_{\tau}+\mathbf{I})^{-1}$.
And the two functions are minimized by $\btheta_{0}^{\mathrm{ma}}$ and $\btheta_{0}^{\mathrm{im}}$ respectively. Let $r^{\mathrm{ma}} = \min_{\alpha  } \mathcal{R}^{\mathrm{ma}}({\btheta_0^{\mathrm{ma}}}, \alpha  ) - c > 0$, 
and when 
$0<\gamma  < \big((r^{\mathrm{ma}})^{-\frac{1}{2}}  C_{\btheta}\bar{\lambda}^{\frac{1}{2}} - 1 \big)^{-1}\ushort{\lambda} $, then $\mathcal{R}^{\mathrm{im}}(\btheta_0^{\mathrm{im}},\gamma ) < \mathcal{R}^{\mathrm{ma}}(\btheta_0^{\mathrm{ma}},\alpha  )$.
\end{theorem}




For \textbf{BaMAML}, the optimal population risk is
\begin{align}\label{eq:bamaml_model_error}
  \mathcal{R}^{\mathrm{ba}}(\btheta_0^{\mathrm{ba}},\gamma )
  = \mathbb{E}_{\tau}\big[\|\btheta_0^{\mathrm{ba}}-\btheta^{\mathrm{gt}}_{\tau}\|^2_{\mathbf{W}^{\mathrm{ba}}_{\tau}(\gamma)}\big] + 1
\end{align}


Similar to the proof for Theorem~\ref{thm:lower_risk_bi_ma}, 
BaMAML also has lower optimal population risk than MAML, as stated in the following theorem. 
\begin{theorem}[BaMAML has lower optimal population risk than MAML]
\label{thm:same_risk_ba_bi}
Under Assumptions~1-2,
for meta-test task $\tau$ and arbitrary ${\btheta}$, 
the optimal population risk for BaMAML, as functions of ${\btheta}$, is
\begin{align*}
  \mathcal{R}^{\mathrm{ba}}({\btheta},\gamma) 
  = \mathbb{E}_{\tau}\big[\|{\btheta}-{\btheta}_{\tau}\|^{2}_{\mathbf{W}_{\tau}^{\mathrm{ba}}(\gamma)}\big] + 1.
\end{align*}
And when 
$0<\gamma  < \big((r^{\mathrm{ma}})^{-\frac{1}{2}}  C_{\btheta}\bar{\lambda}^{\frac{1}{2}} - 1 \big)^{-1}\ushort{\lambda} $, then $\mathcal{R}^{\mathrm{ba}}(\btheta_0^{\mathrm{ba}},\gamma ) < \mathcal{R}^{\mathrm{ma}}(\btheta_0^{\mathrm{ma}},\alpha )$.
\end{theorem}





\subsection{Statistical error} 
\label{app_sub:statistical_error}

To analyze the statistical error of different meta learning methods, we begin with a looser bound under data agnostic case with Assumptions 1 and 2 only.
And to give a sharper analysis of the statistical error in order to make a fair comparison among different methods,
we further make Assumption~3 on the task and data distributions.
In the following sections,
we will first present the supporting lemmas and then the main results for different methods.

\subsubsection{Supporting Lemmas} 
\label{app_ssub:supporting_lemmas}

In this section, we present some supporting lemmas for the proof of the main results for statistical errors of different methods.

\begin{lemma}
\label{lemma:concentration_W_N}  
Suppose Assumptions 1-2 hold.
Define
  $\mathbf{W}_{\tau,N}^{\cal A} \coloneqq \mathbb{E}_{\mathbf{x}_{\tau}}[\hat{\mathbf{W}}_{\tau,N}^{\cal A}]$,
  then
  $$\|\mathbf{W}_{\tau,N}^{\cal A}\|\leq  \|\mathbf{W}_{\tau}^{\cal A}\| 
  + L^{\cal A}\Big(\widetilde{\mathcal{O}}(\frac{d}{N}) 
  + \widetilde{\mathcal{O}}(\sqrt{\frac{d}{N}}) \Big).$$

\end{lemma}

\begin{proof}
For ERM, $\mathbf{W}_{\tau,N}^{\rm er} \coloneqq \mathbb{E}_{\mathbf{x}_{\tau}}[\hat{\mathbf{W}}_{\tau,N}^{\rm er}] 
= \mathbb{E}[\hat{\mathbf{W}}_{\tau,N}^{\rm er}] 
= \mathbb{E}[\hat{\mathbf{Q}}_{\tau, N}] 
= {\mathbf{Q}}_{\tau}
= {\mathbf{W}}_{\tau}^{\rm er} $, $L^{\rm er} = 0$.

For MAML, from \eqref{eq:W_tau_N_ma}, we have 
\begin{align*}
  \mathbf{W}_{\tau,N}^{\mathrm{ma}} 
  &=
  \mathbb{E}_{\hat{\mathbf{Q}}_{\tau,N}}\big[(\mathbf{I} - {\alpha} \hat{\mathbf{Q}}_{\tau,N}) \mathbf{Q}_{\tau}(\mathbf{I} - {\alpha} \hat{\mathbf{Q}}_{\tau,N})\big]  \\
  &= \mathbf{W}_{\tau}^{\mathrm{ma}} 
  +\frac{\alpha^{2}}{N}\Big(\mathbb{E}_{\mathbf{x}_{\tau, i}}\big[\mathbf{x}_{\tau, i} \mathbf{x}_{\tau, i}^{\top} \mathbf{Q}_{\tau} \mathbf{x}_{\tau, i} \mathbf{x}_{\tau, i}^{\top}\big]-\mathbf{Q}_{\tau}^{3}\Big)
\end{align*}
therefore 
\begin{align*}
  \|\mathbf{W}_{\tau,N}^{\rm ma}\|
  &\leq  \|\mathbf{W}_{\tau}^{\rm ma}\| 
  + \frac{\alpha^{2}}{N}\Big\|\mathbb{E}_{\mathbf{x}_{\tau, i}}\big[\mathbf{x}_{\tau, i} \mathbf{x}_{\tau, i}^{\top} \mathbf{Q}_{\tau} \mathbf{x}_{\tau, i} \mathbf{x}_{\tau, i}^{\top}\big]-\mathbf{Q}_{\tau}^{3}\Big\| \\
  &
  = \|\mathbf{W}_{\tau}^{\rm ma}\| 
  + {L^{\rm ma}} \Big(\widetilde{\mathcal{O}}(\frac{d}{N}) 
  + \widetilde{\mathcal{O}}(\sqrt{\frac{d}{N}}) \Big).
\end{align*}

For iMAML, recall
$\mathbf{W}_{\tau,N}^{\mathrm{im}} = \gamma^2{\Sigma}_{\btheta_{\tau}}  \mathbf{Q}_{\tau} {\Sigma}_{\btheta_{\tau}} $,  let $I_0 = 
\gamma {\Sigma}_{\theta_{\tau}}-(\mathbf{I}+\gamma ^{-1}\mathbf{Q}_{\tau})^{-1}$, further derived as
\begin{align*}
  I_0 = &
    \gamma {\Sigma}_{\theta_{\tau}}-(\mathbf{I}+  \gamma ^{-1}\mathbf{Q}_{\tau})^{-1} 
  =
  (\mathbf{I}+  \gamma ^{-1}\hat{\mathbf{Q}}_{\tau})^{-1}
  -(\mathbf{I}+  \gamma ^{-1}\mathbf{Q}_{\tau})^{-1} \\
  =&
  \gamma ^{-1}(\mathbf{I}+  \gamma ^{-1}\mathbf{Q}_{\tau})^{-1}
  (\mathbf{Q}_{\tau} - \hat{\mathbf{Q}}_{\tau})
  (\mathbf{I}+  \gamma ^{-1}\hat{\mathbf{Q}}_{\tau})^{-1}
\end{align*}
Then we have
\begin{align*}
  \mathbf{W}_{\tau,N}^{\mathrm{im}} =
  &\mathbb{E}_{\mathbf{x}_{\tau}}[\hat{\mathbf{W}}_{\tau,N}^{\mathrm{im}}]
  = \mathbb{E}_{\mathbf{x}_{\tau}}[ \gamma {\Sigma}_{\theta_{\tau}} \mathbf{Q}_{\tau} \gamma {\Sigma}_{\theta_{\tau}}] \\
  =& \mathbb{E}_{\mathbf{x}_{\tau}}\Big[ \big(\gamma {\Sigma}_{\theta_{\tau}}+(\mathbf{I}+\gamma ^{-1}\mathbf{Q}_{\tau})^{-1}-(\mathbf{I}+\gamma ^{-1}\mathbf{Q}_{\tau})^{-1}\big) \mathbf{Q}_{\tau} \big(\gamma {\Sigma}_{\theta_{\tau}}+(\mathbf{I}+\gamma ^{-1}\mathbf{Q}_{\tau})^{-1}-(\mathbf{I}+\gamma ^{-1}\mathbf{Q}_{\tau})^{-1}\big)\Big] \\
=& \mathbb{E}_{\mathbf{x}_{\tau}}\Big[ 
(\mathbf{I}+\gamma ^{-1}\mathbf{Q}_{\tau})^{-1}\mathbf{Q}_{\tau}
(\mathbf{I}+\gamma ^{-1}\mathbf{Q}_{\tau})^{-1} \Big] 
+\mathbb{E}_{\mathbf{x}_{\tau}}\Big[ 
  I_0 \mathbf{Q}_{\tau} I_0 \Big] 
+ \mathbb{E}_{\mathbf{x}_{\tau}}\Big[ 
  I_0 \mathbf{Q}_{\tau} (\mathbf{I}+\gamma ^{-1}\mathbf{Q}_{\tau})^{-1} 
  +  (\mathbf{I}+\gamma ^{-1}\mathbf{Q}_{\tau})^{-1}\mathbf{Q}_{\tau} 
  I_0 \Big] \\
=& \mathbf{W}_{\tau}^{\mathrm{im}} 
+\mathbb{E}_{\mathbf{x}_{\tau}}\Big[ 
  I_0 \mathbf{Q}_{\tau} I_0 \Big] 
+ \mathbb{E}_{\mathbf{x}_{\tau}}\Big[
  I_0  (\mathbf{I}+\gamma ^{-1}\mathbf{Q}_{\tau})\mathbf{W}_{\tau,N}^{\mathrm{im}}
  +  \mathbf{W}_{\tau,N}^{\mathrm{im}}(\mathbf{I}+\gamma ^{-1}\mathbf{Q}_{\tau}) 
  I_0\Big] \\
=& \mathbf{W}_{\tau}^{\mathrm{im}} 
  +\mathbb{E}_{\mathbf{x}_{\tau}}\Big[
    {\Sigma}_{\theta_{\tau}}
    (\mathbf{Q}_{\tau} - \hat{\mathbf{Q}}_{\tau})
    \mathbf{W}_{\tau}^{\mathrm{im}}
    (\mathbf{Q}_{\tau} - \hat{\mathbf{Q}}_{\tau})
    {\Sigma}_{\theta_{\tau}}  
  + 
    {\Sigma}_{\theta_{\tau}}
    (\hat{\mathbf{Q}}_{\tau} - \mathbf{Q}_{\tau} )
    \mathbf{W}_{\tau}^{\mathrm{im}}  
    +
    \mathbf{W}_{\tau}^{\mathrm{im}} 
    (\mathbf{Q}_{\tau} - \hat{\mathbf{Q}}_{\tau})
    {\Sigma}_{\theta_{\tau}}
    \Big] 
\end{align*}
where because
\begin{align*}
  \|\mathbb{E}_{\mathbf{x}_{\tau}}[{\Sigma}_{\btheta_{\tau}}(\mathbf{Q}_{\tau}-\hat{\mathbf{Q}}_{\tau}) \mathbf{W}_{\tau}^{\mathrm{im}}]\| 
  \leq 
  \mathbb{E}_{\mathbf{x}_{\tau}}[\|{\Sigma}_{\btheta_{\tau}}\|\|\mathbf{Q}_{\tau}-\hat{\mathbf{Q}}_{\tau}\|]\|\mathbf{W}_{\tau}^{\mathrm{im}}\| 
  \leq 
  \mathbb{E}_{\mathbf{x}_{\tau}}[\|\mathbf{Q}_{\tau}-\hat{\mathbf{Q}}_{\tau}\|]\|\mathbf{W}_{\tau}^{\mathrm{im}}\| 
\end{align*}
and $\|\mathbb{E}_{\mathbf{x}_{\tau}}[{\Sigma}_{\btheta_{\tau}}(\mathbf{Q}_{\tau}-\hat{\mathbf{Q}}_{\tau}) \mathbf{W}_{\tau}^{\mathrm{im}}(\mathbf{Q}_{\tau}-\hat{\mathbf{Q}}_{\tau}) {\Sigma}_{\btheta_{\tau}}]\| 
\leq
\mathbb{E}_{\mathbf{x}_{\tau}}[\|\mathbf{Q}_{\tau}-\hat{\mathbf{Q}}_{\tau}\|^{2}]\|\mathbf{W}_{\tau}^{\mathrm{im}}\|$. 
Based on sub-gaussian concentration inequality,
it holds with probability at least $1-\delta$ that
\begin{align*}
  \Big\|\mathbf{Q}_{\tau}-\hat{\mathbf{Q}}_{\tau, N}\Big\| \leq \bar{\lambda} C K^{2}\Big(\sqrt{\frac{d+\log \frac{2}{\delta}}{N}}+\frac{d+\log \frac{2}{\delta}}{N}\Big).
\end{align*}
Therefore choose $\delta = N^{-1}$ and since $\mathbf{x}_{\tau}$ is bounded, we have
\begin{align*}
  \|\mathbf{W}_{\tau,N}^{\rm im}\|
  &\leq  \|\mathbf{W}_{\tau}^{\rm im}\| 
  + \Big(\widetilde{\mathcal{O}}(\sqrt{\frac{d}{N}}) +
  \widetilde{\mathcal{O}}(\frac{d}{N}) 
  \Big) L^{\rm im} .
\end{align*}

Similarly, for BaMAML,
\begin{align*}
  &\mathbf{W}_{\tau,N}^{\mathrm{ba}} 
= \mathbb{E}_{\mathbf{x}_{\tau}}
\big[ \frac{\gamma_b}{N_1} \big((\mathbf{I}_{d}+\gamma_b^{-1}N_1 \hat{\mathbf{Q}}_{\tau,N_1})^{-1} 
- (\mathbf{I}_{d}+\gamma_b^{-1}(N+N_1) \hat{\mathbf{Q}}_{\tau,N+N_1})^{-1}\big) \big]\\
=& \mathbb{E}_{\mathbf{x}_{\tau}}
\big[ \frac{\gamma_b}{N_1} \big((\mathbf{I}_{d}+\gamma_b^{-1}N_1 {\mathbf{Q}}_{\tau})^{-1} 
- (\mathbf{I}_{d}+\gamma_b^{-1}(N+N_1) {\mathbf{Q}}_{\tau})^{-1} \\
&+(\mathbf{I}_{d}+\gamma_b^{-1}N_1 \hat{\mathbf{Q}}_{\tau,N_1})^{-1} 
- (\mathbf{I}_{d}+\gamma_b^{-1}N \hat{\mathbf{Q}}_{\tau,N})^{-1} 
- (\mathbf{I}_{d}+\gamma_b^{-1}N_1 {\mathbf{Q}}_{\tau})^{-1} 
+ (\mathbf{I}_{d}+\gamma_b^{-1}N {\mathbf{Q}}_{\tau})^{-1} \big) \big]\\
=& \mathbf{W}_{\tau}^{\rm ba}
+ \frac{\gamma_b}{N_1} \mathbb{E}_{\mathbf{x}_{\tau}}
\big[ (\mathbf{I}_{d}+\gamma_b^{-1}N_1 \hat{\mathbf{Q}}_{\tau,N_1})^{-1} 
- (\mathbf{I}_{d}+\gamma_b^{-1}N \hat{\mathbf{Q}}_{\tau,N})^{-1} 
- (\mathbf{I}_{d}+\gamma_b^{-1}N_1 {\mathbf{Q}}_{\tau})^{-1} 
+ (\mathbf{I}_{d}+\gamma_b^{-1}N {\mathbf{Q}}_{\tau})^{-1} \big] \\
=& \mathbf{W}_{\tau}^{\rm ba}
+ \frac{\gamma_b}{N_1} \mathbb{E}_{\mathbf{x}_{\tau}}
\big[\gamma (\mathbf{I}_{d}+\gamma_b^{-1}N_1 {\mathbf{Q}}_{\tau})^{-1}(\mathbf{Q}_{\tau} - \hat{\mathbf{Q}}_{\tau, N_1}) (\mathbf{I}_{d}+\gamma_b^{-1}N_1 \hat{\mathbf{Q}}_{\tau,N_1})^{-1} \\
&\hspace{2cm} - \gamma_b^{-1}N(\mathbf{I}_{d}+\gamma_b^{-1}N {\mathbf{Q}}_{\tau})^{-1} 
(\mathbf{Q}_{\tau} - \hat{\mathbf{Q}}_{\tau, N})
(\mathbf{I}_{d}+\gamma_b^{-1}N \hat{\mathbf{Q}}_{\tau,N})^{-1} \big] \\
=& \mathbf{W}_{\tau}^{\rm ba}
+ {\gamma} \mathbb{E}_{\mathbf{x}_{\tau}}
\big[\gamma (\mathbf{I}_{d}+\gamma^{-1} {\mathbf{Q}}_{\tau})^{-1}(\mathbf{Q}_{\tau} - \hat{\mathbf{Q}}_{\tau, N_1}) (\mathbf{I}_{d}+\gamma^{-1} \hat{\mathbf{Q}}_{\tau,N_1})^{-1} \\
&\hspace{2cm} - (\gamma s)^{-1}(\mathbf{I}_{d}+(\gamma s)^{-1} {\mathbf{Q}}_{\tau})^{-1} 
(\mathbf{Q}_{\tau} - \hat{\mathbf{Q}}_{\tau, N})
(\mathbf{I}_{d}+(\gamma s)^{-1} \hat{\mathbf{Q}}_{\tau,N})^{-1} \big] 
\end{align*}
therefore
\begin{align*}
  \|\mathbf{W}_{\tau,N}^{\rm ba}\|
  &\leq  \|\mathbf{W}_{\tau}^{\rm ba}\| 
  + \Big(\widetilde{\mathcal{O}}(\sqrt{\frac{d}{N}}) +
  \widetilde{\mathcal{O}}(\frac{d}{N}) 
  \Big) L^{\rm ba} .
\end{align*}

\end{proof}

\begin{lemma}[Concentration of $\hat{\mathbf{W}}_{\tau,N}^{\mathcal{A}}$]
\label{lemma:concentration_W_hat}
Denote $d$ as the dimension of $\btheta_{\tau}$, 
$T$ as the number of tasks.
Suppose Assumption~1 holds,
and $\mathbf{x}_{\tau,i}$ is sub-gaussian with parameter $k$,
then with probability at least $1 - Td^{-10}$, 
for $\tau = 1, \dots, T$,
we have the following bounds, given by
\begin{align}\label{eq:bound_A_hat_er}
  \mathbf{0} \preceq \hat{\mathbf{W}}_{\tau,N}^{\mathrm{er}} \preceq \widetilde{\mathcal{O}}(c^{\mathrm{er}}) \mathbf{I}_d
\end{align}
where $c^{\mathrm{er}} \coloneqq 1 + \max\{d/N, \sqrt{d/N}\}$,
and $\widetilde{\mathcal{O}}(\cdot)$ hides the logarithmic factor $\log (NdT)$.\\
And denote $\| \cdot \|_{\mathrm{op}}$ as the operator norm. With probability at least $1 - Td^{-10}$, it holds that
\begin{subequations}
\begin{align}
  \label{eq:bound_A_hat_erm_moment1_norm}
  &\Big\|\frac{1}{T} \sum_{\tau=1}^{T} \hat{\mathbf{W}}_{\tau,N}^{\mathrm{er}}-\mathbb{E}\big[\hat{\mathbf{W}}_{\tau,N}^{\mathrm{er}}\big]\Big\|_{\mathrm{op}} \leq \widetilde{\mathcal{O}}\Big(c^{\mathrm{er}} \sqrt{\frac{d}{T}}+d^{-4}\Big), \\
  \label{eq:bound_A_hat_erm_moment2_norm}
  &\Big\|\frac{1}{T} \sum_{\tau=1}^{T} (\hat{\mathbf{W}}_{\tau,N}^{\mathrm{er}})^2-\mathbb{E}\big[(\hat{\mathbf{W}}_{\tau,N}^{\mathrm{er}})^2\big]\Big\|_{\mathrm{op}} \leq \widetilde{\mathcal{O}}\Big((c^{\mathrm{er}})^2 \sqrt{\frac{d}{T}}+d^{-4}\Big).
\end{align}
\end{subequations}

\end{lemma}

\begin{proof}
\label{proof:concentration_W_hat}
The proof is similar to Lemma C.4 in~\cite{bai2021_trntrn_trnval},
the difference is we do not need Assumption~3, $\mathbf{x}_{\tau, i} \sim \mathcal{N}\left(\mathbf{0}, \mathbf{I}_{d}\right)$, but only requires
$\mathbf{x}_{\tau,i}$ to be sub-gaussian with parameter $K$.
Recall that $\hat{\mathbf{W}}_{\tau,N}^{\mathrm{er}} = \frac{1}{N}\mathbf{X}_{\tau,N}^{\mathrm{all}\top}\mathbf{X}_{\tau,N}^{\mathrm{all}}=\hat{\mathbf{Q}}_{\tau,N} $.
Applying the sub-gaussian covariance concentration (~\cite{vershynin2018high}, Exercise 4.7.3), we have with probability at least $1-d^{-10}$ that
\begin{align*}\label{eq:W_hat_bound}
\hat{\mathbf{W}}_{\tau,N}^{\mathrm{er}} 
= \hat{\mathbf{Q}}_{\tau,N}
&\preceq 
{\mathbf{Q}}_{\tau}
+\Big\|\hat{\mathbf{Q}}_{\tau,N}-{\mathbf{Q}}_{\tau}\Big\|_{\mathrm{op}} \mathbf{I}_{d} \\
&\preceq
\Big(\bar{\lambda}+CK^2 \sqrt{\frac{d+\log d}{N}}+CK^2 \frac{d+\log d}{N}\Big) \mathbf{I}_{d}
\preceq  
K_{\tau} c^{\mathrm{er}} \mathbf{I}_{d}
\numberthis
\end{align*}
where $K_{\tau}=\mathcal{O}(1)$ is an absolute constant dependent on $\bar{\lambda}, C, K$. 
Let $\mathcal{W}_{\tau} \coloneqq \big\{\hat{\mathbf{W}}_{\tau,N}^{\mathrm{er}}  \preceq K_{\tau} c^{\mathrm{er}} \mathbf{I}_{d}\big\}$ denote this event. 
We have $\mathbb{P}(\mathcal{W}_{\tau}) \geq 1-Td^{-10}$. Let $\mathcal{W} \coloneqq \bigcup_{t=1}^{T} \mathcal{W}_{\tau} $ denote the union event. Note that on the event $\mathcal{W}$ we have
\begin{align*}
  \frac{1}{T} \sum_{\tau=1}^{T} \hat{\mathbf{W}}_{\tau,N}^{\mathrm{er}} 
  =\frac{1}{T} \sum_{\tau=1}^{T} \hat{\mathbf{W}}_{\tau,N}^{\mathrm{er}}  \mathbf{1}\{\mathcal{W}_{\tau}\}.
\end{align*}
And on the event $\mathcal{W}_{\tau}, \hat{\mathbf{W}}_{\tau,N}^{\mathrm{er}}$ is bounded by:
$
\mathbf{0} \preceq \hat{\mathbf{W}}_{\tau,N}^{\mathrm{er}} \mathbf{1}\{\mathcal{W}_{\tau}\} \preceq K_{\tau}c^{\mathrm{er}} \mathbf{I}_{d}
$,
which means that for any $\mathbf{v} \in \mathbb{R}^{d}$ and $\|\mathbf{v}\|_{2}=1$, the random variable
$
\mathbf{v}^{\top} \hat{\mathbf{W}}_{\tau,N}^{\mathrm{er}} \mathbf{1}\{\mathcal{W}_{\tau}\} \mathbf{v}-\mathbf{v}^{\top} \mathbb{E}\big[\hat{\mathbf{W}}_{\tau,N}^{\mathrm{er}} \mathbf{1}\{\mathcal{W}_{\tau}\}\big] \mathbf{v}
$
is mean-zero and sub-gaussian with parameter $K_{\tau}c^{\mathrm{er}}$. Therefore by the standard sub-gaussian concentration, we have
\begin{align*}
  \mathbb{P} \Big(\Big|
\mathbf{v}^{\top}\big(\frac{1}{T} \sum_{\tau=1}^{T} \hat{\mathbf{W}}_{\tau,N}^{\mathrm{er}} \mathbf{1}\big\{\mathcal{W}_{\tau}\big\}\big) \mathbf{v}-\mathbf{v}^{\top} \mathbb{E}\big[\hat{\mathbf{W}}_{\tau,N}^{\mathrm{er}} \mathbf{1}\big\{\mathcal{W}_{\tau}\big\}\big] \mathbf{v}
\Big| \geq t\Big) \leq 2 \exp \Big(- \frac{T t^{2}} {2 (K_{\tau}c^{\mathrm{er}})^{2}}\Big).
\end{align*}
Using the fact that for any symmetric matrix $\mathbf{M}$, 
$
\|\mathbf{M}\|_{\mathrm{op}} \leq 2 \sup _{\mathbf{v} \in N_{1 / 4}(\mathbb{S}^{d-1})}|\mathbf{v}^{\top} \mathbf{M} \mathbf{v}|
$
where $N_{1 / 4}(\mathbb{S}^{d-1})$ is a $1 / 4$-covering set of the $(d-1)$-unit sphere $\mathbb{S}^{d-1}$ with $|N_{1 / 4}(\mathbb{S}^{d-1})| \leq 9^{d}$ (~\cite{vershynin2018high} , Exercise 4.4.3), 
we have
\begin{align*}
& \mathbb{P}\Big(\Big\|\frac{1}{T} \sum_{\tau=1}^{T} 
\hat{\mathbf{W}}_{\tau,N}^{\mathrm{er}} \mathbf{1}\Big\{\mathcal{W}_{\tau}\Big\}-\mathbb{E}\big[\hat{\mathbf{W}}_{\tau,N}^{\mathrm{er}} \mathbf{1}\Big\{\mathcal{W}_{\tau}\Big\}\big]\Big\|_{\mathrm{op}} \geq t\Big) \\
\leq &\left|N_{1 / 4}\Big(\mathbb{S}^{d-1}\Big)\right| \cdot \sup _{\|\mathbf{v}\|_{2}=1} \mathbb{P}\Big(\Big|\mathbf{v}^{\top}\Big(\frac{1}{T} \sum_{\tau=1}^{T} \hat{\mathbf{W}}_{\tau,N}^{\mathrm{er}} \mathbf{1}\Big\{\mathcal{W}_{\tau}\Big\}\Big) \mathbf{v}-\mathbf{v}^{\top} \mathbb{E}\big[\hat{\mathbf{W}}_{\tau,N}^{\mathrm{er}} \mathbf{1}\Big\{\mathcal{W}_{\tau}\Big\}\big] \mathbf{v}\Big| \geq t\Big) \\
\leq & \exp \Big(-T t^{2} / 2(K_{\tau}c^{\mathrm{er}})^{2}+3 d\Big).
\end{align*}
Taking $t= \mathcal{O} \big(K_{\tau}c^{\mathrm{er}} \sqrt{\frac{6d+20 \log ( d)}{T}}\big)=\widetilde{\mathcal{O}}\big(K_{\tau}c^{\mathrm{er}} \sqrt{\frac{d}{T}}\big)$, the above probability is upper bounded by $d^{-10} $. In other words, with probability at least $1-Td^{-10} $, we have
\begin{align}
\label{eq:bound_avg_W_hat_concentration_erm}
\Big\|\frac{1}{T} \sum_{\tau=1}^{T} \hat{\mathbf{W}}_{\tau,N}^{\mathrm{er}} \mathbf{1}\left\{\mathcal{W}_{\tau}\right\}-\mathbb{E}\big[\hat{\mathbf{W}}_{\tau,N}^{\mathrm{er}} \mathbf{1}\left\{\mathcal{W}_{\tau}\right\}\big]\Big\|_{\mathrm{op}} \leq \widetilde{\mathcal{O}}\Big(K_{\tau}c^{\mathrm{er}} \sqrt{\frac{d}{T}}\Big).
\end{align}
To bound the difference between 
$\mathbb{E}\big[ \hat{\mathbf{W}}_{\tau,N}^{\mathrm{er}}\big]$ and $\mathbb{E}\big[ \hat{\mathbf{W}}_{\tau,N}^{\mathrm{er}} \mathbf{1}\{\mathcal{W}_{\tau}\}\big]$, it follows 
\begin{align}
\nonumber
&\Big\|\mathbb{E}\big[ \hat{\mathbf{W}}_{\tau,N}^{\mathrm{er}}\big]-\mathbb{E}\big[ \hat{\mathbf{W}}_{\tau,N}^{\mathrm{er}} \mathbf{1}\left\{\mathcal{W}_{\tau}\right\}\big]\Big\|_{\mathrm{op}} \leq \mathbb{E}\Big[\big\| \hat{\mathbf{W}}_{\tau,N}^{\mathrm{er}}\big\|_{\mathrm{op}} 1\left\{\mathcal{W}_{\tau}^{c}\right\}\Big] 
\leq\left(\mathbb{E}\big[\big\| \hat{\mathbf{W}}_{\tau,N}^{\mathrm{er}}\big\|_{\mathrm{op}}^{2}\big] \cdot \mathbb{P}\left(\mathcal{W}_{\tau}^{c}\right)\right)^{\frac{1}{2}} \\
\label{eq:bound_diff_E_W_hat_E_W_hat_event_erm}
&\leq \sqrt{\mathbb{E}\big[\max _{i}\|\mathbf{x}_{\tau, i}\|_{2}^{2}\big] \cdot d^{-10}} \leq 
\sqrt{k^2\left(d+C \log N \right) \cdot d^{-10}}=\widetilde{\mathcal{O}}\left(d^{-4.5}\right)
\end{align}
where $\mathcal{W}_{\tau}^{c}$ is the complement of $\mathcal{W}_{\tau}$, and the last inequality is by sub-gaussian norm concentration. 
%
Combining \eqref{eq:bound_avg_W_hat_concentration_erm} and \eqref{eq:bound_diff_E_W_hat_E_W_hat_event_erm}, with probability at least $1- Td^{-10}$, we have that
\begin{align*}
&\Big\|\frac{1}{T} \sum_{\tau =1}^T \hat{\mathbf{W}}_{\tau,N}^{\mathrm{er}}-\mathbb{E}\big[\hat{\mathbf{W}}_{\tau,N}^{\mathrm{er}}\big]\Big\|_{\text {op }} 
\leq \Big\|\frac{1}{T} \sum_{\tau =1}^T \hat{\mathbf{W}}_{\tau,N}^{\mathrm{er}} \mathbf{1}\left\{\mathcal{W}_{\tau}\right\}-\mathbb{E}\big[\hat{\mathbf{W}}_{\tau,N}^{\mathrm{er}} \mathbf{1}\left\{\mathcal{W}_{\tau}\right\}\big]\Big\|_{\mathrm{op}} \\
&+\Big\|\mathbb{E}\big[\hat{\mathbf{W}}_{\tau,N}^{\mathrm{er}}\big]-\mathbb{E}\big[\hat{\mathbf{W}}_{\tau,N}^{\mathrm{er}} \mathbf{1}\left\{\mathcal{W}_{\tau}\right\}\big]\Big\|_{\mathrm{op}}
\leq \widetilde{\mathcal{O}}\left(c^{\mathrm{er}} \sqrt{\frac{d}{T}}+d^{-4.5}\right).
\end{align*}
Similarly we can prove that with probability at least $1- Td^{-10}$ that
\begin{align*}
\Big\|\frac{1}{T} \sum_{\tau =1}^T (\hat{\mathbf{W}}_{\tau,N}^{\mathrm{er}})^2
-\mathbb{E}\big[(\hat{\mathbf{W}}_{\tau,N}^{\mathrm{er}})^2\big]\Big\|_{\text {op }} 
 \leq \widetilde{\mathcal{O}}\left((c^{\mathrm{er}})^2 \sqrt{\frac{d}{T}}+d^{-4}\right).
\end{align*}
This completes the proof of Lemma~\ref{lemma:concentration_W_hat}.
Note that, similar results apply to random weight matrices $\hat{\mathbf{W}}_{\tau,N}^{\mathcal{A}}$ of other methods 
by replacing $\hat{\mathbf{W}}_{\tau,N}^{\mathrm{er}} = \hat{\mathbf{Q}}_{\tau,N}$ with 
$\hat{\mathbf{W}}_{\tau,N}^{\mathcal{A}} \preceq \hat{\mathbf{Q}}_{\tau,N_2}$ in \eqref{eq:W_hat_bound}.
And Lemma~\ref{lemma:concentration_W_hat} still holds with Assumption~\ref{asmp:linear_centroid_model}.

\end{proof}


  

\begin{lemma}
[Hanson-Wright inequality]
\label{lemma:Hanson-Wright inequality}
(Restatement of Theorem 6.2.1 in~\citep{vershynin2018high}). Let $\mathbf{z} \in \mathbb{R}^{d}$ be a random vector with independent, mean-zero, and $K$-sub-gaussian entries, and let $\mathbf{C} \in \mathbb{R}^{d \times d}$ be a fixed matrix. Then it holds with probability at least $1-\delta$ that
\begin{equation*}
\Big|\mathbf{z}^{\top} \mathbf{C} \mathbf{z}-\mathbb{E}[\mathbf{z}^{\top} \mathbf{C} \mathbf{z}]\Big| 
\leq \mathcal{O}\Big(K^{2}\|\mathbf{C}\|_{\mathrm{F}} \log \frac{2}{\delta} \Big).    
\end{equation*}
\end{lemma}

\begin{lemma}[Linear combination of sub-gaussian]
\label{lemma:combinate_sub_gaussian}
\citep{vershynin2018high}
 Let $\mathbf{z} \in \mathbb{R}^{d}$ be a random vector with independent and $K$-sub-gaussian entries. Then for any $\mathbf{v} \in \mathbb{S}^{d-1}(r), \mathbf{v}^{\top} \mathbf{z}$ is $rK$-sub-gaussian. In other words, it holds with probability at least $1-\delta$ that
 \begin{align*}
   \Big|\mathbf{v}^{\top} \mathbf{z} - \mathbb{E} [\mathbf{v}^{\top} \mathbf{z}] \Big|&
   \leq \mathcal{O}\Big(rK \sqrt{\log \frac{2}{\delta}}\Big).
 \end{align*}
\end{lemma}

\begin{lemma}
[Hanson-Wright inequality with non-zero mean]
\label{lemma:Hanson-Wright inequality non zero mean}
Let $\mathbf{z} \in \mathbb{R}^{d}$ be a random vector with independent, and $K$-sub-gaussian entries, and let $\mathbf{C} \in \mathbb{R}^{d \times d}$ be a fixed matrix. Then it holds with probability at least $1-\delta$ that
\begin{equation*}
\left|\mathbf{z}^{\top} \mathbf{C} \mathbf{z}-\mathbb{E}\big[\mathbf{z}^{\top} \mathbf{C} \mathbf{z}\big]\right| \leq \mathcal{O}\Big(K^{2}\|\mathbf{C}\|_{\mathrm{F}} \log (2 / \delta)\Big) + \mathcal{O}\Big(K\|\mathbb{E}[\mathbf{z}]\|\|\mathbf{C}\|_{\mathrm{op}} \sqrt{\log (2 / \delta)}\Big).    
\end{equation*}
\end{lemma}

\begin{proof}
\label{proof:Hanson-Wright inequality non zero mean}
  \begin{align*}
    & \left|\mathbf{z}^{\top} \mathbf{C} \mathbf{z}-\mathbb{E}\big[\mathbf{z}^{\top} \mathbf{C} \mathbf{z}\big]\right|
    = 
    \Big|(\mathbf{z} - \mathbb{E}[\mathbf{z}])^{\top} \mathbf{C} (\mathbf{z} - \mathbb{E}[\mathbf{z}])-\mathbb{E}\big[(\mathbf{z} - \mathbb{E}[\mathbf{z}])^{\top} \mathbf{C} (\mathbf{z} - \mathbb{E}[\mathbf{z}])\big] 
    + 
    2\mathbb{E}[\mathbf{z}]^{\top} \mathbf{C} (\mathbf{z} - \mathbb{E}[\mathbf{z}])
    \Big| \\
    & \leq
    \Big|(\mathbf{z} - \mathbb{E}[\mathbf{z}])^{\top} \mathbf{C} (\mathbf{z} - \mathbb{E}[\mathbf{z}])-\mathbb{E}\big[(\mathbf{z} - \mathbb{E}[\mathbf{z}])^{\top} \mathbf{C} (\mathbf{z} - \mathbb{E}[\mathbf{z}])\big]  \Big|
    + 
    2 \Big|\mathbb{E}[\mathbf{z}]^{\top} \mathbf{C} (\mathbf{z} - \mathbb{E}[\mathbf{z}])
    \Big| \\
    &\leq 
    \mathcal{O}\Big(K^{2}\|\mathbf{C}\|_{\mathrm{F}} \log (2 / \delta)\Big) + \mathcal{O}\Big(K\|\mathbb{E}[\mathbf{z}]\|\|\mathbf{C}\|_{\mathrm{op}} \sqrt{\log (2 / \delta)}\Big)
  \end{align*}
  where the last inequality follows from Lemma~\ref{lemma:Hanson-Wright inequality} and~\ref{lemma:combinate_sub_gaussian}. Note that when $\mathbb{E}[\mathbf{z}] = \mathbf{0}$, this Lemma reduces to the zero-mean version of Hanson-Wright inequality, i.e. Lemma~\ref{lemma:Hanson-Wright inequality}.
\end{proof}



\begin{lemma}[sub-gaussian random vector concentration]
\label{lemma:subgaussian_vector_concentrate} 
  Let $\mathbf{U}_{\tau} \in \mathbb{R}^{d\times d},
  \mathbf{z}_{\tau} \in \mathbb{R}^{d} $.
  Assume $\big\|\mathbf{U}_{\tau}\big\| \leq \bar{\lambda}$ and $\mathbf{z}_{\tau}$ has independent, mean-zero, $K$-sub-gaussian entries. With probability at least $1-\delta$, it holds that
\begin{align*}
  \Big| \big\|\frac{1}{T}\sum_{\tau=1}^{T} \mathbf{U}_{\tau} \mathbf{z}_{\tau} \big\| 
  - \big\|\mathbb{E}_{\tau}\big[\mathbf{U}_{\tau} \mathbf{z}_{\tau}\big] \big\| \Big| 
  \leq \widetilde{\mathcal{O}}\Big( K \bar{\lambda} \sqrt{\frac{d}{T}} \log \frac{2}{\delta} \Big).
\end{align*}

\end{lemma}

\begin{proof}
\label{proof:subgaussian_vector_concentrate}
  Apply Lemma~\ref{lemma:Hanson-Wright inequality}, the Hanson-Wright inequality, and let $\mathbf{z} = \mathbf{z}_{\tau}$, $ \mathbf{C} = \mathbf{U}_{\tau}^{\top}\mathbf{U}_{\tau} $, we obtain that with probability at least $1 - \delta$,
  \begin{equation*}
  \left|\mathbf{z}_{\tau}^{\top} \mathbf{U}_{\tau}^{\top}\mathbf{U}_{\tau} \mathbf{z}_{\tau}-\mathbb{E}_{\mathbf{z}_{\tau}\mid \mathbf{U}_{\tau}}\big[\mathbf{z}_{\tau}^{\top} \mathbf{U}_{\tau}^{\top}\mathbf{U}_{\tau} \mathbf{z}_{\tau}\big]\right| 
  \leq \mathcal{O}\Big(K^{2}\|\mathbf{U}_{\tau}^{\top}\mathbf{U}_{\tau}\|_{\mathrm{F}} \log \frac{2}{\delta} \Big).    
  \end{equation*}
Since 
\begin{align*}
  &\Big|\|\mathbf{U}_{\tau} \mathbf{z}_{\tau}\| - \|\mathbb{E}_{\mathbf{z}_{\tau}\mid \mathbf{U}_{\tau}}\big[\mathbf{U}_{\tau} \mathbf{z}_{\tau}\big]\| \Big|^2 
  \leq  \Big|\|\mathbf{U}_{\tau} \mathbf{z}_{\tau}\|^2 - \|\mathbb{E}_{\mathbf{z}_{\tau}\mid \mathbf{U}_{\tau}}\big[\mathbf{U}_{\tau} \mathbf{z}_{\tau}\big]\|^2 \Big| \\
  = & \left|\mathbf{z}_{\tau}^{\top} \mathbf{U}_{\tau}^{\top}\mathbf{U}_{\tau} \mathbf{z}_{\tau}-\mathbb{E}_{\mathbf{z}_{\tau}\mid \mathbf{U}_{\tau}}\big[\mathbf{z}_{\tau}^{\top} \mathbf{U}_{\tau}^{\top}\mathbf{U}_{\tau} \mathbf{z}_{\tau}\big]\right|
  \leq \mathcal{O}\Big(K^{2}d\|\mathbf{U}_{\tau}\|^2_{\mathrm{op}} \log \frac{2}{\delta} \Big)
\end{align*} 
where the last equation holds because $\mathbf{z}_{\tau}$ has mean-zero entries.
Therefore, it holds with probability at least $1 - \delta$ that
\begin{align*}\label{eq:bound_W_theta}
  \Big|\|\mathbf{U}_{\tau} \mathbf{z}_{\tau}\| - \|\mathbb{E}_{\mathbf{z}_{\tau}\mid \mathbf{U}_{\tau}}\big[\mathbf{U}_{\tau} \mathbf{z}_{\tau}\big]\| \Big|
  & \leq \mathcal{O}\Big(K \bar{\lambda} \sqrt{d} \log \frac{2}{\delta} \Big). 
  \numberthis
\end{align*}
Also, based on Lemma~\ref{lemma:combinate_sub_gaussian}, it holds with probability at least $1 - \delta$ that
\begin{align*}\label{eq:bound_E_theta_cond_W}
  \Big| \| \mathbb{E}_{\mathbf{z}_{\tau}, \mathbf{U}_{\tau}}\big[\mathbf{U}_{\tau} \mathbf{z}_{\tau}\big] \| - \| \mathbb{E}_{\mathbf{z}_{\tau}\mid \mathbf{U}_{\tau}}\big[\mathbf{U}_{\tau} \mathbf{z}_{\tau}\big] \| \Big| 
  &\leq \Big\| \mathbb{E}_{\mathbf{z}_{\tau}, \mathbf{U}_{\tau}}\big[\mathbf{U}_{\tau} \mathbf{z}_{\tau}\big] - \mathbb{E}_{\mathbf{z}_{\tau}\mid \mathbf{U}_{\tau}}\big[\mathbf{U}_{\tau} \mathbf{z}_{\tau}\big] \Big\| \\
  &\leq \widetilde{\mathcal{O}}\big( K \bar{\lambda} \sqrt{d}\log \frac{2}{\delta} \big) .
  \numberthis
\end{align*}
Combining~\eqref{eq:bound_W_theta} and~\eqref{eq:bound_E_theta_cond_W}, it holds with probability at least $1-\delta$ that
\begin{align*}
  \Big| \|\frac{1}{T}\sum_{\tau=1}^{T} \mathbf{U}_{\tau} \mathbf{z}_{\tau} \|  - \|\mathbb{E}_{\mathbf{z}_{\tau}, \mathbf{U}_{\tau}}\big[\mathbf{U}_{\tau} \mathbf{z}_{\tau}\big] \| \Big| 
  &\leq \widetilde{\mathcal{O}}\big(  K \bar{\lambda} \sqrt{\frac{d}{T}} \log \frac{2}{\delta} \big) .
  \numberthis
\end{align*}

\end{proof}

\begin{lemma}[Bound of statistical error not caused by data noise]
  \label{lemma:bound_stats_err_term1}
  Define 
  \begin{align*}
    &\mathbf{z}_{\mathcal{A}} 
    \coloneqq 
    \begin{bmatrix}
    (\btheta_1^{\mathrm{gt}} - \btheta_0^{\mathcal{A}})^{\top}, 
    \dots ,
    (\btheta_{T}^{\mathrm{gt}} - \btheta_0^{\mathcal{A}})^{\top}
    \end{bmatrix}^{\top} \in \mathbb{R}^{dT}, \\
    &\mathbf{U}_{\mathcal{A}} 
    \coloneqq  
    \Big[
    \hat{\mathbf{W}}^{\mathcal{A}}_{1,N} 
    (\sum_{\tau=1}^{T} \hat{\mathbf{W}}_{\tau, N}^{\mathcal{A}})^{-1} ,
    \dots, 
    \hat{\mathbf{W}}^{\mathcal{A}}_{T,N}
    (\sum_{\tau=1}^{T} \hat{\mathbf{W}}_{\tau, N}^{\mathcal{A}})^{-1} 
    \Big]^{\top}
    \in \mathbb{R}^{dT\times d}.
  \end{align*}
  1) Suppose Assumptions~1-2 hold, the statistical error for method $\mathcal{A}$ is computed by
  \begin{align*}
    \mathcal{E}_{\mathcal{A}}^{2} (\hat{\btheta}_0^{\mathcal{A}})
    = \|\hat{\btheta}_{0}^{{\mathcal{A}}}-\btheta_{0}^{ {\mathcal{A}}}\|_{\mathbb{E}_{\tau}[\mathbf{W}_{\tau}^{\mathcal{A}}]}^2
    = 
    \underbracket{
    \|\mathbf{U}_{\mathcal{A}}^{\top} \mathbf{z}_{\mathcal{A}}\|^2_{\mathbb{E}_{\tau}[\mathbf{W}_{\tau}^{\mathcal{A}}]}}_{I_1^{\mathcal{A}}}
    + \|\Delta_{T}^{\mathcal{A}}\|^2_
    {\mathbb{E}_{\tau}[\mathbf{W}_{\tau}^{\mathcal{A}}]}
    + 2{\mathbf{z}_{\mathcal{A}}^{\top}
    \mathbf{U}_{\mathcal{A}} }
    \mathbb{E}_{\tau}[\mathbf{W}_{\tau}^{\mathcal{A}}]
    \Delta_{T}^{\mathcal{A}}
  \end{align*}
  where with probability at least $1-Td^{-10}$,
  the first term $I_1^{\mathcal{A}}$ can be bounded above by
  \footnote{Note that, we provide bound for $I_1^{\mathcal{A}}$ and $I_2^{\mathcal{A}}$ in this lemma since it has the same form for different methods $\mathcal{A}$.
  And the bound for the rest terms in the statistical error are deferred to later sections for the specific methods.
  }
  \begin{align*}
    I_{1}^{\mathcal{A}} 
    &\leq 
    \frac{R^2 }{T} 
    \Big(\lambda_{\min}(\mathbb{E}[{\mathbf{W}}^{\mathcal{A}}_{\tau}])^{-1}
      \lambda_{\max}(\mathbb{E}[({\mathbf{W}}^{\mathcal{A}}_{\tau})^{2}])
      + \widetilde{\mathcal{O}}(\frac{d}{N})
      + \widetilde{\mathcal{O}}(\frac{1}{\sqrt{d}})  + \widetilde{\mathcal{O}} (\sqrt{\frac{d}{T}})\Big) \\
     & + 
    \Big( \widetilde{\mathcal{O}}(\sqrt{\frac{d}{T}})  + 
    \widetilde{\mathcal{O}}(\frac{d}{N})
    \Big) M^2
    .
  \end{align*}
  2) Suppose Assumptions~1,3 hold,
   the statistical error for method $\mathcal{A}$ can be computed by
  \begin{align*}
    \mathcal{E}_{\mathcal{A}}^{2} (\hat{\btheta}_0^{\mathcal{A}})
    =w_{\mathcal{A}}\|\hat{\btheta}_{0}^{{\mathcal{A}}}-\btheta_{0}^{ {\mathcal{A}}}\|_2^2
    =& w_{\mathcal{A}}
    (\underbracket{
      \|\mathbf{U}_{\mathcal{A}}^{\top} \mathbf{z}_{\mathcal{A}}\|^2}_{I_2^{\mathcal{A}}}
      + {\|\Delta_{T}^{\mathcal{A}}\|^2}
      + 2{\mathbf{z}_{\mathcal{A}}^{\top}
        \mathbf{U}_{\mathcal{A}} }
        \Delta_{T}^{\mathcal{A}} )
  \end{align*}
  Define $\tilde{C}^{\mathcal{A}}_{0} \coloneqq
  \frac{1}{d}\big\langle\mathbb{E}^{-2}\big[\hat{\mathbf{W}}_{\tau, N}^{\mathcal{A}}\big], \mathbb{E}\big[(\hat{\mathbf{W}}_{\tau, N}^{\mathcal{A}})^{2}\big]\big\rangle $. With probability at least $1-Td^{-10}$,
  $I_2^{\mathcal{A}}$ can be bounded above by
  \begin{equation*}
    I_2^{\mathcal{A}}
    \leq \frac{R^{2}}{T}\Big(\tilde{C}^{\mathcal{A}}_{0} + \widetilde{\mathcal{O}}(\frac{1}{\sqrt{d}}) + \widetilde{\mathcal{O}}(\sqrt{\frac{d}{T}})\Big). 
  \end{equation*}
  \end{lemma}
  
  \begin{proof}
  \label{proof:bound_stats_err_term1}
  The derivations in Section~\ref{app_sec:methods_formulation_sln} give  the empirical solutions $\hat{\btheta}_0^{\mathcal{A}}$ as below
  \begin{align}\label{eq:theta_0_hat_general_sln}
    \hat{\btheta}_{0}^{\mathcal{A}} 
    =& \Big(\sum_{\tau=1}^{T}\hat{\mathbf{W}}_{\tau, N}^{\mathcal{A}}\Big)^{-1}  
    \Big(\sum_{\tau=1}^{T}
     \hat{\mathbf{W}}_{\tau, N}^{\mathcal{A}} \btheta_{\tau}^{\mathrm{gt}} \Big) 
     + \Delta_{T}^{\mathcal{A}} .
  \end{align}
  Thus the difference between the estimated model parameter $\hat{\btheta}_0^{\mathcal{A}}$ and the population-wise optimal model parameter $\btheta_0^{\mathcal{A}}$, which is be used to compute the statistical error, is given by
  \begin{align}\label{eq:theta0_hat_diff_theta0_star_general_mat_form2}
    &\hat{\btheta}_0^{\mathcal{A}} - {\btheta}_0^{\mathcal{A}}
    =
    \Big(\sum_{\tau=1}^{T}\hat{\mathbf{W}}_{\tau, N}^{\mathcal{A}}\Big)^{-1}
     \Big(\sum_{\tau=1}^{T}
        \hat{\mathbf{W}}_{\tau, N}^{\mathcal{A}} (\btheta_{\tau}^{\mathrm{gt}} - {\btheta}_0^{\mathcal{A}}) \Big)
     + \Delta_{T}^{\mathcal{A}}
     =\mathbf{U}^{\top}_{\mathcal{A}} \mathbf{z}_{\mathcal{A}}
     + \Delta_{T}^{\mathcal{A}}
  \end{align}
  based on which the statistical error $\mathcal{E}_{\mathcal{A}}^{2} (\hat{\btheta}_0^{\mathcal{A}})$ in \eqref{eq:meta_test_risk_decompose_final} can be computed by
  \begin{align}\label{eq:stats_err_mat_form_general_no_assp3}
    \mathcal{E}_{\mathcal{A}}^{2} (\hat{\btheta}_0^{\mathcal{A}})
    =\|\hat{\btheta}_{0}^{{\mathcal{A}}}-\btheta_{0}^{ {\mathcal{A}}}\|_{\mathbb{E}_{\tau}[\mathbf{W}_{\tau}^{\mathcal{A}}]}^2
    =& 
    \underbracket{\|\mathbf{U}_{\mathcal{A}}^{\top} \mathbf{z}_{\mathcal{A}}\|^2_{\mathbb{E}_{\tau}[\mathbf{W}_{\tau}^{\mathcal{A}}]}}_{I_1^{\mathcal{A}}}
      + {\|\Delta_{T}^{\mathcal{A}}\|_{\mathbb{E}_{\tau}[\mathbf{W}_{\tau}^{\mathcal{A}}]}^2}
      + 2{\mathbf{z}_{\mathcal{A}}^{\top}
        \mathbf{U}_{\mathcal{A}} }
        {\mathbb{E}_{\tau}[\mathbf{W}_{\tau}^{\mathcal{A}}]}
        \Delta_{T}^{\mathcal{A}}
  \end{align}
  where $I_1^{\mathcal{A}}$ is the only term that does not depend on $\Delta_T^{\mathcal{A}}$, which is caused by the random noise $\epsilon$ in the data.
  In other words, when the variance of $\epsilon$ becomes zero, the statistical error $\mathcal{E}_{\mathcal{A}}^{2} (\hat{\btheta}_0^{\mathcal{A}})$ reduces to $I_1^{\mathcal{A}}$, or $I_1^{\mathcal{A}}$ is the statistical error in the noiseless realizable case.
  We next proceed to bound $I_1^{\mathcal{A}}$ by considering the concentration around its mean as follows
  \begin{align*}
    &\|\mathbf{U}_{\mathcal{A}}^{\top} \mathbf{z}_{\mathcal{A}}\|^2_{\mathbb{E}_{\tau}[\mathbf{W}_{\tau}^{\mathcal{A}}]} 
    \leq  \Big| \|\mathbf{U}_{\mathcal{A}}^{\top} \mathbf{z}_{\mathcal{A}}\|^2_{\mathbb{E}_{\tau}[\mathbf{W}_{\tau}^{\mathcal{A}}]}
    - \mathbb{E}_{\mathbf{z}_{\mathcal{A}} \mid \mathbf{U}_{\mathcal{A}}}\big[\|\mathbf{U}_{\mathcal{A}}^{\top} \mathbf{z}_{\mathcal{A}}\|^2_{\mathbb{E}_{\tau}[\mathbf{W}_{\tau}^{\mathcal{A}}]} \big] \Big|
    + \mathbb{E}_{\mathbf{z}_{\mathcal{A}} \mid \mathbf{U}_{\mathcal{A}}}\big[\|\mathbf{U}_{\mathcal{A}}^{\top} \mathbf{z}_{\mathcal{A}}\|^2_{\mathbb{E}_{\tau}[\mathbf{W}_{\tau}^{\mathcal{A}}]}\big].
  \end{align*}
  Next we will bound the above two terms respectively.
  We first bound $\Big| \|\mathbf{U}_{\mathcal{A}}^{\top} \mathbf{z}_{\mathcal{A}}\|^2_{\mathbb{E}_{\tau}[\mathbf{W}_{\tau}^{\mathcal{A}}]}
  - \mathbb{E}_{\mathbf{z}_{\mathcal{A}} \mid \mathbf{U}_{\mathcal{A}}}\big[\|\mathbf{U}_{\mathcal{A}}^{\top} \mathbf{z}_{\mathcal{A}}\|^2_{\mathbb{E}_{\tau}[\mathbf{W}_{\tau}^{\mathcal{A}}]} \big] \Big|$.
  From Assumption~2, $\btheta_{\tau}^{\mathrm{gt}} - \btheta_{0}^{\mathcal{A}}$ are $(R / \sqrt{d})$-sub-gaussian.
  To bound the absolute error around the expectation,
   from the Hanson-Wright inequality in Lemma~\ref{lemma:Hanson-Wright inequality non zero mean}, with probability at least $1 - \delta$, the following inequality holds
  \begin{align*}\label{eq:HW_ineq_general_ref_no_assp3}
    &\left|\|\mathbf{U}_{\mathcal{A}}^{\top} \mathbf{z}_{\mathcal{A}} \|^2_{\mathbb{E}_{\tau}[\mathbf{W}_{\tau}^{\mathcal{A}}]} 
    - \mathbb{E}_{\mathbf{z}_{\mathcal{A}} \mid \mathbf{U}_{\mathcal{A}}}[\|\mathbf{U}_{\mathcal{A}}^{\top} \mathbf{z}_{\mathcal{A}} \|^2_{\mathbb{E}_{\tau}[\mathbf{W}_{\tau}^{\mathcal{A}}]} ]\right| \\
    \leq &
    \widetilde{\mathcal{O}} \Big(\frac{R^2}{d}
    \Big\|\big(\sum_{\tau=1}^T 
    \hat{\mathbf{W}}_{\tau, N}^{\mathcal{A}}\big)^{-1} 
    \mathbb{E}_{\tau}[\mathbf{W}_{\tau}^{\mathcal{A}}] 
    \big(\sum_{\tau=1}^T 
    \hat{\mathbf{W}}_{\tau, N}^{\mathcal{A}}\big)^{-1}
    (\sum_{\tau=1}^T (\hat{\mathbf{W}}_{\tau, N}^{\mathcal{A}})^2 )
     \Big\|_{\mathrm{F}}\Big) \\
    &+ \widetilde{\mathcal{O}} \Big(
    \frac{R}{\sqrt{d}}  M 
    \Big\|\big(\sum_{\tau=1}^T 
    \hat{\mathbf{W}}_{\tau, N}^{\mathcal{A}}\big)^{-1} 
    \mathbb{E}_{\tau}[\mathbf{W}_{\tau}^{\mathcal{A}}] 
    \big(\sum_{\tau=1}^T 
    \hat{\mathbf{W}}_{\tau, N}^{\mathcal{A}}\big)^{-1}
    (\sum_{\tau=1}^T (\hat{\mathbf{W}}_{\tau, N}^{\mathcal{A}})^2 )
     \Big\|_{\mathrm{op}}
    \Big) \\
    \leq& 
    \widetilde{\mathcal{O}}\Big(\frac{R^{2}+RM}{d T} \Big \|\frac{1}{T} \sum_{\tau=1}^T \hat{\mathbf{W}}_{\tau, N}^{\mathcal{A}}\Big \|^{-2} \cdot \sqrt{d} \Big \|\frac{1}{T} \sum_{\tau=1}^T\hat{\mathbf{W}}_{\tau, N}^{\mathcal{A}}\Big \|_{\mathrm{op}}^2 \Big \|\mathbb{E}_{\tau}[\mathbf{W}_{\tau}^{\mathcal{A}}]\Big \|_{\mathrm{op}}\Big) 
    =\widetilde{\mathcal{O}}\Big(\frac{R^{2} + RM}{T\sqrt{d}} \Big) . \numberthis
  \end{align*}
  Note that in the last equation, we ignore the higher order terms in $\|\frac{1}{T} \sum_{\tau=1}^T \hat{\mathbf{W}}_{\tau, N}^{\mathcal{A}}\|_{\mathrm{op}} $, which can be obtained from Lemma~\ref{lemma:concentration_W_hat}.

  To bound the expected statistical error, $\mathbb{E}_{\mathbf{z}_{\mathcal{A}} \mid \mathbf{U}_{\mathcal{A}}}[\|\mathbf{U}_{\mathcal{A}}^{\top} \mathbf{z}_{\mathcal{A}} \|^2_{\mathbb{E}_{\tau}[\mathbf{W}_{\tau}^{\mathcal{A}}]} ]$,
  first note that 
  since for all $\tau$, $\mathbf{W}_{\tau}^{\mathcal{A}}$ is symmetric positive definite (PD) based on Assumption~\ref{asmp:bounded_data_matrix_eigenvalues}, $\mathbb{E}_{\tau}[\mathbf{W}_{\tau}^{\mathcal{A}}]$ is also symmetric PD, who has a Cholesky decomposition, $\mathbb{E}_{\tau}[\mathbf{W}_{\tau}^{\mathcal{A}}] = \mathbb{E}_{\tau}^{\frac{1}{2}}[\mathbf{W}_{\tau}^{\mathcal{A}}] \mathbb{E}_{\tau}^{\frac{1}{2}}[\mathbf{W}_{\tau}^{\mathcal{A}}]^{\top}$ with $\mathbb{E}_{\tau}^{\frac{1}{2}}[\mathbf{W}_{\tau}^{\mathcal{A}}]$ defined as the lower triangular matrix in the decomposition.
  The statistical error can be rewritten as
  \begin{align*}
    &\|\mathbf{U}_{\mathcal{A}}^{\top} \mathbf{z}_{\mathcal{A}} \|^2_{\mathbb{E}_{\tau}[\mathbf{W}_{\tau}^{\mathcal{A}}]}
    = \mathrm{tr}(\|\mathbf{U}_{\mathcal{A}}^{\top} \mathbf{z}_{\mathcal{A}} \|^2_{\mathbb{E}_{\tau}[\mathbf{W}_{\tau}^{\mathcal{A}}]})
    = \mathrm{tr}\Big(
    (\mathbb{E}_{\tau}^{\frac{1}{2}}[\mathbf{W}_{\tau}^{\mathcal{A}}]
      {\mathbf{U}}_{\mathcal{A}}^{\top} \mathbf{z}_{\mathcal{A}})
    (\mathbb{E}_{\tau}^{\frac{1}{2}}[\mathbf{W}_{\tau}^{\mathcal{A}}]
      {\mathbf{U}}_{\mathcal{A}}^{\top} \mathbf{z}_{\mathcal{A}})^{\top}\Big) \\
    =& \mathrm{tr}\Bigg(
    \Big(\mathbb{E}_{\tau}^{\frac{1}{2}}[\mathbf{W}_{\tau}^{\mathcal{A}}]
    \big(\sum_{\tau=1}^{T}\hat{\mathbf{W}}_{\tau, N}^{\mathcal{A}}\big)^{-1}
    \big(\sum_{\tau=1}^{T}
    \hat{\mathbf{W}}_{\tau, N}^{\mathcal{A}} (\btheta_{\tau}^{\mathrm{gt}} - {\btheta}_0^{\mathcal{A}}) \big)
    \Big)
    \Big(\mathbb{E}_{\tau}^{\frac{1}{2}}[\mathbf{W}_{\tau}^{\mathcal{A}}]
    \big(\sum_{\tau=1}^{T}\hat{\mathbf{W}}_{\tau, N}^{\mathcal{A}}\big)^{-1}
    \big(\sum_{\tau=1}^{T}
    \hat{\mathbf{W}}_{\tau, N}^{\mathcal{A}} (\btheta_{\tau}^{\mathrm{gt}} - {\btheta}_0^{\mathcal{A}}) \big)
    \Big)^{\top}
    \Bigg)
  \end{align*}
  whose conditional expectation is given by
  \begin{align*}
  \label{eq:E_stats_err_eq}
    &\mathbb{E}_{\mathbf{z}_{\mathcal{A}} \mid \mathbf{U}_{\mathcal{A}}}[\|\mathbf{U}_{\mathcal{A}}^{\top} \mathbf{z}_{\mathcal{A}} \|^2_{\mathbb{E}_{\tau}[\mathbf{W}_{\tau}^{\mathcal{A}}]}] \\
    =&\mathbb{E}_{\btheta_{\tau}^{\mathrm{gt}} \mid \hat{\mathbf{W}}_{\tau, N}^{\mathcal{A}}}\Big[
    \mathrm{tr}\Big(
       \Big(\mathbb{E}_{\tau}^{\frac{1}{2}}[\mathbf{W}_{\tau}^{\mathcal{A}}]
       \big(\sum_{\tau=1}^{T}\hat{\mathbf{W}}_{\tau, N}^{\mathcal{A}}\big)^{-1}
       \big(\sum_{\tau=1}^{T}
       \hat{\mathbf{W}}_{\tau, N}^{\mathcal{A}} (\btheta_{\tau}^{\mathrm{gt}} - {\btheta}_0^{\mathcal{A}}) \big)
       \Big) \\
       & \quad\quad\quad \Big(\mathbb{E}_{\tau}^{\frac{1}{2}}[\mathbf{W}_{\tau}^{\mathcal{A}}]
       \big(\sum_{\tau=1}^{T}\hat{\mathbf{W}}_{\tau, N}^{\mathcal{A}}\big)^{-1}
       \big(\sum_{\tau=1}^{T}
       \hat{\mathbf{W}}_{\tau, N}^{\mathcal{A}} (\btheta_{\tau}^{\mathrm{gt}} - {\btheta}_0^{\mathcal{A}}) \big)
       \Big)^{\top}
       \Big)
       \Big] \\
    =& \underbracket{\mathrm{tr}\Big(
    \mathrm{Cov}_{\btheta_{\tau}^{\mathrm{gt}} \mid \hat{\mathbf{W}}_{\tau, N}^{\mathcal{A}}}
    \Big[\mathbb{E}_{\tau}^{\frac{1}{2}}[\mathbf{W}_{\tau}^{\mathcal{A}}]
    \big(\sum_{\tau=1}^{T}\hat{\mathbf{W}}_{\tau, N}^{\mathcal{A}}\big)^{-1}
    \big(\sum_{\tau=1}^{T}
    \hat{\mathbf{W}}_{\tau, N}^{\mathcal{A}} (\btheta_{\tau}^{\mathrm{gt}} - {\btheta}_0^{\mathcal{A}}) \big)
    \Big]
    \Big)}_{I_a} \\
    &+ \underbracket{\Big\|\mathbb{E}_{{\btheta_{\tau}^{\mathrm{gt}} \mid \hat{\mathbf{W}}_{\tau, N}^{\mathcal{A}}}}
    \Big[\mathbb{E}_{\tau}^{\frac{1}{2}}[\mathbf{W}_{\tau}^{\mathcal{A}}]
    \big(\sum_{\tau=1}^{T}\hat{\mathbf{W}}_{\tau, N}^{\mathcal{A}}\big)^{-1}
    \big(\sum_{\tau=1}^{T}
    \hat{\mathbf{W}}_{\tau, N}^{\mathcal{A}} (\btheta_{\tau}^{\mathrm{gt}} - {\btheta}_0^{\mathcal{A}}) \big)
    \Big]
    \Big\|^2}_{I_b}
    \numberthis
  \end{align*}
  where the last equation is given by the fact that $\mathbb{E}[\mathbf{z} \mathbf{z}^{\top}] = \mathrm{Cov}[\mathbf{z}] + \mathbb{E}[\mathbf{z}] \mathbb{E}[\mathbf{z}]^{\top}$ for any random vector $\mathbf{z}$, and the linear and cyclic property of trace.
  From Assumption~\ref{asmp:sub_gaussian_task_para}, $\btheta_{\tau}^{\mathrm{gt}} - \btheta_{0}^{\mathcal{A}}$ has independent $(R / \sqrt{d})$-sub-gaussian entries, and Lemma~\ref{lemma:combinate_sub_gaussian}, linear combinations of sub-gaussian random variables are still sub-gaussian,
  $\mathbb{E}_{\tau}^{\frac{1}{2}}[\mathbf{W}_{\tau}^{\mathcal{A}}]
    \big(\sum_{\tau=1}^{T}\hat{\mathbf{W}}_{\tau, N}^{\mathcal{A}}\big)^{-1}
    \big(\sum_{\tau=1}^{T}
    \hat{\mathbf{W}}_{\tau, N}^{\mathcal{A}} (\btheta_{\tau}^{\mathrm{gt}} - {\btheta}_0^{\mathcal{A}}) \big)$
    has sub-gaussian entries with parameter $\big\|\mathbb{E}_{\tau}^{\frac{1}{2}}[\mathbf{W}_{\tau}^{\mathcal{A}}]
    \big(\sum_{\tau=1}^{T}\hat{\mathbf{W}}_{\tau, N}^{\mathcal{A}}\big)^{-1}\big\|^2_{\mathrm{op}}
    \big\|\sum_{\tau=1}^{T}
    (\hat{\mathbf{W}}_{\tau, N}^{\mathcal{A}})^2\big\|_{\mathrm{op}} R^2/d$, which is the upper bound of the variance of each entry
    based on the sub-gaussian property.
    Since the trace of the covariance is the sum of the variance of all entries, it holds that
    \begin{align*}
    \label{eq:bound_I_a}
      I_a \leq &
      d  \Big\|\mathbb{E}_{\tau}^{\frac{1}{2}}[\mathbf{W}_{\tau}^{\mathcal{A}}]
    \big(\sum_{\tau=1}^{T}\hat{\mathbf{W}}_{\tau, N}^{\mathcal{A}}\big)^{-1}\Big\|^2_{\mathrm{op}}
    \Big\|\sum_{\tau=1}^{T}
    (\hat{\mathbf{W}}_{\tau, N}^{\mathcal{A}})^2\Big\|_{\mathrm{op}} \frac{R^2}{d} \\
    \leq &  
     \Big \|\mathbb{E}_{\tau}[\mathbf{W}_{\tau}^{\mathcal{A}}] \Big\|_{\mathrm{op}}
    \Big\|\sum_{\tau=1}^{T}\hat{\mathbf{W}}_{\tau, N}^{\mathcal{A}}\Big\|_{\mathrm{op}}^{-2}
    \Big\|\sum_{\tau=1}^{T}
    (\hat{\mathbf{W}}_{\tau, N}^{\mathcal{A}})^2\Big \|_{\mathrm{op}}  R^2 \\
    \leq &
    \frac{R^2}{T}\Big \|\mathbb{E}_{\tau}[\mathbf{W}_{\tau}^{\mathcal{A}}] \Big\|_{\mathrm{op}}
    \Big\|\frac{1}{T}\sum_{\tau=1}^{T}\hat{\mathbf{W}}_{\tau, N}^{\mathcal{A}}\Big\|_{\mathrm{op}}^{-2}
    \Big\|\frac{1}{T}\sum_{\tau=1}^{T}
    (\hat{\mathbf{W}}_{\tau, N}^{\mathcal{A}})^2\Big \|_{\mathrm{op}}  .
    \numberthis
    \end{align*}
  
  $I_b$ can be further derived as
  {\fontsize{9}{9}\selectfont
  \begin{align*}
    &I_b 
    =\Big\|\mathbb{E}_{{\btheta_{\tau}^{\mathrm{gt}} \mid \hat{\mathbf{W}}_{\tau, N}^{\mathcal{A}}}}
    \Big[\mathbb{E}_{\tau}^{\frac{1}{2}}[\mathbf{W}_{\tau}^{\mathcal{A}}]
    \big(\sum_{\tau=1}^{T}\hat{\mathbf{W}}_{\tau, N}^{\mathcal{A}}\big)^{-1}
    \big(\sum_{\tau=1}^{T}
    \hat{\mathbf{W}}_{\tau, N}^{\mathcal{A}} (\btheta_{\tau}^{\mathrm{gt}} - {\btheta}_0^{\mathcal{A}}) \big)
    \Big]
    \Big\|^2 \\
    =& \Big\langle \mathbb{E}_{\tau}[\mathbf{W}_{\tau}^{\mathcal{A}}]
    \Big(\mathbb{E}_{{\btheta_{\tau}^{\mathrm{gt}} \mid \hat{\mathbf{W}}_{\tau, N}^{\mathcal{A}}}}
      \Big[
      \big(\sum_{\tau=1}^{T}\hat{\mathbf{W}}_{\tau, N}^{\mathcal{A}}\big)^{-1}
      \big(\sum_{\tau=1}^{T}
      \hat{\mathbf{W}}_{\tau, N}^{\mathcal{A}} (\btheta_{\tau}^{\mathrm{gt}} - {\btheta}_0^{\mathcal{A}}) \big)
      \Big]+ 
      \big(\sum_{\tau=1}^{T}\hat{\mathbf{W}}_{\tau, N}^{\mathcal{A}}\big)^{-1}
    \mathbb{E}_{{\btheta_{\tau}^{\mathrm{gt}}, \hat{\mathbf{W}}_{\tau, N}^{\mathcal{A}}}}
    \Big[\sum_{\tau=1}^{T}
    \hat{\mathbf{W}}_{\tau, N}^{\mathcal{A}} (\btheta_{\tau}^{\mathrm{gt}} - {\btheta}_0^{\mathcal{A}}) 
    \Big]\Big), \\
    &
    \underbracket{\mathbb{E}_{{\btheta_{\tau}^{\mathrm{gt}} \mid \hat{\mathbf{W}}_{\tau, N}^{\mathcal{A}}}}
       \Big[
       \big(\sum_{\tau=1}^{T}\hat{\mathbf{W}}_{\tau, N}^{\mathcal{A}}\big)^{-1}
       \big(\sum_{\tau=1}^{T}
       \hat{\mathbf{W}}_{\tau, N}^{\mathcal{A}} (\btheta_{\tau}^{\mathrm{gt}} - {\btheta}_0^{\mathcal{A}}) \big)
       \Big] - 
       \big(\sum_{\tau=1}^{T}\hat{\mathbf{W}}_{\tau, N}^{\mathcal{A}}\big)^{-1}
       \mathbb{E}_{{\btheta_{\tau}^{\mathrm{gt}}, \hat{\mathbf{W}}_{\tau, N}^{\mathcal{A}}}}
       \Big[\sum_{\tau=1}^{T}
       \hat{\mathbf{W}}_{\tau, N}^{\mathcal{A}} (\btheta_{\tau}^{\mathrm{gt}} - {\btheta}_0^{\mathcal{A}}) 
       \Big]}_{I_{b1}} \Big\rangle  \\
    &+ \Big\langle \mathbb{E}_{\tau}[\mathbf{W}_{\tau}^{\mathcal{A}}]
    \Big(\big(\sum_{\tau=1}^{T}\hat{\mathbf{W}}_{\tau, N}^{\mathcal{A}}\big)^{-1}
    \mathbb{E}_{{\btheta_{\tau}^{\mathrm{gt}}, \hat{\mathbf{W}}_{\tau, N}^{\mathcal{A}}}}
    \Big[\sum_{\tau=1}^{T}
    \hat{\mathbf{W}}_{\tau, N}^{\mathcal{A}} (\btheta_{\tau}^{\mathrm{gt}} - {\btheta}_0^{\mathcal{A}}) 
    \Big]+ 
    \big(\sum_{\tau=1}^{T}\hat{\mathbf{W}}_{\tau, N}^{\mathcal{A}}\big)^{-1}
    \mathbb{E}_{{\btheta_{\tau}^{\mathrm{gt}}, {\mathbf{W}}_{\tau, N}^{\mathcal{A}}}}
    \Big[\sum_{\tau=1}^{T}
    {\mathbf{W}}_{\tau, N}^{\mathcal{A}} (\btheta_{\tau}^{\mathrm{gt}} - {\btheta}_0^{\mathcal{A}}) 
    \Big]\Big), \\
    &
    \underbracket{\big(\sum_{\tau=1}^{T}\hat{\mathbf{W}}_{\tau, N}^{\mathcal{A}}\big)^{-1}
      \mathbb{E}_{{\btheta_{\tau}^{\mathrm{gt}}, \hat{\mathbf{W}}_{\tau, N}^{\mathcal{A}}}}
      \Big[\sum_{\tau=1}^{T}
      \hat{\mathbf{W}}_{\tau, N}^{\mathcal{A}} (\btheta_{\tau}^{\mathrm{gt}} - {\btheta}_0^{\mathcal{A}}) 
      \Big]- 
      \big(\sum_{\tau=1}^{T}\hat{\mathbf{W}}_{\tau, N}^{\mathcal{A}}\big)^{-1}
      \mathbb{E}_{{\btheta_{\tau}^{\mathrm{gt}}, {\mathbf{W}}_{\tau, N}^{\mathcal{A}}}}
      \Big[\sum_{\tau=1}^{T}
      {\mathbf{W}}_{\tau, N}^{\mathcal{A}} (\btheta_{\tau}^{\mathrm{gt}} - {\btheta}_0^{\mathcal{A}}) 
      \Big]}_{I_{b2}} \Big\rangle\\
    &+ \underbracket{\Big\langle \mathbb{E}_{\tau}[\mathbf{W}_{\tau}^{\mathcal{A}}]
      \big(\sum_{\tau=1}^{T}\hat{\mathbf{W}}_{\tau, N}^{\mathcal{A}}\big)^{-1}
      \mathbb{E}_{{\btheta_{\tau}^{\mathrm{gt}}, {\mathbf{W}}_{\tau, N}^{\mathcal{A}}}}
      \Big[\sum_{\tau=1}^{T}
      {\mathbf{W}}_{\tau, N}^{\mathcal{A}} (\btheta_{\tau}^{\mathrm{gt}} - {\btheta}_0^{\mathcal{A}}) 
      \Big], 
      \big(\sum_{\tau=1}^{T}\hat{\mathbf{W}}_{\tau, N}^{\mathcal{A}}\big)^{-1}
      \mathbb{E}_{{\btheta_{\tau}^{\mathrm{gt}}, {\mathbf{W}}_{\tau, N}^{\mathcal{A}}}}
      \Big[\sum_{\tau=1}^{T}
      {\mathbf{W}}_{\tau, N}^{\mathcal{A}} (\btheta_{\tau}^{\mathrm{gt}} - {\btheta}_0^{\mathcal{A}}) 
      \Big] \Big\rangle}_{I_{b3}}
  \end{align*}}
  where $I_{b1}$ can be further derived as
  \begin{align*}
    I_{b1}
    &= \big(\sum_{\tau=1}^{T}\hat{\mathbf{W}}_{\tau, N}^{\mathcal{A}}\big)^{-1}
    \Big\{\mathbb{E}_{{\btheta_{\tau}^{\mathrm{gt}} \mid \hat{\mathbf{W}}_{\tau, N}^{\mathcal{A}}}}
       \Big[   
       \sum_{\tau=1}^{T}
       \hat{\mathbf{W}}_{\tau, N}^{\mathcal{A}} (\btheta_{\tau}^{\mathrm{gt}} - {\btheta}_0^{\mathcal{A}}) 
       \Big] - 
       \mathbb{E}_{{\btheta_{\tau}^{\mathrm{gt}}, \hat{\mathbf{W}}_{\tau, N}^{\mathcal{A}}}}
       \Big[\sum_{\tau=1}^{T}
       \hat{\mathbf{W}}_{\tau, N}^{\mathcal{A}} (\btheta_{\tau}^{\mathrm{gt}} - {\btheta}_0^{\mathcal{A}}) 
       \Big]\Big\} \\
    &= \big(\frac{1}{T}\sum_{\tau=1}^{T}\hat{\mathbf{W}}_{\tau, N}^{\mathcal{A}}\big)^{-1}
    \Big\{\frac{1}{T}\sum_{\tau=1}^{T}     
       \hat{\mathbf{W}}_{\tau, N}^{\mathcal{A}} 
       \mathbb{E}_{{\btheta_{\tau}^{\mathrm{gt}} \mid \hat{\mathbf{W}}_{\tau, N}^{\mathcal{A}}}}
       \Big[ 
       \btheta_{\tau}^{\mathrm{gt}} - {\btheta}_0^{\mathcal{A}}
       \Big] - 
       \mathbb{E}_{\hat{\mathbf{W}}_{\tau, N}^{\mathcal{A}}}
       \Big[
       \hat{\mathbf{W}}_{\tau, N}^{\mathcal{A}} 
       \mathbb{E}_{{\btheta_{\tau}^{\mathrm{gt}} \mid \hat{\mathbf{W}}_{\tau, N}^{\mathcal{A}}}}
       [\btheta_{\tau}^{\mathrm{gt}} - {\btheta}_0^{\mathcal{A}}]
       \Big]\Big\} \\
    \|I_{b1}\| 
    &\leq
    \widetilde{\mathcal{O}} \Big(\sqrt{\frac{d}{T}} \Big) K
    M\Big\|\frac{1}{T}\sum_{\tau=1}^{T}\hat{\mathbf{W}}_{\tau, N}^{\mathcal{A}}\Big\|^{-1}_{\mathrm{op}} .
  \end{align*}
  And for $I_{b2}$, it holds that
  \begin{align*}
    I_{b2}
    = &
    \big(\sum_{\tau=1}^{T}\hat{\mathbf{W}}_{\tau, N}^{\mathcal{A}}\big)^{-1}
    \Big\{\mathbb{E}_{{\btheta_{\tau}^{\mathrm{gt}}, \hat{\mathbf{W}}_{\tau, N}^{\mathcal{A}}}}
    \Big[   
    \sum_{\tau=1}^{T}
    \hat{\mathbf{W}}_{\tau, N}^{\mathcal{A}} (\btheta_{\tau}^{\mathrm{gt}} - {\btheta}_0^{\mathcal{A}}) 
    \Big] - 
    \mathbb{E}_{{\btheta_{\tau}^{\mathrm{gt}}, {\mathbf{W}}_{\tau, N}^{\mathcal{A}}}}
    \Big[\sum_{\tau=1}^{T}
    {\mathbf{W}}_{\tau, N}^{\mathcal{A}} (\btheta_{\tau}^{\mathrm{gt}} - {\btheta}_0^{\mathcal{A}}) 
    \Big]\Big\} \\
    = &
    \big(\sum_{\tau=1}^{T}\hat{\mathbf{W}}_{\tau, N}^{\mathcal{A}}\big)^{-1}
    \Big\{\mathbb{E}_{{\btheta_{\tau}^{\mathrm{gt}}, {\mathbf{W}}_{\tau, N}^{\mathcal{A}}, \hat{\mathbf{W}}_{\tau, N}^{\mathcal{A}}}}
    \Big[   
    \sum_{\tau=1}^{T}
    (\hat{\mathbf{W}}_{\tau, N}^{\mathcal{A}} 
    - {\mathbf{W}}_{\tau, N}^{\mathcal{A}})
    (\btheta_{\tau}^{\mathrm{gt}} - {\btheta}_0^{\mathcal{A}})
    \Big]\Big\} \\
    = &
    \big(\frac{1}{T}\sum_{\tau=1}^{T}\hat{\mathbf{W}}_{\tau, N}^{\mathcal{A}}\big)^{-1}
    \Big\{\mathbb{E}_{{\btheta_{\tau}^{\mathrm{gt}}, {\mathbf{W}}_{\tau, N}^{\mathcal{A}}, \hat{\mathbf{W}}_{\tau, N}^{\mathcal{A}}}}
    \Big[   
    \frac{1}{T}\sum_{\tau=1}^{T}
    (\hat{\mathbf{W}}_{\tau, N}^{\mathcal{A}} 
    - \mathbb{E}[\hat{\mathbf{W}}_{\tau, N}^{\mathcal{A}}]
    + \mathbb{E}[\hat{\mathbf{W}}_{\tau, N}^{\mathcal{A}}] 
    - {\mathbf{W}}_{\tau, N}^{\mathcal{A}})
    (\btheta_{\tau}^{\mathrm{gt}} - {\btheta}_0^{\mathcal{A}})
    \Big]\Big\} \\
    \|I_{b_2}\|
    \leq & \Big(\widetilde{\mathcal{O}} (\sqrt{\frac{d}{T}} ) K + 
    \big\| \mathbb{E}[\hat{\mathbf{W}}_{\tau, N}^{\mathcal{A}}
    - {\mathbf{W}}_{\tau, N}^{\mathcal{A}}] \big\|
    \Big)
    M\Big\|\frac{1}{T}\sum_{\tau=1}^{T}\hat{\mathbf{W}}_{\tau, N}^{\mathcal{A}}\Big\|^{-1}_{\mathrm{op}} .
  \end{align*}
  $I_{b3} = 0$ since $\mathbb{E}_{{\btheta_{\tau}^{\mathrm{gt}}, {\mathbf{W}}_{\tau, N}^{\mathcal{A}}}}
    \big[\sum_{\tau=1}^{T}
    {\mathbf{W}}_{\tau, N}^{\mathcal{A}} (\btheta_{\tau}^{\mathrm{gt}} - {\btheta}_0^{\mathcal{A}}) 
    \big] = \mathbf{0}$.
  Combining $I_{b1}, I_{b2}, I_{b3}$ from above discussions we can bound $I_b$ by
  \begin{align*}
  \label{eq:I_b_bound}
    I_b &=
    \Big\|\mathbb{E}_{{\btheta_{\tau}^{\mathrm{gt}} \mid \hat{\mathbf{W}}_{\tau, N}^{\mathcal{A}}}}
    \Big[\mathbb{E}_{\tau}^{\frac{1}{2}}[\mathbf{W}_{\tau}^{\mathcal{A}}]
    \big(\sum_{\tau=1}^{T}\hat{\mathbf{W}}_{\tau, N}^{\mathcal{A}}\big)^{-1}
    \big(\sum_{\tau=1}^{T}
    \hat{\mathbf{W}}_{\tau, N}^{\mathcal{A}} (\btheta_{\tau}^{\mathrm{gt}} - {\btheta}_0^{\mathcal{A}}) \big)
    \Big]
    \Big\|^2 \\
    &\leq 
    \Big(\widetilde{\mathcal{O}} (\sqrt{\frac{d}{T}} ) K + 
    \big\| \mathbb{E}[\hat{\mathbf{W}}_{\tau, N}^{\mathcal{A}}
    - {\mathbf{W}}_{\tau, N}^{\mathcal{A}}] \big\|
    \Big) M^2
    \Big \|\mathbb{E}_{\tau}[\mathbf{W}_{\tau}^{\mathcal{A}}] \Big\|_{\mathrm{op}}
    \Big\|\frac{1}{T}\sum_{\tau=1}^{T}\hat{\mathbf{W}}_{\tau, N}^{\mathcal{A}}\Big\|_{\mathrm{op}}^{-1} .
    \numberthis
  \end{align*}
  Combining the bound for $I_a$ and $I_b$, the expected statistical error conditioned on $\hat{\mathbf{W}}_{\tau,N}^{\mathcal{A}}$ is bounded by
  \begin{align*}
  \label{eq:bound_E_I1}
    &\mathbb{E}_{\mathbf{z}_{\mathcal{A}} \mid \mathbf{U}_{\mathcal{A}}}[\|\mathbf{U}_{\mathcal{A}}^{\top} \mathbf{z}_{\mathcal{A}} \|^2_{\mathbb{E}_{\tau}[\mathbf{W}_{\tau}^{\mathcal{A}}]}]
    \leq 
    \frac{R^2}{T}\Big \|\mathbb{E}_{\tau}[\mathbf{W}_{\tau}^{\mathcal{A}}] \Big\|_{\mathrm{op}}
    \Big\|\frac{1}{T}\sum_{\tau=1}^{T}\hat{\mathbf{W}}_{\tau, N}^{\mathcal{A}}\Big\|_{\mathrm{op}}^{-2}
    \Big\|\frac{1}{T}\sum_{\tau=1}^{T}
    (\hat{\mathbf{W}}_{\tau, N}^{\mathcal{A}})^2\Big \|_{\mathrm{op}} \\
    &+ 
    \Big(\widetilde{\mathcal{O}} (\sqrt{\frac{d}{T}} ) K + 
    \big\| \mathbb{E}[\hat{\mathbf{W}}_{\tau, N}^{\mathcal{A}}
    - {\mathbf{W}}_{\tau, N}^{\mathcal{A}}] \big\|
    \Big) M^2
    \Big \|\mathbb{E}_{\tau}[\mathbf{W}_{\tau}^{\mathcal{A}}] \Big\|_{\mathrm{op}}
    \Big\|\frac{1}{T}\sum_{\tau=1}^{T}\hat{\mathbf{W}}_{\tau, N}^{\mathcal{A}}\Big\|_{\mathrm{op}}^{-1} .
    \numberthis
  \end{align*}

  Finally note that $\big\| \mathbb{E}[\hat{\mathbf{W}}_{\tau, N}^{\mathcal{A}}
  - {\mathbf{W}}_{\tau, N}^{\mathcal{A}}] \big\| 
  \leq \widetilde{\mathcal{O}}(\frac{d}{N})$,
  by combining~\eqref{eq:HW_ineq_general_ref_no_assp3}\eqref{eq:bound_E_I1},
  with probability at least $1-Td^{-10}$, it holds that
  \begin{align*}
    \|\mathbf{U}_{\mathcal{A}}^{\top} \mathbf{z}_{\mathcal{A}} \|^2_{\mathbb{E}_{\tau}[\mathbf{W}_{\tau}^{\mathcal{A}}]}
    &\leq 
    \frac{R^2 }{T} 
    \Big(\lambda_{\min}(\mathbb{E}[{\mathbf{W}}^{\mathcal{A}}_{\tau}])^{-1}
      \lambda_{\max}(\mathbb{E}[({\mathbf{W}}^{\mathcal{A}}_{\tau})^{2}])
      + \widetilde{\mathcal{O}}(\frac{d}{N})
      + \widetilde{\mathcal{O}}(\frac{1}{\sqrt{d}})  + \widetilde{\mathcal{O}} (\sqrt{\frac{d}{T}})\Big) \\
     & + 
    \Big( \widetilde{\mathcal{O}}(\sqrt{\frac{d}{T}}) + 
    \widetilde{\mathcal{O}}({\frac{d}{N}})
    \Big) M^2
    .
    \numberthis
  \end{align*}
  This completes the proof in the data agnostic case without Assumption~3.
  Next we proceed to prove the bound under the case with Assumption~3.
  
  From Assumption~3, $\mathbf{W}^{\mathcal{A}}_{\tau} = w_{\mathcal{A}} \mathbf{I}_d$,
  where $w_{\mathcal{A}}$ is the same for different task $\tau$, therefore
  $\btheta_0^{\mathcal{A}} = \mathbb{E}_{\tau}[\btheta_{\tau}^{\mathrm{gt}}]$,
  and $w_{\mathrm{er}} = 1$, 
  $w_{\mathrm{ma}} = (1-\alpha)^2$, $w_{\mathrm{im}} = (1+\gamma^{-1})^{-2}$, $w_{\mathrm{ba}} = (1+({\gamma}s)^{-1})^{-1} (1+{\gamma}^{-1})^{-1}$.
  By using this property in~\eqref{eq:meta_test_risk_decompose_final},
  we obtain a simplified meta-test risk decomposition of method $\mathcal{A}$ to the statistical and optimal population risk in the linear centroid model
   by
  \begin{align}\label{eq:er_risk_decompose_assumption5_general}
  \lim_{N_a \to \infty} \mathcal{R}_{N_a}^{\mathcal{A}}(\hat{\btheta}_0^{\mathcal{A}})
  =\underbracket{w_{\mathcal{A}}\|\hat{\btheta}_0^{\mathcal{A}}- {\btheta}_0^{\mathcal{A}}\|_2^2}_{\text{statistical error}~\mathcal{E}^{2}_{\mathcal{A}} (\hat{\btheta}_0^{\mathcal{A}})}
  +\underbracket{\lim_{N_a \to \infty} \mathcal{R}_{N_a}^{\mathcal{A}}({\btheta}_0^{\mathcal{A}})}_{\text{optimal population risk}}.
  \end{align}
  Thus the statistical error in \eqref{eq:er_risk_decompose_assumption5_general} can be computed by
  \begin{align}\label{eq:stats_err_mat_form_general}
    \mathcal{E}_{\mathcal{A}}^{2} (\hat{\btheta}_0^{\mathcal{A}})
    =w_{\mathcal{A}}\|\hat{\btheta}_{0}^{{\mathcal{A}}}-\btheta_{0}^{ {\mathcal{A}}}\|_2^2
    =& w_{\mathcal{A}}
    (\underbracket{\| \mathbf{U}_{\mathcal{A}}^{\top} \mathbf{z}_{\mathcal{A}}\|^2}_{I_2^{\mathcal{A}}}
      + {\|\Delta_{T}^{\mathcal{A}}\|_2^2}
      + 2{\mathbf{z}_{\mathcal{A}}^{\top}
        \mathbf{U}_{\mathcal{A}} }
        \Delta_{T}^{\mathcal{A}} ).
  \end{align}
  We will bound term $I_2^{\mathcal{A}}$ in the above equation.
  The only difference of $I_2^{\mathcal{A}}$ from $I_1^{\mathcal{A}}$ is that we can treat $\mathbb{E}_{\tau}[\mathbf{W}_{\tau}] = \mathbf{I}_d$ and $M=0$ when adding Assumption~3.
  Therefore, with probability at least $1 - \delta$, we have 
  \begin{align*}\label{eq:HW_ineq_general_ref}
    &\left|\| \mathbf{U}_{\mathcal{A}}^{\top} \mathbf{z}_{\mathcal{A}}\|^2
    - \mathbb{E}_{\btheta_{\tau}^{\mathrm{gt}}, \mathbf{e}_{\tau} \mid \hat{\mathbf{W}}_{\tau, N}^{\mathcal{A}}}[\| \mathbf{U}_{\mathcal{A}}^{\top} \mathbf{z}_{\mathcal{A}}\|^2]\right|  
    =\widetilde{\mathcal{O}}\Big(\frac{R^{2}}{T\sqrt{d}}  \Big).\numberthis
  \end{align*}
  
  To compute $\mathbb{E}_{\btheta_{\tau}^{\mathrm{gt}}, \mathbf{e}_{\tau} \mid \hat{\mathbf{W}}_{\tau, N}^{\mathcal{A}}}[\mathbf{z}_{\mathcal{A}}^{\top}\mathbf{U}_{\mathcal{A}} \mathbf{U}_{\mathcal{A}}^{\top} \mathbf{z}_{\mathcal{A}} ]$, first we have
  \begin{align*}
  \nonumber
  &\mathbb{E}_{\btheta_{\tau}^{\mathrm{gt}}, \mathbf{e}_{\tau} \mid \hat{\mathbf{W}}_{\tau, N}^{\mathcal{A}}}[\mathbf{z}_{\mathcal{A}}^{\top}\mathbf{U}_{\mathcal{A}} \mathbf{U}_{\mathcal{A}}^{\top} \mathbf{z}_{\mathcal{A}}] \\
  =& \frac{R^{2}}{T d}\Big\langle\Big(\frac{1}{T}{\sum_{\tau=1}^{T} \hat{\mathbf{W}}_{\tau, N}^{\mathcal{A}}}\Big)^{-2}, 
  \frac{1}{T}{\sum_{\tau=1}^{T} (\hat{\mathbf{W}}_{\tau, N}^{\mathcal{A}})^{2}}\Big\rangle 
  \label{eq:E_mse_general_3term}  
  = \frac{R^{2}}{T}
  \Big[\underbracket{\frac{1}{d}\Big\langle\mathbb{E}^{-2}\big[\hat{\mathbf{W}}_{\tau, N}^{\mathcal{A}}\big], \mathbb{E}\big[(\hat{\mathbf{W}}_{\tau, N}^{\mathcal{A}})^{2}\big]\Big\rangle}_{= \tilde{C}_{0}^{\mathcal{A}}} \\
  &+\underbracket{\frac{1}{d}\Big\langle\Big(\frac{1}{T}{\sum_{\tau=1}^{T} \hat{\mathbf{W}}_{\tau, N}^{\mathcal{A}}}\Big)^{-2}-\mathbb{E}^{-2}\big[\hat{\mathbf{W}}_{\tau, N}^{\mathcal{A}}\big], \mathbb{E}\big[(\hat{\mathbf{W}}_{\tau, N}^{\mathcal{A}})^{2}\big]\Big\rangle}_{I_c} \\
  &+\underbracket{\frac{1}{d}\Big\langle\Big(\frac{1}{T}{\sum_{\tau=1}^{T} \hat{\mathbf{W}}_{\tau, N}^{\mathcal{A}}}\Big)^{-2}, \frac{1}{T}{\sum_{\tau=1}^{T} (\hat{\mathbf{W}}_{\tau, N}^{\mathcal{A}})^{2}}-\mathbb{E}\big[(\hat{\mathbf{W}}_{\tau, N}^{\mathcal{A}})^{2}\big]\Big\rangle}_{I_d}\Big]. \numberthis
  \end{align*}
  For term $I_c$ and $I_d$, from Lemma~\ref{lemma:concentration_W_hat}, we have
  with probability at least $1-Td^{-10}$ that
  \begin{equation}\label{eq:E_stats_err_term_a_b_general}
    |I_c| \leq \widetilde{\mathcal{O}}(\sqrt{\frac{d}{T}}),
    |I_d| \leq \widetilde{\mathcal{O}}(\sqrt{\frac{d}{T}})  .
  \end{equation}
  Combining~\eqref{eq:E_mse_general_3term} and \eqref{eq:E_stats_err_term_a_b_general}, we have
  with probability at least $1-Td^{-10}$
  \begin{equation}\label{eq:E_term1_bound_general}
    \mathbb{E}_{\btheta_{\tau}^{\mathrm{gt}}, \mathbf{e}_{\tau} \mid \hat{\mathbf{W}}_{\tau, N}^{\mathcal{A}}}[\mathbf{z}_{\mathcal{A}}^{\top}\mathbf{U}_{\mathcal{A}} \mathbf{U}_{\mathcal{A}}^{\top} \mathbf{z}_{\mathcal{A}}] 
    \leq 
    \frac{R^{2}}{T}\Big(\tilde{C}^{\mathcal{A}}_{0} + \widetilde{\mathcal{O}}(\sqrt{\frac{d}{T}})\Big).
  \end{equation}
  
  Then combining~\eqref{eq:HW_ineq_general_ref},\eqref{eq:E_term1_bound_general}, 
  it holds with probability at least $1-Td^{-10}$ that
  \begin{equation}\label{eq:stats_err_term_I1_general}
    I_2^{\mathcal{A}}=
    \mathbf{z}_{\mathcal{A}}^{\top}\mathbf{U}_{\mathcal{A}} \mathbf{U}_{\mathcal{A}}^{\top} \mathbf{z}_{\mathcal{A}}
    \leq \frac{R^{2}}{T}\Big(\tilde{C}^{\mathcal{A}}_{0} + \widetilde{\mathcal{O}}(\frac{1}{\sqrt{d}}) + \widetilde{\mathcal{O}}(\sqrt{\frac{d}{T}})\Big). 
  \end{equation}
  
  \end{proof}
  
  \begin{lemma}[Dominating constant in statistical error]
    \label{lemma:dominate_constant_stats_err}
    Suppose Assumptions~1,3 hold. The dominating constant in the statistical error of meta learning method $\mathcal{A}$ is computed by 
    \begin{align*}
      \tilde{C}^{\mathcal{A}}_{0} \coloneqq
      \frac{1}{d}\Big\langle\mathbb{E}^{-2}\big[\hat{\mathbf{W}}_{\tau}^{\mathcal{A}}\big], \mathbb{E}\big[(\hat{\mathbf{W}}_{\tau}^{\mathcal{A}})^{2}\big]\Big\rangle
      =\frac{1}{d} \mathbb{E}\Big[  \mathrm{tr}\big((\hat{\mathbf{W}}_{\tau}^{\mathcal{A}})^{2}\big)\Big]
      \Big\{\frac{1}{d}\mathbb{E}\big[ \mathrm{tr}(\hat{\mathbf{W}}_{\tau}^{\mathcal{A}})\big]\Big\}^{-2} \geq 1.
    \end{align*}
  \end{lemma}
  
  \begin{proof}
    We have proved in previous sections that the dominating constant in the statitical error of meta learning method $\mathcal{A}$ adopts the form $\tilde{C}^{\mathcal{A}}_{0} \coloneqq
    \frac{1}{d}\big\langle\mathbb{E}^{-2}\big[\hat{\mathbf{W}}_{\tau}^{\mathcal{A}}\big], \mathbb{E}\big[(\hat{\mathbf{W}}_{\tau}^{\mathcal{A}})^{2}\big]\big\rangle$. Next we will prove the equality by showing that $\mathbb{E}[\hat{\mathbf{W}}_{\tau}^{\mathcal{A}}] = \frac{1}{d}\mathbb{E}[\mathrm{tr}(\hat{\mathbf{W}}_{\tau}^{\mathcal{A}})] \mathbf{I}_d$.
    
    Let $\mathbf{X}_{\tau,N} = \mathbf{U}_{\tau,N}\mathbf{D}_{\tau,N}\mathbf{V}_{\tau,N}^{\top}$ be the SVD of $\mathbf{X}_{\tau,N}$, where $\mathbf{U}_{\tau,N} \in \mathbb{R}^{N\times N}, \mathbf{D}_{\tau,N} \in \mathbb{R}^{N\times d}, \mathbf{V}_{\tau,N} \in \mathbb{R}^{d\times d}$.
    Define $\hat{\mathbf{Q}}_{\tau,N} \coloneqq \frac{1}{N}\mathbf{X}_{\tau,N}^{\top} \mathbf{X}_{\tau,N}$,  and denote $\lambda_{1}^{(N)} \geq \dots \geq \lambda_{d}^{(N)}$ as the eigenvalues of $\hat{\mathbf{Q}}_{\tau,N}$.
    Then $\mathbf{D}_{\tau,N}^{\top}\mathbf{D}_{\tau,N} = 
    N \mathrm{Diag}(\lambda_{1}^{(N)},\dots,\lambda_{d}^{(N)})$.
    
    For ERM, based on the expression of $\hat{\mathbf{W}}_{\tau}^{\mathrm{er}}$, it is apparent that
    \begin{align*}\label{eq:E_W_hat_erm}
    &\mathbb{E}[\hat{\mathbf{W}}_{\tau}^{\mathrm{er}}] 
    = \mathbb{E}[ \hat{\mathbf{Q}}_{\tau,N}] 
    = \mathbf{Q}_{\tau} = \mathbf{I}_d
    = \frac{1}{d}\mathbb{E}[\mathrm{tr}(\hat{\mathbf{W}}_{\tau}^{\mathrm{er}})] \mathbf{I}_d
    \numberthis.
    \end{align*} 
    
  For MAML, using the expression of $\hat{\mathbf{W}}_{\tau}^{\mathrm{ma}}$, we have
    \begin{align*}\label{eq:E_A_hat_ma_tau}
      &\mathbb{E}
      \big[\hat{\mathbf{W}}_{\tau, N}^{\mathrm{ma}}\big]
      = \mathbb{E} \big[ 
      \big(\mathbf{I}-
      {\alpha  }\hat{\mathbf{Q}}_{\tau,N_1}\big) 
      \hat{\mathbf{Q}}_{\tau,N_2}
      \big(\mathbf{I}-
      {\alpha  }\hat{\mathbf{Q}}_{\tau,N_1}\big)\big] 
      =  \mathbb{E} \big[ 
      \big(\mathbf{I}-
      {\alpha  }\hat{\mathbf{Q}}_{\tau,N_1}\big) 
      {\mathbf{Q}}_{\tau}
      \big(\mathbf{I}-
      {\alpha  }\hat{\mathbf{Q}}_{\tau,N_1}\big)\big]  \\
      =& \mathbb{E} \Big[ \mathbf{V}_{\tau,N_1}
      \mathrm{Diag}\Big((1-\alpha \lambda_1^{(N_1)})^2,
      \dots, (1-\alpha \lambda_d^{(N_1)})^2\Big)
      \mathbf{V}_{\tau,N_1}^{\top}\Big] 
    . \numberthis
    \end{align*}
  Then we show that $\mathbb{E}[\hat{\mathbf{W}}_{\tau}^{\mathrm{ma}}] =\frac{1}{d}\mathbb{E}[  
  \mathrm{tr}(\hat{\mathbf{W}}_{\tau}^{\mathrm{ma}})
  ] \mathbf{I}_d$ by the permutation trick.
  We utilize the isotropicity of $\mathbf{X}_{\tau}$. 
  {For notation simplicity, we use $\mathbf{V}_{\tau}$ to replace $\mathbf{V}_{\tau,N}$ in the following discussion since the value of $N$ does not affect the arguments.}
  Recall that $\mathbf{V}_{\tau}$ is uniform on all the orthogonal matrices. Let $\mathbf{P} \in \mathbb{R}^{d \times d}$ be any permutation matrix, then $\mathbf{V}_{\tau} \mathbf{P}$ has the same distribution as $\mathbf{V}_{\tau}$. 
  For this permuted data matrix $\mathbf{V}_{\tau} \mathbf{P}$,
  $\mathbb{E}[  \sum_{i=1}^{d} 
  \lambda_{i}^{(N)} \mathbf{v}_{\tau,i}\mathbf{v}_{\tau,i}^{\top}]
  =\mathbb{E}\big[  
  \sum_{i=1}^{d} \lambda_{i}^{(N)} \mathbf{v}_{\tau, t_{p}(i)} \mathbf{v}_{\tau, t_{p}(i)}^{\top}\big] $ with $t_{p}(i)$ denoting the permutation of the $i$-th element in $\mathbf{P}$.
  
  Summing over all the permutations $\mathbf{P}$ (and there are totally $d !$ instances), we deduce
  \begin{align*}
  d ! \mathbb{E}[\hat{\mathbf{W}}_{\tau}^{\mathrm{ma}}]
  =&\sum_{\text {all } t_{p}} 
  \mathbb{E}\Big[  \sum_{i=1}^{d} (1-\alpha \lambda_i^{(N_1)})^2 \mathbf{v}_{\tau, t_{p}(i)} \mathbf{v}_{\tau, t_{p}(i)}^{\top}\Big] \\
  =&(d-1) ! \mathbb{E}\Big[  \sum_{j=1}^{d} \big(\sum_{i=1}^{d} (1-\alpha \lambda_i^{(N_1)})^2\big) \mathbf{v}_{\tau, j} \mathbf{v}_{\tau, j}^{\top}\Big] \\
  =&(d-1) ! \mathbb{E}\Big[  \mathbf{V}_{\tau} \operatorname{Diag}\big(\sum_{i=1}^{d} (1-\alpha \lambda_i^{(N_1)})^2, \ldots, \sum_{i=1}^{d} (1-\alpha \lambda_i^{(N_1)})^2\big) \mathbf{V}_{\tau}^{\top}\Big] \\
  =& (d-1) ! \mathbb{E}\Big[  \sum_{i=1}^{d} ((1-\alpha \lambda_i^{(N_1)})^2)^2 \mathbf{V}_{\tau} \mathbf{V}_{\tau}^{\top}\Big]
  =(d-1) ! \mathbb{E}[  
  \mathrm{tr}(\hat{\mathbf{W}}_{\tau}^{\mathrm{ma}})] 
  \mathbf{I}_d 
  \end{align*}
  which gives  $\mathbb{E}[\hat{\mathbf{W}}_{\tau}^{\mathrm{ma}}] = \frac{1}{d}\mathbb{E}[\mathrm{tr}(\hat{\mathbf{W}}_{\tau}^{\mathrm{ma}})] \mathbf{I}_d$.

  Following similar arguments, for iMAML, it also holds that
  $\mathbb{E}[\hat{\mathbf{W}}_{\tau}^{\mathrm{im}}] = \frac{1}{d}\mathbb{E}[\mathrm{tr}(\hat{\mathbf{W}}_{\tau}^{\mathrm{im}})] \mathbf{I}_d$.
  And for BaMAML, we use the expression $\hat{\mathbf{W}}_{\tau}^{\mathrm{ba}} = \frac{\gamma s}{1-s} [(\mathbf{I}_{d}+\gamma^{-1}\hat{\mathbf{Q}}_{\tau,N_1})^{-1} - (\mathbf{I}_{d}+(\gamma s)^{-1}\hat{\mathbf{Q}}_{\tau,N})^{-1}]$, which apparently gives $\mathbb{E}[\hat{\mathbf{W}}_{\tau}^{\mathrm{ba}}] = \frac{1}{d}\mathbb{E}[\mathrm{tr}(\hat{\mathbf{W}}_{\tau}^{\mathrm{ba}})] \mathbf{I}_d$ using the permutation trick.
  
  To summarize, we have proved for all four methods ERM, MAML, iMAML and BaMAML, $\mathbb{E}[\hat{\mathbf{W}}_{\tau}^{\mathcal{A}}] = \frac{1}{d}\mathbb{E}[\mathrm{tr}(\hat{\mathbf{W}}_{\tau}^{\mathcal{A}})] \mathbf{I}_d$. Then it is not hard to see that
  \begin{align*}
    \tilde{C}^{\mathcal{A}}_{0} \coloneqq
      \frac{1}{d}\Big\langle\mathbb{E}^{-2}\big[\hat{\mathbf{W}}_{\tau}^{\mathcal{A}}\big], \mathbb{E}\big[(\hat{\mathbf{W}}_{\tau}^{\mathcal{A}})^{2}\big]\Big\rangle
      =\frac{1}{d} \mathbb{E}\Big[  \mathrm{tr}\big((\hat{\mathbf{W}}_{\tau}^{\mathcal{A}})^{2}\big)\Big]
      \Big\{\frac{1}{d}\mathbb{E}\big[ \mathrm{tr}(\hat{\mathbf{W}}_{\tau}^{\mathcal{A}})\big]\Big\}^{-2}.
  \end{align*}
  
  Finally, applying Jensen's inequality, for any PSD matrix $\mathbf{M} \in \mathbb{R}^{d\times d}$,  we have $\frac{1}{d} \mathrm{tr}(\mathbf{M}^2) \geq 
  (\frac{1}{d} \mathrm{tr}(\mathbf{M}))^2$, therefore
  \begin{align*}
    \tilde{C}^{\mathcal{A}}_{0} \coloneqq
      \frac{1}{d}\Big\langle\mathbb{E}^{-2}\big[\hat{\mathbf{W}}_{\tau}^{\mathcal{A}}\big], \mathbb{E}\big[(\hat{\mathbf{W}}_{\tau}^{\mathcal{A}})^{2}\big]\Big\rangle
      =\frac{1}{d} \mathbb{E}\Big[  \mathrm{tr}\big((\hat{\mathbf{W}}_{\tau}^{\mathcal{A}})^{2}\big)\Big]
      \Big\{\frac{1}{d}\mathbb{E}\big[ \mathrm{tr}(\hat{\mathbf{W}}_{\tau}^{\mathcal{A}})\big]\Big\}^{-2}\
      \geq 1.
  \end{align*}
  
  \end{proof}

  \begin{lemma}[Constant in statistical error of MAML]
  \label{lemma:dominate_constant_stats_err_maml}
    Suppose Assumptions~1,3 hold. 
    The dominating constant in the statistical error of MAML is computed by 
  {
  \begin{align*}
    &\tilde{C}^{\mathrm{ma}}_{0} 
    \coloneqq
  \frac{1}{d}\Big\langle\mathbb{E}\big[\hat{\mathbf{W}}_{\tau, N}^{\mathrm{ma}}\big]^{-2}, \mathbb{E}\big[(\hat{\mathbf{W}}_{\tau, N}^{\mathrm{ma}})^{2}\big]\Big\rangle \\
   =&\frac{1}{dN_2} 
    \mathbb{E} \Big[
    \mathrm{tr}^2\left( 
    (\mathbf{I}-\alpha\hat{\mathbf{Q}}_{\tau,N_1})^2
    \right) 
    + (N_2 + 1)
    \mathrm{tr}\left (
    (\mathbf{I}-\alpha\hat{\mathbf{Q}}_{\tau,N_1})^4
    \right)
    \Big] 
     \Big\{\frac{1}{d} 
    \mathbb{E} \big[  \mathrm{tr}
    \big((\mathbf{I}-\alpha \hat{\mathbf{Q}}_{\tau,N_1})^2\big)
    \big] \Big\}^{-2} .
  \end{align*}}

  \end{lemma}
  
  \begin{proof}
  \label{proof:dominate_constant_stats_err_maml}
  We reuse the permutation trick to derive $\mathbb{E}\big[(\hat{\mathbf{W}}_{\tau, N}^{\mathrm{ma}})^2\big]$ below.
  \begin{align}\label{eq:E_A_hat_ma_tau2}
    \mathbb{E}\big[(\hat{\mathbf{W}}_{\tau, N}^{\mathrm{ma}})^2\big]
    &=  \mathbb{E} \big[
    \big(\mathbf{I}-
    {\alpha  }\hat{\mathbf{Q}}_{\tau,N_1}\big) 
    \hat{\mathbf{Q}}_{\tau,N_2}
    \big(\mathbf{I}-
    {\alpha  }\hat{\mathbf{Q}}_{\tau,N_1}\big)^2
    \hat{\mathbf{Q}}_{\tau,N_2}
    \big(\mathbf{I}-
    {\alpha  }\hat{\mathbf{Q}}_{\tau,N_1}\big)
    \big] 
  \end{align}
  We know that $\mathbb{E}\big[(\hat{\mathbf{W}}_{\tau, N}^{\mathrm{ma}})^2\big]$ is equal to a scale factor times $\mathbf{I}_d$, the identity matrix. And the scale factor can be derived below
  {
  \begin{align*}\label{eq:tr_E_A_hat_ma_tau2}
    & \mathrm{tr}\mathbb{E}\big[(\hat{\mathbf{W}}_{\tau, N}^{\mathrm{ma}})^2\big]
    = \mathrm{tr} \mathbb{E} \big[
    \hat{\mathbf{Q}}_{\tau,N_2}
    \big(\mathbf{I}-
    {\alpha  }\hat{\mathbf{Q}}_{\tau,N_1}\big)^2
    \hat{\mathbf{Q}}_{\tau,N_2}
    \big(\mathbf{I}-
    {\alpha  }\hat{\mathbf{Q}}_{\tau,N_1}\big)^2
    \big]  \\
    =&  \frac{1}{N_2^2} \mathrm{tr} \mathbb{E} \big[
    \mathbf{X}_{\tau,N_2}^{\text{val}}
    \big(\mathbf{I}-
    {\alpha  }\hat{\mathbf{Q}}_{\tau,N_1}\big)^2
    \mathbf{X}_{\tau,N_2}^{\text{val}\top}\mathbf{X}_{\tau,N_2}^{\text{val}}
    \big(\mathbf{I}-
    {\alpha  }\hat{\mathbf{Q}}_{\tau,N_1}\big)^2
    \mathbf{X}_{\tau,N_2}^{\text{val}\top}
    \big] \\
    =&  \frac{1}{N_2^2} \mathrm{tr} \mathbb{E} \big[
    \big(\mathbf{X}_{\tau,N_2}^{\text{val}}
    \mathbf{V}_{\tau}^{\text{trn}}
    \mathrm{Diag}\big((1-\alpha  \lambda_1^{(N_1)})^2,
    \dots,
    (1-\alpha  \lambda_d^{(N_1)})^2\big)
    \mathbf{V}_{\tau}^{{\text{trn}}\top}
    \mathbf{X}_{\tau,N_2}^{\text{val}\top}\big)^2
    \big] \\
    =&  \frac{1}{N_2^2} \mathrm{tr} \mathbb{E} \Big[
    \Big(\sum_{i,j=1}^{N_2} 
    \mathrm{Diag}\big((1-\alpha  \lambda_1^{(N_1)})^2,
    \dots, (1-\alpha  \lambda_d^{(N_1)})^2\big)
    \mathbf{v}_j \mathbf{v}_i^{\top} \Big)^2
    \Big] \\
    =&  \frac{1}{N_2^2} \mathbb{E} \Big[
   \sum_{i}^{N_2} \mathrm{tr}
    \Big( \mathrm{Diag}\big((1-\alpha  \lambda_1^{(N_1)})^2,
    \dots, (1-\alpha  \lambda_d^{(N_1)})^2\big)
    \mathbf{v}_i \mathbf{v}_i^{\top} \Big)^2 \\
    &+ \sum_{i\neq j}
    ( \mathrm{Diag}\big((1-\alpha  \lambda_1^{(N_1)})^2,
    \dots, (1-\alpha  \lambda_d^{(N_1)})^2\big)
    \mathbf{v}_j \mathbf{v}_i^{\top} )^2
    \Big]  \\
    =&  \frac{1}{N_2 d} \mathbb{E} \Big[ 
    \mathrm{tr}^2\Big( 
    \mathrm{Diag}\big((1-\alpha  \lambda_1^{(N_1)})^2,
    \dots,
    (1-\alpha  \lambda_d^{(N_1)})^2\big)
    \Big) 
    +2 \Big\|\mathrm{Diag}\Big((1-\alpha  \lambda_1^{(N_1)})^2,\dots,
    (1-\alpha  \lambda_d^{(N_1)})^2\Big) \Big\|_{\mathrm{F}}^2\\
    &+ (N_2 - 1)
    \Big\|
    \mathrm{Diag}\Big((1-\alpha  \lambda_1^{(N_1)})^2,
    \dots,
    (1-\alpha  \lambda_d^{(N_1)})^2\Big)
    \Big\|_{\mathrm{F}}^2
    \Big] .
    \numberthis
  \end{align*}}

  By combining Lemma~\ref{lemma:dominate_constant_stats_err}, \eqref{eq:E_A_hat_ma_tau2}, and \eqref{eq:tr_E_A_hat_ma_tau2},
  we arrive at the following
  \begin{align*}
  & \tilde{C}^{\mathrm{ma}}_{0} 
    \coloneqq
  \frac{1}{d}\Big\langle\mathbb{E}\big[\hat{\mathbf{W}}_{\tau, N}^{\mathrm{ma}}\big]^{-2}, \mathbb{E}\big[(\hat{\mathbf{W}}_{\tau, N}^{\mathrm{ma}})^{2}\big]\Big\rangle \\
  =&\frac{1}{dN_2} 
    \mathbb{E} \Big[
    \Big( \sum_{i=1}^{d}
    (1-\alpha  \lambda_i^{(N_1)})^2
    \Big)^2 
    + (N_2 + 1)
    \Big ( \sum_{i=1}^{d}
    (1-\alpha  \lambda_i^{(N_1)})^4
    \Big) \Big] 
     \Big\{\frac{1}{d} 
    \mathbb{E} \big[  \sum_{i=1}^{d}
    (1-\alpha  \lambda_i^{(N_1)})^2
    \big] \Big\}^{-2} \\
    \label{eq:constant_maml}
    \numberthis
  =&\frac{1}{dN_2} 
    \mathbb{E} \Big[ 
    \mathrm{tr}^2\Big( 
    (\mathbf{I}-\alpha\hat{\mathbf{Q}}_{\tau,N_1})^2
    \Big) 
    + (N_2 + 1)
    \mathrm{tr} \Big(
    (\mathbf{I}-\alpha\hat{\mathbf{Q}}_{\tau,N_1})^4
    \Big) \Big] 
    \Big\{\frac{1}{d} 
    \mathbb{E} \big[  \mathrm{tr} 
    \big((\mathbf{I}-\alpha \hat{\mathbf{Q}}_{\tau,N_1})^2\big)
    \big] \Big\}^{-2} .
  \end{align*}
  
  \end{proof}

  \subsubsection{Bound of statistical errors under Assumptions 1,2}
  \label{app_ssub:loose_bound}
  
  \paragraph{ERM.}
  Following the definition of $\hat{\btheta}_0^{\mathrm{er}}$ in \eqref{eq:theta_0_hat_erm_sln}, we have 
  \begin{align}\label{eq:theta0_hat_diff_theta0_star_erm_mat_form2}
    &\hat{\btheta}_0^{\mathrm{er}} - {\btheta}_0^{\mathrm{er}}
    =
    \Big(\sum_{\tau=1}^{T}\hat{\mathbf{W}}_{\tau}^{\mathrm{er}}\Big)^{-1}
     \Big(\sum_{\tau=1}^{T}
        \hat{\mathbf{W}}_{\tau}^{\mathrm{er}} (\btheta_{\tau}^{\mathrm{gt}} - {\btheta}_0^{\mathrm{er}}) \Big)
     +
     \Big(\sum_{\tau=1}^{T}\hat{\mathbf{W}}_{\tau}^{\mathrm{er}}\Big)^{-1}
     \Big(\sum_{\tau=1}^{T}       
        \frac{1}{N}
        \mathbf{X}^{\text{all}\top}_{\tau, N}\mathbf{e}^{\text{all}}_{\tau,N}\Big).
  \end{align}

  To bound the statistical error, we define 
  \begin{subequations}
  \begin{align}\label{eq:define_z_U_e_erm}
    &\mathbf{z}^{\mathrm{all}}_{e,\mathrm{er}}  \coloneqq 
    \begin{bmatrix}
      \mathbf{e}_1^{\text{all}\top}
      ,\dots,
      \mathbf{e}_T^{\text{all}\top}
    \end{bmatrix}^{\top} \in \mathbb{R}^{NT}, \\
    &\mathbf{U}_{e,\mathrm{er}}  \coloneqq  
    \frac{1}{N}
    \begin{bmatrix}
     \mathbf{X}_{1,N}^{\text{all} \top},\dots, 
     \mathbf{X}_{T,N}^{\text{all} \top}
    \end{bmatrix}^{\top}
    \Big(\sum_{\tau=1}^{T}\hat{\mathbf{W}}_{\tau}^{\mathrm{er}}\Big)^{-1}
    \in \mathbb{R}^{NT\times d}.
  \end{align}
  \end{subequations}
  
  Under Assumptions~1-2, with $\mathbf{U}_{\mathrm{er}}, \mathbf{z}_{\mathrm{er}}$ defined in Lemma~\ref{lemma:bound_stats_err_term1}, the ERM statistical error is given by
  \begin{align*}\label{eq:ERM_stats_err_no_assp3}
    \mathcal{E}_{\mathrm{er}}^{2} (\hat{\btheta}_0^{\mathrm{er}})
    = \|\hat{\btheta}_{0}^{{\mathrm{er}}}-\btheta_{0}^{ {\mathrm{er}}}\|_{\mathbb{E}_{\tau}[\mathbf{W}_{\tau}^{\mathrm{er}}]}^2
    \leq & 
    2\Big( \underbracket{\|
    \mathbf{U}_{\mathrm{er}}^{\top} \mathbf{z}_{\mathrm{er}}\|^2_{\mathbb{E}_{\tau}[\mathbf{W}_{\tau}^{\mathrm{er}}]}}_{I_1^{\mathrm{er}}}
    + \underbracket{\| 
    \mathbf{U}_{e,\mathrm{er}}^{\top} \mathbf{z}_{e,\mathrm{er}}^{\text{all}}\|^2_{\mathbb{E}_{\tau}[\mathbf{W}_{\tau}^{\mathrm{er}}]} }_{I_2} \Big) 
    .\numberthis
  \end{align*}
  We will then bound terms $I_1^{\mathrm{er}}, I_2 $ respectively.
  First the bound for the term $I_1^{\mathrm{er}}$ in \eqref{eq:ERM_stats_err_no_assp3} is provided in Lemma~\ref{lemma:bound_stats_err_term1},
  which states that with probability at least $1-Td^{-10}$, we have
  \begin{align*}\label{eq:stats_err_term_I1_erm_loose}
    I_1^{\mathrm{er}}
    \leq & 
    \frac{R^2 }{T} 
    \Big(\lambda_{\min}(\mathbb{E}[{\mathbf{W}}^{\mathrm{er}}_{\tau}])^{-1}
      \lambda_{\max}(\mathbb{E}[({\mathbf{W}}^{\mathrm{er}}_{\tau})^{2}])
      + \widetilde{\mathcal{O}}(\frac{d}{N})
      + \widetilde{\mathcal{O}}(\frac{1}{\sqrt{d}})  + \widetilde{\mathcal{O}} (\sqrt{\frac{d}{T}})\Big) \\
     & + 
    \Big( \widetilde{\mathcal{O}}(\sqrt{\frac{d}{T}}) 
    \Big) M^2
    . \numberthis
  \end{align*}
  Following similar arguments from Lemma~\ref{lemma:bound_stats_err_term1}, 
  for term $I_2$, first
  \begin{align*}\label{eq:HW_ineq_ze_erm_loose}
    &|\mathbf{z}_{e,\mathrm{er}}^{\text{all}\top}\mathbf{U}_{e,\mathrm{er}} 
    \mathbb{E}_{\tau}[\mathbf{W}_{\tau}^{\mathrm{er}}]
    \mathbf{U}_{e,\mathrm{er}}^{\top} \mathbf{z}_{e,\mathrm{er}}^{\text{all}}
    - \mathbb{E}_{\btheta_{\tau}^{\mathrm{gt}}, \mathbf{e}_{\tau} \mid \hat{\mathbf{W}}_{\tau}^{\mathrm{er}}}[\mathbf{z}_{e,\mathrm{er}}^{\text{all}\top}\mathbf{U}_{e,\mathrm{er}} 
    \mathbb{E}_{\tau}[\mathbf{W}_{\tau}^{\mathrm{er}}]
    \mathbf{U}_{e,\mathrm{er}}^{\top} \mathbf{z}_{e,\mathrm{er}}^{\text{all}}]| 
      \\
    \leq & \widetilde{\mathcal{O}} \Big(
    \big\|\mathbf{U}_{e,\mathrm{er}} 
    \mathbb{E}_{\tau}[\mathbf{W}_{\tau}^{\mathrm{er}}]
    \mathbf{U}_{e,\mathrm{er}}^{\top}\big\|_{\mathrm{F}}\Big) 
    = \widetilde{\mathcal{O}} \Big(\frac{1}{N}\Big\|
    \mathbb{E}_{\tau}[\mathbf{W}_{\tau}^{\mathrm{er}}]
    \Big(\sum_{\tau=1}^{T} \hat{\mathbf{W}}_{\tau}^{\mathrm{er}} \Big)^{-2}
    \Big(\sum_{\tau=1}^{T}  \hat{\mathbf{W}}_{\tau}^{\mathrm{er}}
     \Big)
    \Big\|_{\mathrm{F}} \Big)\\
    =& \widetilde{\mathcal{O}} \Big(\frac{1}{TN}\Big\|
    \mathbb{E}_{\tau}[\mathbf{W}_{\tau}^{\mathrm{er}}]
    \Big(\frac{1}{T}\sum_{\tau=1}^{T} \hat{\mathbf{W}}_{\tau}^{\mathrm{er}} \Big)^{-2}
    \Big(\frac{1}{T}\sum_{\tau=1}^{T} \hat{\mathbf{W}}_{\tau}^{\mathrm{er}}
     \Big)
    \Big\|_{\mathrm{F}} \Big)
    =\widetilde{\mathcal{O}}\Big(\frac{\sqrt{d}}{TN}\Big) 
     \numberthis
  \end{align*}
  and the expectation is given by
  \begin{align*}\label{eq:concentrate_cov_z_erm_loose}
  &\mathbb{E}_{\btheta_{\tau}^{\mathrm{gt}}, \mathbf{e}_{\tau} \mid \hat{\mathbf{W}}_{\tau}^{\mathrm{er}}}[\mathbf{z}_{e,\mathrm{er}}^{\text{all}\top}\mathbf{U}_{e,\mathrm{er}} 
  \mathbb{E}_{\tau}[\mathbf{W}_{\tau}^{\mathrm{er}}]
  \mathbf{U}_{e,\mathrm{er}}^{\top} \mathbf{z}_{e,\mathrm{er}}^{\text{all}}] 
  = \mathrm{tr}\Big(
  \mathbf{U}_{e,\mathrm{er}} \mathbb{E}_{\tau}[\mathbf{W}_{\tau}^{\mathrm{er}}]
  \mathbf{U}_{e,\mathrm{er}}^{\top}
  \Big) \\
  =& \frac{d}{TN}\frac{1}{d}
  \Big\langle \mathbb{E}_{\tau}[\mathbf{W}_{\tau}^{\mathrm{er}}]
  \Big(\frac{1}{T}\sum_{\tau=1}^{T}
  \hat{\mathbf{W}}_{\tau}^{\mathrm{er}} \Big)^{-2} ,
  \Big(\frac{1}{T}\sum_{\tau=1}^{T} \hat{\mathbf{W}}_{\tau}^{\mathrm{er}} \Big)
  \Big\rangle 
  = \frac{d}{TN} 
  \frac{1}{d} \mathrm{tr}\Big(\mathbb{E}_{\tau}[\mathbf{W}_{\tau}^{\mathrm{er}}]
  \Big(\frac{1}{T}\sum_{\tau=1}^{T}
  \hat{\mathbf{W}}_{\tau}^{\mathrm{er}} \Big)^{-1}\Big) \\
  =& \frac{d}{TN}\Big\{
  \frac{1}{d} \mathrm{tr}\Big(\mathbb{E}_{\tau}[\mathbf{W}_{\tau}^{\mathrm{er}}]\Big(\frac{1}{T}\sum_{\tau=1}^{T}
  \hat{\mathbf{W}}_{\tau}^{\mathrm{er}} \Big)^{-1}
  - \mathbf{I}\Big)
  + \underbracket{\frac{1}{d}
  \mathrm{tr}\Big(
  \mathbf{I}_d
  \Big)}_{=C^{\mathrm{er}}_{1} }
  \Big\}
  \numberthis
  \end{align*}
  Therefore combining~\eqref{eq:HW_ineq_ze_erm_loose} and~\eqref{eq:concentrate_cov_z_erm_loose}, we have with probability at least $1-Td^{-10}$
  \begin{align}\label{eq:stats_err_term_II_erm_loose}
      I_2
    &\leq 
    \frac{d}{TN}\Big(
    C^{\mathrm{er}}_{1} +\widetilde{\mathcal{O}}(\sqrt{\frac{d}{T}})+\widetilde{\mathcal{O}}( \frac{1}{\sqrt{d}}) + \widetilde{\mathcal{O}}( \frac{d}{N}) \Big) 
    .
  \end{align}
  
  Finally, 
  by combining the bound for $I_1$ and $I_2$,
  we conclude that with probability at least $1 - Td^{-10}$,
  the statistical error of ERM is bounded by 
  \begin{align*}
  \label{eq:stats_err_final_erm_loose}
    \mathcal{E}^{2}_{\mathrm{er}} (\hat{\btheta}_0^{\mathrm{er}})
    \leq &
  \frac{R^{2} }{T}\Big(2C^{\mathrm{er}}_{0}
  +\widetilde{\mathcal{O}}(\sqrt{\frac{d}{T}})
  +\widetilde{\mathcal{O}}( \frac{1}{\sqrt{d}})
  +\widetilde{\mathcal{O}}( \frac{d}{N})
  \Big)
  \nonumber
  +\frac{d}{TN}\Big(2C^{\mathrm{er}}_{1}
  +\widetilde{\mathcal{O}}(\sqrt{\frac{d}{T}})
  +\widetilde{\mathcal{O}}( \frac{1}{\sqrt{d}})
  +\widetilde{\mathcal{O}}( \frac{d}{N}) \Big) \\
  &
  + \Big( \widetilde{\mathcal{O}}(\sqrt{\frac{d}{T}})  + 
    \widetilde{\mathcal{O}}( \frac{d}{N})
    \Big) M^2
    . \numberthis
  \end{align*}
  
  
  \paragraph{MAML.}
  From  the expressions of $\hat{\btheta}_0^{\mathrm{ma}}$ 
  and $\hat{\mathbf{W}}_{\tau}^{\mathrm{ma}}$, 
  we have
  \begin{align*}\label{eq:theta0_hat_diff_theta0_star_maml_mat_form2}
    &\hat{\btheta}_0^{\mathrm{ma}} - {\btheta}_0^{\mathrm{ma}}
    = 
    \Big(\sum_{\tau=1}^{T}
    \hat{\mathbf{W}}_{\tau}^{\mathrm{ma}}\Big)^{-1} 
    \Big(\sum_{\tau=1}^{T}\hat{\mathbf{W}}_{\tau}^{\mathrm{ma}}
    (\btheta_{\tau}^{\text{gt}} - {\btheta}_0^{\mathrm{ma}}) \Big) \\
    &+ \Big(\sum_{\tau=1}^{T} 
    \hat{\mathbf{W}}_{\tau}^{\mathrm{ma}}\Big)^{-1} 
    \Big(\sum_{\tau=1}^{T} 
    \big(\mathbf{I}- 
    {\alpha}\hat{\mathbf{Q}}_{\tau,N_1}\big)
    \big(
    \frac{1}{N_2}\mathbf{X}_{\tau,N_2}^{\text{val}\top} \mathbf{e}_{\tau,N_2}^{\text{val}} - \frac{\alpha}{N_1}\hat{\mathbf{Q}}_{\tau,N_2}
    \mathbf{X}_{\tau,N_1}^{\text{trn}\top} \mathbf{e}_{\tau,N_1}^{\text{trn}}
    \big)
    \Big)  .
  \numberthis
  \end{align*}
  To bound the statistical error of MAML, we define 
  \begin{subequations}
  \begin{align}\label{eq:define_z_e1_val_maml}
    &\mathbf{z}^{\mathrm{val}}_{e1,\mathrm{ma}} \coloneqq
    \begin{bmatrix}
      \mathbf{e}_1^{\text{val}\top}
      ,\dots,
      \mathbf{e}_T^{\text{val}\top}
    \end{bmatrix}^{\top} \in \mathbb{R}^{N_2T}, \\
    \label{eq:define_U_e1_val_maml}
    &\mathbf{U}_{e1,\mathrm{ma}}^{\top} \coloneqq
    \frac{1}{N_2}
  \Big(\sum_{\tau=1}^{T}\hat{\mathbf{W}}_{\tau, N}^{\mathrm{ma}}\Big)^{-1}
  \Big[
  \big(\mathbf{I}- {\alpha}\hat{\mathbf{Q}}_{1,N_1}\big)
    \mathbf{X}_{1,N_2}^{\text{val} \top},
    \dots, 
    \big(\mathbf{I}- \alpha \hat{\mathbf{Q}}_{T,N_1}\big)
  \mathbf{X}_{T,N_2}^{\text{val} \top}\Big] \in \mathbb{R}^{d \times N_2 T }
  \end{align}
  \end{subequations}
  \begin{subequations}
  \begin{align}\label{eq:define_z_e2_trn_maml}
    &\mathbf{z}^{\mathrm{trn}}_{e2,\mathrm{ma}} \coloneqq
    \begin{bmatrix}
      \mathbf{e}_1^{\text{trn}\top}
      ,\dots,
      \mathbf{e}_T^{\text{trn}\top}
    \end{bmatrix}^{\top} \in \mathbb{R}^{N_1 T}, \\
    \label{eq:define_U_e2_trn_maml}
    &\mathbf{U}_{e2,\mathrm{ma}}^{\top} \coloneqq
    \frac{\alpha}{N_1}
  \Big(\sum_{\tau=1}^{T}\hat{\mathbf{W}}_{\tau, N}^{\mathrm{ma}}\Big)^{-1} 
  \Big[   
  \big(\mathbf{I}- 
    {\alpha}\hat{\mathbf{Q}}_{1,N_1}\big)
    \hat{\mathbf{Q}}_{1,N_2}
    \mathbf{X}_{1,N_1}^{\text{trn} \top},\dots, 
    \big(\mathbf{I}- 
    {\alpha}\hat{\mathbf{Q}}_{T,N_1}\big)
  \hat{\mathbf{Q}}_{T,N_2} \mathbf{X}_{T,N_1}^{\text{trn} \top}\Big] \in \mathbb{R}^{d \times N_1 T}.
  \end{align}
  \end{subequations}

  Under Assumptions~1-2, with $\mathbf{U}_{\mathrm{ma}}, \mathbf{z}_{\mathrm{ma}}$ defined in Lemma~\ref{lemma:bound_stats_err_term1}, the MAML statistical error is given by
  \begin{align}\label{eq:MAML_stats_err_no_assp3}
    &\mathcal{E}_{\mathrm{ma}}^{2} (\hat{\btheta}_0^{\mathrm{ma}})
    \leq 2\Big( 
    \underbracket{
    \|\mathbf{U}_{\mathrm{ma}}^{\top} \mathbf{z}_{\mathrm{ma}}\|^2_{\mathbb{E}_{\tau}[\mathbf{W}_{\tau}^{\mathrm{ma}}]}}_{I_1^{\mathrm{ma}}}
    + \underbracket{
    \|\mathbf{U}_{e1,\mathrm{ma}}^{\top} \mathbf{z}_{e1,\mathrm{ma}}^{\text{val}}\|^2_{\mathbb{E}_{\tau}[\mathbf{W}_{\tau}^{\mathrm{ma}}]}}_{I_2}
    + \underbracket{
    \|\mathbf{U}_{e2,\mathrm{ma}}^{\top} \mathbf{z}_{e2,\mathrm{ma}}^{\text{trn}}\|^2_{\mathbb{E}_{\tau}[\mathbf{W}_{\tau}^{\mathrm{ma}}]}}_{I_3} \Big)
    .
  \end{align}
  We will then bound terms $I_1^{\mathrm{ma}}, I_2$-$I_6$ respectively.
  First the bound for the term $I_1^{\mathrm{ma}}$ is provided in Lemma~\ref{lemma:bound_stats_err_term1},
  which states that with probability at least $1-Td^{-10}$, we have
  \begin{align*}\label{eq:stats_err_term_I1_maml_loose}
    I_1^{\mathrm{ma}}
    \leq &
    \frac{R^2 }{T} 
    \Big(\lambda_{\min}(\mathbb{E}[{\mathbf{W}}^{\mathrm{ma}}_{\tau}])^{-1}
      \lambda_{\max}(\mathbb{E}[({\mathbf{W}}^{\mathrm{ma}}_{\tau})^{2}])
      + \widetilde{\mathcal{O}}( \frac{d}{N})
      + \widetilde{\mathcal{O}}(\frac{1}{\sqrt{d}})  + \widetilde{\mathcal{O}} (\sqrt{\frac{d}{T}})\Big) \\
     & + 
    \Big( \widetilde{\mathcal{O}}(\sqrt{\frac{d}{T}})  + 
    \widetilde{\mathcal{O}}(\frac{d}{N})
    \Big) M^2
    . \numberthis
  \end{align*}
  
  Following similar arguments as \eqref{eq:HW_ineq_ze_erm_loose},  from Lemma~\ref{lemma:bound_stats_err_term1}, 
  for term $I_2$, first
  {
  \begin{align}
    &\Big| \|\mathbf{U}_{e1,\mathrm{ma}}^{\top} \mathbf{z}_{e1,\mathrm{ma}}^{\text{val}}\|^2_{\mathbb{E}_{\tau}[\mathbf{W}_{\tau}^{\mathrm{ma}}]} 
    - \mathbb{E}_{\btheta_{\tau}^{\mathrm{gt}}, \mathbf{e}_{\tau} \mid \hat{\mathbf{W}}_{\tau, N}^{\mathrm{ma}}}
    [\|\mathbf{U}_{e1,\mathrm{ma}}^{\top} \mathbf{z}_{e1,\mathrm{ma}}^{\text{val}}\|^2_{\mathbb{E}_{\tau}[\mathbf{W}_{\tau}^{\mathrm{ma}}]} ] \Big| 
    =\widetilde{\mathcal{O}}\Big(\frac{\sqrt{d}}{TN_2} \Big) 
    \label{eq:HW_ineq_ze1_maml_loose}
  \end{align}}
  and the expectation is given by
  \begin{align*}\label{eq:concentrate_cov_z_maml_loose}
  &\mathbb{E}_{\btheta_{\tau}^{\mathrm{gt}}, \mathbf{e}_{\tau} \mid \hat{\mathbf{W}}_{\tau, N}^{\mathrm{ma}}}
    [\mathbf{z}_{e1,\mathrm{ma}}^{\text{val}\top}
    \mathbf{U}_{e1,\mathrm{ma}} 
    \mathbb{E}_{\tau}[\mathbf{W}_{\tau}^{\mathrm{ma}}]\mathbf{U}_{e1,\mathrm{ma}}^{\top}
    \mathbb{E}_{\tau}[\mathbf{W}_{\tau}^{\mathrm{ma}}] \mathbf{z}_{e1,\mathrm{ma}}^{\text{val}}]
  = \mathrm{tr}\Big(
  \mathbf{U}_{e1,\mathrm{ma}} 
  \mathbb{E}_{\tau}[\mathbf{W}_{\tau}^{\mathrm{ma}}]
  \mathbf{U}_{e1,\mathrm{ma}}^{\top}
  \Big) \\
  = & \frac{d}{TN_2} 
  \frac{1}{d} \mathrm{tr}\Big(
  \mathbb{E}_{\tau}[\mathbf{W}_{\tau}^{\mathrm{ma}}]
  \big(\frac{1}{T}\sum_{\tau=1}^{T}
  \hat{\mathbf{W}}_{\tau, N}^{\mathrm{ma}} \big)^{-1}\Big) \\
  = &  \frac{d}{TN_2}\Big\{ 
  \frac{1}{d} \mathrm{tr}\Big(
  \mathbb{E}_{\tau}[\mathbf{W}_{\tau}^{\mathrm{ma}}]
  \big(\frac{1}{T}\sum_{\tau=1}^{T}
  \hat{\mathbf{W}}_{\tau, N}^{\mathrm{ma}} \big)^{-1}
  - \mathbb{E}^{-1}[\hat{\mathbf{W}}_{\tau, N}^{\mathrm{ma}}]\Big)
  + \underbracket{\frac{1}{d}
  \mathrm{tr}\Big( 
  \mathbf{I}_d
  \Big)}_{=C^{\mathrm{ma}}_{1,1} }
  \Big\}.
  \numberthis
  \end{align*}
  Therefore combining~\eqref{eq:HW_ineq_ze_erm_loose} and~\eqref{eq:concentrate_cov_z_erm_loose}, we have
  \begin{equation}\label{eq:stats_err_term_II_maml_loose}
      I_2
    \leq 
    \frac{d}{TN_2}\Big(
    C^{\mathrm{ma}}_{1,1} 
    +\widetilde{\mathcal{O}}(\sqrt{{\frac{d}{T}}})+\widetilde{\mathcal{O}}( \frac{1}{\sqrt{d}})\Big).
  \end{equation}
  
  Following similar arguments from \eqref{eq:HW_ineq_ze1_maml_loose}-\eqref{eq:stats_err_term_II_maml_loose},
  $I_3$ satisfies
    \begin{align}\label{eq:bound_MSE_e2_maml_loose}
  I_3
    &\leq \frac{d}{TN_1}\Big(
    C^{\mathrm{ma}}_{1,2} +\widetilde{\mathcal{O}}(\sqrt{\frac{d}{T}})+\widetilde{\mathcal{O}}(\frac{1}{\sqrt{d }})\Big)
  \end{align}
  with $C^{\mathrm{ma}}_{1,2}$ defined by
  \begin{align*}
    C^{\mathrm{ma}}_{1,2}
    &\coloneqq 
    \frac{1}{d} \Big \langle 
  \mathbb{E}^{-1}[{\mathbf{W}}_{\tau}^{\mathrm{ma}}], 
  \alpha^{2}\mathbb{E}(\mathbf{I} - \alpha {\mathbf{Q}}_{\tau})
  {\mathbf{Q}}_{\tau}{\mathbf{Q}}_{\tau}{\mathbf{Q}}_{\tau}(\mathbf{I} - \alpha {\mathbf{Q}}_{\tau})
  \Big \rangle. 
  \end{align*}

  Finally, 
  define $C_1^{\rm ma} \coloneqq (1-s)^{-1} C_{1,1}^{\rm ma} + s^{-1} C_{1,2}^{\rm ma}$.
  By combining the bound of $I_1$-$I_3$,
  we conclude that with probability at least $1 - Td^{-10}$,
  the statistical error of MAML is bounded above by 
  \begin{align*}
  \label{eq:stats_err_final_maml_loose}
  \mathcal{E}^{2}_{\mathrm{ma}} (\hat{\btheta}_0^{\mathrm{ma}})
  \leq & 
  \frac{R^{2} }{T}\Big(2C^{\mathrm{ma}}_{0}
  +\widetilde{\mathcal{O}}(\sqrt{\frac{d}{T}})
  +\widetilde{\mathcal{O}}( \frac{1}{\sqrt{d}})
  +\widetilde{\mathcal{O}}( \frac{d}{N})
  \Big)
  +\frac{d}{TN}\Big(2C^{\mathrm{ma}}_{1}
    +\widetilde{\mathcal{O}}(\sqrt{\frac{d}{T}})+\widetilde{\mathcal{O}}( \frac{1}{\sqrt{d}})
    \Big) \\
  &
  +\Big( \widetilde{\mathcal{O}}(\sqrt{\frac{d}{T}})  + 
  \widetilde{\mathcal{O}}(\frac{d}{N})
  \Big) M^2
  . \numberthis
  \end{align*}

  \paragraph{iMAML.}
  
    Based on $\hat{\btheta}_0^{\mathrm{im}}$ in \eqref{eq:theta_0_hat_biMAML_sln_app}, and
    $\hat{\mathbf{W}}_{\tau}^{\mathrm{im}}$ in~\eqref{eq:W_hat_bi},
    we have
    \begin{align}\label{eq:theta0_hat_diff_theta0_star_bimaml_mat_form2}
      \hat{\btheta}_0^{\mathrm{im}} - {\btheta}_0^{\mathrm{im}}
      =& 
      \Big(\sum_{\tau=1}^{T}\hat{\mathbf{W}}_{\tau}^{\mathrm{im}}\Big)^{-1}
       \Big(\sum_{\tau=1}^{T}
       \hat{\mathbf{W}}_{\tau}^{\mathrm{im}} (\btheta_{\tau}^{\mathrm{gt}} - {\btheta}_0^{\mathrm{im}}) \Big)\\
       \nonumber
       &+
       \Big(\sum_{\tau=1}^{T}\hat{\mathbf{W}}_{\tau}^{\mathrm{im}}\Big)^{-1}
       \Big(\sum_{\tau=1}^{T}\gamma {\Sigma}_{\btheta_{\tau}}
       \frac{1}{N_2}\mathbf{X}^{\text{val}\top}_{\tau}\mathbf{e}_{\tau}^{\text{val}}
       -\gamma ^{-1}\hat{\mathbf{W}}_{\tau}^{\mathrm{im}}\frac{1}{N_1}  \mathbf{X}_{\tau}^{\text{trn} \top} \mathbf{e}_{\tau,N}^{}\Big).
    \end{align}
    To bound the iMAML statistical error, define
    \begin{align}\label{eq:define_z_e1_val_bimaml}
      &\mathbf{z}^{\mathrm{val}}_{e1,\mathrm{im}} \coloneqq
      \begin{bmatrix}
        \mathbf{e}_1^{\text{val}\top}
        ,\dots,
        \mathbf{e}_T^{\text{val}\top}
      \end{bmatrix}^{\top} \in \mathbb{R}^{N_2T}, \\
      \label{eq:define_U_e1_val_bimaml}
      &\mathbf{U}_{e1,\mathrm{im}}^{\top} \coloneqq
      \frac{1}{N_2}
    \left(\sum_{\tau=1}^{T}\hat{\mathbf{W}}_{\tau, N}^{\mathrm{im}}\right)^{-1}
    \Big[\gamma {\Sigma}_{\btheta_{1},N_1}
    \mathbf{X}_{1,N_2}^{\text{val} \top},\dots, 
    \gamma {\Sigma}_{\btheta_{T},N_1} 
    \mathbf{X}_{T,N_2}^{\text{val} \top}\Big] \in \mathbb{R}^{ d  \times N_2 T}
    \end{align}
    where  ${\Sigma}_{\btheta_{\tau}, N_1} = (\frac{1}{N_1} \mathbf{X}_{\tau, N_1}^{\mathrm{trn}\top} \mathbf{X}_{\tau, N_1}^{\mathrm{trn}}+\gamma \mathbf{I})^{-1}$,
    and
    \begin{align}\label{eq:define_z_e2_trn_bimaml}
      &\mathbf{z}^{\mathrm{trn}}_{e2,\mathrm{im}} \coloneqq
      \begin{bmatrix}
        \mathbf{e}_1^{\text{trn}\top}
        ,\dots,
        \mathbf{e}_T^{\text{trn}\top}
      \end{bmatrix}^{\top} \in \mathbb{R}^{N_1 T}, \\
      \label{eq:define_U_e2_trn_bimaml}
      &\mathbf{U}_{e2,\mathrm{im}}^{\top} \coloneqq
      \frac{1}{N_1}\left(\sum_{\tau=1}^{T}\hat{\mathbf{W}}_{\tau}^{\mathrm{im}}\right)^{-1}[\gamma ^{-1}\hat{\mathbf{W}}^{\mathrm{im}}_{1,N}  \mathbf{X}_{1}^{\text{trn} \top},\dots, 
    \gamma ^{-1}\hat{\mathbf{W}}^{\mathrm{im}}_{T,N}  \mathbf{X}_{T}^{\text{trn} \top}]
     \in \mathbb{R}^{d \times N_1 T}.
    \end{align}

  Following similar arguments as the derivation for MAML in \eqref{eq:MAML_stats_err_no_assp3}-\eqref{eq:stats_err_final_maml_loose}, 
  with probability at least $1-Td^{-10}$, 
  we have
  \begin{align*}
  \mathcal{E}_{\mathrm{im}}^2 (\hat{\btheta}_0^{\mathrm{im}})
  \leq &
    \frac{R^2 }{T} 
    \Big(2 C_0^{\rm im}
      + \widetilde{\mathcal{O}}( \frac{d}{N})
      + \widetilde{\mathcal{O}}(\frac{1}{\sqrt{d}})  + \widetilde{\mathcal{O}} (\sqrt{\frac{d}{T}})\Big) 
  +\frac{d}{TN}\Big(2C^{\mathrm{im}}_{1}
    +\widetilde{\mathcal{O}}(\sqrt{\frac{d}{T}})+\widetilde{\mathcal{O}}( \frac{1}{\sqrt{d}})
    \Big)  \\
    &+ 
    \Big(\widetilde{\mathcal{O}}(\sqrt{\frac{d}{T}}) + 
    \widetilde{\mathcal{O}}(\frac{d}{N})
    \Big) M^2 
    .\numberthis
  \end{align*}
  with  $C_1^{\rm im} \coloneqq (1-s)^{-1} C_{1,1}^{\rm im} + s^{-1} C_{1,2}^{\rm im}$, and
  \begin{align*}
    C^{\mathrm{im}}_{1,1}
    &\coloneqq 1,  \\
    C^{\mathrm{im}}_{1,2}
    &\coloneqq 
    \frac{1}{d} \Big\langle
    \mathbb{E}[{\mathbf{W}}_{\tau}^{\mathrm{im}}]^{-1},
    \frac{1}{T}\sum_{\tau=1}^{T}
    (\gamma   )^{-2} 
    \mathbb{E}[\hat{\mathbf{W}}_{\tau}^{\mathrm{im}}]
      \mathbb{E}[{\Sigma}_{\btheta_{\tau}}^{-1}]
      \mathbb{E}[\hat{\mathbf{W}}_{\tau}^{\mathrm{im}}]
    -\mathbb{E}[\gamma ^{-1}(\hat{\mathbf{W}}_{\tau}^{\mathrm{im}})^2]
    \Big\rangle .
  \numberthis
  \end{align*}
  


  \paragraph{BaMAML.}
  Based on the expressions of $\hat{\btheta}_0^{\mathrm{ba}}$ and $\hat{\mathbf{W}}_{\tau}^{\mathrm{ba}}$, we have
  \begin{align*}\label{eq:theta0_hat_diff_theta0_star_bamaml_mat_form2}
    \hat{\btheta}_0^{\mathrm{ba}} - {\btheta}_0^{\mathrm{ba}}
    =& \Big(\sum_{\tau=1}^{T}\hat{\mathbf{W}}_{\tau, N}^{\mathrm{ba}}\Big)^{-1}
     \Big(\sum_{\tau=1}^{T}
     \hat{\mathbf{W}}_{\tau, N}^{\mathrm{ba}} 
     (\btheta_{\tau}^{\mathrm{gt}} - \btheta_0^{\mathrm{ba}}) \Big) \\
     &+
     \Big(\sum_{\tau=1}^{T}\hat{\mathbf{W}}_{\tau, N}^{\mathrm{ba}}\Big)^{-1}
     \Big(\sum_{\tau=1}^{T}\mathbf{X}_{\tau,N}^{\text{all}\top}\Sigma_{y,N}^{-1} \mathbf{e}_{\tau,N}^{\text{all}} - \mathbf{X}_{\tau,N_1}^{\text{trn}\top}\Sigma_{y,N_1}^{-1} \mathbf{e}_{\tau,N_1}^{\text{trn}}
     \Big) \numberthis
  \end{align*}
  where $\Sigma_{y,N}^{-1} = (\mathbf{I}_N + \gamma_b^{-1}\mathbf{X}_{\tau,N}\mathbf{X}^{\top}_{\tau,N})^{-1} $.

  To bound the statistical error of BaMAML, define
  \begin{align}\label{eq:define_z_e_all_bamaml}
    &\mathbf{z}^{\mathrm{all}}_{e,\mathrm{ba}} \coloneqq
    \begin{bmatrix}
      \mathbf{e}_1^{\text{all}\top}
      ,\dots,
      \mathbf{e}_T^{\text{all}\top}
    \end{bmatrix}^{\top} \in \mathbb{R}^{NT}, \\
    \label{eq:define_U_e_all_bamaml}
    &\mathbf{U}_{e1,\mathrm{ba}}^{\top} \coloneqq
    \frac{1}{N_2}
  \Big(\sum_{\tau=1}^{T}\hat{\mathbf{W}}_{\tau, N}^{\mathrm{ba}}\Big)^{-1}
  \Big[ 
  \mathbf{X}_{1,N}^{ \top}{\Sigma}_{y, N}^{-1},
   \dots, 
  \mathbf{X}_{T,N}^{ \top}{\Sigma}_{y, N}^{-1} 
  \Big] \in \mathbb{R}^{ d  \times N T}, \\
  &\mathbf{z}^{\mathrm{trn}}_{e2,\mathrm{ba}} \coloneqq
    \begin{bmatrix}
      \mathbf{e}_1^{\text{trn}\top}
      ,\dots,
      \mathbf{e}_T^{\text{trn}\top}
    \end{bmatrix}^{\top} \in \mathbb{R}^{N_1 T}, \\
  \label{eq:define_U_e_trn_bamaml}
  &\mathbf{U}_{e2,\mathrm{ba}}^{\top} \coloneqq
    \frac{1}{N_2}
  \Big(\sum_{\tau=1}^{T}\hat{\mathbf{W}}_{\tau, N}^{\mathrm{ba}}\Big)^{-1}
  \Big[ \mathbf{X}_{1,N_1}^{ \top} {\Sigma}_{y, N_1}^{-1}, \dots,   
  \mathbf{X}_{T,N_1}^{ \top} {\Sigma}_{y, N_1}^{-1} \Big] \in \mathbb{R}^{ d  \times N T}.
  \end{align}
  
  Following similar arguments as the derivation for ERM in \eqref{eq:ERM_stats_err_no_assp3}-\eqref{eq:stats_err_final_erm_loose}, 
  with probability at least $1-Td^{-10}$, 
  we have
  \begin{align*}
    \mathcal{E}_{\mathrm{ba}}^2 (\hat{\btheta}_0^{\mathrm{ba}})
  \leq &
  \frac{R^2 }{T} 
    \Big(2C_0^{\rm ba}
    + \widetilde{\mathcal{O}}( \frac{d}{N})
    + \widetilde{\mathcal{O}}(\frac{1}{\sqrt{d}})  + \widetilde{\mathcal{O}} (\sqrt{\frac{d}{T}})\Big)
    +\frac{d}{TN}\Big(2C^{\mathrm{ba}}_{1}
  +\widetilde{\mathcal{O}}(\sqrt{\frac{d}{T}})
  +\widetilde{\mathcal{O}}( \frac{1}{\sqrt{d}})\Big) \\
    &+  
    \Big( \widetilde{\mathcal{O}}(\sqrt{\frac{d}{T}})  + 
    \widetilde{\mathcal{O}}({\frac{d}{N}})
    \Big) M^2
  . \numberthis
  \end{align*}
  with 
  $C^{\mathrm{ba}}_{1}$ derived by
  \begin{align*}
  C^{\mathrm{ba}}_{1} = 1
  \geq & \frac{1}{d}\Big\langle
  \mathbb{E}[{\mathbf{W}}_{\tau, N}^{\mathrm{ba}}]^{-1},
  (1-s)^{-1} \mathbb{E}[(\mathbf{I}_d + (\gamma s)^{-1}{\mathbf{Q}}_{\tau})^{-1} {\mathbf{Q}}_{\tau}
  (\mathbf{I}_d + (\gamma s)^{-1}{\mathbf{Q}}_{\tau})^{-1} \\
  &- s(\mathbf{I}_d + \gamma ^{-1}{\mathbf{Q}}_{\tau})^{-1} {\mathbf{Q}}_{\tau}
  (\mathbf{I}_d + \gamma ^{-1}{\mathbf{Q}}_{\tau})^{-1}]
  \Big\rangle. 
  \numberthis
  \end{align*}

  \subsubsection{Bound of statistical error under Assumptions 1,3}
  \label{app_ssub:tight_bound}

  \paragraph{ERM.}
  Then under Assumptions~1,3, with $\mathbf{U}_{\mathrm{er}}, \mathbf{z}_{\mathrm{er}}$ defined in Lemma~\ref{lemma:bound_stats_err_term1}, 
  we have
  \begin{align}
    \label{eq:stats_err_mat_form_erm}
    \mathcal{E}^{2}_{\mathrm{er}} (\hat{\btheta}_0^{\mathrm{er}})
    = w_{\mathrm{er}}\|\hat{\btheta}_{0}^{{\mathrm{er}}}-\btheta_{0}^{ {\mathrm{er}}}\|_2^2
    =& w_{\mathrm{er}}
    (\underbracket{\|\mathbf{U}_{\mathrm{er}}^{\top} \mathbf{z}_{\mathrm{er}}\|^2 }_{I_1^{\mathrm{er}}}
      + \underbracket{\| \mathbf{U}_{e,\mathrm{er}}^{\top} \mathbf{z}_{e,\mathrm{er}}^{\text{all}}\|^2 }_{I_2}
      + 2\underbracket{\mathbf{z}_{\mathrm{er}}^{\top}
        \mathbf{U}_{\mathrm{er}} \mathbf{U}_{e,\mathrm{er}}^{\top}
        \mathbf{z}_{e,\mathrm{er}}^{\text{all}}}_{I_3}).
  \end{align}
  We will then bound the redefined terms $I_1^{\mathrm{er}}, I_2, I_3$ in \eqref{eq:stats_err_mat_form_erm} respectively.
  The bound for the term $I_1^{\mathrm{er}}$ in \eqref{eq:stats_err_mat_form_erm} is provided in Lemma~\ref{lemma:bound_stats_err_term1},
  which states that with probability at least $1-Td^{-10}$, the following holds
  \begin{equation}\label{eq:stats_err_term_I1_erm}
    I_1=
    \mathbf{z}_{\mathrm{er}}^{\top}\mathbf{U}_{\mathrm{er}} \mathbf{U}_{\mathrm{er}}^{\top} \mathbf{z}_{\mathrm{er}}
    \leq \frac{R^{2}}{T}\Big(\tilde{C}^{\mathrm{er}}_{0} + \widetilde{\mathcal{O}}(\frac{1}{\sqrt{d}}) + \widetilde{\mathcal{O}}(\sqrt{\frac{d}{T}})\Big). 
  \end{equation}
  
  For $\tilde{C}^{\mathrm{er}}_{0}$, 
  since $\mathbb{E}[\hat{\mathbf{W}}_{\tau}^{\mathrm{er}}] = \mathbf{Q}_{\tau}$,
  and by Lemma~\ref{lemma:dominate_constant_stats_err}, we have
  \begin{align}\label{eq:term1_const_erm}
    \tilde{C}^{\mathrm{er}}_{0}
    =&\frac{1}{d} \mathbb{E}\Big[  \mathrm{tr}(\hat{\mathbf{Q}}_{\tau, N}^{2})\Big]
    \Big(\frac{1}{d}\mathbb{E}\big[ \mathrm{tr}(\mathbf{Q}_{\tau})\big]\Big)^{-2}
  \end{align}

  Following similar arguments from Lemma~\ref{lemma:bound_stats_err_term1}, 
  for term $I_2$, first
  \begin{align*}\label{eq:HW_ineq_ze_erm}
    &|\mathbf{z}_{e,\mathrm{er}}^{\text{all}\top}\mathbf{U}_{e,\mathrm{er}} \mathbf{U}_{e,\mathrm{er}}^{\top} \mathbf{z}_{e,\mathrm{er}}^{\text{all}}
    - \mathbb{E}_{\btheta_{\tau}^{\mathrm{gt}}, \mathbf{e}_{\tau} \mid \hat{\mathbf{W}}_{\tau}^{\mathrm{er}}}[\mathbf{z}_{e,\mathrm{er}}^{\text{all}\top}\mathbf{U}_{e,\mathrm{er}} \mathbf{U}_{e,\mathrm{er}}^{\top} \mathbf{z}_{e,\mathrm{er}}^{\text{all}}]| 
    =\widetilde{\mathcal{O}}\Big(\frac{\sqrt{d}}{TN}\Big) 
     \numberthis
  \end{align*}

  \begin{align*}\label{eq:concentrate_cov_z_erm}
  &\mathbb{E}_{\btheta_{\tau}^{\mathrm{gt}}, \mathbf{e}_{\tau} \mid \hat{\mathbf{W}}_{\tau}^{\mathrm{er}}}[\mathbf{z}_{e,\mathrm{er}}^{\text{all}\top}\mathbf{U}_{e,\mathrm{er}} \mathbf{U}_{e,\mathrm{er}}^{\top} \mathbf{z}_{e,\mathrm{er}}^{\text{all}}] 
  = \mathrm{tr}\Big(
  \mathbf{U}_{e,\mathrm{er}} \mathbf{U}_{e,\mathrm{er}}^{\top}
  \Big)
  = \frac{d}{TN}\frac{1}{d}
  \Big\langle
  \Big(\frac{1}{T}\sum_{\tau=1}^{T}
  \hat{\mathbf{W}}_{\tau}^{\mathrm{er}} \Big)^{-2} ,
  \Big(\frac{1}{T}\sum_{\tau=1}^{T} \hat{\mathbf{W}}_{\tau}^{\mathrm{er}} \Big)
  \Big\rangle \\
  =& \frac{d}{TN} 
  \frac{1}{d} \mathrm{tr}\Big(\Big(\frac{1}{T}\sum_{\tau=1}^{T}
  \hat{\mathbf{W}}_{\tau}^{\mathrm{er}} \Big)^{-1}\Big)
  = \frac{d}{TN}\Big\{
  \frac{1}{d} \mathrm{tr}\Big(\Big(\frac{1}{T}\sum_{\tau=1}^{T}
  \hat{\mathbf{W}}_{\tau}^{\mathrm{er}} \Big)^{-1}
  - \mathbb{E}^{-1}[\hat{\mathbf{W}}_{\tau}^{\mathrm{er}}]\Big)
  + \underbracket{\frac{1}{d}
  \mathrm{tr}\Big(
  \mathbb{E}^{-1}[\hat{\mathbf{W}}_{\tau}^{\mathrm{er}}]
  \Big)}_{=\tilde{C}^{\mathrm{er}}_1 }
  \Big\}
  \numberthis
  \end{align*}
  
  Therefore combining~\eqref{eq:HW_ineq_ze_erm} and~\eqref{eq:concentrate_cov_z_erm}, we have
  \begin{equation}\label{eq:stats_err_term_II_erm}
      I_2=
    \mathbf{z}_{e,\mathrm{er}}^{\text{all}\top}
    \mathbf{U}_{e,\mathrm{er}} \mathbf{U}_{e,\mathrm{er}}^{\top}
    \mathbf{z}_{e,\mathrm{er}}^{\text{all}}
    \leq 
    \frac{d}{TN}\Big(
    \tilde{C}^{\mathrm{er}}_1 +\widetilde{\mathcal{O}}(\sqrt{{\frac{d}{T}}})+\widetilde{\mathcal{O}}( \frac{1}{\sqrt{d}})\Big).
  \end{equation}

  For term $I_3$, 
  note that $\mathbb{E}_{\btheta_{\tau}^{\mathrm{gt}}, \mathbf{e}_{\tau} \mid \hat{\mathbf{W}}_{\tau}^{\mathrm{er}}}[\mathbf{z}_{\mathrm{er}}^{\top}
    \mathbf{U}_{\mathrm{er}} \mathbf{U}_{e,\mathrm{er}}^{\top}
    \mathbf{z}_{e,\mathrm{er}}^{\text{all}}] = 0$.
  Following a similar argument from~\eqref{eq:HW_ineq_ze_erm} to~\eqref{eq:stats_err_term_II_erm}, 
  with probability at least $1-\delta$,
    $ \label{eq:stats_err_term_I3_erm}
    |I_3| = | \mathbf{z}_{\mathrm{er}}^{\top}
      \mathbf{U}_{\mathrm{er}} \mathbf{U}_{e,\mathrm{er}}^{\top}
      \mathbf{z}_{e,\mathrm{er}}^{\text{all}} |
      \leq 
      \widetilde{\mathcal{O}} (\frac{R}{T\sqrt{N}}).$
  Finally, 
  by combining \eqref{eq:stats_err_term_I1_erm}-\eqref{eq:stats_err_term_II_erm}, 
  and applying the weight $w_\mathrm{er}$,
  we conclude that with probability at least $1 - Td^{-10}$,
  the statistical error of ERM is bounded by 
  \begin{align*}
  \label{eq:stats_err_final_erm}
    \mathcal{E}^{2}_{\mathrm{er}} (\hat{\btheta}_0^{\mathrm{er}})
    =& w_{\mathrm{er}} \|\hat{\btheta}_0^{\mathrm{er}}(\gamma ) - \btheta_0^{\mathrm{er}}(\gamma )\|_2^2
    = 
  \frac{R^{2}}{T}\Big(w_{\mathrm{er}}\tilde{C}^{\mathrm{er}}_{0}
  +\widetilde{\mathcal{O}}(\sqrt{\frac{d}{T}})
  +\widetilde{\mathcal{O}}( \frac{1}{\sqrt{d}})\Big)\\
  \nonumber
  &+\frac{d}{TN}\Big(w_{\mathrm{er}}\tilde{C}^{\mathrm{er}}_1
  +\widetilde{\mathcal{O}}(\sqrt{\frac{d}{T}})
  +\widetilde{\mathcal{O}}( \frac{1}{\sqrt{d}})\Big) 
  + \widetilde{\mathcal{O}}\Big(\frac{R}{T\sqrt{N}} \Big). 
  \numberthis
  \end{align*}

  \paragraph{MAML.}
  Under Assumptions~1,3 and $\mathbf{U}_{\mathrm{ma}}, \mathbf{z}_{\mathrm{ma}}$ defined in Lemma~\ref{lemma:bound_stats_err_term1}, 
  we have
  \begin{align}
    \label{eq:stats_err_mat_form_maml}
    &\mathcal{E}^2_{\mathrm{ma}} = 
    w_{\mathrm{ma}}\|\hat{\btheta}_{0}^{{\mathrm{ma}}}-\btheta_{0}^{ {\mathrm{ma}}}\|_2^2
    = w_{\mathrm{ma}}
    (\underbracket{\| \mathbf{U}_{\mathrm{ma}}^{\top} \mathbf{z}_{\mathrm{ma}}\|^2 }_{I_1}
      + \underbracket{\| \mathbf{U}_{e1,\mathrm{ma}}^{\top} \mathbf{z}_{e1,\mathrm{ma}}^{\text{val}}\|^2 }_{I_2}
      + \underbracket{\| \mathbf{U}_{e2,\mathrm{ma}}^{\top} \mathbf{z}_{e2,\mathrm{ma}}^{\text{trn}}\|^2 }_{I_3} 
      \nonumber \\
      &+ 2\underbracket{\mathbf{z}_{\mathrm{ma}}^{\top}
        \mathbf{U}_{\mathrm{ma}} \mathbf{U}_{e1,\mathrm{ma}}^{\top}
        \mathbf{z}_{e1,\mathrm{ma}}^{\text{val}}}_{I_4}
      - 2\underbracket{\mathbf{z}_{\mathrm{ma}}^{\top}
        \mathbf{U}_{\mathrm{ma}} \mathbf{U}_{e2,\mathrm{ma}}^{\top}
        \mathbf{z}_{e2,\mathrm{ma}}^{\text{trn}}}_{I_5}
      - 2\underbracket{\mathbf{z}_{e1,\mathrm{ma}}^{\text{val}\top}
        \mathbf{U}_{e1,\mathrm{ma}} \mathbf{U}_{e2,\mathrm{ma}}^{\top}
        \mathbf{z}_{e2,\mathrm{ma}}^{\text{trn}}}_{I_6}).
  \end{align}
  We will then bound these terms $I_1$-$I_6$ in \eqref{eq:stats_err_mat_form_maml} one by one.
  
  To bound term $I_1$ in \eqref{eq:stats_err_mat_form_maml}, from Lemma~\ref{lemma:bound_stats_err_term1},
  we have with probability at least $1-Td^{-10}$
  \begin{equation}\label{eq:stats_err_term_I_maml}
    I_1=
    \mathbf{z}_{\mathrm{ma}}^{\top}\mathbf{U}_{\mathrm{ma}} \mathbf{U}_{\mathrm{ma}}^{\top} \mathbf{z}_{\mathrm{ma}}
    \leq \frac{R^{2}}{T}\left(
    \tilde{C}^{\mathrm{ma}}_{0} + \widetilde{\mathcal{O}}(\frac{1}{\sqrt{d }}) + \widetilde{\mathcal{O}}(\sqrt{\frac{d}{T}})\right). 
  \end{equation}
  
  For $\tilde{C}^{\mathrm{ma}}_{0}$, 
  by Lemma~\ref{lemma:dominate_constant_stats_err_maml}
  \begin{align}\label{eq:constant_ma}
  \tilde{C}^{\mathrm{ma}}_{0}
  =&\frac{1}{dN_2} 
    \mathbb{E} \Big[
    \mathrm{tr}^2\Big( 
    (\mathbf{I}-\alpha\hat{\mathbf{Q}}_{\tau,N_1})^2
    \Big) 
    + (N_2 + 1)
    \mathrm{tr}\Big (
    (\mathbf{I}-\alpha\hat{\mathbf{Q}}_{\tau,N_1})^4
    \Big)
    \Big] 
     \Big[\frac{1}{d} 
    \mathbb{E} \big[  \mathrm{tr}
    \big((\mathbf{I}-\alpha \hat{\mathbf{Q}}_{\tau,N_1})^2\big)
    \big] \Big]^{-2} .
  \end{align}

  For $I_2$, from Lemma~\ref{lemma:Hanson-Wright inequality}, we have the absolute error around the expectation is given by
  \begin{align}
    &|\mathbf{z}_{e1,\mathrm{ma}}^{\text{val}\top}
    \mathbf{U}_{e1,\mathrm{ma}} \mathbf{U}_{e1,\mathrm{ma}}^{\top} \mathbf{z}_{e1,\mathrm{ma}}^{\text{val}}
    - \mathbb{E}_{\btheta_{\tau}^{\mathrm{gt}}, \mathbf{e}_{\tau} \mid \hat{\mathbf{W}}_{\tau, N}^{\mathrm{ma}}}
    [\mathbf{z}_{e1,\mathrm{ma}}^{\text{val}\top}
    \mathbf{U}_{e1,\mathrm{ma}} \mathbf{U}_{e1,\mathrm{ma}}^{\top}
     \mathbf{z}_{e1,\mathrm{ma}}^{\text{val}}]| 
    =\widetilde{\mathcal{O}}\Big(\frac{\sqrt{d}}{TN_2} \Big) 
    \label{eq:HW_ineq_ze1_maml}
  \end{align}
  and the expectation is given by
  \begin{align*}\label{eq:concentrate_cov_z1_maml}
  &\mathbb{E}_{\btheta_{\tau}^{\mathrm{gt}}, \mathbf{e}_{\tau} \mid \hat{\mathbf{W}}_{\tau, N}^{\mathrm{ma}}}
    [\mathbf{z}_{e1,\mathrm{ma}}^{\text{val}\top}
    \mathbf{U}_{e1,\mathrm{ma}} \mathbf{U}_{e1,\mathrm{ma}}^{\top}
    \mathbf{z}_{e1,\mathrm{ma}}^{\text{val}}] 
  = \mathrm{tr}\Big(
  \mathbf{U}_{e1,\mathrm{ma}} \mathbf{U}_{e1,\mathrm{ma}}^{\top}
  \Big) 
  = \frac{d}{TN_2} 
  \frac{1}{d} \mathrm{tr}\Big(\big(\frac{1}{T}\sum_{\tau=1}^{T}
  \hat{\mathbf{W}}_{\tau, N}^{\mathrm{ma}} \big)^{-1}\Big) \\
  = & \frac{d}{TN_2}\Big\{
  \frac{1}{d} \mathrm{tr}\Big(\big(\frac{1}{T}\sum_{\tau=1}^{T}
  \hat{\mathbf{W}}_{\tau, N}^{\mathrm{ma}} \big)^{-1}
  - \mathbb{E}^{-1}[\hat{\mathbf{W}}_{\tau, N}^{\mathrm{ma}}]\Big)
  + \underbracket{\frac{1}{d}
  \mathrm{tr}\Big(
  \mathbb{E}^{-1}[\hat{\mathbf{W}}_{\tau, N}^{\mathrm{ma}}]
  \Big)}_{=\tilde{C}^{\mathrm{ma}}_{1,1} }
  \Big\}.
  \numberthis
  \end{align*}
  
  And by combining~\eqref{eq:HW_ineq_ze1_maml} and~\eqref{eq:concentrate_cov_z1_maml} , with probability at least $1 - \delta$, we have
  \begin{align}\label{eq:bound_MSE_e1_maml}
    I_2 
    =\mathbf{z}_{e1,\mathrm{ma}}^{\text{val}\top}
    \mathbf{U}_{e1,\mathrm{ma}} \mathbf{U}_{e1,\mathrm{ma}}^{\top} \mathbf{z}_{e1,\mathrm{ma}}^{\text{val}}
    &=\frac{d}{TN_2}\Big(
    \tilde{C}^{\mathrm{ma}}_{1,1} +\widetilde{\mathcal{O}}(\sqrt{\frac{d}{T}})+\widetilde{\mathcal{O}}(\frac{1}{\sqrt{d }})\Big) .
  \end{align}
  
  For $I_3$, the absolute error around the expectation is given by
  \begin{align}\label{eq:HW_ineq_ze2_maml}
    &|\| \mathbf{U}_{e2,\mathrm{ma}}^{\top} 
    \mathbf{z}_{e2,\mathrm{ma}}^{\text{trn}}\|^2
    - \mathbb{E}_{\btheta_{\tau}^{\mathrm{gt}}, \mathbf{e}_{\tau} \mid \hat{\mathbf{W}}_{\tau, N}^{\mathrm{ma}}}
    [\| \mathbf{U}_{e2,\mathrm{ma}}^{\top}
    \mathbf{z}_{e2,\mathrm{ma}}^{\text{trn}}\|^2 ]| 
    =
    \widetilde{\mathcal{O}}\Big(\frac{\sqrt{d}}{TN_1} \Big)
  \end{align}
  and the expectation is given by
  \begin{align*}\label{eq:concentrate_cov_z2_maml}
  &\mathbb{E}_{\btheta_{\tau}^{\mathrm{gt}}, \mathbf{e}_{\tau} \mid \hat{\mathbf{W}}_{\tau, N}^{\mathrm{ma}}}
    [\mathbf{z}_{e2,\mathrm{ma}}^{\text{trn}\top}
    \mathbf{U}_{e2,\mathrm{ma}} \mathbf{U}_{e2,\mathrm{ma}}^{\top} 
    \mathbf{z}_{e2,\mathrm{ma}}^{\text{trn}}] 
  = \mathrm{tr}\Big(
  \mathbf{U}_{e2,\mathrm{ma}} \mathbf{U}_{e2,\mathrm{ma}}^{\top} 
  \Big) \\
  = &\frac{d}{TN_1} 
  \frac{1}{d} \mathrm{tr}\Big(\Big(\frac{1}{T}\sum_{\tau=1}^{T}
  \hat{\mathbf{W}}_{\tau, N}^{\mathrm{ma}} \Big)^{-2}
  \Big( \frac{1}{T}\sum_{\tau=1}^{T} 
  \alpha^{2}(\mathbf{I} - \alpha \hat{\mathbf{Q}}_{\tau,N_1})
  \hat{\mathbf{Q}}_{\tau,N_2}\hat{\mathbf{Q}}_{\tau,N_1}\hat{\mathbf{Q}}_{\tau,N_2}
  (\mathbf{I} - \alpha \hat{\mathbf{Q}}_{\tau,N_1})
  \Big)
  \Big) \\
  = & \frac{d}{TN_1}\Big\{
  \frac{1}{d} \Big \langle \Big(\frac{1}{T}\sum_{\tau=1}^{T}
  \hat{\mathbf{W}}_{\tau, N}^{\mathrm{ma}} \Big)^{-2}
  - \mathbb{E}^{-2}[\hat{\mathbf{W}}_{\tau, N}^{\mathrm{ma}}], 
  \frac{1}{T}\sum_{\tau=1}^{T} 
  \alpha^{2}(\mathbf{I} - \alpha \hat{\mathbf{Q}}_{\tau,N_1})
  \hat{\mathbf{Q}}_{\tau,N_2}\hat{\mathbf{Q}}_{\tau,N_1}\hat{\mathbf{Q}}_{\tau,N_2}
  (\mathbf{I} - \alpha \hat{\mathbf{Q}}_{\tau,N_1})
  \Big \rangle \\
  &+ \frac{1}{d} \Big \langle 
   \mathbb{E}^{-2}[\hat{\mathbf{W}}_{\tau, N}^{\mathrm{ma}}], 
  \frac{1}{T}\sum_{\tau=1}^{T} 
  \alpha^{2}(\mathbf{I} - \alpha \hat{\mathbf{Q}}_{\tau,N_1})
  \hat{\mathbf{Q}}_{\tau,N_2}\hat{\mathbf{Q}}_{\tau,N_1}\hat{\mathbf{Q}}_{\tau,N_2}(\mathbf{I} - \alpha \hat{\mathbf{Q}}_{\tau,N_1}) \\
  &- \mathbb{E}[\alpha^{2}(\mathbf{I} - \alpha \hat{\mathbf{Q}}_{\tau,N_1})
  \hat{\mathbf{Q}}_{\tau,N_2}\hat{\mathbf{Q}}_{\tau,N_1}\hat{\mathbf{Q}}_{\tau,N_2}(\mathbf{I} - \alpha \hat{\mathbf{Q}}_{\tau,N_1})]
  \Big \rangle  \\
  &+ \underbracket{\frac{1}{d} \Big \langle 
  \mathbb{E}^{-2}[\hat{\mathbf{W}}_{\tau, N}^{\mathrm{ma}}], 
  \mathbb{E}[\alpha^{2}(\mathbf{I} - \alpha \hat{\mathbf{Q}}_{\tau,N_1})
  \hat{\mathbf{Q}}_{\tau,N_2}\hat{\mathbf{Q}}_{\tau,N_1}\hat{\mathbf{Q}}_{\tau,N_2}(\mathbf{I} - \alpha \hat{\mathbf{Q}}_{\tau,N_1})]
  \Big \rangle}_{=\tilde{C}^{\mathrm{ma}}_{1,2} }
  \Big\} .
  \numberthis
  \end{align*}
  
  For $I_3$, with probability at least $1 - \delta$
  \begin{align}\label{eq:bound_MSE_e2_maml}
  I_3 =
    \mathbf{z}_{e2,\mathrm{ma}}^{\text{trn}\top}
    \mathbf{U}_{e2,\mathrm{ma}} \mathbf{U}_{e2,\mathrm{ma}}^{\top} \mathbf{z}_{e2,\mathrm{ma}}^{\text{trn}}
    &=\frac{d}{TN_1}\Big(
    \tilde{C}^{\mathrm{ma}}_{1,2} +\widetilde{\mathcal{O}}(\sqrt{\frac{d}{T}})+\widetilde{\mathcal{O}}(\frac{1}{\sqrt{d }})\Big) .
  \end{align}

  For $I_4,I_5,I_6$, with probability at least $1-\delta$
  \begin{align}\label{eq:stats_err_term_I456_maml}
    |I_4|  
    \leq 
    \widetilde{\mathcal{O}} (\frac{R}{T\sqrt{N_2}}),
    |I_5|  
    \leq 
    \widetilde{\mathcal{O}} (\frac{R}{T\sqrt{N_1}}),
    |I_6| 
    \leq 
    \widetilde{\mathcal{O}} (\frac{\sqrt{d}}{T\sqrt{N_1N_2}}).
  \end{align}


  Finally, 
  applying the weight $w_\mathrm{ma}$, we have
  with probability $1-Td^{-10}$
  the statistical error of MAML is bounded by
  \begin{align*}\label{eq:stats_err_final_maml}
    &\mathcal{E}^{2}_{\mathrm{ma}} (\hat{\btheta}_0^{\mathrm{ma}})
    =w_{\mathrm{ma}} \|\hat{\btheta}_0^{\mathrm{ma}}(\gamma ) - \btheta_0^{\mathrm{ma}}(\gamma )\|_2^2
    = 
  \frac{R^{2}}{T}\Big(w_\mathrm{ma}\tilde{C}^{\mathrm{ma}}_{0}
  +\widetilde{\mathcal{O}}(\sqrt{\frac{d}{T}})+\widetilde{\mathcal{O}}( \frac{1}{\sqrt{d}})\Big) \\
  &+\frac{d}{TN}\Big(w_\mathrm{ma}\tilde{C}^{\mathrm{ma}}_{1} 
  +\widetilde{\mathcal{O}}(\sqrt{\frac{d}{T}})+\widetilde{\mathcal{O}}( \frac{1}{\sqrt{d}})
    \Big)
  + \widetilde{\mathcal{O}}\Big(\frac{R}{T \sqrt{N}}\Big)
  .\numberthis
  \end{align*}

  \paragraph{iMAML.}
    With $\mathbf{U}_{\mathrm{im}}, \mathbf{z}_{\mathrm{im}}$ defined in Lemma~\ref{lemma:bound_stats_err_term1}, we can rewrite ~\eqref{eq:theta0_hat_diff_theta0_star_bimaml_mat_form2} as
    \begin{align*}\label{eq:stats_diff_mat_form_bimaml}
      \hat{\btheta}_0^{\mathrm{im}} - {\btheta}_0^{\mathrm{im}} 
      =&  \mathbf{U}^{\top}_{\mathrm{im}} \mathbf{z}_{\mathrm{im}}
       + \mathbf{U}_{e1,\mathrm{im}}^{\top}
      \mathbf{z}_{e1,\mathrm{im}}^{\text{val}}
      - \mathbf{U}_{e2,\mathrm{im}}^{\top}
      \mathbf{z}_{e2,\mathrm{im}}^{\text{trn}}.\numberthis
    \end{align*}
  Thus the squared error can be computed by
    \begin{align}\label{eq:stats_err_mat_form_bimaml}
      \|\hat{\btheta}_{0}^{{\mathrm{im}}}-\btheta_{0}^{ {\mathrm{im}}}\|_2^2
      =& 
      \underbracket{\|\mathbf{U}_{\mathrm{im}}^{\top} \mathbf{z}_{\mathrm{im}}\|^2}_{I_1}
      + \underbracket{\| \mathbf{U}_{e1,\mathrm{im}}^{\top} \mathbf{z}_{e1,\mathrm{im}}^{\text{val}}\|^2}_{I_2}
      + \underbracket{\| \mathbf{U}_{e2,\mathrm{im}}^{\top} \mathbf{z}_{e2,\mathrm{im}}^{\text{trn}}\|^2 }_{I_3} \\
      \nonumber
      &+ 2\underbracket{\mathbf{z}_{\mathrm{im}}^{\top}
        \mathbf{U}_{\mathrm{im}} \mathbf{U}_{e1,\mathrm{im}}^{\top}
        \mathbf{z}_{e1,\mathrm{im}}^{\text{val}}}_{I_4}
      - 2\underbracket{\mathbf{z}_{\mathrm{im}}^{\top}
        \mathbf{U}_{\mathrm{im}} \mathbf{U}_{e2,\mathrm{im}}^{\top}
        \mathbf{z}_{e2,\mathrm{im}}^{\text{trn}}}_{I_5}
      - 2\underbracket{\mathbf{z}_{e1,\mathrm{im}}^{\text{val}\top}
        \mathbf{U}_{e1,\mathrm{im}} \mathbf{U}_{e2,\mathrm{im}}^{\top}
        \mathbf{z}_{e2,\mathrm{im}}^{\text{trn}}}_{I_6}.
    \end{align}
  To bound term $I_1$ in \eqref{eq:stats_err_mat_form_bimaml}, from Lemma~\ref{lemma:bound_stats_err_term1},
    we have with probability at least $1-Td^{-10}$
    \begin{equation}\label{eq:stats_err_term_I_bimaml}
      I_1=
      \mathbf{z}_{\mathrm{im}}^{\top}\mathbf{U}_{\mathrm{im}} \mathbf{U}_{\mathrm{im}}^{\top} \mathbf{z}_{\mathrm{im}}
      \leq \frac{R^{2}}{T}\Big(
      \tilde{C}^{\mathrm{im}}_{0} + \widetilde{\mathcal{O}}(\frac{1}{\sqrt{d }}) + \widetilde{\mathcal{O}}(\sqrt{\frac{d}{T}})\Big). 
    \end{equation}
  For $\tilde{C}_{0}^{\mathrm{im}}$, by Lemma C.2 in ~\citep{bai2021_trntrn_trnval}, 
    \begin{align}\label{eq:constant_bi}
       \tilde{C}_{0}^{\mathrm{im}}
    =&\frac{\frac{1}{d N_2} \mathbb{E}\big[\operatorname{tr}\big(\gamma ^{2}\big(\hat{\mathbf{Q}}_{\tau, N_1}+\gamma  \mathbf{I}_{d}\big)^{-2}\big)^{2}+\big(N_2+1\big) \operatorname{tr}\big(\gamma ^{4}\big(\hat{\mathbf{Q}}_{\tau, N_1}+\gamma  \mathbf{I}_{d}\big)^{-4}\big)\big]}{\big(\frac{1}{d} \mathbb{E}\big[\operatorname{tr}\big(\gamma ^{2}\big(\hat{\mathbf{Q}}_{\tau, N_1}+\gamma  \mathbf{I}_{d}\big)^{-2}\big)\big]\big)^{2}}  
    \end{align}
  For $I_2$, first from Lemma~\ref{lemma:Hanson-Wright inequality}
    \begin{align*}\label{eq:HW_ineq_ze1_bimaml}
      &\Big|\mathbf{z}_{e1,\mathrm{im}}^{\text{val}\top}\mathbf{U}_{e1,\mathrm{im}} \mathbf{U}_{e1,\mathrm{im}}^{\top} \mathbf{z}_{e1,\mathrm{im}}^{\text{val}}
      - \mathbb{E}_{\btheta_{\tau}^{\mathrm{gt}}, \mathbf{e}_{\tau} \mid \hat{\mathbf{W}}_{\tau}^{\mathrm{im}}}[\mathbf{z}_{e1,\mathrm{im}}^{\text{val}\top}\mathbf{U}_{e1,\mathrm{im}} \mathbf{U}_{e1,\mathrm{im}}^{\top} \mathbf{z}_{e1,\mathrm{im}}^{\text{val}}] \Big| 
      =\widetilde{\mathcal{O}}\Big(\frac{\sqrt{d}}{T N_2} \Big) \numberthis
    \end{align*}
    \begin{align*}\label{eq:concentrate_z1_bimaml}
    &\mathbb{E}_{\btheta_{\tau}^{\mathrm{gt}}, \mathbf{e}_{\tau} \mid \hat{\mathbf{W}}_{\tau}^{\mathrm{im}}}[
    \mathbf{z}_{e1,\mathrm{im}}^{\text{val}\top}\mathbf{U}_{e1,\mathrm{im}} \mathbf{U}_{e1,\mathrm{im}}^{\top} \mathbf{z}_{e1,\mathrm{im}}^{\text{val}}] 
    = \frac{1}{TN_2}
    \mathrm{tr}
    \Big(\Big(\frac{1}{T}\sum_{\tau=1}^{T} \hat{\mathbf{W}}_{\tau}^{\mathrm{im}}\Big)^{-1}\Big) \\
    =& \frac{d}{TN_2} \Bigg(
      \underbracket{\frac{1}{d} 
      \mathrm{tr}\Big(
    \mathbb{E}^{-1}[\hat{\mathbf{W}}_{\tau}^{\mathrm{im}}]
    \Big) }_{= C^{\mathrm{im}}_{1,1}}
    + \underbracket{\frac{1}{d} \mathrm{tr}\Big( 
    \Big(\frac{1}{T}\sum_{\tau=1}^{T} \hat{\mathbf{W}}_{\tau}^{\mathrm{im}}\Big)^{-1}
    - \mathbb{E}^{-1}[\hat{\mathbf{W}}_{\tau}^{\mathrm{im}}]
    \Big)}_{I_9}
    \Bigg),\numberthis
    \end{align*}
  and by Lemma~\ref{lemma:concentration_W_hat}, 
    with probability at least $1-Td^{-10}$
    \begin{equation}\label{eq:E_stats_err_term_9_bimaml}
      |I_9| \leq \widetilde{\mathcal{O}}(\sqrt{\frac{d}{T}}).
    \end{equation}
    Therefore, combining 
    ~\eqref{eq:HW_ineq_ze1_bimaml},\eqref{eq:concentrate_z1_bimaml},\eqref{eq:E_stats_err_term_9_bimaml}
    with probability at least $1-Td^{-10}$
    \begin{align}\label{eq:I2_bound_bimaml}
      &I_2 \leq
      \frac{d}{TN_2}\Big(
      \tilde{C}^{\mathrm{im}}_{1,1}
      + \widetilde{\mathcal{O}}(\sqrt{\frac{d}{T}})
      + \widetilde{\mathcal{O}}(\frac{1}{\sqrt{d }}) \Big) \\
      \label{eq:e1_const_bimaml}
      &\tilde{C}^{\mathrm{im}}_{1,1} \coloneqq \frac{1}{d}\mathrm{tr}\Big(\mathbb{E}^{-1} [\hat{\mathbf{W}}_{\tau}^{\mathrm{im}}]\Big).
    \end{align}
  Similarly, for $I_3$, based on Lemma~\ref{lemma:Hanson-Wright inequality} we have
    \begin{align*}\label{eq:HW_ineq_ze2_bimaml}
      &|\| \mathbf{U}_{e2,\mathrm{im}}^{\top} \mathbf{z}_{e2,\mathrm{im}}^{\text{trn}}\|^2
      - \mathbb{E}_{\mathbf{e}_{\tau} \mid \hat{\mathbf{W}}_{\tau}^{\mathrm{im}}}[\| \mathbf{U}_{e2,\mathrm{im}}^{\top} \mathbf{z}_{e2,\mathrm{im}}^{\text{trn}}\|^2 ]| 
      =\widetilde{\mathcal{O}}\Big(\frac{\sqrt{d}}{TN_1} \Big). \numberthis
    \end{align*}
  And similar to the derivations in ERM and MAML, with Lemma~\ref{lemma:concentration_W_hat}, it holds that
  \begin{align*}\label{eq:concentrate_z2_bimaml}
    &\mathbb{E}_{\mathbf{e}_{\tau} \mid \hat{\mathbf{W}}_{\tau}^{\mathrm{im}}}[\| \mathbf{U}_{e2,\mathrm{im}}^{\top} \mathbf{z}_{e2,\mathrm{im}}^{\text{trn}}\|^2 ] 
    = \frac{1}{TN_1}
    \Big\langle
    \Big(\frac{1}{T}\sum_{\tau=1}^{T}
    \hat{\mathbf{W}}_{\tau}^{\mathrm{im}} \Big)^{-2} ,
    \Big(\frac{1}{T}\sum_{\tau=1}^{T}
    (\gamma   )^{-2} \hat{\mathbf{W}}_{\tau}^{\mathrm{im}}
      \frac{1}{N_1}\mathbf{X}^{\text{trn}\top}_{\tau}\mathbf{X}^{\text{trn}}_{\tau}\hat{\mathbf{W}}_{\tau}^{\mathrm{im}}
      \Big)
    \Big\rangle \\
    =&\frac{d}{TN_1} \Bigg\{
    \underbracket{\frac{1}{d} \Big\langle
    \mathbb{E}[\hat{\mathbf{W}}_{\tau}^{\mathrm{im}}]^{-2},
    \frac{1}{T}\sum_{\tau=1}^{T}
    \gamma^{-2} 
    \mathbb{E}[\hat{\mathbf{W}}_{\tau}^{\mathrm{im}}]
      \mathbb{E}[{\Sigma}_{\btheta_{\tau}}^{-1}]
      \mathbb{E}[\hat{\mathbf{W}}_{\tau}^{\mathrm{im}}]
    -
    \mathbb{E}[\gamma ^{-1}(\hat{\mathbf{W}}_{\tau}^{\mathrm{im}})^2]
    \Big\rangle}_{=\tilde{C}^{\mathrm{im}}_{1,2}} 
    + \widetilde{\mathcal{O}} (\sqrt{\frac{d}{T}})
    \Bigg\}
    \numberthis
    \end{align*}
  Combining ~\eqref{eq:HW_ineq_ze2_bimaml} and ~\eqref{eq:concentrate_z2_bimaml}, with probability at least $1-Td^{-10}$, we have
    \begin{align}\label{eq:concentrate_cov_e2_bimaml2}
      I_3 = \mathbf{z}_{e2,\mathrm{im}}^{\text{trn}\top}\mathbf{U}_{e2,\mathrm{im}} \mathbf{U}_{e2,\mathrm{im}}^{\top} \mathbf{z}_{e2,\mathrm{im}}^{\text{trn}}
      \leq 
      \frac{d}{TN_1}\Big(\tilde{C}^{\mathrm{im}}_{1,2}
      +\widetilde{\mathcal{O}}(\frac{1}{\sqrt{d}})
      +\widetilde{\mathcal{O}}(\sqrt{ \frac{d}{T}})
      \Big)
    \end{align}
  Following a similar argument, for $I_4,I_5,I_6$, with probability at least $1-\delta$
    \begin{align}\label{eq:stats_err_term_I456_bimaml}
      |I_4|  
      \leq 
      \widetilde{\mathcal{O}} (\frac{R}{T\sqrt{N_2}}),
      |I_5|  
      \leq 
      \widetilde{\mathcal{O}} (\frac{R}{T\sqrt{N_1}}),
      |I_6| 
      \leq 
      \widetilde{\mathcal{O}} (\frac{\sqrt{d}}{T\sqrt{N_1N_2}}).
    \end{align}

  Finally, define $\tilde{C}_1^{\rm im} \coloneqq (1-s)^{-1} \tilde{C}_{1,1}^{\rm im} + s^{-1} \tilde{C}_{1,2}^{\rm im}$,
  applying the weight $w_\mathrm{im}$ we have
  with probability at least $1-Td^{-10}$,
  the statistical error of iMAML is bounded by
  \begin{align*}\label{eq:stats_err_final_bimaml}
    &\mathcal{E}^{2}_{\mathrm{im}} (\hat{\btheta}_0^{\mathrm{im}})
    = w_{\mathrm{im}} \|\hat{\btheta}_0^{\mathrm{im}}(\gamma ) - \btheta_0^{\mathrm{im}}(\gamma )\|_2^2
    = 
  \frac{R^{2}}{T}\Big(w_\mathrm{im}\tilde{C}^{\mathrm{im}}_{0}+\widetilde{\mathcal{O}}(\sqrt{\frac{d}{T}})+\widetilde{\mathcal{O}}( \frac{1}{\sqrt{d }})\Big)\\
  &+\frac{d}{TN}\Big(
    w_\mathrm{im}\tilde{C}^{\mathrm{im}}_{1}
    +\widetilde{\mathcal{O}}(\frac{1}{\sqrt{d }})
    +\widetilde{\mathcal{O}}(\sqrt{\frac{d}{T}}) \Big)
  + \widetilde{\mathcal{O}}\Big(\frac{R}{T \sqrt{N}} \Big)
  .\numberthis
  \end{align*}

  \paragraph{BaMAML.}
  With $\mathbf{U}_{\mathrm{ba}}, \mathbf{z}_{\mathrm{ba}}$ defined in Lemma~\ref{lemma:bound_stats_err_term1}
    \begin{align}\label{eq:diff_mat_form_bamaml}
      \hat{\btheta}_0^{\mathrm{ba}} - {\btheta}_0^{\mathrm{ba}} 
      =&  \mathbf{U}_{\mathrm{ba}}^{\top} \mathbf{z}_{\mathrm{ba}}
       + \mathbf{U}_{e1,\mathrm{ba}}^{\top} \mathbf{z}_{e1,\mathrm{ba}}
       - \mathbf{U}_{e2,\mathrm{ba}}^{\top} \mathbf{z}_{e2,\mathrm{ba}} \\
       \label{eq:stats_err_mat_form_bamaml}
      \|\hat{\btheta}_{0}^{{\mathrm{ba}}}-\btheta_{0}^{ {\mathrm{ba}}}\|_2^2
      =& \underbracket{\| \mathbf{U}_{\mathrm{ba}}^{\top} \mathbf{z}_{\mathrm{ba}}\|^2 }_{I_1}
        + \underbracket{\| \mathbf{U}_{e1,\mathrm{ba}}^{\top} \mathbf{z}_{e1,\mathrm{ba}}^{\text{all}}\|^2 }_{I_2}
        + \underbracket{\| \mathbf{U}_{e2,\mathrm{ba}}^{\top} \mathbf{z}_{e2,\mathrm{ba}}^{\text{trn}} \|^2 }_{I_3} \\
        \nonumber
        &+ 2\underbracket{\mathbf{z}_{\mathrm{ba}}^{\top}
          \mathbf{U}_{\mathrm{ba}} \mathbf{U}_{e1,\mathrm{ba}}^{\top}
          \mathbf{z}_{e1,\mathrm{ba}}^{\text{all}}}_{I_4}
        - 2\underbracket{\mathbf{z}_{\mathrm{ba}}^{\top}
          \mathbf{U}_{\mathrm{ba}} \mathbf{U}_{e2,\mathrm{ba}}^{\top}
          \mathbf{z}_{e2,\mathrm{ba}}^{\text{trn}}}_{I_5}
        - 2\underbracket{\mathbf{z}_{e1,\mathrm{ba}}^{\text{all}\top}
          \mathbf{U}_{e1,\mathrm{ba}} \mathbf{U}_{e2,\mathrm{ba}}^{\top}
          \mathbf{z}_{e2,\mathrm{ba}}^{\text{trn}}}_{I_6}.
    \end{align}
    
    
    To bound term $I_1$ in \eqref{eq:stats_err_mat_form_bamaml}, 
    from Lemma~\ref{lemma:bound_stats_err_term1},
    we have with probability at least $1-Td^{-10}$
    \begin{equation}\label{eq:stats_err_term_I1_bamaml}
      I_1=
      \mathbf{z}_{\mathrm{ba}}^{\top}\mathbf{U}_{\mathrm{ba}} \mathbf{U}_{\mathrm{ba}}^{\top} \mathbf{z}_{\mathrm{ba}}
      \leq \frac{R^{2}}{T}\Big(\tilde{C}^{\mathrm{ba}}_{0} + \widetilde{\mathcal{O}}(\frac{1}{\sqrt{d }}) + \widetilde{\mathcal{O}}(\sqrt{\frac{d}{T}})\Big). 
    \end{equation}
    
    To compute $\tilde{C}_{0}^{\mathrm{ba}}$, by Lemma~\ref{lemma:dominate_constant_stats_err}
    \begin{align}\label{eq:term1_bound_bamaml}
      \tilde{C}_{0}^{\mathrm{ba}} 
      =& \frac{1}{d} \mathbb{E}\Big[\mathrm{tr}\Big(
      (\hat{\mathbf{W}_{\tau,N_a}^{\rm ba}})^2\Big)\Big] 
      \cdot \Big\{\frac{1}{d} 
      \mathbb{E}\Big[\mathrm{tr}\Big( \hat{\mathbf{W}}_{\tau,N}^{\rm ba} \Big) \Big] \Big\}^{-2}
    \end{align}

    For term $I_2$, based on Lemma~\ref{lemma:Hanson-Wright inequality}, the absolute error around the expectation is given by
    \begin{align*}\label{eq:HW_ineq_ze_bamaml}
      &|\| \mathbf{U}_{e,\mathrm{ba}}^{\top} \mathbf{z}_{e,\mathrm{ba}}^{\text{all}}\|^2
      - \mathbb{E}_{\btheta_{\tau}^{\mathrm{gt}}, \mathbf{e}_{\tau} \mid \hat{\mathbf{W}}_{\tau}^{\mathrm{ba}}}[\| \mathbf{U}_{e,\mathrm{ba}}^{\top} \mathbf{z}_{e,\mathrm{ba}}^{\text{all}}\|^2 ]| 
      =\widetilde{\mathcal{O}}\Big(\frac{\sqrt{d}}{TN}\Big) 
       \numberthis
    \end{align*}
    where the expectation is given by
    \begin{align*}\label{eq:concentrate_cov_z_bamaml}
    &\mathbb{E}_{\btheta_{\tau}^{\mathrm{gt}}, \mathbf{e}_{\tau} \mid \hat{\mathbf{W}}_{\tau, N}^{\mathrm{ba}}}[\mathbf{z}_{e1,\mathrm{ba}}^{\text{all}\top}\mathbf{U}_{e1,\mathrm{ba}} \mathbf{U}_{e1,\mathrm{ba}}^{\top} \mathbf{z}_{e1,\mathrm{ba}}^{\text{all}}] 
    = \mathrm{tr}\Big(
    \mathbf{U}_{e1,\mathrm{ba}} \mathbf{U}_{e1,\mathrm{ba}}^{\top}
    \Big)\\
    =& \frac{d}{TN}\frac{1}{d}
    \Big\langle
    \Big(\frac{1}{T}\sum_{\tau=1}^{T}
    \hat{\mathbf{W}}_{\tau, N}^{\mathrm{ba}} \Big)^{-2} ,
    (1-s)^{-1}\Big(\frac{1}{T}\sum_{\tau=1}^{T} (\mathbf{I}_d + (\gamma s)^{-1}\hat{\mathbf{Q}}_{\tau,N})^{-1} \hat{\mathbf{Q}}_{\tau,N} (\mathbf{I}_d + (\gamma s)^{-1}\hat{\mathbf{Q}}_{\tau,N})^{-1} \Big)
    \Big\rangle \\
    \leq & \frac{d}{TN} \Bigg\{\widetilde{\mathcal{O}} (\sqrt{\frac{d}{T}})
    + \underbracket{\frac{1}{d}\Big\langle
    \mathbb{E}[\hat{\mathbf{W}}_{\tau, N}^{\mathrm{ba}}]^{-2},
    (1-s)^{-1}\mathbb{E}[(\mathbf{I}_d + (\gamma s)^{-1}\hat{\mathbf{Q}}_{\tau,N})^{-1} \hat{\mathbf{Q}}_{\tau,N} (\mathbf{I}_d + (\gamma s)^{-1}\hat{\mathbf{Q}}_{\tau,N})^{-1} ]
    \Big\rangle}_{=\tilde{C}^{\mathrm{ba}}_{1,1} }
    \Bigg\} .
    \numberthis
    \end{align*}
    Similarly, 
    \begin{align*}\label{eq:concentrate_cov_z_e2_bamaml}
      &\mathbb{E}_{\btheta_{\tau}^{\mathrm{gt}}, \mathbf{e}_{\tau} \mid \hat{\mathbf{W}}_{\tau, N}^{\mathrm{ba}}}[\mathbf{z}_{e2,\mathrm{ba}}^{\top}\mathbf{U}_{e2,\mathrm{ba}} \mathbf{U}_{e2,\mathrm{ba}}^{\top} \mathbf{z}_{e2,\mathrm{ba}}]  \\
      \leq & \frac{d}{TN} \Bigg\{\widetilde{\mathcal{O}} (\sqrt{\frac{d}{T}})
      + \underbracket{\frac{1}{d}\Big\langle
      \mathbb{E}[\hat{\mathbf{W}}_{\tau, N}^{\mathrm{ba}}]^{-2},
      s(1-s)^{-1}\mathbb{E}[(\mathbf{I}_d + \gamma^{-1}{\mathbf{Q}}_{\tau})^{-1} {\mathbf{Q}}_{\tau} (\mathbf{I}_d + \gamma^{-1}{\mathbf{Q}}_{\tau})^{-1} ]
      \Big\rangle}_{=\tilde{C}^{\mathrm{ba}}_{1,2} }
      \Bigg\} .
      \numberthis
    \end{align*}
    $\tilde{C}_1^{\rm ba} \coloneqq \tilde{C}_{1,1}^{\rm ba} - \tilde{C}_{1,2}^{\rm ba}$, combining the above derivations with Lemma~\ref{lemma:concentration_W_hat} gives the higher order terms in~\eqref{eq:concentrate_cov_z_bamaml}, which leads to
    \begin{align}\label{eq:concentrate_cov_e_bamaml}
      \mathbb{E}_{\btheta_{\tau}^{\mathrm{gt}}, \mathbf{e}_{\tau} \mid \hat{\mathbf{W}}_{\tau, N}^{\mathrm{ba}}}[\mathbf{z}_{e,\mathrm{ba}}^{\text{all}\top}\mathbf{U}_{e,\mathrm{ba}} \mathbf{U}_{e,\mathrm{ba}}^{\top} \mathbf{z}_{e,\mathrm{ba}}^{\text{all}}]
      \leq 
      \frac{d}{TN}
      (\tilde{C}^{\mathrm{ba}}_1 
      +\widetilde{\mathcal{O}} (\sqrt{\frac{d}{T}}))
    \end{align}
    Combining~\eqref{eq:HW_ineq_ze_bamaml} and~\eqref{eq:concentrate_cov_e_bamaml}, with probability at least $1-Td^{-10}$, we have
    \begin{align}\label{eq:concentrate_cov_e_bamaml2}
      I_2 = \mathbf{z}_{e,\mathrm{ba}}^{\text{all}\top}\mathbf{U}_{e,\mathrm{ba}} \mathbf{U}_{e,\mathrm{ba}}^{\top} \mathbf{z}_{e,\mathrm{ba}}^{\text{all}}
      \leq 
      \frac{d}{TN}\Big(\tilde{C}^{\mathrm{ba}}_1 
      +\widetilde{\mathcal{O}}(\frac{1}{\sqrt{d}})
      +\widetilde{\mathcal{O}}(\sqrt{ \frac{d}{T}})
      \Big)
    \end{align}
    Following a similar argument, 
    with probability $1-\delta$,
    $|I_3| \leq 
    \widetilde{\mathcal{O}} (\frac{R}{T \sqrt{N}}).$

    Finally, 
    applying the weight $w_\mathrm{ba}$,
    we have with probability $1-Td^{-10}$,
    the statistical error of BaMAML is bounded by
    \begin{align*}\label{eq:stats_err_final_bamaml}
    \mathcal{E}^{2}_{\mathrm{ba}} (\hat{\btheta}_0^{\mathrm{ba}})
    =&  w_\mathrm{ba}\|\hat{\btheta}_0^{\mathrm{ba}} - \btheta_0^{\mathrm{ba}}\|_2^2
    =
    \frac{R^{2}}{T}\Big(w_\mathrm{ba}\tilde{C}^{\mathrm{ba}}_{0 }+\widetilde{\mathcal{O}}(\sqrt{\frac{d}{T}})+\widetilde{\mathcal{O}}(\frac{1}{\sqrt{d }})\Big)\\
    &+\frac{d}{TN}\Big(w_\mathrm{ba}\tilde{C}^{\mathrm{ba}}_1+\widetilde{\mathcal{O}}(\sqrt{\frac{d}{T}})+\widetilde{\mathcal{O}}(\frac{1}{\sqrt{d }})\Big) 
    + \widetilde{\mathcal{O}}\Big(\frac{R}{T\sqrt{N}} \Big)
    .\numberthis
    \end{align*}

  
  \subsubsection{Asymptotic dominating constant under Assumptions 1,3}
  \label{app_sub:asymptotic_constant}
  
  \begin{theorem}[Asymptotic ERM constant]
    \label{thm:asymptotic_ERM_constant}
    As $d,N \to \infty$, $d/N \to \eta$, the optimal constant of the ERM method $\tilde{C}_{0}^{\mathrm{er}} $ satisfies 
    \begin{align*} 
      \lim _{\tiny\makecell{d, N \rightarrow \infty\\ d / N \rightarrow \eta}}
      \tilde{C}_{0}^{\mathrm{er}} 
      = 1 + \eta.
    \end{align*}
  \end{theorem}

  \begin{proof}
    Recall that $\tilde{C}_{0}^{\mathrm{er}}=\frac{1}{d} \mathbb{E}\Big[  \mathrm{tr}(\hat{\mathbf{Q}}_{\tau, N}^{2})\Big]
    \Big(\frac{1}{d}\mathbb{E}\big[ \mathrm{tr}(\mathbf{Q}_{\tau})\big]\Big)^{-2} $. Based on Assumption~\ref{asmp:linear_centroid_model}, $\mathbb{E}\big[ \mathrm{tr}(\mathbf{Q}_{\tau})\big] = \mathbb{E}\big[ \mathrm{tr}(\mathbf{I}_d)\big] = d$, And $\mathbb{E}[ (\hat{\mathbf{Q}}_{\tau, N}^{2})]$ can be derived by 
    \begin{align*}
      \mathbb{E}\big[  \hat{\mathbf{Q}}_{\tau, N}^{2}\big] 
      &= \mathbb{E}\big[ (\frac{1}{N} \mathbf{X}_{\tau, N}^{\top} \mathbf{X}_{\tau, N})^{2} \big] 
      = \mathbb{E}\big[ (\frac{1}{N} \mathbf{X}_{\tau, N}^{\top} \mathbf{X}_{\tau, N})^{2} \big] 
      = \mathbb{E}\big[ (\frac{1}{N} \sum_i  \mathbf{x}_{\tau, i} \mathbf{x}_{\tau, i}^{\top})^{2} \big] \\
      &= \frac{1}{N}\mathbb{E}\big[ ( \mathbf{x}_{\tau, i} \mathbf{x}_{\tau, i}^{\top})^{2} \big] + \frac{N-1}{N} \mathbf{I}_d
    \end{align*}
  where $\mathrm{tr}(\mathbb{E}\big[ ( \mathbf{x}_{\tau, i} \mathbf{x}_{\tau, i}^{\top})^{2} \big])=
  \mathbb{E}\big[ ( \sum_j {x}_{\tau, ij}^2 )^{2} \big] 
  = d(d+2).
  $
  
  Therefore 
  \begin{align*}
    \tilde{C}_{0}^{\mathrm{er}} = \frac{d+ N+1}{N}, 
    \quad\quad
    \lim _{\tiny\makecell{d, N \rightarrow \infty\\ d / N \rightarrow \eta}}
      \tilde{C}_{0}^{\mathrm{er}} 
      = 1 + \eta.
  \end{align*}
  
  \end{proof}

  %
  
  \begin{theorem}[Asymptotic MAML constant]
  \label{thm:asymptotic_MAML_constant}
    As $d,N \to \infty$, $d/N \to \eta$,
  the optimal constant of the MAML method,
  $\tilde{C}_{0}^{\mathrm{ma}} $, 
  by tuning the step size $\alpha \in (0, 1/\bar{\lambda})$
  and the train-val split ratio $s \in (0,1)$,
  satisfies 
  \begin{align*}
    \inf _{\tiny\makecell{\alpha>0\\ s \in(0,1)}}
    \lim _{\tiny\makecell{d, N \rightarrow \infty\\ d / N \rightarrow \eta}}
    \tilde{C}_{0}^{\mathrm{ma}} 
    = 1 + \eta
  \end{align*}
  \end{theorem}
  
  \begin{proof}
  \label{proof:asymptotic_MAML_constant}
  We first derive a lower bound for $\mathop{\inf }_{\tiny\makecell{\alpha>0\\ s \in(0,1)}}
  \mathop{\lim} _{\tiny\makecell{d, N \rightarrow \infty\\ d / N \rightarrow \eta}}$ by
    \begin{align*}\label{eq:constant_maml_inf_limit_geq}
    &\inf _{\tiny\makecell{\alpha>0\\ s \in(0,1)}}
    \lim _{\tiny\makecell{d, N \rightarrow \infty\\ d / N \rightarrow \eta}} 
    \tilde{C}_{0}^{\mathrm{ma}} \\
    \geq & 
    \lim _{\tiny\makecell{d, N \rightarrow \infty\\ d / N \rightarrow \eta}}
    \inf _{\tiny\makecell{\alpha>0\\ s \in(0,1)}}
    \frac{\frac{1}{dN_2} 
      \mathbb{E} \Big[
      \Big( \sum_{i=1}^{d}
      (1-\alpha  \lambda_i^{(N_1)})^2
      \Big)^2 
      + (N_2 + 1)
      \Big (
      \sum_{i=1}^{d}
      (1-\alpha  \lambda_i^{(N_1)})^4
      \Big)
      \Big]} 
    {\frac{1}{d^2} 
       \mathbb{E}^2 \Big[  \sum_{i=1}^{d}
       (1-\alpha  \lambda_i^{(N_1)})^2
       \Big]} \\
    \geq & 
    \lim _{\tiny\makecell{d, N \to \infty\\ d / N \to \eta}}
    \inf _{\tiny\makecell{\alpha > 0 \\s \in(0,1)}}
    \frac{d + N_2+1}{N_2} 
    = 1 + \eta. \numberthis
  \end{align*}
  
  Next we derive the upper bound for $\inf _{\tiny\makecell{\alpha>0\\ s \in(0,1)}}
  \lim _{\tiny\makecell{d, N \rightarrow \infty\\ d / N \rightarrow \eta}} \tilde{C}_{0}^{\text{ma}}$.
  As for any PD matrix $\mathbf{M} \in \mathbb{R}^{d\times d}$, $\frac{1}{d} \mathrm{tr}(\mathbf{M}^2) \geq 
  (\frac{1}{d} \mathrm{tr}(\mathbf{M}))^2$, then applying this inequality we obtain
  \begin{align*}\label{eq:constant_maml_limit}
    & \lim _{\tiny\makecell{d, N \rightarrow \infty\\ d / N \rightarrow \eta}} \tilde{C}_{0}^{\text{ma}} 
    \leq  
     \lim _{\tiny\makecell{d, N \rightarrow \infty\\ d / N \rightarrow \eta}}
    \frac{\frac{1}{N_2} 
    (d + N_2 + 1)
      } 
    {\frac{1}{d^2} 
       \mathbb{E}^2  \big[  \sum_{i=1}^{d}
       (1-\alpha  \lambda_d^{(N_1)})^2
       \big]} 
    \leq  
     \lim _{\tiny\makecell{d, N \rightarrow \infty\\ d / N \rightarrow \eta}}
    \frac{\frac{1}{N_2} 
    (d + N_2 + 1)
      } 
    {
    \mathbb{E}^2 \big[  
    \big(1-\frac{\alpha}{d}\sum_{i=1}^{d}\lambda_i^{(N_1)}\big)^2 \big]} \\
    \leq &
     \lim _{\tiny\makecell{d, N \rightarrow \infty\\ d / N \rightarrow \eta}}
    \frac{\frac{1}{N_2} 
    (d + N_2 + 1) } 
    { \mathbb{E}^2 \big[  
       \big(1-\alpha  
       \mathbb{E}[\frac{1}{d}\sum_{i=1}^{d}\lambda_i^{(N_1)}]\big)^2
       \big]} 
    =
     \lim _{\tiny\makecell{d, N \rightarrow \infty\\ d / N \rightarrow \eta}}
    \frac{\frac{1}{N_2} 
    (d + N_2 + 1)} 
    { (1-\alpha )^2} .
       \numberthis
  \end{align*}
  
  Therefore 
  \begin{align}\label{eq:constant_maml_inf_limit_leq}
    \inf _{\tiny\makecell{\alpha>0\\ s \in(0,1)}}
    \lim _{\tiny\makecell{d, N \rightarrow \infty\\ d / N \rightarrow \eta}} \tilde{C}_{0}^{\text{ma}} 
    \leq 
    \inf _{\tiny\makecell{s \in(0,1)}}
    \lim _{\tiny\makecell{d, N \rightarrow \infty\\ d / N \rightarrow \eta}}
    \frac{d + N_2+1}{N_2} 
    = 1 + \eta
  \end{align}
  
  Based on~\eqref{eq:constant_maml_inf_limit_geq} and~\eqref{eq:constant_maml_inf_limit_leq} we arrive at
  \begin{align}\label{eq:constant_maml_inf_limit}
    \inf _{\tiny\makecell{\alpha>0\\ s \in(0,1)}}
    \lim _{\tiny\makecell{d, N \rightarrow \infty\\ d / N \rightarrow \eta}} \tilde{C}_{0}^{\text{ma}} 
    = 1 + \eta
  \end{align}
  \end{proof}

  \begin{theorem}[Asymptotic dominating constant of iMAML]~\citep{bai2021_trntrn_trnval}
    \label{thm:asymptotic_biMAML_constant}
      As $d,N \to \infty$, $d/N \to \eta$,
    the optimal constant of the iMAML method,
    $\tilde{C}_{0}^{\text{im}} $, 
    by tuning the regularization $\gamma \in (0, \infty)$
    and the train-val split ratio $s \in (0,1)$,
    satisfies 
    \begin{align*}
    \inf _{\tiny\makecell{\gamma>0\\ s \in(0,1)}}
    \lim _{\tiny\makecell{d, N \rightarrow \infty\\ d / N \rightarrow \eta}} \tilde{C}_{0}^{\mathrm{im}} 
      = 1 + \eta.
    \end{align*}
  \end{theorem}

  \begin{theorem}[Asymptotic BaMAML constant]
  \label{thm:asymptotic_baMAML_constant}
    As $d,N \to \infty$, $d/N \to \eta$,
  the optimal constant of the BaMAML method,
  $\tilde{C}_{0}^{\mathrm{ba}} $, 
  by tuning the regularization $\gamma \in (0, \infty)$
  and the train-val split ratio $s \in (0,1)$,
  satisfies 
  \begin{align*}
  \inf _{\tiny\makecell{\gamma>0\\ s \in(0,1)}}
  \lim _{\tiny\makecell{d, N \rightarrow \infty\\ d / N \rightarrow \eta}} \tilde{C}_{0}^{\mathrm{ba}} 
  \begin{cases}
    = 1, & \eta \in (0, 1], \\
    \leq \eta, & \eta \in (1, \infty) .
  \end{cases} 
  \end{align*}
  \end{theorem}

  \begin{proof}
  Adopt the Stieltjes transform to obtain $\lim _{\tiny\makecell{d, N \rightarrow \infty\\ d / N \rightarrow \eta}} \tilde{C}_{0}^{\text{ba}} $ as a function of $\gamma, s, \eta$, given below. For all $\omega_1, \omega_2 > 0, \eta > 0$, define
  \begin{align*}
    s(\omega_{1}, \omega_{2})
    \coloneqq \lim _{d, N \rightarrow \infty, d / N \rightarrow \eta} \frac{1}{d} \mathbb{E}\Big[\operatorname{tr}\Big(\big(\omega_{1} \mathbf{I}_{d}+\omega_{2} \hat{\mathbf{Q}}_{N}\big)^{-1}\Big)\Big]
  \end{align*}
  whose closed form solution is given by
  \begin{align*}
    s(\omega_{1}, \omega_{2})
    &=\frac{\eta-1-\omega_{1} / \omega_{2}+\sqrt{(\omega_{1} / \omega_{2}+1+\eta)^{2}-4 \eta}}{2 \eta \omega_{1}} \\
    &=\frac{\sqrt{(\omega_{1} / \omega_{2}+1+\eta)^{2}-4 \eta} - (\omega_{1} / \omega_{2}+1+\eta) + 2\eta}{2 \eta \omega_{1}}
    \leq \frac{1}{\omega_1}.
  \end{align*}
  \begin{align*}
    \frac{d}{d \omega_1} s(\omega_{1}, \omega_{2})
    =&\big[\big( (1 / \omega_{2}+ (1 + \eta) / \omega_1)^{2}-4 \eta / \omega_1^2 \big)^{-\frac{1}{2}} \cdot \\
    &\big((1 / \omega_{2}+(1+\eta)/ \omega_1)(1+\eta) -4\eta/\omega_1\big) (-{\omega_1^{-2}}) 
    + (1 - \eta) \omega_1^{-2} \big] /2\eta \\
    \frac{d}{d \omega_1} s(\omega_{1}, \omega_{2}) \big|_{\omega_1=1} 
    =& [- \big( (1 / \omega_{2}+1 +\eta)^{2}-4 \eta \big)^{-\frac{1}{2}}\big((1 / \omega_{2}+(1+\eta))(1+\eta) -4\eta\big) + (1 - \eta) ] /2\eta .
  \end{align*}
  Therefore 
  \begin{align*}
    \lim _{d, N \rightarrow \infty, d / N \rightarrow \eta} \frac{1}{d} \mathbb{E}\Big[\operatorname{tr}\Big(\big( \mathbf{I}_{d}+\gamma^{-1} \hat{\mathbf{Q}}_{N}\big)^{-1}\Big)\Big] 
    = s(1, \gamma^{-1}) \leq 1
  \end{align*}
  where by L'Hospital's rule, 
  \begin{align*}
    \lim_{\gamma \to \infty} s(1, \gamma^{-1})
    &=\lim_{\gamma \to \infty} \frac{\sqrt{(\gamma+1+\eta)^{2}-4 \eta} - (\gamma+1+\eta) }{2 \eta } + 1 \\
    &=\lim_{\gamma \to \infty} \frac{\sqrt{(1/\gamma+1+\eta/\gamma)^{2}-4 \eta/\gamma^2} - (1/\gamma+1+\eta/\gamma) }{2 \eta / \gamma} + 1  
    = 1
  \end{align*}
  
  \begin{align*}
    \lim_{\gamma \to 0} s(1, \gamma^{-1})
    &=\lim_{\gamma \to 0} \frac{\sqrt{(\gamma+1+\eta)^{2}-4 \eta} - (\gamma+1-\eta) }{2 \eta }  \\
    &=\lim_{\gamma \to 0} \frac{|\eta - 1 | - (1 - \eta) }{2 \eta } 
    = \begin{cases}
      0 , &\eta \in (0, 1];\\
      1 - \frac{1}{\eta} , &\eta \in (1, \infty).
    \end{cases}
  \end{align*}

  By the derivative trick,
  \begin{align*}
    &\lim _{d, N \rightarrow \infty, d / N \rightarrow \eta} \frac{1}{d} \mathbb{E}\Big[\operatorname{tr}\Big(\big(\omega_{1} \mathbf{I}_{d}+\omega_{2} \hat{\mathbf{Q}}_{N}\big)^{-2}\Big)\Big]
    = \frac{d}{d \omega_1} s(\omega_{1}, \omega_{2}) \\
    &\lim _{d, N \rightarrow \infty, d / N \rightarrow \eta} \frac{1}{d} \mathbb{E}\Big[\operatorname{tr}\Big(\big(\mathbf{I}_{d}+\gamma^{-1} \hat{\mathbf{Q}}_{N}\big)^{-2}\Big)\Big] \\
    =& [\big( (\gamma +1 +\eta)^{2}-4 \eta \big)^{-\frac{1}{2}}\big((\gamma +(1+\eta))(1+\eta) -4\eta\big) - (1 - \eta) ] / 2\eta
  \end{align*}
  Therefore 
  \begin{align*}
    \lim_{\gamma \to \infty}\lim _{d, N \rightarrow \infty, d / N \rightarrow \eta} \frac{1}{d} \mathbb{E}\Big[\operatorname{tr}\Big(\big(\omega_{1} \mathbf{I}_{d}+\omega_{2} \hat{\mathbf{Q}}_{N}\big)^{-2}\Big)\Big]
    =\frac{1 + \eta  - (1 - \eta) }{2 \eta } = 1
  \end{align*}
  \begin{align*}
    \lim_{\gamma \to 0}\lim _{d, N \rightarrow \infty, d / N \rightarrow \eta} \frac{1}{d} \mathbb{E}\Big[\operatorname{tr}\Big(\big(\omega_{1} \mathbf{I}_{d}+\omega_{2} \hat{\mathbf{Q}}_{N}\big)^{-2}\Big)\Big]
    = \frac{|\eta - 1 | - (1 - \eta) }{2 \eta } 
    = \begin{cases}
      0 , &\eta \in (0, 1];\\
      1 - \frac{1}{\eta} , &\eta \in (1, \infty).
    \end{cases}
  \end{align*}

  For $\frac{1}{d} \mathbb{E}\Big[\operatorname{tr}\Big(\big(\mathbf{I}_{d}+\gamma^{-1} \hat{\mathbf{Q}}_{N_1}\big)^{-1} \big(\mathbf{I}_{d} + (\gamma s)^{-1} \hat{\mathbf{Q}}_{N}\big)^{-1} \Big)\Big] $, first
  $
   \lim_{\gamma \to \infty} (\mathbf{I}_{d}+\gamma^{-1} \hat{\mathbf{Q}}_{N_1}\big)^{-1}
   =\mathbf{I}_d
  $,
  therefore
  \begin{align*}
    &\lim_{\gamma \to \infty} \frac{1}{d} \mathbb{E}\Big[\operatorname{tr}\Big(\big(\mathbf{I}_{d}+\gamma^{-1} \hat{\mathbf{Q}}_{N_1}\big)^{-1} \big(\mathbf{I}_{d} + (\gamma s)^{-1} \hat{\mathbf{Q}}_{N}\big)^{-1} \Big)\Big] 
    = \lim_{\gamma \to \infty} \frac{1}{d} \mathbb{E}\Big[\operatorname{tr}\Big( \big(\mathbf{I}_{d} + (\gamma s)^{-1} \hat{\mathbf{Q}}_{N}\big)^{-1} \Big)\Big] \\
    & \lim_{\tiny\makecell{\gamma \to \infty \\ s\gamma \to 0}}
    \lim _{\tiny\makecell{d, N \rightarrow \infty\\ d / N \rightarrow \eta}} \frac{1}{d} \mathbb{E}\Big[\mathrm{tr}\Big(\big(\mathbf{I}_{d}+\gamma^{-1} \hat{\mathbf{Q}}_{N_1}\big)^{-1} \big(\mathbf{I}_{d} + (\gamma s)^{-1} \hat{\mathbf{Q}}_{N}\big)^{-1}  \Big) \Big] \\
    = &
    \lim_{\tiny\makecell{\gamma \to \infty \\ s\gamma \to 0}}
    \lim _{\tiny\makecell{d, N \rightarrow \infty\\ d / N \rightarrow \eta}} \frac{1}{d} \mathbb{E}\Big[\mathrm{tr}\Big(\big(\mathbf{I}_{d} + (\gamma s)^{-1} \hat{\mathbf{Q}}_{N}\big)^{-1} \Big) \Big] 
    = \begin{cases}
      0 , &\eta \in (0, 1];\\
      1 - \frac{1}{\eta} , &\eta \in (1, \infty).
    \end{cases}
   \end{align*}
  Then
  \begin{align*}
    \inf _{\tiny\makecell{\gamma>0 \\ s \in(0,1)}}
      \lim _{\tiny\makecell{d, N \rightarrow \infty\\ d / N \rightarrow \eta}} \tilde{C}_{0}^{\text{ba}}
    \leq \lim_{\tiny\makecell{\gamma \to \infty \\ s\gamma \to 0}}
    \lim _{\tiny\makecell{d, N \rightarrow \infty\\ d / N \rightarrow \eta}} \tilde{C}_{0}^{\text{ba}} 
    = \begin{cases}
      1 , &\eta \in (0, 1];\\
      {\eta} , &\eta \in (1, \infty).
    \end{cases}
  \end{align*}
  Note $\tilde{C}_0^{\rm ba} \geq 1$, therefore
    \begin{align}\label{eq:constant_bamaml_inf_limit}
      \inf _{\tiny\makecell{\gamma>0 \\ s \in(0,1)}}
      \lim _{\tiny\makecell{d, N \rightarrow \infty\\ d / N \rightarrow \eta}} \tilde{C}_{0}^{\text{ba}} 
       \begin{cases}
       = 1, & \eta \in (0, 1], \\
       \leq \eta, & \eta \in (1, \infty) .
      \end{cases} 
    \end{align}

  \end{proof}

  \subsubsection{Comparison of the dominating constants} 
  \label{app_ssub:comparison_of_maml_and_bamaml_in_terms_of_constant_term}
  
  Based on Theorems~\ref{thm:asymptotic_MAML_constant},
  \ref{thm:asymptotic_biMAML_constant},
  \ref{thm:asymptotic_baMAML_constant}, Suppose Assumptions~1,3 hold, we have
  \begin{align}
  \label{eq:maml_bimaml_const_ineq}
    \lim _{\tiny\makecell{d, N \rightarrow \infty\\ d / N \rightarrow \eta}} \tilde{C}_{0}^{\text {er}} =
    &\inf _{\tiny\makecell{\alpha \in(0,1/\bar{\lambda})\\ s \in(0,1)}}
    \lim _{\tiny\makecell{d, N \rightarrow \infty\\ d / N \rightarrow \eta}} \tilde{C}_{0}^{\text {ma}} 
    =
    \inf _{\tiny\makecell{\gamma>0\\ s \in(0,1)}} \lim _{\tiny\makecell{d, N \rightarrow \infty \\ d / N \rightarrow \eta}} \tilde{C}_{0}^{\text {im}} 
    > \inf _{\tiny\makecell{\gamma>0\\ s \in(0,1)}}
    \lim _{\tiny\makecell{d, N \rightarrow \infty\\ d / N \rightarrow \eta}} \tilde{C}_{0}^{\text {ba}} 
  .
  \end{align}

  Considering the weighted version, 
  recall that $w_{\mathrm{ma}} = (1-\alpha)^2 > (1-1/{\bar{\lambda}})^2 >0$, 
  $w_{\mathrm{im}} = (1+\gamma^{-1})^{-2}$, $\lim_{\gamma\to 0} w_{\mathrm{im}} = 0$, $w_{\mathrm{ba}} = (1+\gamma^{-1})^{-1}(1+(\gamma s)^{-1})^{-1} < w_{\rm im}$, $\lim_{\gamma s\to 0} w_{\mathrm{ba}} = 0$.
  Therefore 
  $\inf_{\gamma >0} w_{\mathrm{im}} = \inf_{\gamma >0, s\in(0,1)} w_{\mathrm{ba}} = 0
  < \inf_{\alpha \in (0,1/\bar{\lambda})} w_{\mathrm{ma}} $,
  and
  \begin{align}
    \label{eq:maml_bimaml_excess_risk_const_ineq}
      \inf _{\tiny\makecell{\alpha \in(0,1/\bar{\lambda})\\ s \in(0,1)}}
      \lim _{\tiny\makecell{d, N \rightarrow \infty\\ d / N \rightarrow \eta}} w_{\rm ma} \tilde{C}_{0}^{\text {ma}} 
      \geq 
      \inf _{\tiny\makecell{\gamma>0\\ s \in(0,1)}} \lim _{\tiny\makecell{d, N \rightarrow \infty \\ d / N \rightarrow \eta}} w_{\rm im} \tilde{C}_{0}^{\text {im}} 
      \geq \inf _{\tiny\makecell{\gamma>0\\ s \in(0,1)}}
      \lim _{\tiny\makecell{d, N \rightarrow \infty\\ d / N \rightarrow \eta}} w_{\rm ba} \tilde{C}_{0}^{\text {ba}} 
    .
    \end{align}
  Combining \eqref{eq:maml_bimaml_const_ineq} \eqref{eq:maml_bimaml_excess_risk_const_ineq} with the comparison in optimal population risk, we can conclude that, under Assumptions~1,3, when $\gamma$ is sufficiently small, it is guaranteed that iMAML and BaMAML will have smaller meta-test risk than MAML.
  Furthermore, BaMAML has strictly smaller dominating constant in statistical error compared to iMAML under optimal choice of $\gamma $ and $s$, as $d, N \to \infty, d/N \to \eta > 0$.
  

  \section{Additional experiments and details} 
  \label{app_sec:experiment_details}

  \subsection{Experimental details} 
  \label{app_sub:experimental_details}

  For sinewave regression and real image classification,
  Adam optimizer is used. 
  The hyperparameters of all experiments are chosen based on grid search.
  In sinewave regression experiments,
  the learning rate for ERM is initially 0.0001,
  while the  learning rates for both base-learner and meta-learner in MAML and BaMAML are initially 0.001,
  except that in the experiments with $N=1000,T=100,s=0.5$, 
  the initial learning rate of BaMAML for both base-learner and meta-learner are initially 0.0001.
  The learning rate decay is set to be 0.98 for all methods.
  The number of Monte-Carlo samples of model parameters used for BaMAML is 10.
  In real image classification experiments,
  the CNN architecture used is ResNet18.
  The initial learning rate of MAML and BaMAML are 0.001.

  
  \subsection{Real datasets} 
  \label{sub:real_datasets}
  \noindent\textbf{Experiment settings.} 
  We test the performance on the 5-way miniImageNet classification~\citep{vinyals2016matching} and TieredImageNet. MiniImageNet consists 100 classes of images, each with 600 examples. The classes are split into 64, 12, and 24 for train, validation and test, respectively, following~\citep{Finn2017_maml}.
  Note that, since in this setting ERM without adaptation to new classes does not have practical meaning, we do not compare with ERM in this setting.

  \noindent\textbf{Results.} 
  The meta-test classification accuracy under different settings are provided in Table~\ref{tab:image_class}, where BaMAML shows comparable testing loss on miniImageNet and higher testing loss on the TieredImageNet dataset.

  \begin{table}[t]
    \caption{ 
    Comparison of different MAML on image classification 
    (testing loss (NLL) with std., code modified from~\cite{nguyen2020_VAMPIRE})}
    \label{tab:image_class}
    \centering
    \begin{tabular}{ l| c c c c c}
    \hline
    & \multicolumn{2}{c}{miniImageNet } & \multicolumn{2}{c}{TieredImageNet } \\
    \hline
    Method & 1-shot 5-way &5-shot 5-way 
           & 1-shot 5-way &5-shot 5-way  \\
    \hline
     MAML & $1.41 \pm 0.04$ & $1.18 \pm 0.06$ 
          & $1.36 \pm 0.08$  & $0.99 \pm 0.02$ \\
     BaMAML & $1.38 \pm 0.05$ & $1.15 \pm 0.05$ 
            & $1.05 \pm 0.06 $  & $0.76 \pm 0.01$ \\
    \hline
    \end{tabular}
  \end{table}

\end{document}